\documentclass[11pt, oneside]{article}
\usepackage{geometry}
\usepackage{graphicx}	
\usepackage{amssymb}
\usepackage{amsfonts,latexsym,amsthm,amssymb,amsmath,amscd,euscript}
\usepackage{bbm}
\usepackage{framed}
\usepackage{fullpage}
\usepackage{mathrsfs}
\usepackage{enumitem}
\usepackage{accents}
\usepackage{setspace}
\usepackage{tikz-cd}
\usepackage{xspace}
\usepackage[hypertexnames=false]{hyperref}

\allowdisplaybreaks

\hypersetup{
  colorlinks=true,
  linkcolor=blue,
  filecolor=blue,
  citecolor = black,      
  urlcolor=cyan,
}

\hypersetup{colorlinks=true,citecolor=blue,urlcolor=black,linkbordercolor={1 0 0}}

\newcount\Comments  %
\usepackage{algorithm}

\usepackage{verbatim}
\usepackage[noend]{algpseudocode}

\ifnum\Comments=1
\usepackage[inline]{showlabels}

\fi
\usepackage{turnstile}
\usepackage{mathpartir}
\usepackage[nameinlink,capitalize]{cleveref}
\usepackage{tikz}
\crefformat{equation}{#2(#1)#3}
\Crefformat{equation}{#2(#1)#3}
\Crefname{construction}{Construction}{Constructions}

\Crefformat{figure}{#2Figure #1#3}
\Crefname{assumption}{Assumption}{Assumptions}
\Crefformat{assumption}{#2Assumption #1#3}
\Crefname{subsubsection}{Section}{Sections}
\crefformat{subsubsection}{#2Section #1#3}
\Crefformat{subsubsection}{#2Section #1#3}

\newcommand{\LineComment}[1]{\Statex \(\triangleright\) \emph{#1}}
\newcommand{\nc}{\newcommand}

\nc{\sups}[1]{^{\scriptscriptstyle{#1}}}
\nc{\subs}[1]{_{\scriptscriptstyle{#1}}}

\input{widebar}
\newcommand{\wb}{\widebar}

\newcommand{\abs}[1]{\left| {#1} \right|}

\makeatletter
\newtheorem*{rep@theorem}{\rep@title}
\newcommand{\newreptheorem}[2]{%
\newenvironment{rep#1}[1]{%
 \def\rep@title{#2 \ref{##1}}%
 \begin{rep@theorem}}%
 {\end{rep@theorem}}}
\makeatother

\makeatletter
\newcommand\xlabel[2][]{\phantomsection\def\@currentlabelname{#1}\label{#2}}
\makeatother

\theoremstyle{plain}
\newtheorem{theorem}{Theorem}
\newreptheorem{theorem}{Theorem}
\newtheorem{lemma}[theorem]{Lemma}
\newtheorem{corollary}[theorem]{Corollary}

\newtheorem{proposition}[theorem]{Proposition}
\newtheorem{fact}[theorem]{Fact}

\newtheorem{claim}[theorem]{Claim}
\newtheorem{assumption}[theorem]{Assumption}
\theoremstyle{definition}
\newtheorem{definition}[theorem]{Definition}
\newtheorem{defn}[theorem]{Definition}

\newtheorem{remark}[theorem]{Remark}

\newtheorem{observation}[theorem]{Observation}

\newtheorem{mainquestion*}{Main Question}

\numberwithin{theorem}{section}

\nc{\DMO}{\DeclareMathOperator}

\newcommand{\coreset}{emulator\xspace}
\newcommand{\trunccore}{truncated emulator\xspace}

\newenvironment{namedproof}[1]{\paragraph{Proof of #1.}\hspace{-1em}}{\hfill$\qed$\vspace{1em}}

\DeclareMathOperator*{\argmin}{arg\,min} %
\DeclareMathOperator*{\argmax}{arg\,max}

\nc{\Ccov}{C_{\mathsf{cov}}}
\nc{\Cnrm}{C_{\mathsf{nrm}}}
\nc{\Cemp}{C_{\mathsf{empnrm}}}
\newcommand{\IdealPC}{{\tt IdealGreedyCover}\xspace}

\DMO{\prox}{prox}
\DMO{\UCB}{UCB}
\DMO{\LCB}{LCB}
\nc{\pexp}{q_{\mathrm{exp}}}
\nc{\nn}{\nonumber}
\nc{\rk}{\mathrm{rk}}
\nc{\brk}[3]{{\rm br}_{#1}^{#2}({#3})}
\nc{\co}{{\rm co}}
\nc{\br}[2]{{\rm br}^{#1}({#2})}
\nc{\depth}[1]{{\rm d}({#1})}
\nc{\tA}{\textsc{A}}
\nc{\child}[2]{{\rm ch}_{#1}({#2})}
\nc{\parent}[1]{{\rm pa}({#1})}
\nc{\dg}{\dagger}
\nc{\bB}{\mathbf{B}}
\nc{\Span}{{\rm Span}}
\nc{\indsig}[2]{\mathcal{I}_{#1}({#2})}
\nc{\total}{{\rm fin}}
\nc{\early}{{\rm pre}}
\nc{\zsink}{z_{\rm sink}}
\nc{\lowv}{{\rm low}}
\nc{\ol}{\overline}
\nc{\ul}{\underline}
\nc{\madec}[3]{\texttt{ma-dec}_{#1}({#2}, {#3})}
\nc{\madeco}[1]{\texttt{ma-dec}_{#1}}
\nc{\madecd}[3]{\texttt{ma-dec}^{\texttt{d}}_{#1}({#2}, {#3})}
\nc{\SF}{\mathscr{F}}
\nc{\PiMarkov}{\Pi^{\rm markov}}
\nc{\Mbar}{\wb{M}}
\nc{\Mhat}{{\hat{M}}}
\nc{\phiavg}{\phi^{\mathsf{avg}}}
\nc{\phiavgm}[1]{\phi\sups{{#1},\mathsf{avg}}}
\nc{\phiall}{\phi^{\mathsf{all}}}
\nc{\phiorig}{\phi^{\mathsf{orig}}}
\nc{\muorig}{\mu^{\mathsf{orig}}}
\nc{\thetaorig}{\theta^{\mathsf{orig}}}
\nc{\phitravg}{\bar{\phi}^{\mathsf{avg}}}

\nc{\gamvec}{\gamma}
\nc{\til}{\widetilde}
\nc{\td}{\tilde}
\nc{\wh}{\widehat}
\nc{\todo}[1]{\ifnum\Comments=1 {\color{red}  [TODO: #1]}\fi}
\nc{\old}[1]{\ifnum\Comments=1 {\color{brown}  [OLD: #1]}\fi}
\nc{\noah}[1]{\ifnum\Comments=1 {\color{purple} [ng: #1]}\fi}
\nc{\dhruv}[1]{\ifnum\Comments=1 {\color{magenta} [dr: #1]}\fi}
\nc{\ankur}[1]{\ifnum\Comments=1 {\color{green} [am: #1]}\fi}
\nc{\BP}{\mathbb{P}}
\nc{\BI}{\mathbb{I}}

\nc{\fools}[3]{\MF_{#3}({#1}, {#2})}
\nc{\fool}[2]{\MF({#1},{#2})}
\nc{\clip}[2]{{\rm clip}\left[ \left. {#1} \right| {#2} \right]}
\nc{\imax}{\omega}
\DMO{\conv}{conv}
\nc{\MH}{\mathcal{H}}
\nc{\traj}{\mathscr{H}}
\nc{\MV}{\mathcal{V}}
\nc{\MC}{\mathcal{C}}
\nc{\MI}{\mathcal{I}}
\nc{\st}{\star}
\nc{\lng}{\langle}
\nc{\rng}{\rangle}
\DMO{\OOPT}{opt}
\nc{\dopt}[2]{\ell_{\OOPT}({#1},{#2})}
\nc{\grad}{\nabla}
\nc{\MG}{\mathcal{G}}
\nc{\MP}{\mathcal{P}}
\nc{\PP}{\mathbb{P}}
\nc{\TT}{\mathbb{T}}
\nc{\TTmax}{\TT_{\max}}
\DMO{\REG}{Reg}
\DMO{\WREG}{wReg}
\nc{\reg}[2]{{\Delta}_{{#1}}({#2})}
\nc{\wreg}[2]{{\Delta}^{\rm w}_{{#1}}({#2})}
\nc{\Reg}[2]{{\REG}_{{#1}}({#2})}
\nc{\wReg}[2]{{\WREG}_{{#1}}({#2})}
\DMO{\Ham}{Ham}
\DMO{\Gap}{Gap}
\DMO{\GD}{GD}
\DMO{\GDA}{GDA}
\DMO{\EG}{EG}
\nc{\TE}{\til{\E}}
\nc{\Var}{\mathbb{V}}
\DMO{\Cov}{Cov}
\DMO{\OGDA}{OGDA}
\DMO{\Unif}{\unif}
\nc{\unif}{\mathsf{unif}}
\DMO{\Tr}{Tr}
\nc{\Qu}{\ul{Q}}
\nc{\Qo}{\ol{Q}}
\nc{\Ro}{\ol{R}}
\nc{\Vu}{\ul{V}}
\nc{\Vo}{\ol{V}}
\nc{\RanQ}{\Delta Q}
\nc{\RanV}{\Delta V}
\nc{\clipQ}{\Delta \breve{Q}}
\nc{\frzQ}{\Delta \mathring{Q}}
\nc{\clipV}{\Delta \breve{V}}
\nc{\clipdelta}{\breve{\delta}}
\nc{\cliptheta}{\breve{\theta}}
\nc{\delmin}{\Delta_{{\rm min}}}
\nc{\delmins}[1]{\Delta_{{\rm min},{#1}}}
\nc{\gapfinal}[1]{\max \left\{ \frac{\frzQ_{{#1}}^{k^\st}(x,a)}{2H}, \frac{\delmin}{4H} \right\}}
\nc{\post}[2]{R({#1}; {#2})}
\nc{\posts}[3]{R_{#3}({#1}; {#2})}

\nc{\algnst}[1]{\begin{align*}#1\end{align*}}
\nc{\algn}[1]{\begin{align}#1\end{align}}
\nc{\matx}[1]{\left(\begin{matrix}#1\end{matrix}\right)}
\renewcommand{\^}[1]{^{(#1)}}

\nc{\nuu}{\nu}

\nc{\bel}[1]{\mathbf{b}({#1})}
\nc{\nbel}[1]{\bar{\mathbf{b}}({#1})}
\nc{\sbel}[2]{\mathbf{b}'_{#1}({#2})}
\nc{\nsbel}[2]{\bar{\mathbf{b}}'_{#1}({#2})}

\nc{\bv}{\mathbf{v}}
\nc{\bfr}{\mathbf{r}}
\nc{\bone}{\mathbf{1}}
\nc{\bX}{\mathbf{X}}
\nc{\bY}{\mathbf{Y}}
\nc{\bG}{\mathbf{G}}
\nc{\bz}{\mathbf{z}}
\nc{\bw}{\mathbf{w}}
\nc{\bA}{\mathbf{A}}
\nc{\bJ}{\mathbf{J}}
\nc{\bK}{\mathbf{K}}
\nc{\bb}{\mathbf{b}}
\nc{\ba}{\mathbf{a}}
\nc{\bc}{\mathbf{c}}
\nc{\bC}{\mathbf{C}}
\nc{\BR}{\mathbb R}
\nc{\BA}{\mathbb{A}}
\nc{\BC}{\mathbb C}
\nc{\bx}{\mathbf{x}}
\nc{\bS}{\mathbf{S}}
\nc{\bM}{\mathbf{M}}
\nc{\bR}{\mathbf{R}}
\nc{\bN}{\mathbf{N}}
\nc{\NN}{\mathbb{N}}
\nc{\by}{\mathbf{y}}
\nc{\sy}{y}
\nc{\sx}{x}

\nc{\MO}{\mathcal O}
\nc{\MU}{\mathcal{U}}
\nc{\MUr}{\MU^{\mathsf{r}}}
\nc{\ME}{\mathcal{E}}
\nc{\MN}{\mathcal{N}}
\nc{\MK}{\mathcal{K}}
\nc{\MM}{\mathcal{M}}
\nc{\MS}{\mathcal{S}}
\nc{\MT}{\mathcal{T}}
\nc{\BF}{\mathbb F}
\nc{\BQ}{\mathbb Q}
\nc{\MX}{\mathcal{X}}
\nc{\Xreach}{\MX^{\mathsf{rch}}}
\nc{\preach}{p^{\mathsf{rch}}}
\nc{\tilXreach}{\til{\MX}^{\mathrm{rch}}}
\nc{\MA}{\mathcal{A}}
\nc{\MD}{\mathcal{D}}
\nc{\MB}{\mathcal{B}}
\nc{\MZ}{\mathcal{Z}}
\nc{\MJ}{\mathcal{J}}
\nc{\MW}{\mathcal{W}}
\nc{\MR}{\mathcal{R}}
\nc{\MY}{\mathcal{Y}}
\nc{\BZ}{\mathbb Z}
\nc{\BN}{\mathbb N}
\nc{\ep}{\epsilon}
\nc{\etarch}{\eta_{\mathsf{rch}}}
\nc{\trunc}{\sigma_{\mathsf{trunc}}}
\nc{\tsmall}{\sigma_{\mathsf{bkup}}}
\nc{\phitr}{\bar{\phi}}
\nc{\phidt}[1]{\phi\sups{\mathsf{DT}(#1)}}
\nc{\mutr}{\bar{\mu}}
\nc{\term}{\mathfrak{t}}
\nc{\epstat}{\varepsilon_{\mathsf{stat}}}
\nc{\epnnnt}{\varepsilon_{\mathsf{nnnt}}}
\nc{\nstat}{n_{\mathsf{PSDP}}}
\nc{\nstatrch}{n_{\mathsf{PSDPrch}}}
\nc{\npsdprew}{n_{\mathsf{PSDPrew}}}
\nc{\nfe}{n_{\mathsf{FE}}}
\nc{\epdisc}{\varepsilon_{\mathsf{disc}}}
\nc{\epneg}{\varepsilon_{\mathsf{neg}}}
\nc{\epapx}{\varepsilon_{\mathsf{apx}}}
\nc{\eprelax}{\varepsilon_{\mathsf{relax}}}
\nc{\pidisc}{\pi_{\mathsf{disc}}}
\nc{\Pidisc}{\Pi_{\mathsf{disc}}}
\nc{\pilin}{\pi^{\mathsf{lin}}}
\nc{\avgphi}{\phi^{\mathsf{avg}}}
\nc{\VI}{{\tt ValIteration}\xspace}
\nc{\PSDP}{{\tt PSDP}\xspace}
\nc{\PSDPrew}{{\tt PSDPrew}\xspace}
\nc{\PC}{{\tt GreedyCover}\xspace}
\nc{\FE}{{\tt FeatureEstimation}\xspace}
\nc{\OPT}{{\tt POEM}\xspace}
\nc{\SLM}{{\tt ExploreRchLMDP}\xspace}
\nc{\DrawData}{{\tt DrawReachableTrajectoryData}\xspace}
\nc{\DrawDataTrunc}{{\tt DrawTrajectoryData}\xspace}
\nc{\SLMt}{{\tt ExploreLMDP}\xspace}
\nc{\Paramst}{{\tt ParamSettings}}
\nc{\ESCt}{{\tt EstTruncEmulator}\xspace}
\nc{\ESC}{{\tt EstEmulator}\xspace}
\nc{\OLIVE}{{\tt OLIVE}\xspace}
\nc{\GOLF}{{\tt GOLF}\xspace}
\nc{\BilinUCB}{{\tt BiLin-UCB}\xspace}
\nc{\VOX}{{\tt VOX}\xspace}
\nc{\pifinal}{\pi_{\mathsf{final}}}
\nc{\pifinals}[1]{\pi_{\mathsf{final}}^{#1}}
\nc{\Psiapx}{\Psi^{\mathsf{apx}}}
\nc{\Psiapxt}[1]{\Psi^{\mathsf{apx},#1}}

\nc{\vep}{\varepsilon}
\nc{\gapfn}[1]{\varepsilon_{#1}}
\nc{\ggapfn}[2]{\varphi_{#1}({#2})}
\nc{\epsahk}{\gapfn{0}}
\nc{\BH}{\mathbb H}
\nc{\BG}{\mathbb{G}}
\nc{\D}{\Delta}
\nc{\MF}{\mathcal{F}}
\nc{\One}[1]{\mathbbm{1}\left[{#1}\right]}
\nc{\eptrunc}{\varepsilon_{\mathsf{trunc}}}
\nc{\bOne}{\mathbf{1}}
\nc{\Aopt}{\mathcal{A}^{\rm opt}}
\nc{\Amul}{\mathcal{A}^{\rm mul}}

\nc{\SP}{\mathsf P}
\nc{\SQ}{\mathsf Q}

\nc{\DO}{\accentset{\circ}{\D}}
\nc{\mf}{\mathfrak}
\nc{\mfp}{\mathfrak{p}}
\nc{\mfq}{\mf{q}}
\nc{\Sp}{\mbox{Spec}}
\nc{\Spm}{\mbox{Specm}}
\nc{\hookuparrow}{\mathrel{\rotatebox[origin=c]{90}{$\hookrightarrow$}}}
\nc{\hookdownarrow}{\mathrel{\rotatebox[origin=c]{-90}{$\hookrightarrow$}}}
\nc{\hra}{\hookrightarrow}
\nc{\tra}{\twoheadrightarrow}
\nc{\sgn}{{\rm sgn}}
\nc{\aut}{{\rm Aut}}
\nc{\Hom}{{\rm Hom}}
\nc{\img}{{\rm Im}}
\DMO{\id}{Id}
\DMO{\supp}{supp}
\DMO{\KL}{KL}
\nc{\kld}[2]{\KL({#1}||{#2})}
\nc{\ren}[2]{D_2({#1}||{#2})}
\nc{\chisq}[2]{\chi^2({#1}||{#2})}
\nc{\tvd}[2]{D_{\mathsf{TV}}({#1}, {#2})}
\nc{\hell}[2]{D_{\mathsf{H}}^2({#1}, {#2})}
\nc{\dbi}[3][\pi]{D_{\mathsf{bi}}^{#1}({#2} \| {#3})}
\DMO{\BSS}{BSS}
\DMO{\BES}{BES}
\DMO{\BGS}{BGS}
\DMO{\poly}{poly}
\nc{\indep}{\perp}
\DMO{\sink}{sink}
\nc{\fp}[1]{\MP_1({#1})}
\nc{\BO}{\mathbb{O}}
\nc{\BT}{\mathbb{T}}

\nc{\RR}{\mathbb{R}}
\nc{\Gradient}{\nabla}
\DMO{\diag}{diag}
\nc{\norm}[1]{\left \lVert #1 \right \rVert}
\nc{\EE}{\mathbb{E}}
\nc{\MQ}{\mathcal{Q}}
\nc{\epcvx}{\varepsilon_{\mathsf{cvx}}}
\nc{\mures}[1]{\mu^{\mathsf{res},#1}}
\nc{\ff}{{\mathbf{f}}}
\nc{\epfinal}{\varepsilon_{\mathsf{final}}}

\DMO{\PR}{Pr}
\renewcommand{\Pr}{\PR}
\nc{\E}{\mathbb{E}}
\nc{\ra}{\rightarrow}
\renewcommand{\t}{\top}
\nc{\hc}{\{0,1\}^n}
\nc{\pmhc}[1]{\{-1,1\}^{#1}}

\title{Exploring and Learning in Sparse Linear MDPs \\ without Computationally Intractable Oracles}
\author{Noah Golowich\thanks{Email: \texttt{nzg@mit.edu}. Supported by a Fannie \& John Hertz Foundation Fellowship and an NSF Graduate Fellowship.} \\ MIT \and Ankur Moitra\thanks{Email: \texttt{moitra@mit.edu}. Supported in part by a Microsoft Trustworthy AI Grant, an ONR grant and a David and Lucile Packard Fellowship.} \\ MIT \and Dhruv Rohatgi\thanks{Email: \texttt{drohatgi@mit.edu}. Supported by a U.S. DoD NDSEG Fellowship.} \\ MIT}
\date{\today}

\begin{document}
\maketitle

\begin{abstract}

The key assumption underlying linear Markov Decision Processes (MDPs) is that the learner has access to a known feature map $\phi(x, a)$ that maps state-action pairs to $d$-dimensional vectors, and that the rewards and transition probabilities are linear functions in this representation. But where do these features come from? In the absence of expert domain knowledge, a tempting strategy is to use the ``kitchen sink" approach and hope that the true features are included in a much larger set of potential features. In this paper we revisit linear MDPs from the perspective of feature selection. In a \emph{$k$-sparse} linear MDP, there is an unknown subset $S \subset [d]$ of size $k$ containing all the relevant features, and the goal is to learn a near-optimal policy in only $\poly(k,\log d)$ interactions with the environment. Our main result is the first polynomial-time algorithm for this problem. In contrast, earlier works either made prohibitively strong assumptions that obviated the need for exploration, or required solving computationally intractable optimization problems. 

Along the way we introduce the notion of an \emph{emulator}: a succinct approximate representation of the transitions, that still suffices for computing certain Bellman backups.
Since linear MDPs are a non-parametric model, it is not even obvious whether polynomial-sized emulators exist. We show that they do exist, and moreover can be computed efficiently via convex programming.

As a corollary of our main result, we give an algorithm for learning a near-optimal policy in block MDPs whose decoding function is a low-depth decision tree; the algorithm runs in quasi-polynomial time and takes a polynomial number of samples (in the size of the decision tree). This can be seen as a reinforcement learning analogue of classic results in computational learning theory. Furthermore, it gives a natural model where improving the sample complexity via representation learning is computationally feasible.

\end{abstract}

\newpage
\tableofcontents
\newpage

\section{Introduction}

In sequential decision-making tasks, an agent interacts with a changing environment over a sequence of timesteps. At each step, the agent chooses an action \--- which stochastically affects the state of the environment \--- and receives feedback in the form of an immediate reward and information about the subsequent state. The agent wants to find a policy \--- i.e. a mapping from states to actions \--- that maximizes the cumulative reward. \emph{Episodic reinforcement learning (RL)} is the algorithmic problem of learning a good policy via successive, independent episodes of interaction with the environment through trial and error.

Striking empirical advances have been made in recent years by applying reinforcement learning algorithms, augmented with tools from deep learning, to decision-making tasks such as manipulating robots \cite{gu2017deep}, playing strategic games like Go \cite{silver2018general}, personalizing treatment plans in healthcare \cite{yu2021reinforcement}, and learning to navigate complex traffic situations with self-driving cars \cite{sallab2017deep}. Nevertheless, these algorithms also suffer from well-documented and challenging failure modes, such as sample inefficiency, instability, non-convergence, and sensitivity to hyperparameters \cite{ilyas2020closer,engstrom2020implementation} that limit their real-world deployment, which underscores the need for algorithms with strong provable guarantees. On the theoretical front, the textbook model where one can show rigorous guarantees is that of \emph{tabular} Markov Decision Processes (MDPs). This model assumes that the states, actions and transitions can be explicitly written down in a table; the running time and sample complexity of the learning algorithm are allowed to depend polynomially on the size of this representation, i.e., on the number of states and actions. By now, we have a wide range of algorithms that achieve strong provable guarantees in this setting (e.g., \cite{kearns2002near, brafman2002r,jaksch2010nearoptimal,azar2017minimax,jin2018q}).

However there is a sharp disconnect between tabular MDPs and real-world applications, in which the number of possible states of the environment is often astronomical \--- far too large to explicitly write down, let alone exhaustively explore. For this reason, a major challenge in theoretical reinforcement learning is to move beyond the tabular setting, and find new frameworks that more accurately model the difficulties faced in empirical applications: %
are there natural and well-motivated models where we can, simultaneously,

\begin{enumerate}
    \item[(1)] permit a very large or even infinite number of states, and
    \item[(2)] develop statistically efficient methods for learning a near optimal policy?
\end{enumerate}

\noindent There are by now many models that meet these two conditions \--- containing, for example, block MDPs \cite{du2019provably}, MDPs with low Eluder dimension \cite{wang2020reinforcement}, low Bellman rank MDPs \cite{jiang2017contextual}, bilinear classes of MDPs \cite{du2021bilinear}, and MDPs with bounded decision estimation coefficient \cite{foster2021statistical}. However, this literature typically sidesteps computational considerations. In particular, while learning a near-optimal policy in these models is statistically tractable (in that it only requires a small number of episodes), the algorithms for doing so require access to oracles that can solve computationally hard optimization subproblems. While sometimes justifiable on the grounds that stochastic gradient descent is a powerful heuristic for minimizing over function classes such as neural networks, more often than not the required subproblems are far more complex than minimization. For example, many algorithms rely on oracles that can solve max-min objectives (see e.g. recent work on low-rank MDPs \cite{modi2021model, zhang2022efficient, mhammedi2023efficient}), which is computationally challenging both in theory \cite{daskalakis2021complexity} and in practice \cite{razaviyayn2020nonconvex}. %

In this paper, our main focus is on what happens when we add a third goal to our list:

\begin{enumerate}
    \item[(3)] end-to-end computational efficiency \--- i.e., enable frameworks for learning where the running time is not conditional on unimplementable oracles.
\end{enumerate}

\noindent Unfortunately, ``computationally tractable'' models (by which we mean models that meet all three conditions) are exceedingly rare in the existing literature, as discussed further in \cref{sec:related}. %
While focusing exclusively on sample complexity has led to increasingly general models, these have necessitated increasingly abstract algorithmic frameworks. In the same way that VC dimension is often not the correct abstraction for understanding computationally efficient supervised learning, identifying computationally tractable models for reinforcement learning remains a significant challenge.

\paragraph{Linear MDPs.} An important example of a computationally tractable model is a \emph{linear MDP (LMDP)}. %
In this model, it is assumed that the agent knows a feature mapping $\phi$ that maps each state-action pair $(x, a)$ to a $d$-dimensional \emph{feature vector}. Moreover, the rewards and transition probabilities can be expressed as linear functions in these features. Notice that there could be an exponential or even infinite number of states. Thus it is not even obvious how to describe a policy in the first place, because we can no longer do so by exhaustively listing the action to take at each possible state.

Nevertheless, the seminal work %
\cite{jin2020provably} gave a computationally efficient algorithm, \texttt{LSVI-UCB}, for learning linear MDPs. %
As previously mentioned, the more restrictive tabular MDP setting admits a wide variety of algorithms. Some directly learn the parameters of the model (e.g., \cite{azar2017minimax}); such algorithms are called \emph{model-based}. In contrast, others learn a smaller subset of information that still determines the optimal policy (e.g., \cite{jin2018q}); these algorithms are called \emph{model-free}. In linear MDPs, model-based learning seems like a dead end because there are an infinite number of parameters. Instead, \texttt{LSVI-UCB} employs a model-free approach: it learns the \emph{value function} of the optimal policy, i.e. the cumulative future reward obtained by the policy as a function of the current state and action. The value function, which also uniquely determines the optimal policy, turns out to be linear in $\phi(x, a)$, and with an appropriate potential function for incentivizing exploration, a near optimal policy can be found through a combination of classic techniques and linear regression. 

While linear MDPs are therefore computationally tractable, they overlook an important issue: Where do the features come from? The feature mapping needs to somehow distill all the relevant information about $(x, a)$ into a low-dimensional vector, and this is no easy task.

\subsection{Model} 

\paragraph{Learning the right features?} An MDP that is linear in a set of features $\phi_1,\dots,\phi_d$ is also linear in any set of features that contains $\phi_1,\dots,\phi_d$. Hence, in the absence of expert domain knowledge, a tempting strategy for designing $\phi$ is the ``kitchen-sink approach'': enumerate \emph{all} features that could potentially be useful, whether manually by collating a variety of heuristics and rich feature classes, or automatically by learning a neural network and extracting an intermediate layer. One could then hope that the true features are included in this much larger set. This overparametrization may work, but since we have more features than we truly need, we pay the price of a dramatically larger sample complexity for learning a near-optimal policy. 
  
The above approach motivates \emph{sparse linear MDPs} (formally defined in \cref{sec:prelim}), which will be our main model of interest. This model was first formally introduced in \cite{hao2021online} but has had numerous precedents in the bandit \cite{abbasi2012online} and reinforcement learning \cite{wen2013efficient} literature. %
As before, we assume that there is a known $d$-dimensional feature mapping $\phi(x, a)$. But rather than just assuming that the MDP is linear in $\phi$, we assume that there is an unknown subset $S \subseteq [d]$ of size $k \ll d$, such that the rewards and transition probabilities are each linear functions in the \emph{$k$-dimensional} feature mapping $(x,a) \mapsto \phi(x,a)|_S$. This model captures the problem of feature selection in reinforcement learning. More precisely, we ask: 
\begin{mainquestion*}\label{eq:main-question}
    Given $\poly(k,\log d) \ll d$ interactions with the environment, can we efficiently leverage the sparsity of the model to learn a near-optimal policy?
\end{mainquestion*}

As suggested by the terminology, sparse linear MDPs are connected to the well-studied supervised learning problem of sparse linear regression. In particular, consider the static setting where in each step we observe a covariate $x$ and a reward $y$ that is assumed to be a sparse linear function of $x$. %
Sparse linear regression usually refers to the problem of estimating $\E[y|x]$ from few samples. There is a rich literature on the sample complexity of this problem, not only from the perspective of information-theoretic rates \cite{scarlett2016limits, reeves2019all} but also from the perspective of computationally efficient algorithms such as Orthogonal Matching Pursuit \cite{tropp2007signal, cai2011orthogonal}, the Lasso \cite{tibshirani1996regression, wainwright2009sharp}, and variants thereof \cite{candes2007dantzig, needell2009cosamp, belloni2011square, kelner2022power}. The Lasso, in particular, is the workhorse behind feature selection in a wide range of applications \cite{usai2009lasso, zhang2019forecasting, ilyas2022datamodels} because of its simplicity and efficacy.

More generally, the sparse linear MDP model is motivated by the same principles as the notion of attribute-efficiency in PAC learning \cite{klivans2006toward}, where the goal is to give algorithms whose sample complexity depends on the description complexity of the learned hypothesis \--- the intuition being that learning simpler hypotheses should require fewer mistakes. Here, too, the hope is that linear MDPs that admit particularly succinct optimal policies (as is the case for a sparse linear MDP) ought to enable learning a near-optimal policy from a far fewer number of episodes.

\paragraph{Prior work.} Recent works \cite{hao2021online, zhu2023provably} studied sparse linear MDPs but require prohibitively strong assumptions that mitigate the need for exploration. %
 In particular, they assume bounds on the condition number of certain covariance matrices. In sparse linear regression, if one wishes to recover the true parameters of the underlying regressor, then such assumptions on the covariate distribution are typically made (e.g. \cite{wainwright2009sharp, raskutti2010restricted}) and, to an extent, are necessary. Though they are often considered tame in the setting of supervised learning, this is not the case in RL, where there is no fixed distribution. Instead, the distribution over covariates depends on which state-action pairs $(x, a)$ we reach, which in turn depends on the policy we play; thus well-conditionedness assumptions in RL must specify a policy. In particular, the work of \cite{hao2021online} assumes that the algorithm is given a policy up front that induces a well-conditioned feature distribution, thus obviating the need for exploration, which is a basic component of online RL. The work of \cite{zhu2023provably} assumes that \emph{every} policy induces a well-conditioned feature distribution. This assumption also removes the need for exploration since it implies that every policy is essentially exploratory.\footnote{We also remark that it is not clear if the approach of \cite{zhu2023provably}, which treats a more general setting of function approximation in RL, can actually be made computationally efficient in the case of sparse linear regression.} %

Interestingly, if one only cares about \emph{statistical} efficiency, then \cref{eq:main-question} has a positive answer, without any well-conditionedness assumptions, as a straightforward consequence of the statistical efficiency of online RL in MDPs with low Bellman rank \cite{jiang2017contextual, jin2021bellman,du2021bilinear}. Again, these existing algorithms all rely on computationally intractable oracles (see further discussion in \cref{sec:related}) and so do not answer our main question about \emph{computational} efficiency.

We remark that, in terms of analogies to sparse linear regression, reinforcement learning (where \emph{any} near-optimal policy suffices) is closer in spirit to the task of prediction error minimization than parameter estimation. For the former, sparse linear regression is statistically tractable without condition number bounds and with minimal distributional assumptions; such a result is analogous to the implications of the results of \cite{jiang2017contextual, jin2021bellman,du2021bilinear} for sparse linear MDPs. Moreover
computationally efficient algorithms such as Lasso achieve good rates for minimizing prediction error in sparse regression so long as the true sparse regressor has bounded entries (see e.g. Theorem 7.20 in \cite{wainwright2019high}). At a technical level, the focus of this paper is obtaining an analogue of such a result for sparse linear MDPs.%

\subsection{Main result} 

We study the problem of learning a near-optimal policy in a sparse linear MDP in the standard episodic RL interaction model (formalized in \cref{sec:prelim}), where each episode lasts $H$ timesteps (i.e. the horizon is $H$), and at each timestep the agent chooses from a finite set of $A$ actions. %
Our main result is the first end-to-end algorithmic guarantee for learning sparse linear MDPs without condition number assumptions:

\begin{theorem}[Efficient learning of sparse linear MDPs; informal version of \cref{thm:main}]\label{thm:main-informal}
  Let $d,k,A,H \in \NN$ and $\epsilon,\delta>0$. Let $M$ be a $d$-dimensional $k$-sparse linear MDP (\cref{def:sparse-lmdp}) with $A$ actions and planning horizon $H$. %
  With probability at least $1-\delta$, the algorithm $\OPT(\epsilon,\delta)$ outputs a policy with suboptimality at most $\epsilon$. Moreover, the sample complexity of the algorithm is $\poly(k,A,H,\epsilon^{-1},\log(d/\delta))$ and the time complexity is $\poly(d,A,H,\epsilon^{-1},\log(1/\delta))$. %
\end{theorem}

We remark that the sample complexity and running time depend polynomially on the number of actions $A$ and the planning horizon $H$; such dependence is known to be necessary \cite{hao2021online}, even ignoring computational considerations. On the other hand, we re-emphasize that our result does \emph{not} require any type of well-conditionedness assumption on the distribution of feature vectors drawn from the MDP, such as the restricted eigenvalue conditions \cite{bickel2009simultaneous, raskutti2010restricted} required in the prior works \cite{hao2021online, zhu2023provably}. All we require are minimal boundedness assumptions (stated formally in \cref{def:sparse-lmdp}), and there is evidence that such assumptions are necessary (see \cref{app:slr-to-rl}). %

Unsurprisingly, the algorithm \OPT incorporates key tools from the sparse regression literature, e.g., the Lasso convex program. But it requires significant additional innovations to address the challenging problem of exploration in online RL. As intuition, sparse linear regression can be solved by finding a low-complexity regressor that fits the given data. In our setting, though, the data distribution, namely the distribution of the feature vectors $\phi(x,a)$, is not static \--- it depends on our choice of policy. This seemingly causes a chicken-and-egg%
problem common to representation learning in RL: if we haven't explored the MDP completely, then we cannot guarantee that our learned features are accurate on all policies. However, without accurate feature selection, exploration would require too many episodes. Often, generic representation learning frameworks in RL resolve this by repeatedly (a) learning (possibly inaccurate) features from data collected so far, and (b) exploring the MDP using the current estimated features to improve the feature estimates in future rounds. Unfortunately, implementing this framework for sparse linear MDPs runs into a fundamental issue: 
the statistical complexity of exploration in these frameworks is only small if the estimated features are also sparse, i.e., the feature learning step must be \emph{proper} (see further discussion in \cref{sec:related}). But proper sparse linear regression is likely computationally intractable (see e.g. \cite{gupte2020fine}). We circumvent this issue with a completely new algorithm that explicitly works even for \emph{analytically} sparse linear MDPs (i.e., satisfying an $\ell_1$-relaxation of sparsity; see \cref{defn:l1lmdp}), which were only very recently shown to be even statistically tractable \cite{xie2022role}.

\paragraph{Our techniques.} The bulk of our algorithm \OPT is an efficient subroutine \SLMt{}  (\cref{alg:slm-trunc}) that constructs a small \emph{policy cover} for the MDP, namely a set of policies that visits nearly every state $x$ in proportion to the maximum probability that any policy can reach $x$. As we are considering the episodic finite-horizon setting of RL, this is accomplished iteratively, step-by-step. To construct the policy cover at any step $h$, we require that we already have policy covers at previous steps, and we require two more key ingredients:
\begin{itemize}
\item First, we give a procedure that greedily chooses policies so as to cover as many uncovered states in each iteration as possible. An idealized version of this procedure, which assumes that the transitions of the MDP are known, is given in \IdealPC (\cref{alg:idealpc}). The actual version used in our algorithm is \PC (\cref{alg:pc}), which essentially runs \IdealPC using \emph{estimates} of the true transitions on a representative set of states.
\item Second, we introduce the notion of an \emph{emulator} (\cref{def:not-core-set}) and give a convex program for constructing one, which is then passed as input to \PC. An emulator is a set of estimated transitions that satisfy certain natural properties, encapsulating what it means for the estimates to be ``representative'' of the whole MDP. Roughly, the estimated transitions must represent valid probabilities (i.e., be non-negative) and must allow one to approximately compute Bellman backups of certain functions. One of our conceptual (and technical) contributions is in showing that a succinct emulator even \emph{exists} \--- and thus there is a succinct representation of all pertinent information about the transitions of a linear MDP that suffices for exploration.%
\end{itemize}
The success of the two components at any one step relies crucially on their success at constructing a policy cover at previous steps. Typically, arguments which proceed in this manner are prone to suffering from an exponential blowup of errors: A policy cover which suffers from $\ep$ error at one step may lead to errors growing with $k\ep$ or $d\ep$ at the following step, and so on. To counteract such exponential growth, we apply the technique of analyzing a \emph{truncated MDP}, as discussed in \cref{section:trunc-overview}. While this type of technique has been used for similar purposes recently (see e.g. \cite{golowich2022learning, mhammedi2023representation, mhammedi2023efficient}), its application in our setting faces unique challenges due to %
the interaction between \PC and the approximate nature of an emulator. To compensate, we need to define a sequence of truncated MDPs that adapt to the progress that the algorithm makes over time (\cref{sec:trunc-mdps}). 

\paragraph{Emulators: A model-based approach to learning nonparametric systems.} A major challenge in analyzing linear MDPs is their non-parametric nature: the full set of parameters in a linear MDP is unidentifiable in general, and most prior algorithms accounted for this fact by exploiting that certain model-free characteristics of a linear MDP, such as the optimal value function, are parametric and thus can be approximately learned. While at first it may seem like model-based learning is a dead end, our approach using emulators is in some relaxed sense model-based, given that an emulator is used as a direct replacement for the true transitions in the \PC algorithm. It is an interesting direction for future work to determine whether emulators can be used to construct new model-based algorithms, or provide new analyses of existing algorithms, for other problems in the study of RL with linear function approximation. 

\subsection{Application: Block MDPs}
\label{sec:bmdp}
In this subsection, we discuss an application of \cref{thm:main-informal} which yields an end-to-end efficient learning algorithm for a certain family of \emph{block MDPs}. %
A block MDP (defined formally in \cref{def:bmdp}) is an MDP with a potentially infinite state space $\MX$ but a small \emph{latent} state space $\MS$, and an unknown decoding function $\rho^\st: \MX \to \MS$ such that the transition dynamics between two states $x,x' \in \MX$ only depend on $\rho^\st(x)$, $\rho^\st(x')$ and the action taken. %
Moreover, the reward at state $x$ only depends on $\rho^\st(x)$ and the action taken at $x$. Block MDPs are a useful theoretical abstraction for environments such as robot navigation with rich observations (e.g., images) but much simpler underlying state (e.g., the robot's location in space). It is generally assumed that $\rho^\st$ comes from some known function class $\Phi$, and a recent line of work has shown sample complexity bounds for learning block MDPs that are polynomial in $|\MA|$, $|\MS|$ and $\log |\Phi|$ (and have no dependence on $|\MX|$) \cite{du2019provably, misra2020kinematic, zhang2022efficient, mhammedi2023representation}. As usual, these algorithms rely on optimization oracles that are computationally intractable to implement for essentially any interesting class $\Phi$ (unless $\poly(|\Phi|)$ is an acceptable time complexity \cite{modi2021model}). Even worse, as shown in \cref{sec:block-lb}, there are natural parametric function classes $\Phi$ where the time complexity of \emph{any} learning algorithm for the corresponding family of block MDPs must scale essentially polynomially in $|\Phi|$. In fact, we have the following (informal) observation, phrased formally in \cref{rmk:supervised-rl}:

\begin{observation}\label{obs:pac-to-rl}
For any class of decoding functions $\Phi$, if the associated supervised (improper) learning problem is computationally hard with stochastic noise, then the associated reinforcement learning problem is computationally hard too.
\end{observation}

But a converse reduction is not known. This leaves open a fundamental question:

\begin{mainquestion*}
    Are there {\em any} well-motivated classes of block MDPs for which we can give computationally efficient learning algorithms?
\end{mainquestion*}

\noindent Given that efficient supervised learning is a prerequisite, it is natural to start from decoding functions for which the associated concept class already has a distribution-independent\footnote{The reduction in \cref{sec:block-lb} shows that distribution-specific supervised learning is a prerequisite for learning in horizon-$2$ block MDPs, where the distribution for the supervised learning problem is the initial state distribution of the block MDP. However, as the horizon increases beyond 2, the state distributions at later timesteps depend on the policy in increasingly complex ways. Thus, while distribution-independent supervised learning may not be a strict prerequisite, it's unclear how to avoid it, even if the initial state distribution is assumed to be ``nice''.} PAC learning algorithm in the presence of stochastic noise. 

We consider the family of block MDPs where $\MX = \{0,1\}^n$, $\MS = [s]$, and $\Phi = \Phi_{n,s}$ is the class of depth-$\log(s)$ decision tree decoders $\rho: \{0,1\}^n \to \MS$, where the internal nodes of the tree are labelled by variables $i \in [n]$, and the leaves of the tree are labelled by elements of $\MS$. We call these \emph{decision tree block MDPs}. Decision trees can be PAC-learned in quasi-polynomial time \cite{ehrenfeucht1989learning, rivest1987learning, sakakibara1993noise} and have wide-ranging applications because they are simple, flexible and interpretable. In the same spirit, they are a natural choice for the decoding function in a block MDP and they lead us to reinforcement learning style generalizations of classic supervised learning problems. Observe that $|\Phi| = O(n^{s})$; as discussed in \cref{sec:related}, the previously best known algorithm for reinforcement learning in this family (e.g., \OLIVE \cite{jiang2017contextual}) had sample complexity $O(s \log n)$, but running time at least $n^{O(s)}$, which is exponential in the size of the decision tree. However, as an immediate corollary of our main result for sparse linear MDPs, we can obtain an algorithm with running time \emph{quasi-polynomial} in the size of the decision tree and polynomial sample complexity:

\begin{corollary}[Informal version of \cref{cor:dec-tree-mdp}]
  \label{cor:dec-tree-mdp-informal}
Let $n,s,A,H \in \NN$ and $\epsilon,\delta>0$. Let $M$ be a decision tree block MDP on $\{0,1\}^n$ with $A$ actions, horizon $H$, and (unknown) decoding function $\rho^\st \in \Phi_{n,s}$. Then with probability at least $1-\delta$, $\OPT(\epsilon,\delta)$, with an appropriate feature mapping, outputs a policy with suboptimality at most $\epsilon$. Moreover, the sample complexity of the algorithm is $\poly(s, A, H, \epsilon^{-1}, \log(n/\delta))$, and the time complexity is $\poly(n^{\log s}, A, H, \epsilon^{-1},\log(1/\delta))$.
\end{corollary}

\cref{cor:dec-tree-mdp-informal} follows from \cref{thm:main} by observing that for any unknown $\rho \in \Phi_{n,s}$, the corresponding block MDP is actually \emph{linear} with respect to a fixed and known feature mapping $\phidt{n,s}: \{0,1\}^n \to \RR^{\binom{2n}{\log s} \cdot A}$, where each dimension is identified with an action $a$ and a size-$\log(s)$ Boolean clause $\MC$ on $\{0,1\}^n$, and a state-action pair $(\bar x,\bar a)$ is mapped to the indicator vector $\mathbbm{1}[a=\bar a] \cdot \MC(\bar x)$ (see \cref{eq:phi-dt} in \cref{sec:dec-tree-linear} for a formal definition).  %
Moreover, these linear MDPs turn out to be $sA$-sparse (\cref{prop:phidt-sparse}). Thus, based on \cref{thm:main}, we can add decision tree block MDPs to our list of tractable models, at least if we allow quasi-polynomial running time. 

\begin{remark}
  One might wonder if it is possible to get \emph{polynomial} time (i.e., $\poly(n, A, H,1/\ep, \log(1/\delta))$) in \cref{cor:dec-tree-mdp-informal} as opposed to quasipolynomial time. However, in light of \cref{obs:pac-to-rl}, this would imply an improvement of the best-known time bound of $n^{O(\log s)}$ \cite{sakakibara1993noise} for the problem of learning depth-$\log s$ decision trees with stochastic noise. Moreover, it would then imply an improvement for the problem of learning $\log(s)$-sparse parities with noise, for which the best known algorithm still requires $n^{O(\log s)}$ time \cite{valiant2012finding}.
\end{remark}

\begin{remark}
A limitation of our result is that it can only handle \emph{depth-$\log(s)$} decision tree decoding functions. The classical algorithms for PAC learning decision trees \cite{ehrenfeucht1989learning, rivest1987learning, sakakibara1993noise} actually handle \emph{size-$s$} decision trees in the same amount of time, using the fact that any size-$s$ decision tree has at least one leaf at depth at most $\log(s)$. Thus, it's natural to ask whether our result extends to this level of generality as well. However, there are several obstacles that would have to be resolved first. First, the above PAC learning algorithms require not only $n^{O(\log s)}$ time but also $n^{O(\log s)}$ samples (for learning size-$s$ decision trees on $n$ variables). While $O(s\log n)$ samples suffice given $n^{\Omega(s)}$ time, no known algorithm achieves comparable sample complexity with $n^{s^{o(1)}}$ time \cite{klivans2006toward, servedio2012attribute}. Sparse linear regression achieves the best of both worlds (time complexity $n^{O(\log s)}$ and sample complexity $O(s\log n)$) but only for depth-$\log s$ decision trees.

Second, even if the goal is to solve the RL problem with $n^{O(\log s)}$ time and $n^{O(\log s)}$ samples (i.e. matching the PAC learning results), a prerequisite would be solving the \emph{decision tree regression} problem for size-$s$ decision trees. Concretely, \cite{ehrenfeucht1989learning, rivest1987learning} only work in the realizable (i.e. noiseless) PAC learning setting. These results were subsequently extended to the stochastic noise setting where the noise variance is the same at every covariate \cite{sakakibara1993noise}. However, none of these algorithms apply when each leaf of the decision tree has a different noise level (or more generally, when each leaf has an associated distribution on $[0,1]$, and the goal is to find a potentially improper hypothesis that approximately minimizes squared prediction error). We call this the problem of decision tree regression, and by the techniques in \cref{sec:block-lb,app:slr-to-rl}, it can be shown to be necessary for reinforcement learning in decision tree block MDPs.
\end{remark}

\subsection{Additional related work}
\label{sec:related}
To our knowledge, the only works that give algorithms for (even special cases of) sparse linear MDPs while considering the question of computational efficiency are \cite{hao2021online, zhu2023provably}, and as previously mentioned, both make strong additional assumptions. %
On the other hand, there is a vast literature on sample-efficient algorithms for learning various classes of MDPs with large state spaces. Some of these classes subsume sparse linear MDPs and/or decision tree block MDPs, with no additional assumptions. Thus, it is natural to ask whether the corresponding algorithms can be made computationally efficient, without losing sample-efficiency, for our models of interest. It is also natural to ask whether simple modifications of existing algorithms for linear MDPs could be attribute-efficient (i.e. obtaining sample complexity $\poly(k,\log d)$ for $k$-sparse $d$-dimensional linear MDPs) without losing computational efficiency. We discuss these various approaches and explain the obstacles to making them computationally efficient.

\paragraph{Global optimism frameworks.} Some of the most general-purpose algorithms in theoretical RL are \OLIVE \cite{jiang2017contextual}, \GOLF \cite{jin2021bellman}, and \BilinUCB \cite{du2021bilinear}, which apply to MDPs with low Bellman, low Bellman-Eluder dimension, and low Bilinear rank, respectively.  These complexity measures scale linearly with $k$ for $k$-sparse linear MDPs: in particular it follows from \cite[Proposition 9.9]{agarwal2022reinforcement} that the $V$-type Bellman rank of $k$-sparse linear MDPs is at most $k$, which implies in turn that the Bellman-Eluder dimension and Bilinear rank scale linearly with $k$. 
Thus, these works imply that $k$-sparse linear MDPs are statistically tractable with only $\poly(k,A,H,\log d)$ samples.\footnote{See \cite[Corollary 9.19]{agarwal2022reinforcement}, which shows that the sample complexity for learning an $\ep$-optimal policy is bounded above by $\tilde O(H^6 A^2 k^3 \ep^{-2})$.} In fact, recent work showed that the statistical complexity of \GOLF can actually be bounded in terms of the \emph{coverability coefficient} \cite{xie2022role}, which gave the first statistical guarantee for learning analytically sparse linear MDPs.\footnote{This is not an explicit consequence of the results of \cite{xie2022role}, which have sample complexity scaling with $\log |\mathcal{F}|$, where $\mathcal{F}$ is the value function class. In our setting $\mathcal{F}$ is the class of sparse linear functions in $\phi$, which is infinite. However, typically $\log|\mathcal{F}|$ can be replaced by the Rademacher complexity of $\mathcal{F}$, which in our case is $\poly(k,\log d)$.\label{fn:rademacher}} %

However, this generality comes at a price. The above algorithms use the principle of \emph{global optimism}, which requires optimizing over the set of all value functions within the class that approximately fit the current dataset. For sparse linear MDPs, the function class has covering number roughly $d^k$, and does not have nice geometric properties such as convexity, so it is unclear how one would modify \OLIVE, \GOLF, or \BilinUCB to achieve time complexity better than $\poly(d^k)$ \--- note that in the analogous setting of sparse linear regression, this is the time complexity of the brute-force estimator.

\paragraph{Oracle-efficient frameworks for low-rank MDPs.} %
In an effort to develop more practical algorithms than those relying on global optimism, a recent line of work has given ``oracle-efficient'' algorithms in the more concrete setting of \emph{low-rank MDPs} (see \cite{agarwal2020flambe,uehara2022representation,cheng2023improved,misra2020kinematic,modi2021model, zhang2022efficient, mhammedi2023representation}, which require additional structural assumptions such as block latent structure, and the more recent work \cite{mhammedi2023efficient}, which solves this problem given an oracle which can solve certain max-min optimization problems). A low-rank MDP with rank $k$ is a generalization of a linear MDP: the $k$-dimensional feature mapping $\phi$ is now unknown, but it lies in a known function class $\Phi$, thus formalizing the problem of representation learning in RL. Any low-rank MDP with rank $k$ has $V$-type Bellman rank at most $k$, and thus can be learned by \OLIVE \cite{jiang2017contextual}, \GOLF \cite{jin2021bellman}, or \BilinUCB \cite{du2021bilinear} with sample complexity $\poly(k,A,H,\log |\Phi|)$ (where for infinite function classes $\Phi$, the term $|\Phi|$ can be replaced by a covering number of $\Phi$). However, these algorithms are not oracle-efficient, because they work with $\Phi$ directly. In contrast, the recent work of \cite{mhammedi2023efficient} proposed an algorithm \VOX that matches the sample complexity of these algorithms, and only accesses $\Phi$ through optimization oracles (and otherwise has time complexity $\poly(k,A,H)$).

For any known $d$-dimensional feature mapping $\phi$, the $k$-sparse linear MDPs with feature mapping $\phi$ form a family of low-rank MDPs, where the feature class $\Phi$ is the set of $k$-dimensional feature mappings $\phi|_S$, as $S$ ranges over subsets of $[d]$ of size at most $k$. Thus, \cite{mhammedi2023efficient} gives an oracle-efficient algorithm for learning sparse linear MDPs with sample complexity $\poly(k,A,H,\log d)$. However, the required oracle needs to perform proper sparse linear regression (see Line 13 of Algorithm 5 in \cite{mhammedi2023efficient}) as a key step in the representation learning subroutine, which finds a set $S$ (of size at most $k$) so that the features $\phi|_S$ have low representation error on a class of discriminator functions. The naive way takes time scaling with $(d/\ep)^k$, by iterating over an $\ep$-cover of $\Phi$.\footnote{Technically, to fully implement the required step in the representation learning subroutine, one has to combine the cover-based approach with \cite[Section 7]{modi2021model} to perform optimization over the class of discriminator functions.}

Unfortunately, proper $k$-sparse linear regression is widely-believed to require $d^{\Omega(k)}$ time \cite{gupte2020fine, zhang2014lower}.\footnote{\cite{gupte2020fine} gives a lower bound in the Gaussian random design setting with noiseless responses, and \cite{zhang2014lower} gives a lower bound in the worst-case fixed-design setting with Gaussian noise (both lower bounds are conditional on popular worst-case complexity hypotheses). Neither of these directly address the specific flavor of sparse linear regression at stake here \--- proper, approximate risk minimization given noisy samples with i.i.d. covariates drawn from a distribution with bounded support \--- but nor are there non-trivial algorithms known for this problem.} To avoid this intractability, one option is to relax the representation learning subroutine to return a set of $k/\epsilon$ features that have $\epsilon$-near optimal representation error. This approach could plausibly have a computationally efficient implementation, namely via a greedy algorithm for sparse linear regression, such as Orthogonal Matching Pursuit \cite{cai2011orthogonal}. However, there is a catch: exploration given $\ep$-misspecified features, where the misspecification results from the representation error, typically leads to error terms of at least $\ep \cdot k' > \Omega(1)$, where $k' := k/\ep$ denotes the dimension of the estimated features. %

Earlier works on oracle-efficient RL in block/low-rank MDPs \cite{modi2021model, zhang2022efficient} use the same representation learning subroutine and therefore encounter the same issue of compounding feature complexity. Even for decision tree block MDPs, all known learning algorithms for the decoding function \--- whether the classic PAC learning algorithms \cite{ehrenfeucht1989learning, rivest1987learning, sakakibara1993noise} or the sparse linear regression algorithm (induced by the embedding of decision tree block MDPs as sparse linear MDPs) \--- are improper learners, producing hypotheses that are either unbounded-length decision lists or dense linear combinations of subcubes, respectively. As a result, trying to make e.g., the \texttt{BRIEE} algorithm \cite{zhang2022efficient} computationally efficient for decision tree MDPs still runs into the issues discussed above. 

Our algorithm circumvents this issue, explicitly achieving sample-efficiency for all $\ell_1$-bounded linear MDPs (\cref{defn:l1lmdp}), which are a natural relaxation of sparse linear MDPs. Note that unlike sparse linear MDPs, $\ell_1$-bounded linear MDPs with norm bound $k$ do not fall into the class of low-rank MDPs with rank $k$, so are not captured by the above works \cite{modi2021model, zhang2022efficient, mhammedi2023efficient}, even ignoring computational efficiency. Avoiding the ``curse of improper representation learning'' as discussed above in a more general setting (beyond sparse linear MDPs) is an interesting direction for future research.

\paragraph{Linear MDPs, \texttt{LSVI-UCB}, and Eluder dimension.} Since every sparse linear MDP is a linear MDP, \texttt{LSVI-UCB} \cite{jin2020provably} does give a computationally efficient (i.e. $\poly(d,A,H)$-time) algorithm for learning sparse linear MDPs. Of course the sample complexity is also $\poly(d,H)$; without the sparsity assumption, this is necessary even to have accurate regression estimates. Under the sparsity assumption, can the sample complexity be improved by replacing all instances of ridge regression in \texttt{LSVI-UCB} with Lasso? Unfortunately, this seems unlikely: part of the reason that \texttt{LSVI-UCB} incurs $\poly(d)$ sample complexity is the need to bound the number of different directions of feature space that need to be explored (via the elliptic potential lemma). It is the case that $k$-sparse linear MDPs only have $k$ different ``important directions'', but without some form of representation learning to identify these directions, $\poly(d)$ dependence seems unavoidable. On the other hand, incorporating representation learning into \texttt{LSVI-UCB} would seem to lead back to the issues discussed above for low-rank MDP frameworks.

The importance of some form of representation learning for obtaining the correct sample complexity is made more explicit by considering  \cite{wang2020reinforcement}, which generalizes \texttt{LSVI-UCB} to an algorithm $\mathcal{F}$-\texttt{LSVI}, addressing the setting of general function approximation (i.e. where Bellman backups are assumed to lie in some potentially non-linear function class $\mathcal{F}$). The metric that governs the sample complexity of $\mathcal{F}$-\texttt{LSVI} is the \emph{Eluder dimension} \cite{russo2013eluder} of $\mathcal{F}$. For the class of bounded linear functions, the Eluder dimension can be shown to be $\poly(d)$. Unfortunately, even for the class of $1$-sparse bounded linear functions, the Eluder dimension is still $\poly(d)$. Thus, even ignoring computational issues, $\mathcal{F}$-\texttt{LSVI} is not known to be sample-efficient for sparse linear MDPs.

\paragraph{Policy-cover approaches.} As we discuss further in \cref{sec:overview}, our algorithm for sparse linear MDPs proceeds by iteratively learning a \emph{policy cover} for each layer of the MDP. This has been done in various other settings (e.g. \cite{golowich2022learning, mhammedi2023efficient}), but sparse linear MDPs pose unique challenges. In \cref{sec:prior-failures} we discuss why existing approaches for constructing policy covers do not suffice in our setting without incurring $\poly(d)$ sample complexity or $d^k$ time.

\paragraph{Broader perspective: computationally efficient RL.} Besides linear MDPs, there are very few models in theoretical RL that admit computationally and statistically efficient learning algorithms with provable guarantees and also permit large state spaces. %
Some examples can be found within the class of partially observable MDPs (POMDPs), which can be thought of as highly structured MDPs with an infinite state space, taken to be the space of distributions over latent states. For example, \cite{kwon2023reward} establishes a polynomial-time algorithm for learning mixtures of a constant number of MDPs with identical transitions \cite{kwon2023reward}, and \cite{golowich2023planning, golowich2022learning} prove a quasipolynomial-time algorithm for learning  POMDPs satisfying a non-degeneracy condition on the emission distribution. Other examples include certain models with \emph{deterministic} dynamics (e.g. \cite{jin2020sample, uehara2023computationally}). As mentioned in \cref{sec:bmdp}, block MDPs with only a polynomial number of possible decoding functions can also be learned efficiently \cite{modi2021model}, i.e. with time complexity $\poly(|\MA|,|\MS|,|\Phi|)$ and sample complexity $\poly(|\MA|,|\MS|,\log|\Phi|)$. 
Finally, see also \cite{deng2022polynomial} for a tractable variant of linear MDPs.

\subsection{Organization of the paper}

In \cref{sec:prelim} we introduce all necessary notation and background, including formal definition of a sparse linear MDP, and the episodic RL interaction model. In \cref{sec:learning}, we formally prove that \cref{thm:main-informal} follows from an efficient algorithm for constructing policy covers. The intervening sections are devoted our algorithms for constructing policy covers. Since the full analysis is quite technically involved, we include a complete exposition of a slightly simplified algorithm (with a much simpler analysis that nonetheless contains most of the main ideas) that works under an additional \emph{reachability} assumption. Specifically:

\begin{itemize}
    \item In \cref{sec:overview} we give an overview of \SLM (\cref{alg:slm}), our algorithm for constructing policy covers in a reachable sparse linear MDP. We also sketch the analysis. The complete analysis is given in \cref{sec:policy-cover-reachable}.
    \item In \cref{section:trunc-overview}, we give an overview of \SLMt (\cref{alg:slm-trunc}), our algorithm for constructing policy covers in an arbitrary sparse linear MDP. In particular, we discuss how the algorithm and analysis differ from in the reachable case. The complete analysis is given in \cref{sec:trunc-mdps,sec:policy-cover-unreachable} (culminating in \cref{theorem:main-pc-trunc}, which is used in \cref{sec:learning} to prove \cref{thm:main-informal}).
\end{itemize}

\section{Preliminaries}\label{sec:prelim}

In this section, we give background on \emph{linear MDPs} and the \emph{episodic RL model}, and define the two special cases we will be interested in \--- \emph{sparse linear MDPs}, and \emph{$\ell_1$-bounded linear MDPs}.

\subsection{Linear MDPs and Episodic RL}\label{sec:prelim-linear}

For $d \in \BN$, a \emph{$d$-dimensional linear MDP (LMDP) $M$} is a tuple $M = (H, \MX, \MA, \BP_1\sups{M}, (\phi_h\sups{M})_{h \in [H]},\allowbreak (\mu_{h+1}\sups{M})_{h \in [H-1]},\allowbreak (\theta_h\sups{M})_{h \in [H]})$, where $H$ denotes the \emph{horizon}, $\MX$ denotes the \emph{state space}, $\MA$ denotes the \emph{action space}, $\BP_1\sups{M} \in \Delta(\MX)$ denotes the \emph{initial state distribution}, $\theta_h\sups{M} \in\BR^d$ denote the \emph{environmental reward vectors}, and \emph{feature mappings} $\phi_h\sups{M} : \MX \times \MA \ra \BR^d, \mu_{h+1}\sups{M} : \MX \ra \BR^d$ are used to define the transition distributions. Associated to the vectors $\theta_h\sups{M}$ (for $h \in [H]$) are the environmental reward functions $\bfr = (\bfr_1, \ldots, \bfr_{H})$, where $\bfr_h : \MX \times \MA \ra [0,1]$ is defined by $\bfr_h(x,a) = \lng \theta_h\sups{M}, \phi_h\sups{M}(x,a) \rng$.  
Finally, the transitions are defined as follows: for timestep $1 \leq h < H$, states $x,x' \in \MX$ and action $a \in \MA$, we define 
\[\BP\sups{M}_h(x_{h+1}=x'|x_h=x,a_h=a) := \langle \phi_h\sups{M}(x,a), \mu_{h+1}\sups{M}(x')\rangle.\]

For notational convenience, we let $\BP\sups{M}_h(x,a)$ denote the distribution (on $\MX$) with the above density. For this distribution to be well-defined, the following properties are necessary:
\begin{itemize}
    \item For all $h \in [H-1]$, $x,x'\in\MX$, $a \in \MA$, it holds that $\langle \phi_h\sups{M}(x,a),\mu_{h+1}\sups{M}(x')\rangle \geq 0$.
    \item For all $h \in [H-1]$, $x \in \MX$, $a \in \MA$, it holds that $\sum_{x' \in \MX} \langle \phi_h\sups{M}(x,a),\mu_{h+1}\sups{M}(x')\rangle = 1.$ 
\end{itemize}

\paragraph{Episodic RL access model.} We assume that our learning algorithm can interact with its environment in episodes, and that the environment is specified by a linear MDP $M = (H,\MX,\MA,\BP_1,\allowbreak(\phi_h)_{h \in [H]},(\mu_{h+1})_{h\in[H-1]},(\theta_h)_{h \in [H]})$. Concretely, a single episode proceeds as follows. First, a state $x_1 \sim \BP_1$ is drawn and observed. At each step $1 \leq h \leq H$, the algorithm chooses an action $a_h$, and observes reward $r_h = \bfr_h(x_h,a_h) = \langle\theta_h,\phi_h(x_h,a_h)\rangle$.\footnote{As in \cite{jin2020provably}, we assume (for notational simplicity) that rewards are deterministic, but our results extend unchanged to the setting where the reward is a random variable on $[0,1]$ with mean $\bfr_h(x_h,a_h)$.} If $h < H$, then a new state $x_{h+1} \sim \BP_h(x_h,a_h)$ is drawn and observed. As is standard in the linear MDP setting, we assume that the mappings $(\phi_h\sups{M})_h$ are known (i.e. for any $x \in \MX$, $a \in \MA$, and $h \in [H]$, the algorithm can query $\phi_h(x,a)$) but the mappings $(\mu_{h+1}\sups{M})_h$ are unknown. 

It is typical in the literature to make some assumptions on the norms of the feature mappings $\phi_h\sups{M}, \mu_{h+1}\sups{M}$. \cref{def:bounded-lmdp} below gives one such assumption, which requires boundedness of the $\ell_\infty$ norms of the features.\footnote{It is somewhat more standard to require that the $\ell_2$ norm be bounded (as in \cite{jin2020provably,agarwal2022vo}), but due to equivalence of the $\ell_2, \ell_\infty$ norms, assuming that $\ell_\infty$ norms be bounded only changes the results in \cite{jin2020provably,agarwal2022vo} by a $\poly(d)$ factor.}  %

\begin{definition}
  \label{def:bounded-lmdp}
Let $B>0$. We say that a linear MDP $M = (H,\MX,\MA,\BP_1,(\phi_h)_h,(\mu_{h+1})_h,(\theta_h)_h)$ is a \emph{bounded linear MDP with bound $B$} if for all $h \in [H]$, $x,x' \in \MX$, and $a \in \MA$, it holds that $\norm{\phi_h(x,a)}_\infty \leq 1$ and $\norm{\theta_h}_\infty \leq B$, and for all $h \in [H-1]$ it holds that $\sum_{x' \in \MX} \norm{\mu_{h+1}(x')}_\infty \leq B$.
\end{definition}

\begin{remark}
Prior works \cite{jin2020provably, agarwal2022vo} claim to rely on only the weaker assumption that $\norm{\sum_{x'\in\MX} \mu_{h+1}(x')}_\infty \leq B$ rather than $\sum_{x' \in \MX} \norm{\mu_{h+1}(x')}_\infty \leq B$. However, this claim is incorrect; their proofs actually require the above stronger condition (see Lemma B.1 in \cite{jin2020provably}, which is implicitly cited in \cite{agarwal2022vo} to show that linear MDPs are complete under linear function approximation with bounded weights).
\end{remark}

\subsection{Sparse linear MDPs and $\ell_1$-bounded linear MDPs}

\cref{def:sparse-lmdp} below formalizes the definition of a sparse linear MDP: it simply requires that the MDP be a bounded linear MDP (as per \cref{def:bounded-lmdp}) in a small subset of unknown features.%

\begin{definition}
  \label{def:sparse-lmdp}
  Let $d,k,B \in \NN$. We say that a $d$-dimensional linear MDP $M = (H,\MX,\MA,\BP_1,(\phi_h)_h,\allowbreak(\mu_{h+1})_h,(\theta_h)_h)$ is a \emph{$k$-sparse linear MDP ($k$-SLMDP) with bound $B$} if there is an (unknown) set $S \subseteq [d]$ of size $|S| \leq k$ such that:
  \begin{itemize} 
  \item $\supp(\mu_{h+1}(x')) \subseteq S$ for all $h \in [H-1]$ and $x' \in \MX$, and
  \item $\supp(\theta_h) \subseteq S$ for all $h \in [H]$.
  \end{itemize}
  Moreover, we require that for all $h \in [H], x,x' \in \MX, a \in \MA$, it holds that $\| \phi_h(x,a) \|_\infty \leq 1$ and $\norm{\theta_h}_\infty \leq B$ and for all $h \in [H-1]$, it holds that $\sum_{x' \in \MX} \| \mu_{h+1}(x') \|_\infty \leq B$. 
\end{definition}

We omitted the norm bound $B$ in \cref{thm:main} for simplicity; to be fully precise, one may think of the guarantee of \cref{thm:main} as applying whenever $B = \poly(k)$.

Our formal (and most general) guarantees require making an additional definition. This is because our algorithms will actually apply in the (strictly more general) setting of ``analytic sparsity'', which we also refer to as the setting of $\ell_1$-bounded linear MDPs:

\begin{definition}[``analytically sparse'' / $\ell_1$-bounded linear MDPs]\label{defn:l1lmdp}
Let $\Cnrm \geq 1$. We say that a $d$-dimensional linear MDP $M = (H,\MX,\MA,\BP_1,(\phi_h)_h,\allowbreak(\mu_{h+1})_h,(\theta_h)_h)$ with feature mappings $(\phi_h\sups{M})_h$ and $(\mu_{h+1}\sups{M})_h$ is a \emph{$\ell_1$-bounded linear MDP with bound $\Cnrm$} if for all $h \in [H]$, $x,x' \in \MX$, and $a \in \MA$, it holds that $\norm{\phi_h(x,a)}_\infty \leq 1$ and $\norm{\theta_h}_1 \leq \Cnrm$, and for all $h \in [H-1]$, it holds that $\sum_{x' \in \MX} \norm{\mu_{h+1}(x')}_1 \leq \Cnrm$.%
\end{definition}

In particular, a $k$-sparse linear MDP with $\ell_\infty$ bound $B$ is immediately an $\ell_1$-bounded linear MDP with bound $kB$.

\subsection{Policies and value functions}

For any state space $\MX$ and action space $\MA$, a \emph{(randomized Markovian) policy} $\pi = (\pi_h)_{h=1}^{H}$ is a collection of mappings $\pi_h: \MX \to \Delta(\MA)$; we let $\Pi$ denote the set of randomized Markovian policies. Any policy $\pi$ and MDP $M = (H,\MX,\MA,\BP_1,(\phi^M_h)_h,\allowbreak(\mu^M_{h+1})_h,(\theta^M_h)_h)$ together define a distribution over \emph{trajectories} $(x_1,a_1,r_1,x_2,\dots,x_H,a_H,r_H)$ as follows: first, we draw $x_1 \sim \BP\sups{M}_1(\cdot)$. Then, for $1 \leq h \leq H$, we draw $a_h \sim \pi_h(x_h)$ and observe $r_h \in [0,1]$ with $r_h = \langle \phi_h^M(x_h,a_h), \theta^M_h\rangle$. If $h < H$, we then draw $x_{h+1} \sim \BP\sups{M}_h(\cdot|x_h,a_h)$. Let $\traj_H$ denote the set of trajectories. For a (possibly vector-valued) function $f: \traj_H \to \RR^k$ we will often write 
\[ \E\sups{M, \pi}[f(x_1,a_1,r_1,\dots,x_H,a_H,r_H)]\]
to denote the expectation of $f$ over trajectories $(x_1,a_1,r_1,\dots,x_H,a_H,r_H)$ drawn from $M$ under policy $\pi$. Similarly, for a subset $\ME \subseteq \traj_H$ we may write $\Pr\sups{M,\pi}[\ME]$ to denote $\Pr[(x_1,a_1,r_1,\dots,x_H,a_H,r_H) \in \ME]$ where $(x_1,a_1,r_1,\dots,x_H,a_H,r_H)$ is drawn from $M$ under policy $\pi$.

\begin{definition}[$Q$-functions and $V$-functions]\label{def:qv}
For any MDP $M = (H,\MX,\MA,\BP_1,(\phi_h)_h,\allowbreak(\mu_{h+1})_h,\allowbreak(\theta_h)_h)$, policy $\pi \in \Pi$, and collection of reward functions $\bfr = (\bfr_1,\dots,\bfr_{H})$ where $\bfr_h: \MX\times\MA \to [0,1]$ for $h \in [H]$ (which may or may not be the environmental reward functions), we define $Q\sups{M,\pi,\bfr}_1,\dots,Q\sups{M,\pi,\bfr}_H: \MX\times\MA \to \RR$ and $V\sups{M,\pi,\bfr}_1,\dots,V\sups{M,\pi,\bfr}_{H+1}:\MX\to\RR$ by backwards induction as follows. First, $V\sups{M,\pi,\bfr}_{H+1} = 0$. Then for $h = H,\dots,1$, we define
\[ Q\sups{M,\pi,\bfr}_h(x,a) := \bfr_h(x,a) + \E\sups{M,\pi}[V\sups{M,\pi,\bfr}_{h+1}(x_{h+1}) | x_h = x, a_h = a]\]
and
\[V\sups{M,\pi,\bfr}_h(x) := \E_{a \sim \pi_h(x)} Q\sups{M,\pi,\bfr}_h(x,a).\]
For convenience, we also write $Q\sups{M,\pi}_h$ and $V\sups{M,\pi}_h$ to denote $Q\sups{M,\pi,\bfr}_h$ and $V\sups{M,\pi,\bfr}_h$ respectively, when $\bfr$ denotes the environmental reward functions for $M$.
\end{definition}

\begin{lemma}\label{lemma:linear-q}
For any $\ell_1$-bounded $d$-dimensional linear MDP $M = (H,\MX,\MA,\BP_1,(\phi_h)_h,\allowbreak(\mu_{h+1})_h,(\theta_h)_h)$ with bound $\Cnrm\geq 1$, policy $\pi \in \Pi$, and vectors $\theta'_1,\dots,\theta'_H \in \RR^d$, define $\bfr = (\bfr_1,\dots,\bfr_H)$ by
\[\bfr_h(x,a) := \langle \phi\sups{M}_h(x,a), \theta'_h\rangle.\] 
Then for every $1 \leq h \leq H$ there is some $\bv_h \in \RR^d$ such that $Q\sups{M,\pi,\bfr}_h(x,a) = \langle \phi\sups{M}_h(x,a), \bv_h\rangle$ for all $x \in\MX$, $a \in \MA$, and $\norm{\bv_h}_1 \leq \Cnrm \cdot \sum_{k=h}^H \norm{\theta'_k}_1$.
\end{lemma}

\begin{proof}
When $h = H$ we can set $\bv_h := \theta'_H$. When $h < H$, we have by \cref{def:qv} that
\begin{align*}
Q\sups{M,\pi,\bfr}_h(x,a)
&= \bfr_h(x,a) + \E\sups{M,\pi}[V\sups{M,\pi,\bfr}_{h+1}(x_{h+1})|x_h=x,a_h=a] \\
&= \langle \phi_h(x,a), \theta'_h\rangle + \sum_{x'\in\MX} \langle \phi_h(x,a), \mu_{h+1}(x')\rangle \cdot V\sups{M,\pi,\bfr}_{h+1}(x') \\ 
&= \left\langle \phi_h(x,a), \theta'_h + \sum_{x'\in\MX} V\sups{M,\pi,\bfr}_{h+1}(x') \mu_{h+1}(x')\right\rangle.
\end{align*}
Setting $\bv_h := \theta'_h + \sum_{x'\in\MX} V\sups{M,\pi,\bfr}_{h+1}(x')\mu_{h+1}(x')$, the bound $\norm{\bv_h}_1 \leq \Cnrm \sum_{k=h}^H \norm{\theta'_k}_1$ follows from the fact that $|V\sups{M,\pi,\bfr}_{h+1}(x')| \leq \sum_{k=h+1}^H \norm{\theta'_k}_1$ for all $x'\in\MX$, and $\sum_{x'\in\MX} \norm{\mu_{k+1}(x')}_1 \leq \Cnrm$, and $\Cnrm \geq 1$.
\end{proof}

A similar argument (using the fact that the environmental rewards lie in $[0,1]$ almost surely) gives the following lemma:

\begin{lemma}\label{lemma:linear-q-env}
For any $\ell_1$-bounded $d$-dimensional linear MDP $M = (H,\MX,\MA,\BP_1,(\phi_h)_h,\allowbreak(\mu_{h+1})_h,(\theta_h)_h)$ with bound $\Cnrm\geq 1$, and policy $\pi \in \Pi$.
Then for every $1 \leq h \leq H$ there is some $\bv_h \in \RR^d$ such that $Q\sups{M,\pi}_h(x,a) = \langle \phi\sups{M}_h(x,a), \bv_h\rangle$ for all $x \in\MX$, $a \in \MA$, and $\norm{\bv_h}_1 \leq \Cnrm H$.
\end{lemma}

\begin{proof}
The vectors $\bv_h$ are defined as in \cref{lemma:linear-q}. The bound on $\bv_h$ follows from the fact that $|V_{h+1}\sups{M,\pi}(x')| \leq H-h$ for all $x' \in \MX$ together with the assumptions that $\sum_{x'\in\MX} \norm{\mu_{k+1}(x')}_1 \leq \Cnrm$ and $\norm{\theta_h}_1 \leq \Cnrm$.
\end{proof}

\subsection{Additional notation}

\begin{definition}[Compositions of policies]
For policies $\pi, \pi' \in \Pi$ and $h \in [H+1]$, we let $\pi \circ_h \pi'$ denote the policy which follows $\pi$ up to step $h-1$, and then follows $\pi'$ at step $h$ and thereafter. In particular, $\pi \circ_1 \pi' = \pi'$, and $\pi \circ_{H+1} \pi' = \pi$. 
\end{definition}

\begin{definition}[Visitation distribution]
For an MDP $M$, policy $\pi$, and step $h \in [H]$, we define the \emph{state visitation density} by $d\sups{M,\pi}_h: \MX \to [0,1]$ by $d\sups{M,\pi}_h(x) = \Pr\sups{M,\pi}[x_h=x].$ Overloading the notation, we also define the \emph{state-action visitation density} by  $d\sups{M,\pi}_h(x,a) = \Pr\sups{M,\pi}[x_h=x, a_h = a]$.
\end{definition}

\begin{definition}[Average feature vector]
  \label{def:phi-avg}
Let $M = (H,\MX,\MA,\BP_1,(\phi^M_h)_h,\allowbreak(\mu^M_{h+1})_h,(\theta^M_h)_h)$ be a linear MDP. For $h \in [H]$ and $x \in \MX$, we define 
\[ \phiavgm{M}_h(x) := \frac{1}{|\MA|} \sum_{a \in \MA} \phi\sups{M}_h(x,a).\]
\end{definition}

\paragraph{Sampling a trajectory.} For a policy $\pi$, we write $(x_1,a_1,r_1,\dots,a_H,r_H) \sim \pi$ in our algorithms to denote that the algorithm interacts with the MDP for an episode, following policy $\pi$, and observes $(x_1,a_1,r_1,\dots,a_H,r_H)$. For a set of policies $\Psi \subset \Pi$, we write $(x_1,a_1,r_1,\dots,a_H,r_H) \sim \unif(\Psi)$ to denote that the algorithm draws a policy $\pi \sim \Psi$ and then follows policy $\pi$. Finally, we will sometimes use the notation $(x_1,a_1,r_1,\dots,a_H,r_H) \sim \unif(\Psi) \circ_h \unif(\MA) \circ_{h+1} \pi'$. This denotes that the algorithm draws a policy $\pi \sim \Psi$, and then follows the composition policy $\pi \circ_h \unif(\MA) \circ_{h+1} \pi'$ (which takes a uniformly random action at step $h$).

\paragraph{Implicit superscripts.} When the MDP $M$ is clear from context (e.g. in \cref{sec:overview} and \cref{sec:policy-cover-reachable}), we may omit the superscripts of $M$ in $\BP^M_h$, $\E\sups{M,\pi}$, $\Pr\sups{M,\pi}$, $d\sups{M,\pi}_h$; in such cases we are always referring to the MDP that the learning algorithm has sample access to.

\section{Technical Overview I: $H=1$ \& reachable case}\label{sec:overview}

Throughout \cref{sec:overview} and \cref{sec:policy-cover-reachable}, we let $M = (H,\MX,\MA,\BP_1,(\phi_h)_h,\allowbreak(\mu_{h+1})_h,(\theta_h)_h)$ refer to the $\ell_1$-bounded $d$-dimensional linear MDP (\cref{defn:l1lmdp}), with $\ell_1$-bound $\Cnrm \geq 1$, that the learning algorithm has sample access to. Thus, the feature mappings for $M$ are denoted by $\phi_1,\dots,\phi_{H}: \MX\times\MA \to \RR^d$ and $\mu_2,\dots,\mu_H: \MX\to\RR^d$ (for notational purposes, we will later implicitly extend these mappings to $d+1$ dimensions and a slightly larger state space \--- see Section~\ref{sec:trunc-mdp-def}). The average feature vector at step $h$ and state $x$ is denoted by $\phiavg_h(x)$. We assume that the $\ell_1$-bound $\Cnrm$ is known.

As observed after \cref{defn:l1lmdp}, any $k$-sparse linear MDP with norm bound $B$ is immediately an $\ell_1$-bounded linear MDP with $\ell_1$-bound $\Cnrm := kB$. Our goal is to learn a near-optimal policy in $M$, using only $\poly(\Cnrm,A,H,\log(d))$ samples and $\poly(d,A,H)$ time.

\subsection{The case $H=1$}

Since sparse linear regression exhibits an apparent computational/statistical tradeoff in some regimes, it may not be immediately apparent why we could hope that learning in sparse linear MDPs (let alone $\ell_1$-bounded linear MDPs) is both statistically and computationally tractable. We start by explaining why this is possible in the simplest case, where the horizon has length $1$ (i.e. the contextual bandits case). The key is the assumption (indirectly made in \cref{def:sparse-lmdp} and directly made in \cref{defn:l1lmdp}) that the unknown reward vector $\theta_1$ has bounded $\ell_1$ norm. As a result, we can apply the ``slow rate'' error guarantees for Lasso, which hold regardless of covariate distribution.

Concretely, when $H=1$, an episode proceeds as follows. First, a state $x \sim \BP_1$ is drawn and observed. Next, the algorithm chooses an action $a \in \MA$ (and may query $\phi_1(x,a)$). Finally, the reward $r = \langle \phi_1(x,a), \theta_1\rangle$ is independently sampled and observed. 

Suppose that the algorithm interacts for $N = \Theta(\Cnrm^4 A^2 \epsilon^{-4} \log d)$ episodes, and in each episode plays a uniformly random action, obtaining trajectories $(x^i,a^i,r^i)_{i=1}^N$. Estimating $\theta_1$ from the samples $(\phi_1(x^i,a^i), r^i)_{i=1}^N$ is exactly the problem of (noiseless) sparse linear regression. In particular, ordinary least squares would be inconsistent since $N \ll d$. But by standard guarantees (see e.g. \cref{cor:random-design-prediction-error}), since $\norm{\theta_1}_1 \leq \Cnrm$, the constrained Lasso estimator

\[ \hat{\theta} := \argmin_{\theta \in \RR^d: \norm{\theta}_1 \leq \Cnrm} \sum_{i=1}^N \left(\langle \phi_1(x^i,a^i), \theta\rangle - r^i\right)^2 \] 
satisfies, with high probability, that

\[ \E_{(x,a) \sim \BP_1 \times \unif(\MA)} \langle \phi_1(x,a), \hat\theta - \theta_1\rangle^2 \lesssim \frac{\Cnrm^2 \sqrt{\log(d)}}{\sqrt{N}} \leq \frac{\epsilon^2}{A}.\]
Summing over all actions, it follows that $\hat\theta$ and $\theta_1$ are close under the covariate distribution induced by any policy:

\[\max_{\pi \in \Pi} \E_{x\sim\BP_1} \langle \phi_1(x,\pi(x)), \hat\theta - \theta_1\rangle^2 \leq \E_{x\sim\BP_1} \sum_{a\in\MA} \langle\phi_1(x,a),\hat\theta-\theta_1\rangle^2 \leq \epsilon^2.\]

It then follows straightforwardly that the policy $\hat\pi$ defined by $\hat\pi(x) := \argmax_{a \in \MA} \langle \phi_1(x,a), \hat\theta\rangle$ is at most $O(\epsilon)$-suboptimal compared to the optimal policy $\pi^\st(x) := \argmax_{a \in \MA} \langle \phi_1(x,a),\theta_1\rangle$.

\begin{remark}
The sample complexity dependence on $\Cnrm$ could be improved in the above argument, by employing more sophisticated generalization bounds than those stated in \cref{sec:sparse-regression-lemmas} (see e.g. \cite{srebro2010optimistic}). However, in the setting of general $H$ we will pay factors of $\Cnrm$ in the sample complexity for many other reasons, so we will not seek to optimize the polynomial dependence.
\end{remark}

\subsection{General $H$: the challenge of exploration}

When $H=1$ (or more generally, when $H$ is a constant), reinforcement learning tends to be much simpler, because the policy $\pi_{\textsf{unif}}$ that takes a uniformly random action at every step has the following \emph{exploratory} property: no policy can visit a state too much more often than $\pi_{\textsf{unif}}$. However, as $H$ increases, this form of random exploration becomes less useful. Any reinforcement learning algorithm that avoids exponential sample complexity must perform more directed exploration of the state space. Over the years, two paradigms have emerged to perform this task: \emph{optimism}, which adds confidence bonuses to value function estimates to encourage exploration of unvisited states (e.g., \cite{jin2020provably}), and \emph{policy cover}-based approaches, which explicitly construct sets of policies to explore as many ``regions'' of state space as possible (e.g., \cite{golowich2022learning}). Our algorithm follows the latter of these approaches, constructing a policy cover at each step $h = 1, \ldots, H$.

Before formalizing our definition of a policy cover, we make a simplifying assumption. Specifically, for the remainder of \cref{sec:overview}, we will assume that the MDP is \emph{reachable}, as defined below. We will show how to remove this assumption in Part II of the technical overview (\cref{section:trunc-overview}). Most of the algorithmic ideas already show up in the reachable case, but the analysis in the general case is significantly more technical. %

\begin{defn}\label{defn:rch}
Let $\eta \in (0,1)$. We say that the MDP $M$ is $\eta$-\emph{reachable} if for all $h \in [H-1]$ and $x \in \MX$, it holds that
  \begin{align}
\max_{\pi \in \Pi} d\sups{M,\pi}_{h+1}(x) \geq \eta \cdot \| \mu_{h+1}(x) \|_1\nonumber.
  \end{align}
\end{defn}
Reachability dictates that for each state $x$ at step $h+1$, there is some policy $\pi$ visiting $x$ with sufficiently large probability, as a function of $\| \mu_{h+1}(x) \|_1$; note that we must necessarily have $\max_{\pi \in \Pi} d_{h+1}\sups{M, \pi}(x) \leq \| \mu_{h+1}(x) \|_1$, since $\| \phi_h(x,a) \|_\infty \leq 1$ for all $x,a$.

We now define our notion of a policy cover (for the reachable case).

\begin{definition}\label{defn:pc}
Let $\alpha \in (0,1)$. We say that a collection of policies $\Psi \subset \Pi$ is an \emph{$\alpha$-approximate policy cover at step $h\in [H]$} if for all $x \in \MX$, %
\begin{align}
\frac{1}{|\Psi|} \sum_{\pi' \in \Psi} d_h^{M,\pi'}(x) \geq \alpha \cdot \max_{\pi \in \Pi} d_h^{M,\pi}(x)\label{eq:approx-pc}. 
\end{align}
\end{definition}

Once we have a policy cover $\Psi_h$ for each step $h \in [H]$, we can find a near-optimal policy for the MDP via the \emph{Policy Search by Dynamic Programming} (PSDP) algorithm (\cref{alg:psdp-rew}).  \cref{alg:psdp-rew} is a variant of the standard PSDP algorithm \cite{bagnell2003policy} that uses $\ell_1$-regularized least squares rather than ordinary least squares to achieve sample complexity polynomial in $\Cnrm$ (and $1/\alpha$) rather than in $d$.

\begin{lemma}[PSDP; special case of Lemma~\ref{lem:psdp-rew}]\label{lem:psdp-rew-overview}
 There is a sufficiently large $C > 0$ so that the following holds. Fix $\alpha, \delta \in (0,1)$, and suppose we are given $\Psi_1, \ldots, \Psi_H \subset \Pi$, so that each $\Psi_h$ is an $\alpha$-approximate policy cover at step $h$ (per \cref{defn:pc}). Then for any $n \in \BN$ (denoting the sample complexity), $\PSDPrew(\Psi_{1:H}, n)$ (\cref{alg:psdp-rew}) outputs a policy $\hat \pi$ that, with probability $1-\delta$, satisfies
  \begin{align}
V\sups{M, \hat\pi}_1(x_1) \geq \max_{\pi \in \Pi} V\sups{M, \pi}_1(x_1) - \frac{C_{\ref{lem:psdp-rew}} \Cnrm H^2 \sqrt{A} \log^{1/4}(Hd/\delta)}{\alpha n^{1/4}}\label{eq:rew-psdp-bound-overview}.
  \end{align}
\end{lemma}

The outstanding challenge is how to efficiently construct a policy cover $\Psi_{h+1}$ at step $h+1$, given policy covers $\Psi_{1:h}$ at steps $1,\dots,h$. Via a slight variant of \cref{lem:psdp-rew-overview}, we can use $\Psi_{1:h}$ to efficiently (approximately) optimize $\E^{M,\pi}\langle \phi_{h}(x_{h},a_{h}), \theta\rangle$ over policies $\pi \in \Pi$ for any fixed $\theta \in \RR^d$. Leveraging this fact, a natural approach is to construct a cover $\Psi_{h+1}$ by adaptively choosing ``unexplored'' directions $\theta \in \BR^d$, and for each such $\theta$, adding to $\Psi_{h+1}$ a policy $\pi\^\theta$ that maximizes $\E\sups{M, \pi\^\theta}\lng \phi_h(x_h, a_h), \theta \rng$. A key question is how to choose the directions $\theta$, and in particular, how to formalize the notion of what it means for a direction $\theta$ to be ``unexplored''. There are several natural approaches for this question which are considered in prior work, including the use of barycentric spanners \cite{golowich2022learning, mhammedi2023efficient} and various approaches to adapt them to the setting of sparse linear MDPs, such as representation learning \cite{mhammedi2023efficient}. However, there are issues with making these approaches computationally efficient. The main difficulty is that we need the parameter $\alpha$ in \cref{eq:approx-pc} to scale with $1/\Cnrm$ rather than $1/d$, in order for PSDP to be sample-efficient at subsequent steps. For the purposes of intuition, this is roughly equivalent to needing that the $\Psi_{h+1}$ is fairly small, i.e. $|\Psi_{h+1}|$ scales with $\Cnrm$ rather than $d$.

As we discuss further in \cref{sec:prior-failures}, in the general case of $\ell_1$-bounded linear MDPs (\cref{defn:l1lmdp}), all of the approaches from prior work do not even certify that there \emph{exists} a small cover $\Psi_{h+1}$, let alone addressing the question of how to efficiently construct such a cover. %
In \cref{sec:existence-pc}, we discuss how to prove that there exists a small policy cover in $\ell_1$-bounded linear MDPs. Subsequently, in \cref{sec:greedycover-overview,sec:emulator-overview}, we discuss how such a cover can be found efficiently. %

\subsection{Why does a small policy cover exist?}
\label{sec:existence-pc}
As discussed above, at the very least we need an \emph{existential} proof that there is some policy cover $\Psi_{h+1}$ of size scaling with $\Cnrm$ rather than $d$. %
If the features $(\mu_{h+1}(x): x \in \MX)$ are truly sparse, with a common support $U$ of size $\Cnrm$, existence of such a cover follows from the existence of a barycentric spanner for the $\Cnrm$-dimensional polytope $\{\E^{M,\pi}[\phi_h(x_h,a_h)_U]: \pi \in \Pi\}$. However, under the weaker assumption that $\sum_{x \in \MX} \norm{\mu_{h+1}(x)}_1 \leq \Cnrm$, this argument no longer works.

Instead, we describe, in \cref{alg:idealpc}, a procedure that generates a cover of size $O(\Cnrm)$ for (reachable) $\ell_1$-bounded linear MDPs with norm bound $\Cnrm$. The construction maintains a set of states that have not been covered so far. At every step, it computes the policy that maximizes the probability of reaching one of these uncovered states. If this probability is small, the procedure terminates. Otherwise, the policy is added to the cover and the procedure repeats. We remark in passing that \cref{alg:idealpc} can be modified to work more generally for (non-linear) MDPs which have a bounded \emph{coverability coefficient}, as defined in \cite{xie2022role}, using the equivalence of coverability coefficient and the notion of \emph{cumulative reachability} \cite[Lemma 3]{xie2022role}. 

To be clear, \cref{alg:idealpc} is not implementable, let alone computationally efficient, as stated, since it requires knowledge of the $\mu_{h+1}(x)$ vectors. However, it will motivate our final  computationally efficient algorithm.

\makeatletter
\renewcommand*{\ALG@name}{Construction}
\makeatother
\algblockdefx{MRepeat}{EndRepeat}{\textbf{repeat}}{}
\algnotext{EndRepeat}
\begin{algorithm}[t]
    \caption{$\IdealPC(M,h,\Cnrm,\xi)$: Idealized Policy Cover Construction}
    \label[construction]{alg:idealpc}
    
    \begin{algorithmic}[1]
    \State Initialize $\Psiapx_{h+1} = \emptyset$ and $\MB_1 = \MX$, and $t = 1$.
    \MRepeat
        \State Set $\mures{t} \gets \sum_{x \in \MB_t} \mu_{h+1}(x)$.
        \State Choose $\pi^t \in \argmax_{\pi \in \Pi} \E^{M,\pi}\langle \phi_h(x_h,a_h), \mures{t}\rangle$.
        \State \textbf{Break} if $\E^{M,\pi^t}\langle \phi_h(x_h,a_h), \mures{t}\rangle < \xi$. Otherwise, add $\pi^t$ to $\Psiapx_{h+1}$.
        \State Set $\MG_t \gets \{x \in \MB_t: \E^{M,\pi^t}\langle \phi_h(x_h,a_h), \mures{t}\rangle \geq \frac{\xi}{2\Cnrm} \cdot \norm{\mu_{h+1}(x)}_1\}$.
        \State Set $\MB_{t+1} \gets \MB_t \setminus \MG_t$ and $t \gets t+1$.
    \EndRepeat
    \State $\MB \gets \MB_t$.
    \State\textbf{return} $\Psiapx_{h+1}, \MB$.
    \end{algorithmic}
\end{algorithm}
\makeatletter
\renewcommand*{\ALG@name}{Algorithm}
\makeatother

The following lemma states the guarantee of \cref{alg:idealpc}.

\begin{lemma}\label{lem:idealpc-guarantees}
Fix $h \in [H-1]$ and $\xi>0$. The outputs $\Psiapx_{h+1} \subset \Pi$ and $\MB \subseteq \MX$ of the algorithm $\IdealPC(M,h,\Cnrm,\xi)$ have the property that $|\Psiapx_{h+1}| \leq 2\Cnrm/\xi$, and moreover for all $\pi \in\Pi$,
\begin{enumerate}
    \item $\Pr^{M,\pi}[x_{h+1}\in\MB] < \xi$.
    \item For all $x \in \MX\setminus \MB$ there is some $\pi' \in \Psiapx_{h+1}$ such that $d_{h+1}^{M,\pi'}(x) \geq \frac{\xi}{2\Cnrm} d_{h+1}^{M,\pi}(x)$.
\end{enumerate}
\end{lemma}
\begin{proof}
The first claim is immediate from the \textbf{Break} condition in the construction: by definition of $\mures{t}$, $\pi^t$ is a policy that maximizes $\Pr^{M,\pi}[x_{h+1} \in \MB_t]$. The second claim is similarly immediate from the definition of $\MG_t$. It remains to bound $|\Psiapx_{h+1}|$. For this, note that for every $\pi^t$ added to $\Psiapx_{h+1}$, on the one hand $\Pr^{M,\pi^t}[x_{h+1} \in \MB_t] \geq \xi$, and on the other hand $\Pr^{M,\pi^t}[x_{h+1} \in \MB_t\setminus \MG_t] \leq \frac{\xi}{2\Cnrm}\sum_{x \in \MG_t}\norm{\mu_{h+1}(x)}_1 \leq \frac{\xi}{2}$. Thus $\sum_{x \in \MG_t} \norm{\mu_{h+1}(x)}_1 \geq \Pr^{M,\pi^t}[x_{h+1} \in \MG_t] \geq \frac{\xi}{2}$. But the sets $\MG_1,\MG_2,\dots \subseteq \MX$ are disjoint and $\sum_{x \in \MX} \norm{\mu_{h+1}(x)}_1 \leq \Cnrm$, so it must be that $t \leq 2\Cnrm/\xi$.
\end{proof}

The set $\Psiapx_{h+1}$ is not quite a policy cover at step $h+1$ in the sense of \cref{defn:pc}, since there is a set $\MB$ of potentially uncovered states. However, if $M$ is $\eta$-reachable and $\xi \ll \eta$, then we \emph{can} get a policy cover at step $h+2$: %
\begin{lemma}
  \label{lem:idealpc-pc}
  In the setting of \cref{lem:idealpc-guarantees}, if $\xi \leq \eta/(2A)$, then the set
  \begin{align} \Psi_{h+2} := \{\pi' \circ_{h+1} \unif(\MA): \pi' \in \Psiapx_{h+1}\}\label{eq:define-psih2}
  \end{align}
  is an $\alpha$-approximate policy cover at step $h+2$. 
\end{lemma}
\begin{proof}[Proof of \cref{lem:idealpc-pc}]
  For any state $x' \in \MX$, let $\pi \in \Pi$ be the policy that maximizes $d^{M,\pi}_{h+2}(x')$. We have that
\begin{align*}
\frac{1}{A} d^{M,\pi}_{h+2}(x')
&\leq d^{M,\pi \circ_{h+1} \unif(\MA)}_{h+2}(x') \\
&= \E^{M,\pi \circ_{h+1}\unif(\MA)} \langle \phi_{h+1}(x_{h+1},a_{h+1}), \mu_{h+2}(x')\rangle \\
&= \sum_{x \in \MX} d^{M,\pi}_{h+1}(x) \cdot \langle \phiavg_{h+1}(x,a), \mu_{h+2}(x')\rangle \\
&= \sum_{x \in \MB} d^{M,\pi}_{h+1}(x) \cdot \langle \phiavg_{h+1}(x), \mu_{h+2}(x')\rangle + \sum_{x \in \MX\setminus\MB} d^{M,\pi}_{h+1}(x) \cdot \langle \phiavg_{h+1}(x), \mu_{h+2}(x')\rangle.
\end{align*}

By the first claim of \cref{lem:idealpc-guarantees}, the first term is at most $\xi \cdot \sup_{x \in \MB} |\langle\phiavg_{h+1}(x),\mu_{h+2}(x')\rangle| \leq \xi \cdot \norm{\mu_{h+2}(x')}_1$. By the second claim of the lemma and the fact that $|\Psiapx_{h+1}| \leq 2\Cnrm/\xi$, we can bound the second term as
\begin{align} 
\sum_{x \in \MX\setminus\MB} d^{M,\pi}_{h+1}(x) \cdot \langle \phiavg_{h+1}(x), \mu_{h+2}(x')\rangle 
&\leq \frac{4\Cnrm^2}{\xi^2} \cdot \frac{1}{|\Psiapx_{h+1}|}\sum_{\pi' \in\Psiapx_{h+1}} \sum_{x \in \MX\setminus\MB} d^{M,\pi'}_{h+1}(x) \cdot \langle \phiavg_{h+1}(x), \mu_{h+2}(x')\rangle \label{eq:cover-application-overview}\\
&\leq \frac{4\Cnrm^2}{\xi^2} \cdot \frac{1}{|\Psiapx_{h+1}|}\sum_{\pi' \in\Psiapx_{h+1}} \sum_{x \in \MX} d^{M,\pi'}_{h+1}(x) \cdot \langle \phiavg_{h+1}(x), \mu_{h+2}(x')\rangle \nonumber\\
&= \frac{4\Cnrm^2}{\xi^2} \cdot \frac{1}{|\Psiapx_{h+1}|}\sum_{\pi' \in\Psiapx_{h+1}} d^{M,\pi' \circ_{h+1} \unif(\MA)}_{h+2}(x')\nonumber.
\end{align}
Thus, summing the two bounds and using the definition of $\Psi_{h+2}$, we get that
\begin{align} \frac{1}{A} d^{M,\pi}_{h+2}(x') \leq \xi \cdot \norm{\mu_{h+2}(x')}_1 + \frac{4\Cnrm^2}{\xi^2} \cdot \frac{1}{|\Psi_{h+2}|} \sum_{\pi' \in\Psi_{h+2}} d^{M,\pi'}_{h+2}(x').\label{eq:h2-pc-intro}\end{align}
So long as $\xi \leq \eta/(2A)$, we get from $\eta$-reachability that $\xi\cdot\norm{\mu_{h+2}(x')}_1 \leq \frac{1}{2A} d^{M,\pi}_{h+2}(x')$. Thus,
\[d^{M,\pi}_{h+2}(x') \leq \frac{8\Cnrm^2 A}{\xi^2} \cdot \frac{1}{|\Psi_{h+2}|} \sum_{\pi' \in \Psi_{h+2}} d^{M,\pi'}_{h+2}(x'),\]
so $\Psi_{h+2}$ is an $\alpha$-policy cover (\cref{defn:pc}) with $\alpha = 8\Cnrm^2 A/\xi^2$.
\end{proof}

\subsection{Making the construction efficient}
\label{sec:greedycover-overview}
There are two glaring issues with the construction of $\Psiapx_{h+1},\Psi_{h+2}$ described above. First, it requires iterating over all states $x \in \MX$. Second, it requires access to the features $\mu_{h+1}(x)$, which are unknown to the algorithm. To overcome these issues, we implement an analogue of \IdealPC that works with a smaller set of \emph{simulated} feature vectors. In particular, \cref{alg:pc} gives a computationally-efficient and sample-efficient algorithm \PC that takes as input a set of $m$ vectors $\hat\mu_{h+1}^1,\dots,\hat\mu_{h+1}^m \in \RR^d$, which are to be interpreted as a sort of simulation of the set of true feature vectors $\{ \mu_{h+1}(x) \ : \ x \in \MX \}$ at step $h+1$.

\algblockdefx{MRepeat}{EndRepeat}{\textbf{repeat}}{}
\algnotext{EndRepeat}
\begin{algorithm}[t]
	\caption{$\PC(h, (\hat \mu_{h+1})_{j=1}^m, \xi, \Cemp, \Psi_{1:h}, N)$: Policy Cover Construction}
	\label{alg:pc}
	\begin{algorithmic}[1]\onehalfspacing
 \vspace{1em}
		\Require~ 
  \begin{minipage}{\dimexpr\textwidth-3cm}
  Step $h \in [H]$; target rewards $\hat\mu_{h+1}^1,\dots,\hat\mu_{h+1}^m$; tolerance $\xi > 0$; empirical norm parameter $\Cemp \in \RR$; policy covers $\Psi_1, \ldots, \Psi_{h}$; sample complexity $N$ (per call to \PSDP and \FE).
  \end{minipage}\vspace{1em}
		\State Initialize $\Psiapx_{h+1} = \emptyset$ and $\MB_1 =[m]$, and $t = 1$.
        \MRepeat
            \State Set $\mures{t} \gets \sum_{j \in \MB_t} \hat \mu_{h+1}^j$.
            \State Choose $\pi^t \gets \PSDP(h, \mures{t}, \Psi_{1:h-1}, N)$. \label{line:pc-psdp}\Comment{\emph{See \cref{alg:psdp} for \PSDP.}}
            \State Set $\hat \phi_h^t \gets \FE(h, \pi^t, N)$. \label{line:pc-fe} \Comment{\emph{See \cref{alg:fe} for \FE.}}
            \State \textbf{Break} if $\langle \hat{\phi}_h^t, \mures{t}\rangle < \xi$. Otherwise, add $\pi^t$ to $\Psiapx_{h+1}$.\label{alg:pc:line:break}
            \State Set $\MG_t \gets \left\{ j \in \MB_t \ : \ \lng \hat \phi_h^{t}, \hat \mu_{h+1}^j \rng \geq (\xi/(2\Cemp)) \cdot \| \hat \mu_{h+1}^j \|_1  \right\}$.\label{alg:pc:line:jidef}
            \State Set $\MB_{t+1} \gets \MB_t \setminus \MG_t$ and $t \gets t+1$. 
        \EndRepeat
        \State Set $\pifinal \gets \pi^t$. 
		\State \textbf{Return:} Intermediate policy cover $\Psiapx_{h+1}$, the sets $\MB := \MB_t$, $\MG := [m] \setminus \MB_t$, and final policy $\pifinal$. 
	\end{algorithmic}
\end{algorithm}

Just like \IdealPC, this algorithm maintains a set of ``uncovered'' vectors indexed by $\MB_t \subseteq [m]$, and repeatedly computes a policy $\pi^t$ (\cref{line:pc-psdp}) that approximately optimizes in the direction $\mures{t} = \sum_{j \in \MB_t} \hat\mu^j_{h+1}$, this time using the efficient \PSDP algorithm (\cref{alg:psdp}). In \cref{line:pc-fe}, it uses the efficient \FE algorithm (\cref{alg:fe}) to compute an estimate $\hat \phi_h^t$ of the expected feature vector $\E\sups{M, \pi^t}[\phi_h(x_h, a_h)]$ at step $h$ under $\pi^t$. The estimate $\hat \phi_h^t$ is in turn used to add indices to the set $\MG_t$ of explored vectors (\cref{alg:pc:line:jidef}). Apart from statistical errors resulting from the use of \PSDP and \FE,  %
 the guarantees of \PC are analogous to those of \IdealPC. In particular, we have the following:

\begin{lemma}[Informal version of \cref{lem:emp-pc-guarantee-rch}]
  \label{lem:emp-pc-guarantee-rch-informal}
 Suppose that $\sum_{j=1}^m \norm{\hat\mu^j_{h+1}}_1 \leq \Cemp$. Then, with high probability, the output $(\Psiapx_{h+1},\MB,\pifinal)$ of \PC satisfies $|\Psiapx_{h+1}| \leq 2\Cemp/\xi$, as well as  the following conditions for all $\pi \in \Pi$:  %
\begin{enumerate}
    \item $\sum_{j \in \MB} \left\langle \E^{M,\pi}[\phi_h(x_h,a_h)], \hat\mu_{h+1}^j\right\rangle \leq 3\xi/2$.
    \item For all $j \in [m]\setminus \MB$, there is some $\pi' \in \Psiapx_{h+1}$ so that $\lng\E^{\pi'}[\phi_h(x_h, a_h)], \hat \mu_{h+1}^j \rng \geq \frac{\xi}{4 \Cemp} \cdot \lng \E^\pi[\phi_h(x_h, a_h)], \hat \mu_{h+1}^j \rng$. 
    \end{enumerate}
  \end{lemma}

  That is, for fixed $h$, given a set of $m$ target reward vectors $\hat\mu^j_{h+1}$ (for $j \in [m]$) with total $\ell_1$ norm at most $\Cemp$, we can construct a set of $2\Cemp/\xi$ policies that approximately optimize most of these rewards. At first glance, it may seem that there is a much simpler way of obtaining such a guarantee: optimize each vector $\hat \mu_{h+1}^j$ independently. %
  However, the fact that \PC yields only $2\Cemp/\xi$ polices, as opposed to $m$ policies, as the simpler approach would ensure, %
  will prove crucial in the downstream analysis (we will have $m \gg 2\Cemp/\xi$).

Of course, the challenge remains of finding a small set of vectors $\{\hat\mu_{h+1}^1,\dots,\hat\mu_{h+1}^m\}$ that somehow ``simulates'' the MDP $M$ sufficiently well for the guarantee of \cref{lem:emp-pc-guarantee-rch-informal} to be useful. However, the insights from the inefficient construction in the previous section \--- in particular, using $\Psi_{h+1}$ to construct a policy cover for step $h+2$, rather than directly trying to cover step $h+1$\--- provide a blueprint for the properties that we need $\{\hat\mu^1_{h+1},\dots,\hat\mu^m_{h+1}\}$ to satisfy.

\paragraph{An \coreset for the MDP.} We wish to show that the guarantee of \cref{lem:emp-pc-guarantee-rch-informal} implies that the output of \PC yields a policy cover for $M$, in a similar way that \cref{lem:idealpc-pc} uses \cref{lem:idealpc-guarantees} to guarantee that the output of \IdealPC yields a policy cover of $M$.
The key step in the proof of \cref{lem:idealpc-pc} (see \cref{eq:cover-application-overview}) uses the guarantee of \IdealPC to obtain that, for any $\pi \in \Pi$ and $x' \in \MX$, we have
\begin{align}
  &  \sum_{x \in \MX} \lng \E\sups{M, \pi}[\phi_h(x_h, a_h)], \mu_{h+1}(x)\rng \cdot \lng \phiavg_{h+1}(x), \mu_{h+2}(x') \rng\nonumber\\ & \lesssim \frac{4\Cnrm^2}{\xi^2} \cdot \E_{\pi' \sim \unif(\Psiapx_{h+1})} \sum_{x \in \MX} \lng \E\sups{M, \pi'}[\phi_h(x_h, a_h)], \mu_{h+1}(x)\rng \cdot \lng \phiavg_{h+1}(x), \mu_{h+2}(x') \rng \label{eq:cover-main-step},
\end{align}
where the $\lesssim$ ignores additive error terms arising from the ``bad'' set $\MB$. 
For context, we note that the left-hand side is bounded below by $\frac 1A \cdot d_{h+2}\sups{M, \pi}(x')$, and the right-hand side is bounded above by $\frac{4\Cnrm^2}{\xi^2} \E_{\pi' \sim \unif(\Psiapx_{h+1})} d_{h+2}\sups{M, \pi' \circ_{h+1} \unif(\MA)}(x')$, so \cref{eq:cover-main-step} establishes that, up to error terms, we obtain the desired cover property. 

If we instead use \PC, we cannot obtain a guarantee in terms of the true vectors $\hat \mu_{h+1}(x)$; instead, we can show that
for any states $\tilde{x}^1,\dots,\tilde{x}^m \in \MX$, %
\begin{align}
&\sum_{j \in [m]} \langle\E^{M,\pi}[\phi_h(x_h,a_h)],\hat\mu_{h+1}^j\rangle \cdot \langle\phiavg_{h+1}(\tilde{x}^j),\mu_{h+2}(x')\rangle \nonumber\\
&\qquad\lesssim \frac{8\Cnrm^2}{\xi^2} \cdot \E_{\pi' \sim \unif(\Psiapx_{h+1})}\sum_{j \in [m]} \langle\E^{M,\pi'}[\phi_h(x_h,a_h)],\hat\mu_{h+1}^j\rangle \cdot \langle\phiavg_{h+1}(\tilde{x}^j),\mu_{h+2}(x')\rangle.\label{eq:greedycover-result-intro}
\end{align}

Our goal is to choose the feature vectors $\hat \mu_{h+1}^j$, as well as the states $\tilde x^j$, so that \emph{validity of \cref{eq:cover-main-step} for any policy $\pi$ is {equivalent} to validity of \cref{eq:greedycover-result-intro} for $\pi$}, up to small error terms. 
This will be the case if the vectors $\{\hat\mu_{h+1}^1,\dots,\hat\mu_{h+1}^m\}$ and states $\{\tilde{x}^1,\dots,\tilde{x}^m\}$ are chosen so as to satisfy that for all policies $\pi \in \Pi$, the following discrepancy is small:

\[ \norm{\sum_{x \in \MX} \langle \EE\sups{M, \pi}[\phi_h(x_h,a_h)], \mu_{h+1}(x)\rangle \cdot \phiavg_{h+1}(x) - \sum_{j = 1}^m \langle \EE\sups{M, \pi}[\phi_h(x_h,a_h)], \hat\mu_{h+1}^j\rangle \cdot \phiavg_{h+1}(\tilde{x}^j)}_\infty.\]

Besides this we need two other properties: first, to even apply \PC we need a bound on $\sum_{j\in[m]} \norm{\hat\mu_{h+1}^j}_1$. Second, at various points we need the inner products $\langle \E\sups{M, \pi}[\phi_h(x_h,a_h)], \hat\mu_{h+1}^j\rangle$ to be approximately non-negative, for certain policies $\pi$. Thus, we are motivated to make the following definition.

\begin{definition}[Emulator for reachable $\ell_1$-bounded linear MDP]\label{def:not-core-set}
    Fix $h \in [H]$. For any $m \in \NN$ and $\epapx, \epneg, C > 0$, a set of vectors $(\hat\mu_{h+1}^j)_{j=1}^m \subset \RR^d$ is a \emph{$(\epapx, \epneg, C)$-\coreset} for the MDP $M$ at step $h$ if the following conditions hold:
\begin{enumerate}
    \item\label{it:emulator-C-bound} $\sum_{j=1}^m \norm{\hat\mu_{h+1}^j}_1 \leq C$.
    \item\label{it:emulator-epneg-bound} For any policy $\pi \in \Pi$ and $j \in [m]$,
    \[ \langle \EE\sups{M, \pi}[\phi_h(x_h,a_h)], \hat\mu_{h+1}^j\rangle \geq -\epneg\norm{\hat\mu_{h+1}^j}_1.\]
    \item\label{it:emulator-epapx-bound} There are states $(\tilde x^j)_{j=1}^m \subseteq \MX$ so that for any policy $\pi \in \Pi$,
    \[ \norm{\sum_{x \in \MX} \langle \EE\sups{M, \pi}[\phi_h(x_h,a_h)], \mu_{h+1}(x)\rangle \cdot \phiavg_{h+1}(x) - \sum_{j = 1}^m \langle \EE\sups{M, \pi}[\phi_h(x_h,a_h)], \hat\mu_{h+1}^j\rangle \cdot \phiavg_{h+1}(\tilde{x}^j)}_\infty \leq \epapx.\]
\end{enumerate}
\end{definition}

Our main result about \PC is that this definition encapsulates the needed properties: if $\{\hat\mu_{h+1}^1,\dots,\hat\mu_{h+1}^m\}$ is an \coreset with appropriately small error parameters (relative to the reachability parameter $\eta$), then the output $\Psiapx_{h+1}$ satisfies that $\Psi_{h+2} := \{\pi\circ_{h+1}\unif(\MA): \pi \in\Psiapx_{h+1}\}$ is an approximate policy cover for step $h+2$. 
See \cref{lemma:pc-coreset-guarantee} for the formal statement.

\subsection{A convex program for finding an \coreset}
\label{sec:emulator-overview}
Notice that there trivially exists a $(0,0,\Cnrm)$-\coreset of size $|\MX|$: namely, the set of vectors $\{\mu_{h+1}(x): x \in \MX\}$. The key question is whether there is a \emph{small} \coreset, and how to construct it efficiently. We answer both of these questions together. In \cref{alg:esc}, we give a convex program that takes as input a policy cover $\Psi_h$ for step $h$ as well as parameters $n,m \in \BN$, uses the cover to draw trajectory data $\MC_h = (x_h^i,a_h^i,x_{h+1}^i)_{i=1}^n$ and $\MD_h = (\tilde{x}_{h+1}^j)_{j=1}^m$, and then solves a convex program to produce an \coreset of size $m$ at step $h$, with one vector for each element of $\MD_h$. There are two criteria to check: feasibility of the program, and correctness of any feasible solution.

\begin{algorithm}[t]
	\caption{$\ESC(h, \Psi_h, \epcvx, \Cnrm,n,m)$}
	\label{alg:esc}
	\begin{algorithmic}[1]\onehalfspacing
		\Require~ Step $h \in [H]$; policy cover $\Psi_h$; tolerance $\epcvx>0$; norm parameter $\Cnrm > 0$; ``transition'' dataset size $n$; ``state'' dataset size $m$
        \State $\MC_h,\MD_h \gets \DrawData(h,\Psi_h,n,m)$ \Comment{\cref{alg:dd}}
  \LineComment{{We write $\MC_h = \{ (x_h^i, a_h^i, x_{h+1}^i) \}_{i=1}^n$, and $\MD_h = \{ \tilde x_{h+1}^j \}_{j=1}^m$. }}        
        \State For each $\ell \in [d]$, solve $\ell_1$-constrained regression
        \begin{equation} \hat \bw_{\ell} := \argmin_{w \in \RR^d: \norm{w}_1 \leq \Cnrm} \sum_{i=1}^n \left(\langle \phi_h(x_h^i, a_h^i), w\rangle - \phiavg_{h+1}(x_{h+1}^i)_\ell\right)^2
        \label{eq:wj-guarantee}
        \end{equation}
        \State Find $\hat\mu_{h+1}^1,\dots,\hat\mu_{h+1}^m \in \RR^d$ satisfying the following convex program \cref{eq:program}:\footnotemark
        
        \begin{subequations}
        \label{eq:program}
        \begin{align}
        \sum_{j=1}^m \norm{\hat\mu_{h+1}^j}_1 &\leq \Cnrm & \label{eq:hatmu-cnrm-constraint}\\
        \langle \phi_h(x_h^i,a_h^i), \hat\mu_{h+1}^j\rangle &\geq 0 &\forall i \in [n], j \in [m] \label{eq:hatmu-nneg-constraint} \\
        \frac{1}{n} \sum_{i=1}^n \left(\langle \phi_h(x_h^i,a_h^i), \hat \bw_\ell\rangle - \sum_{j=1}^m \left\langle \phi_h(x_h^i,a_h^i), \hat\mu_{h+1}^j \right \rangle \phiavg_{h+1}(\tilde x_{h+1}^j)_\ell \right)^2 &\leq \epcvx^2 &\forall \ell \in [d] \label{eq:what-wprime-constraint}
        \end{align}
        \end{subequations}
		\State \textbf{Return:} $(\hat\mu_{h+1}^j)_{j=1}^m$ if \cref{eq:program} is feasible, otherwise $\perp$
	\end{algorithmic}
\end{algorithm}

\footnotetext{For simplicity, we assume in the main body of the paper that it is possible to find an exact solution to the convex program (whenever it is feasible) in polynomial time. However, solving it approximately suffices; we postpone these optimization details to \cref{sec:opt-details}.}

\paragraph{Correctness.} Correctness depends on choosing the datasets $\MC_h, \MD_h$ appropriately. We show (\cref{thm:muhat-coreset}) that so long as $n,m \gg \poly(\Cnrm,A,H,\alpha^{-1},\epneg^{-1},\epapx^{-1})$ (up to logarithmic factors) and $\MC_h,\MD_h$ are drawn from $\pi \circ_h \unif(\MA)$ for a uniformly random policy $\pi \sim \Unif(\Psi_h)$, then with high probability the output of \ESC is an $(\epapx,\epneg,\Cnrm)$-emulator at step $h$. In particular, constraint \cref{eq:hatmu-cnrm-constraint} of \cref{alg:esc} immediately implies that \cref{it:emulator-C-bound} of \cref{def:not-core-set} is satisfied. By a concentration argument, constraint \cref{eq:hatmu-nneg-constraint} implies that \cref{it:emulator-epneg-bound} is satisfied for policies in the policy cover $\Psi_h$. A rejection sampling argument, together with the coverage property of $\Psi_h$, establishes that this guarantee can be extended to all policies. Finally, constraint \cref{eq:what-wprime-constraint} together with generalization bounds implies that \cref{it:emulator-epapx-bound} holds for policies in $\Psi_h$, which again can be extended to all policies using the fact that $\Psi_h$ is an $\alpha$-approximate policy cover at step $h$.

\paragraph{Feasibility.} In \cref{lem:randomize-no-trunc,lem:feasibility}, we give a constructive proof of feasibility of the convex program. For each state $\tilde{x}_{h+1}^j$ in $\MD_h$, we let $\psi_{h+1}^j$ be an importance-weighted rescaling of the true feature vector $\mu_{h+1}(\tilde{x}_{h+1}^j)$, so that if $\tilde{x}_{h+1}^j$ had low likelihood of being drawn under the distribution of $\MD_h$, then $\psi_{h+1}^j$ is scaled up, and vice versa. As a result of the importance-weighting, for any $(x,a)\in\MX\times\MA$ and function $g: \MX\to\RR$, the vector $\psi_{h+1}^j$ yields an unbiased estimate $\langle \phi_h(x,a), \psi_{h+1}^j\rangle \cdot g(\tilde{x}_{h+1}^j)$ of the population mean $\E_{x'\sim\BP_h(x,a)}[g(x')]$. Moreover, the coverage property of $\Psi_h$ and the reachability assumption on the MDP imply that none of the vectors $\psi_{h+1}^j$ are too large, so empirical averages concentrate. This is the key to showing that $(\psi_{h+1}^j)_{j=1}^m$ satisfy \cref{eq:what-wprime-constraint} with high probability; the other constraints hold by similar (or simpler) arguments.

\paragraph{Putting it all together.} Summarizing, once we have shown that the convex program in \cref{alg:esc} is feasible and that any solution to it yields an emulator at any given step, we can then pass this emulator to \PC (\cref{alg:pc}), which constructs a policy cover at this step. Our final algorithm \SLM (\cref{alg:slm}) simply repeats this procedure for each step $h \in [H]$. The complete analysis of \SLM is given in \cref{sec:policy-cover-reachable}.

\begin{algorithm}[t]
    \caption{$\DrawData(h,\Psi_h,n,m)$}
    \label{alg:dd}
    \begin{algorithmic}[1]\onehalfspacing
            \State $\MC_h, \MD_h \gets \emptyset$
            \For{$1 \leq i \leq n$}
                \State Sample $(x_1^i,a_1^i,\dots,x_h^i,a_h^i ,x_{h+1}^i) \sim \unif(\Psi_h) \circ_h \unif(\MA)$
                \State Update dataset: $\MC_h \gets \MC_h \cup \{(x_h^i,a_h^i,x_{h+1}^i)\}$
            \EndFor
            \For{$1 \leq j \leq m$}
                \State Sample $(\tilde{x}_1^j,\tilde{a}_1^j,\dots,\tilde x_h^j, \tilde a_h^j, \tilde x_{h+1}^j) \sim \unif(\Psi_h) \circ_h \unif(\MA)$
                \State Update dataset: $\MD_h \gets \MD_h \cup \{\tilde x_{h+1}^j\}$
            \EndFor
            \State \textbf{Return:} $\MC_h$, $\MD_h$
    \end{algorithmic}
\end{algorithm}

\begin{algorithm}[t]
	\caption{$\SLM(\delta)$: Explore Reachable $\ell_1$-Bounded Linear MDP}
	\label{alg:slm}
	\begin{algorithmic}[1]\onehalfspacing
		\Require Failure probability $\delta\in(0,1)$.
		\State Let $\pi_\unif = \unif(\MA) \circ \unif(\MA) \circ \dots \circ \unif(\MA)$ be the uniform policy
        \State Set $\Psi_1 = \Psi_2 = \{\pi_\unif\}$
        \State \label{line:xi-reachable} $\xi \gets \etarch/(9A)$
        \State $\epapx \gets \xi^2 \etarch /(54\Cnrm^2 A)$ and $\epneg \gets \xi^2 \etarch/(102\Cnrm^3 A)$ and $\alpha \gets \xi^2/(16\Cnrm^2 A)$ \label{line:eps-reachable}
        \State \label{line:N-reachable} $N \gets \max(\nstatrch(\xi/(4\Cnrm),\alpha,\delta\xi/(12H\Cnrm)), \nfe(\xi/(4\Cnrm), \delta\xi/(12H\Cnrm))$
        \State \label{line:nm-reachable} Set $n,m$ as
        \[n \gets C_{\ref{thm:muhat-coreset}}\max\left(\frac{\Cnrm^2 A H^3 \log(8AH\Cnrm/\epneg)\log(24dH/\delta)}{\alpha \epneg^4}, \frac{\Cnrm^4 A^2 \log(12dH/\delta)}{\alpha^2 \epapx^4}\right)\]
        \[m \gets C_{\ref{thm:muhat-coreset}} \frac{A^3\log(24dnH/\delta)}{\alpha^3\etarch\epapx^2}.\]
        \For{$1 \leq h \leq H-2$}

            \State $(\hat\mu_{h+1}^j)_{j=1}^m \gets \ESC(h,\Psi_h,\epapx\sqrt{\alpha/(4A)},\Cnrm,n,m)$\Comment{\cref{alg:esc}}\label{line:call-esc}
            \State $\Psiapx_{h+1} \gets \PC(h,(\hat\mu_{h+1}^j)_{j=1}^m,\xi,\Cnrm,\Psi_{1:h-1},N)$\Comment{\cref{alg:pc}}
            \State $\Psi_{h+2} \gets \{\pi \circ_{h+1} \unif(\MA): \pi \in \bar\Psi_{h+1}\}$
        \EndFor
        \State \textbf{Return:} $\Psi_{1:H}$
	\end{algorithmic}
\end{algorithm}

\section{Technical Overview II: beyond reachability}\label{section:trunc-overview}

We now drop the assumption that $M$ is $\eta$-reachable (\cref{defn:rch}). In the previous section, this assumption was crucially used in the final step of proving that $\Psi_{h+2}$ is an $\alpha$-policy cover: an empirical version of \cref{eq:h2-pc-intro} (which takes into account an approximation error $\ep$ between the emulator and the true feature vectors $\mu_{h+1}(x)$) states that for any policy $\pi$ and state $x$, the visitation probability $d^{M,\pi}_{h+2}(x)$ is upper bounded by $\poly(\Cnrm,\xi^{-1}) \cdot \E_{\pi' \sim \unif(\Psi_{h+2})} d^{M,\pi'}_{h+2}(x)$ plus some additive error term $O(\epsilon + \xi) \cdot \norm{\mu_{h+2}(x)}_1$. In the case that $M$ is reachable, we can use reachability to show that the error term can be removed (see the discussion after \cref{eq:h2-pc-intro}). But without reachability, if $\max_{\pi \in \Pi} d^{M,\pi}_{h+2}(x)$ is very small (relative to $\norm{\mu_{h+2}(x)}_1$), then it could be that $d^{M,\pi'}_{h+2}(x) = 0$ for all $\pi' \in \Psi_{h+2}$, in which case $\Psi_{h+2}$ is not an approximate policy cover in the sense of \cref{defn:pc}. 

To try to fix this analysis, one approach may be to essentially ignore the hard-to-reach states because they cannot contribute significant reward to any policy anyways. Relaxing the definition of policy cover, we could define an $(\alpha,\epsilon)$-policy cover $\Psi$ at step $h$ to be a set of policies satisfying the bound \[\E_{\pi'\sim\unif(\Psi)}\left[ d^{\pi'}_h(x)\right] \geq \alpha \cdot \max_{\pi\in\Pi} d^\pi_h(x) - \epsilon \cdot \norm{\mu_h(x)}_1\] for all $x \in \MX$. Then consider the relaxed inductive hypothesis that $\Psi_{1:h}$ are $(\alpha,\epsilon)$-policy covers at steps $1,\dots,h$. Unfortunately, this has an issue of compounding error: the best we could hope to show is that $\Psi_{h+2}$ is an $(\alpha,C\epsilon)$-policy cover for a constant $C>1$. As a result, by the final layer we will incur additive error exponential in the horizon $H$.

This is a common technical obstacle in the analysis of many reinforcement learning algorithms that explore layer by layer: a relatively simple analysis that works in the reachable setting breaks down in the general case. Fortunately, often the algorithm still provably works in the general case, and all that needs to change is the proof. The proof for the general case typically proceeds by analyzing the exploratory guarantees of the algorithm on input $M$, with respect to a \emph{truncated} MDP $\Mbar$ that is both statistically close to $M$ and also \emph{reachable}. In a learning algorithm based on constructing policy covers layer by layer, the inductive hypothesis is (roughly) that $\Psi_{1:h}$ are approximate policy covers for steps $1,\dots,h$ with respect to the truncated MDP $\Mbar$. 

\paragraph{A first attempt: same algorithm, different analysis.} For a general MDP $M$, a truncated MDP $\Mbar(\emptyset)$ can be obtained from $M$ by adding a ``terminal'' state $\term$ at every layer (we use the notation $\emptyset$ for reasons that will become clear later). We construct the transitions of $\Mbar(\emptyset)$  by ``truncating'' those of $M$ step-by-step, in the order of increasing $h$. At each step $h$, we will have an intermediate MDP $\Mbar_h(\emptyset)$ for which the transitions at steps $1$ through $h-1$ are truncated, and the transitions at step $h$ onwards are identical to those of $M$. We then let $\Xreach_{h+1}(\emptyset)$ denote the set of ``reachable'' states, i.e. the states for which $\max_{\pi\in\Pi} d\sups{\Mbar_h(\emptyset),\pi}_{h+1}(x) \geq \trunc \cdot \norm{\mu_{h+1}(x)}_1$ for some threshold $\trunc>0$. Whenever some state at layer $h$ is about to transition to a state \emph{outside} of $\Xreach_{h+1}(\emptyset)$, we redefine the corresponding transition in $\Mbar_{h+1}(\emptyset)$ to instead go to $\term$. Also, $\term$ always transitions to $\term$. We then let $\Mbar(\emptyset)$ denote the final truncated MDP in this process, i.e., $\Mbar(\emptyset) = \Mbar_H(\emptyset)$. %

By an inductive argument, one can show that $\Mbar(\emptyset)$ is both $\trunc$-reachable (except possibly at the state $\term$) and close to $M$ (since every policy transitions to $\term$ with probability at most $\trunc \cdot \poly(\Cnrm,H)$). Moreover, when $M$ is a linear MDP, one can define $\Mbar(\emptyset)$ to respect the linearity by augmenting the feature vectors with an extra dimension. We formally define $\Mbar(\emptyset)$ in \cref{sec:trunc-mdps}.

Now, we would like to inductively argue that the sets $\Psi_{h}$ computed by \SLM are ``truncated policy covers'' in the sense that a uniformly random policy from $\Psi_h$ covers each state in proportion to its maximum visitation probability  \emph{under the truncated MDP} $\Mbar(\emptyset)$, as defined formally in \cref{def:ih-trunc}. Since $\Mbar(\emptyset)$ is close to $M$, once we have constructed truncated policy covers at all steps $h \in [H]$, it can be shown that \PSDP finds a near-optimal policy (\cref{lem:psdp-rew}), just like in the reachable setting.

\begin{defn}[Truncated policy cover]
  \label{def:ih-trunc}
Consider $h \in [H]$ and $\Psi_h \subset \Pi$.  For $\alpha > 0$, we say that $\Psi_h$ is an \emph{$\alpha$-truncated policy cover} at step $h$ if 
  for all $x \in \MX$,  %
\begin{align}
\frac{1}{|\Psi_h|} \sum_{\pi' \in \Psi_h} d_h\sups{M, \pi'}(x) \geq \alpha \cdot \max_{\pi \in \Pi} d_h\sups{\bar M(\emptyset), \pi}(x)\label{eq:pc-trunc}.
\end{align}
\end{defn}

We will show in  \cref{lem:reachability-in-trunc} that the truncated MDP $\Mbar(\emptyset)$ is $\trunc$-reachable. Thus, a similar argument to that of the reachable case guarantees that we can avoid the aforementioned issue of an additive error term in the coverage inequality \cref{eq:pc-trunc}. 
However, since \cref{def:ih-trunc} is weaker than \cref{defn:pc}, it remains to check whether the inductive hypothesis that $\Psi_{1:h}$ are $\alpha$-\emph{truncated} policy covers (per \cref{def:ih-trunc}) for steps $1,\dots,h$ suffices for the rest of the proof.

Under the new inductive hypothesis, we can prove that the output of \ESCt (\cref{alg:esc-trunc}) is with high probability a \emph{truncated} emulator, as defined below. This definition is (roughly) the analogue of \cref{def:not-core-set} for the truncated MDP $\Mbar(\emptyset)$.

\begin{definition}[Weaker version of \cref{def:trunc-core}]\label{def:trunc-core-overview}
    Fix $h \in [H]$. For any $m \in \NN$ and $\epapx, \epneg, C > 0$, a set of vectors $(\hat\mu^j)_{j=1}^m$ is a \emph{$(\epapx, \epneg, C)$-\trunccore} for the MDP at step $h$ if the following conditions hold:
\begin{enumerate}
    \item\label{it:norm-bound-ov} $\sum_{j=1}^m \norm{\hat\mu^j}_1 \leq C$
    \item \label{it:approx-nonneg-ov} For any policy $\pi \in \Pi$ and $j \in [m]$,
      \begin{align} \EE\sups{\Mbar(\emptyset), \pi}\left[ \max \left\{ 0, \max_{a \in \MA}- \langle \phi_h(x_h,a), \hat\mu^j\rangle \right\} \right] &\leq   \epneg\norm{\hat\mu^j}_1.
        \end{align}
    \item \label{it:pol-approx-ov} There are states $(\tilde x^j)_{j=1}^m \subseteq \MX$ so that for any policy $\pi \in \Pi$,
      \begin{align}
                \norm{\sum_{x \in \MX} \langle \EE\sups{\Mbar(\emptyset), \pi}[\phi_h(x_h,a_h)], \mu_{h+1}(x)\rangle \cdot \phiavg_{h+1}(x) - \sum_{j = 1}^m \langle \EE\sups{\Mbar(\emptyset), \pi}[\phi_h(x_h,a_h)], \hat\mu^j\rangle \cdot \phiavg_{h+1}(\tilde{x}^j)}_\infty \leq & \epapx\nonumber.
      \end{align}
\end{enumerate}
\end{definition}

The next step would be to analyze \PC (\cref{alg:pc}), proving an analogue of \cref{lem:emp-pc-guarantee-rch} under the weaker assumption that the inputs $\Psi_{1:h}$ provided to \PC are only \emph{truncated} policy covers. Unfortunately, at this point we encounter an issue. Recall that the sets $\Psi_{1:h}$ are needed for the policy optimization algorithm \PSDP, which is invoked by \PC. Whereas true policy covers certify an upper bound on $\max_{\pi \in \Pi} d^{M,\pi}_h(x)$ (i.e., by \cref{eq:approx-pc}), truncated policy covers only certify an upper bound on $\max_{\pi \in \Pi} d^{\Mbar(\emptyset),\pi}_h(x)$. Hence, when \PSDP is passed truncated policy covers as input, we incur an additional error term, due to the discrepancy between $\Mbar(\emptyset)$ and $M$, in the suboptimality of the output policy. With a naive analysis, this error term scales with the probability that the optimal policy $\pi^\st$ (for the given reward vector) ever visits a ``truncated'' state $x \in \MX\setminus\Xreach_g(\emptyset)$ for all $g \leq h$, which can bounded by a multiple of the truncation threshold $\trunc$ that was used to define $\Mbar(\emptyset)$. Incurring this error at the last step of the algorithm, once $\Psi_{1:H}$ have been constructed and we are applying \PSDP with the environmental rewards, is fine. Indeed, this is what we do in \cref{lem:psdp-rew}. However, incurring this error in each induction step, while constructing the covers $\Psi_{1:H}$, would be problematic: completing the induction step and proving that we have found a policy cover for step $h+2$ requires that all additive errors are smaller than the reachability parameter of $\Mbar(\emptyset)$ \--- which is precisely $\trunc$!

The key to solving this issue is a refined analysis of \PSDP with truncated policy covers, which lets us implement a ``win-win'' argument. Above, we observed that \PSDP only fails if the optimal policy $\pi^\st$ visits the truncated states with non-trivial probability. However, since we do not know $\pi^\st$, this is not algorithmically useful. In \cref{lem:psdp-trunc} (with $\Gamma=\emptyset$), we show that the new error term in the suboptimality of the policy $\hat\pi$  produced by \PSDP can be bounded in terms of the probabilities that the policies $\Gamma^1 := \{\pi' \circ_{h-1} \unif(\MA) \circ_h \hat\pi: h \in [H]\}$ visit truncated states. Crucially, we know all of these policies. Thus, if \PSDP fails and hence the policies in $\Gamma^1$ visit truncated states with non-trivial probability, we can rerun the entire algorithm, constructing new policy covers at every step. Moreover, in this rerun of the algorithm, $\DrawDataTrunc$ and $\PSDP$ will draw additional data from $\unif(\Gamma^1)$. We can then define a ``less-truncated'' MDP $\Mbar(\Gamma^1)$ which does not truncate states that are well-covered by policies in $\Gamma^1$ (see \cref{sec:trunc-mdps}). In this way, we have made progress, since \PSDP will only fail if some policy visits states truncated by $\Mbar(\Gamma^1)$. %
Repeating this argument, we may have to define a larger set $\Gamma^2 \subset \Pi$ when analyzing the rerun of the algorithm, and so forth, but in each repetition where PSDP fails, we must discover new states. Thus, after a bounded number of repetitions, there is at least one repetition where we do not discover any new states, and thus PSDP will succeed at all steps; in particular, we will succeed in constructing all the truncated policy covers $\Psi_{1:H}$.

\paragraph{Full algorithm.}
The full algorithm (\SLMt; \cref{alg:slm-trunc}) proceeds in multiple \emph{phases} $1 \leq t \leq T$. In each phase $t$, the algorithm has a set $\Gamma^t \subset \Pi$ of policies, as discussed in the previous paragraph. %
Each phase proceeds in $H$ steps: at each step $h$, the algorithm has already constructed policy covers $\Psi_{1:h}^t$. First (\cref{line:trunc-emulator}), it constructs a truncated emulator at step $h$ (\cref{def:trunc-core}). Then (\cref{line:trunc-greedycover}) it passes this emulator to \PC to construct a set of policies that yields a truncated policy cover at step $h+2$ (\cref{line:odd-psi}). It then uses the output of \PC to append a collection of new policies to the set $\Gamma^t$ (\cref{line:append-gammat}), which must explore new states in the event that some call to \PSDP within \PC fails. We will be able to show that, at some phase $t$, all calls to \PSDP succeed, and thus the resulting collection $\Psi_{1:H}^t$ is a truncated policy cover at all steps $h \in [H]$. Rather than attempt to determine such a value of $t$ directly, \cref{alg:slm-trunc} simply returns all of the policy covers $\Psi_{1:H}^t$ (for $t \in [T]$).  \OPT (\cref{alg:opt}) will then concatenate these policy covers together (potentially losing a factor of $T$ in the coverage bound) and compute a near-optimal policy using \PSDPrew (\cref{alg:psdp-rew}).%

Finally, we mention that in the inductive steps over steps $h \in [H]$, \cref{alg:slm-trunc} only considers odd values of $h$: notice that the argument sketched above establishes that if we have policy covers at all steps up to step $h$, then we can efficiently construct one at step $h+2$. To yield policy covers at even steps, we simply append all policies from a cover at the previous step with a uniformly random action (\cref{line:even-psi}), at the cost of losing  a factor of $A$ in the coverage parameter. 

The formal guarantees of \SLMt are given in \cref{theorem:main-pc-trunc}, which is then used to analyze \OPT and prove \cref{thm:main-informal} in \cref{sec:learning}.

\section{Constructing a policy cover in a reachable LMDP}
\label{sec:policy-cover-reachable}

Throughout this section we assume that $M$ is an $\etarch$-reachable (\cref{defn:rch}), $\ell_1$-bounded $d$-dimensional linear MDP (\cref{defn:l1lmdp}) with bound $\Cnrm$, for known parameters $\etarch>0$ and $\Cnrm \geq 1$. The main result of the section is an analysis of \SLM (\cref{alg:slm}); we show that given interactive access to $M$ (as formalized in \cref{sec:prelim-linear}), the algorithm finds a policy cover for $M$ at all steps:

\begin{theorem}\label{theorem:main-pc}
  Let $\delta>0$. Then with probability at least $1-\delta$, the output $\Psi_{1:H}$ of $\SLM(\delta)$ (\cref{alg:slm}) satisfies that $\Psi_h$ is an $(\etarch^2 / (1296\Cnrm^2 A^3))$-approximate policy cover (\cref{defn:pc}) for all $h \in [H]$. Moreover, the sample complexity of $\SLM(\delta)$ is $\poly(\Cnrm,A,H,\etarch^{-1}, \allowbreak\log(d/\delta))$, and the time complexity is $\poly(d,\Cnrm,A,H,\etarch^{-1},\log(1/\delta))$.
\end{theorem}

The proof is by induction on the step $h$. Suppose that the algorithm has constructed policy covers $\Psi_{1:h+1}$ for steps $1,\dots,h+1$; we then want to show that at step $h$, the algorithm constructs a policy cover for step $h+2$. As discussed in the overview, this proof is split into two modular pieces. First, we show that \ESC (\cref{alg:esc}) produces an \coreset for $M$ at step $h$ (\cref{def:not-core-set}):

\begin{lemma}\label{thm:muhat-coreset}
There is a universal constant $C_{\ref{thm:muhat-coreset}}$ so that the following holds. Fix $n,m \in \NN$, $h \in [H]$, and $\epneg, \epapx,\alpha,\delta > 0$. Let $\Psi_h$ be an $\alpha$-approximate policy cover for step $h$. Suppose that the following bounds hold:
\begin{align}
  \epcvx \leq & \sqrt{\frac{\alpha}{4A}} \cdot \epapx\nonumber\\
  n \geq &  C_{\ref{thm:muhat-coreset}}\max\left(\frac{\Cnrm^2 A H^3 \log(8AH\Cnrm/\epneg)\log(12d/\delta)}{\alpha \epneg^4}, \frac{\Cnrm^4 A^2 \log(6d/\delta)}{\alpha^2 \epapx^4}\right)\nonumber\\
  m \geq &  C_{\ref{thm:muhat-coreset}} \frac{A^3\log(12dn/\delta)}{\alpha^3\etarch\epapx^2}.\nonumber
  \end{align}

Then with probability at least $1-\delta$, the output of $\ESC(h, \Psi_h,\epapx\sqrt{\frac{\alpha}{4A}}, \Cnrm,n,m)$ is a $(\epapx,\epneg,\Cnrm)$-\coreset{} (\cref{def:not-core-set}) for step $h$. 
\end{lemma}

\begin{remark}[Sample and computational efficiency of \ESC]
  \label{rmk:convex-program-efficient}
  In the context of \cref{thm:muhat-coreset}, it is straightforward to see that the the sample cost of \ESC is $n+m$. For simplicity, we have assumed that the convex program \cref{eq:program} can be solved exactly in time $\poly(n,m,d)$, in which case the computational cost of \ESC is $\poly(n,m,d)$. In fact, since it is not known how to efficiently solve generic convex programs exactly, we need to be more careful with our analysis of the computational complexity. As shown in \cref{sec:opt-details}, in the event that \cref{eq:program} is feasible, the ellipsoid algorithm can compute, in time 
  $\poly(n,m,d, \log(\Cnrm/(\epapx\epcvx\epneg)))$, a solution to a relaxation of \cref{eq:program}, which suffices to compute a $(3\epapx, 3\epneg, \Cnrm)$-\coreset. By decreasing the parameters $\epapx, \epneg$ that are passed to \ESC in \SLM (\cref{alg:slm}) by a factor of 3 and adjusting $\epcvx, n,m$ appropriately, we obtain that an $(\epapx, \epneg, \Cnrm)$-\coreset can be computed in $\poly(n,m,d, \log(\Cnrm/(\epapx\epcvx\epneg)))$ time. 
\end{remark}

Second, we show in \cref{lemma:pc-coreset-guarantee} that \PC (\cref{alg:pc}), given an \coreset (and the earlier policy covers $\Psi_{1:h}$), produces a set of policies $\Psiapx_{h+1}$ such that extending each policy with a random action at step $h+1$ yields a policy cover for step $h+2$. We remark that \PC takes as input a parameter $N$ which is passed to its \PSDP and \FE subroutines and determines the number of samples used in the calls to \PSDP and \FE; the value of $N$ is determined by the function $\nstatrch(\cdot, \cdot, \cdot)$, which is defined in \cref{lem:psdp}, and the function $\nfe(\cdot,\cdot)$, which is defined in \cref{lem:fe}. 

\begin{lemma}\label{lemma:pc-coreset-guarantee}
Fix $h \in [H]$, $N \in \NN$, and $\epapx, \epneg,\alpha,\Cemp,\delta > 0$. Fix $\xi \in (0,\Cemp)$, and let $\Psi_{1:h}$ be $\alpha$-approximate policy covers for steps $1,\dots,h$ respectively. Let $(\hat\mu_{h+1}^\ell)_{\ell=1}^m$ be an $(\epapx,\epneg,\Cemp)$-\coreset{} (\cref{def:not-core-set}) for step $h$.

Suppose that $N \geq \max\{\nstatrch(\xi/(4\Cemp), \alpha, \delta\xi/(6\Cemp)), \nfe(\xi/(4\Cemp), \delta\xi/(6\Cemp))\}$. Also suppose that
\begin{equation} \frac{9\Cemp^2}{\xi^2} \epapx + \frac{17\Cemp^3}{\xi^2} \epneg + \frac{3}{2}\xi \leq \frac{\etarch}{2A}.
\label{eq:rch-param-constraint}
\end{equation}

Let $\Psiapx_{h+1}$ denote the output of \PC (\cref{alg:pc}) with parameters $\xi,\Cemp,N$. Define 
\[\Psi_{h+2} := \{\pi \circ_{h+1} \unif(\MA): \pi \in \Psiapx_{h+1}\}.\] Then with probability at least $1-\delta$, $\Psi_{h+2}$ is an $\xi^2/(16\Cemp^2 A)$-approximate policy cover (\cref{defn:pc}) for step $h+2$. %
\end{lemma}

Given these two pieces, the proof of \cref{theorem:main-pc} is immediate.

\begin{proof}[Proof of \cref{theorem:main-pc}]
  The parameters $\xi, \epapx, \epneg, \alpha, N, n, m$ are set in \cref{line:xi-reachable,line:eps-reachable,line:N-reachable,line:nm-reachable} of \cref{alg:slm}. 
  We prove the following statement for each $h \in [H]$ by induction: with probability at least $1-h\delta/H$, the sets $\Psi_{1:h}$ are $\alpha$-approximate policy covers for steps $1,\dots,h$ respectively. Since $\Psi_1$ is non-empty, it is a $1$-approximate policy cover (with probability $1$). Similarly, since $\Psi_2$ contains a policy that takes uniformly random actions at step 1, it is a $1/A$-approximate policy cover with probability $1$. Note that $1/A \geq \alpha$, so the induction statement holds for $h \in \{1,2\}$. 

Now fix $h \in [H-2]$ and suppose that $\Psi_k$ is an $\alpha$-approximate policy cover for step $k$, for all $k \in [h+1]$. By the induction hypothesis, this event holds with probability at least $1-(h+1)\delta/H$ over the randomness of the first $h-1$ iterations within \SLM. We now consider iteration $h$. By \cref{thm:muhat-coreset} and choice of $n,m$ (\cref{line:nm-reachable}), and the fact that the parameter $\epcvx$ passed to \ESC (on \cref{line:call-esc}) is given by $\epcvx = \sqrt{\alpha/(4A)} \cdot \epapx$, the set of vectors $(\hat\mu_{h+1}^j)_{j=1}^m$ is a $(\epapx,\epneg,\Cnrm)$-\coreset{} for step $h$ with probability at least $1-\delta/(2H)$. Suppose that this event occurs. Then applying \cref{lemma:pc-coreset-guarantee}, by choice of $N$ (\cref{line:N-reachable}) and $\epapx,\epneg,\xi$ (\cref{line:xi-reachable,line:eps-reachable}; so that \cref{eq:rch-param-constraint} is satisfied), we get that $\Psi_{h+2}$ is a $\xi^2/(16\Cnrm^2 A)$-approximate policy cover for step $h+2$, with probability at least $1-\delta/(2H)$. By choice of $\alpha$ (\cref{line:eps-reachable}) and a union bound, we conclude that with probability at least $1-(h+2)\delta/H$, we have that $\Psi_{1:k}$ are $\alpha$-approximate policy covers for steps $1$ through $h+2$. This completes the induction.

It remains to analyze the sample and time complexity of the algorithm, which are dominated by $H-2$ calls to each of \ESC and \PC. By \cref{rmk:convex-program-efficient}, the sample complexity of each call to \ESC is $n+m$, and by \cref{lem:pc-size-bound}, the sample complexity of each call to \PC is $O(\Cnrm N/\xi)$. By definition of $N$ (see \cref{lem:psdp} and \cref{lem:fe} for the definitions of $\nstatrch$ and $\nfe$), we have $N = O(H^4 A^2 \Cnrm^8 \alpha^{-4}\xi^{-4} \log(Hd\Cnrm/(\delta\xi)))$. %
By choice of $n$, $m$, $\xi$, and $\alpha$, the overall sample complexity of \SLM is $\poly(\Cnrm,A,H,\eta^{-1},\log(d/\delta))$.
A similar analysis shows that the overall time complexity is $\poly(d,\Cnrm,A,H,\eta^{-1},\log(1/\delta))$, assuming that the convex program \cref{eq:program} can be solved in polynomial time (as discussed in \cref{rmk:convex-program-efficient}, we remove this assumption in \cref{sec:opt-details}, showing that it suffices to approximately solve the program, which is possible in polynomial time via the ellipsoid algorithm). We also remark that linear policies (in particular, the ones output by \PSDP; \cref{alg:psdp}) can simply be represented at all points in the algorithm by the vectors $\hat{\bw}_h$ defining the policy. %
\end{proof}

\subsection{Convex Program for constructing \coreset}\label{sec:convex-program}
In this section we prove \cref{thm:muhat-coreset} by analyzing \ESC (\cref{alg:esc}). This algorithm takes as input a policy cover $\Psi_h$ for step $h$, and constructs datasets $\MC_h = (x_h^i,a_h^i,x_{h+1}^i)_{i=1}^n$ and $\MD_h = (\tilde{x}_h^j)_{j=1}^m$ according to \DrawData (\cref{alg:dd}), i.e. by repeatedly drawing a trajectory $(x_1,\dots,x_h)$ from a random policy $\pi$ in the given policy cover $\Psi_h$, then picking a uniformly random action $a_h \sim \unif(\MA)$, and drawing a subsequent state $x_{h+1}$. The algorithm \ESC then solves a convex program to construct an \coreset.

Since $\Psi_h$ is assumed to be an $\alpha$-approximate policy cover for step $h$ (\cref{defn:pc}), the datasets $\MC_h$ and $\MD_h$ ``cover'' the MDP in the following distributional sense:

\begin{lemma}\label{lemma:beta-nu-coverage}
Fix $n,m,h \in \NN$ and $\alpha>0$, and let $\Psi_h$ be an $\alpha$-approximate policy cover for the MDP $M$ at step $h$. Let $\MC_h,\MD_h$ be the outputs of $\DrawData(h,\Psi_h,n,m)$. Let $\iota_h \in \Delta(\MX\times\MA\times\MX)$ be the distribution of the i.i.d. samples $(x_h^i,a_h^i,x_{h+1}^i) \in \MC_h$, and let $\beta_{h+1} \in \Delta(\MX)$ be the distribution of the i.i.d. samples $(\tilde{x}_{h+1}^j) \in \MD_h$. Let $\nu_h(x,a)$ denote the marginal distribution of $(x,a)$ under $(x,a,x') \sim \iota_h$. Then we have the following guarantees:
\begin{align*}
\forall  x \in \MX, \ a \in \MA, \quad \nu_h(x,a) \geq \frac{\alpha}{A} \cdot \max_{\pi \in \Pi} d_h^\pi(x,a), \qquad \beta_{h+1}(x) \geq \frac{\alpha}{A} \cdot \max_{\pi \in \Pi} d_{h+1}^\pi(x)
\end{align*}

\end{lemma}

\begin{proof}
    Consider a sample $(x,a,x') \sim \iota_h$. The first bound is immediate from the definition of an $\alpha$-approximate policy cover (\cref{defn:pc}) together with the fact that $a \sim \unif(\MA)$ is independent of $x$. Next, for any $x' \in \MX$ and $\pi \in \Pi$, \[\beta_{h+1}(x') = \sum_{x\in\MX,a\in\MX} \nu_h(x,a)\BP_h(x'|x,a) \geq \frac{\alpha}{A} \sum_{x\in\MX,a\in\MA} d_h^\pi(x,a) = d_{h+1}^\pi(x')\]
    which proves the second bound.
\end{proof}

The two datasets serve two distinct purposes in \ESC, and the coverage property is crucial for both.

\paragraph{The ``state'' dataset $\MD_h$.} The \coreset produced by \ESC will consist of one vector for each of the $m$ states $\tilde{x}_{h+1}^j \in \MD_h$; in fact, the requisite states in \cref{it:emulator-epapx-bound} of \cref{def:not-core-set} will precisely be the elements of $\MD_h$. Intuitively, a set of states that ``misses'' some important parts of the MDP should not be able to emulate the entire MDP; indeed, the coverage property of $\MD_h$ is crucial in establishing feasibility of the convex program \cref{eq:program}.

\paragraph{The ``transitions'' dataset $\MC_h$.} The convex program will pick one vector for each state in $\MD_h$. The role of $\MC_h$ is to enforce empirical analogues of the non-negativity property (\cref{it:emulator-epneg-bound}) and approximation property (\cref{it:emulator-epapx-bound}) of an \coreset. Standard generalization bounds ensure that these properties hold for policies in $\Psi_h$, and the coverage property then ensures that in fact the properties hold for all policies. 

\begin{remark}
One might wonder whether the constraint \cref{eq:what-wprime-constraint} can be simplified; indeed, a more obvious empirical analogue of \cref{it:emulator-epapx-bound} might be the constraint
\[\frac{1}{n}\sum_{i=1}^n \left(\phiavg_{h+1}(x_{h+1}^i)_\ell - \sum_{j=1}^m \left\langle \phi_h(x_h^i,a_h^i),\hat\mu_{h+1}^j\right\rangle \phiavg_{h+1}(\tilde{x}_{h+1}^j)_\ell\right)^2 \leq \epcvx^2\]
which avoids needing to compute the Lasso solution $\hat{\bw}_\ell$. The term $\phiavg_{h+1}(x_{h+1}^i)_\ell$ is an unbiased estimate of the quantity that $\langle \phi_h(x_h^i,a_h^i), \hat{\bw}_\ell\rangle$ is trying to approximate. However, since it may have $\Omega(1)$ variance, the above constraint may not be feasible for $\epcvx \ll 1$. Pre-computing the Lasso predictions for each coordinate of $\phiavg$ decreases the variance, thus avoiding this issue.
\end{remark}

\subsubsection{Feasibility}
We begin by showing that the convex program \cref{eq:program} is feasible, with high probability over the datasets $\MC_h,\MD_h$. In fact, for any choice of $(x_h^i,a_h^i)_{i=1}^n$, the program is feasible with high probability over the conditional samples $x_{h+1}^i \sim \BP_h(\cdot|x_h^i,a_h^i)$ and the dataset $\MD_h$. The construction is simple. Let $\beta_{h+1} \in \Delta(\MX)$ be the distribution of the states $\tilde{x}_{h+1}^j$. Then for each $j \in [m]$, we can define 
\[\psi^j_{h+1} := \frac{1}{m} \cdot \frac{\mu_{h+1}(\tilde{x}_{h+1}^j)}{\beta_{h+1}(\tilde{x}_{h+1}^j)}.\]
By definition of $\beta_{h+1}$, for any fixed $(x,a)\in\MX$, it's clear that $\langle \phi_h(x,a), m\psi_{h+1}^j\rangle \cdot \phiavg_{h+1}(\tilde{x}_{h+1}^j)_\ell$ is an unbiased estimate of $\E_{x' \sim \BP_h(x,a)} \phiavg_{h+1}(x')_\ell$. Moreover, by the coverage property of $\beta_{h+1}$ (together with the assumption that the MDP is $\etarch$-reachable), the vectors $\psi_{h+1}^j$ are bounded, so the empirical average $\sum_{j=1}^m \langle \phi_h(x,a), \psi_{h+1}^j\rangle \cdot \phiavg_{h+1}(\tilde{x}_{h+1}^j)_\ell$ concentrates with high probability over $\MD_h$. Next, standard prediction error bounds for Lasso imply that \[\frac{1}{n}\sum_{i=1}^n \left(\langle \phi_h(x_h^i,a_h^i), \hat{\bw}_\ell\rangle - \E_{x' \sim \BP_h(x_h^i,a_h^i)} \left[\phiavg_{h+1}(x')_\ell\right]\right)^2\]
is small with high probability over the conditional samples $x_{h+1}^i \sim \BP_h(\cdot|x_h^i,a_h^i)$. Combining these two pieces shows that the above choice of $\psi_{h+1}^1,\dots,\psi_{h+1}^m$ satisfies \cref{it:emulator-epapx-bound} of \cref{def:not-core-set} with high probability (with respect to the states $\tilde{x}_{h+1}^1,\dots,\tilde{x}_{h+1}^m$). \cref{it:emulator-epneg-bound} is immediate from the fact that $\langle\phi_h(x,a),\mu_{h+1}(x')\rangle \geq 0$ for all $x,x'\in\MX$, $a \in \MA$. Finally, \cref{it:emulator-C-bound} holds again by analyzing the expectation and proving concentration \--- there is a small loss, so in our formal feasibility construction we actually scale the vectors $\psi_{h+1}^j$ down slightly, but this is not important and could be avoided by slightly loosening the constraint in \cref{it:emulator-C-bound}.

The following lemma formalizes the unbiasedness/concentration parts of the above argument, and \cref{lem:feasibility} combines \cref{lem:randomize-no-trunc} with standard Lasso guarantees to complete the proof of feasibility.

\begin{lemma}
  \label{lem:randomize-no-trunc}
  Let $\delta,\zeta > 0$ and $m \in \NN$. Let $\MG$ be a set of functions $g : \MX \ra [-1,1]$. 
  Fix $h \in [H-1]$, and a distribution $\beta_{h+1} \in \Delta(\MX)$ satisfying \begin{equation}
  \beta_{h+1}(x) \geq \zeta \cdot \max_{\pi \in \Pi} d_{h+1}^\pi(x) \qquad \forall x \in \MX. 
  \label{eq:beta-hyp}
  \end{equation} Let $\MD_h = \{ \tilde{x}_{h+1}^j \}_{j=1}^m$ consist of i.i.d. draws $\tilde x_{h+1}^j \sim \beta_{h+1}$.
  Then there are vectors $\psi_{h+1}^1, \ldots, \psi_{h+1}^m \in \BR^d$, depending on the dataset $\MD_h$, with the following property. For any fixed $(x,a) \in \MX \times \MA$, with probability at least $1-\delta$ over the draw of the dataset $\MD_h$ (and the ensuing $\psi_{h+1}^j$), we have that:
  \begin{enumerate}
  \item For all $i \in [m]$, $\lng \phi_h(x,a), \psi_{h+1}^i \rng \geq 0$.
  \item For all $g \in \MG$, it holds that
    \begin{align}
\left| \E_{x' \sim \BP_h(x,a)} [g(x')] - \sum_{i=1}^m \lng \phi_h(x,a), \psi_{h+1}^i \rng \cdot g(\tilde x_{h+1}^i) \right| \leq \frac{2}{\zeta\etarch}\sqrt{\frac{2\log(2|\MG|/\delta)}{m}}\nonumber.
    \end{align}
  \item It holds that $\sum_{j=1}^m \| \psi_{h+1}^j \|_1 \leq \Cnrm$. 
  \end{enumerate}
\end{lemma}
\begin{proof}
  Set $\epsilon := \frac{2}{\zeta\etarch} \sqrt{\frac{2\log(2|\MG|/\delta)}{m}}$. For each $j \in [m]$, we define the (random) vector $\psi_{h+1}^j := \frac{1-\ep/2}{m}\cdot  \frac{\mu_{h+1}(\tilde x_{h+1}^j)}{\beta_{h+1}(\tilde x_{h+1}^j)}$.
  Note that for any $x' \in \MX$, we have by \cref{eq:beta-hyp} and reachability (\cref{defn:rch}) that
  \begin{equation} \beta_{h+1}(x') \geq \zeta \cdot \max_{\pi \in \Pi} d_{h+1}^\pi(x') \geq \zeta \etarch \cdot \norm{\mu_{h+1}(x')}_1.
\label{eq:beta-rch-lb}
  \end{equation}

   Now fix $(x,a) \in \MX\times \MA$. We use \cref{eq:beta-rch-lb} to prove the lemma's claims. 
  
  \paragraph{Proof of 1.} The first claimed statement of the lemma follows immediately from the fact that each $\psi_{h+1}^j$ is a rescaling of some $\mu_{h+1}(x')$ by a positive scalar.
  
  \paragraph{Proof of 2.} For each $j \in [m]$ and $g \in \MG$, we have with probability $1$ that
  \begin{align}
\left|\lng \phi_h(x,a), m\psi_{h+1}^j \rng \cdot g(\tilde x_{h+1}^j)\right| \leq \left\lng \phi_h(x,a), \frac{\mu_{h+1}(\tilde x_{h+1}^j)}{\beta_{h+1}(\tilde x_{h+1}^j)} \right\rng \leq \frac{\norm{\mu_{h+1}(\tilde x_{h+1}^j)}_1}{\beta_{h+1}(\tilde x_{h+1}^j)} \leq \frac{1}{\zeta \etarch}\nonumber,
  \end{align}
  by the bounds $|g| \leq 1$, $\| \phi_h(x,a) \|_\infty \leq 1$, and \cref{eq:beta-rch-lb}. Also, for any $g \in \MG$ we have that in expectation over the draw of $\tilde x_{h+1}^j \sim \beta_{h+1}$ (which determines $\hat{\mu}_{h+1}^j$),
  \begin{align}
    \E \left[ \lng \phi_h(x,a), m\psi_{h+1}^j \rng \cdot g(\tilde x_{h+1}^j) \right] =& (1 - \ep/2) \cdot \sum_{x' \in \MX} \beta_{h+1}(x') \cdot \left \lng \phi_h(x,a), \frac{\mu_{h+1}(x')}{\beta_{h+1}(x')} \right\rng \cdot g(x') \nonumber\\
    = & (1-\ep/2) \cdot \E_{x' \sim \BP_h(x,a)} [ g(x')]\nonumber.
  \end{align}
  By Hoeffding's inequality and a union bound over $g \in \MG$, it follows that with probability at least $1-\delta/2$ over the draw of $\MD_{h}$, for all $g \in \MG$,%
  \begin{align}
\left| (1-\ep/2) \cdot \E_{x' \sim \BP_h(x,a)}[g(x')] - \frac{1}{m} \sum_{j=1}^m \lng \phi_h(x,a), m\psi_{h+1}^j \rng \cdot g(\tilde x_{h+1}^j) \right| \leq \frac{1}{\zeta \etarch} \cdot \sqrt{ \frac{2\log (2|\MG|/\delta)}{m}} = \frac{\ep}{2}\label{eq:chernoff-mu-g}.
  \end{align}
  by choice of $\epsilon$. In the event that \cref{eq:chernoff-mu-g} holds, it follows that
\begin{align}
\left| \E_{x' \sim \BP_h(x,a)}[g(x')] - \frac{1}{m} \sum_{j=1}^m \lng \phi_h(x,a), m\psi_{h+1}^j \rng \cdot g(\tilde x_{h+1}^j) \right| \leq \ep
  \end{align}
  since $\EE_{x'\sim\BP_h(x,a)}[g(x')] \in [-1,1]$, which establishes the second claim of the lemma.

  \paragraph{Proof of 3.} To establish the final claimed statement, for each $j \in [m]$, we compute that in expectation over the draw of $\tilde x_{h+1}^j \sim \beta_{h+1}$,
  \begin{align}
\E \left[ \| m\psi_{h+1}^j \|_1 \right] = & (1-\ep/2) \cdot \sum_{x' \in \MX}\beta_{h+1}(x')  \cdot \frac{\norm{\mu_{h+1}(x')}_1}{\beta_{h+1}(x')} \leq (1-\ep/2) \cdot \Cnrm\nonumber.
  \end{align}
  Additionally, by \cref{eq:beta-rch-lb}, $\| m\psi_{h+1}^j \|_1 \leq 1/(\zeta \etarch)$ for all $j \in [m]$ with probability $1$, so Hoeffding's inequality gives that with probability at least $1-\delta/2$ over the draw of $\MD_{h+1}$,
  \begin{align}
\frac{1}{m} \sum_{j=1}^m \| m\psi_{h+1}^j \|_1 \leq (1-\ep/2) \cdot \Cnrm + \frac{1}{\zeta\etarch} \cdot \sqrt{\frac{2\log 2/\delta}{m}} \leq \Cnrm\nonumber,
  \end{align}
  by choice of $\epsilon$ and the fact that $\Cnrm \geq 1$. 
\end{proof}

\begin{lemma}\label{lem:feasibility}
There is a universal constant $C_{\ref{lem:feasibility}}$ so that the following holds. Let $\epcvx, \delta,\alpha > 0$ and $n,m \in \NN$. Fix $h \in [H-1]$ and suppose that $\Psi_h$ is an $\alpha$-approximate policy cover for step $h$. Suppose that the following bounds hold: \[n \geq C_{\ref{lem:feasibility}}\epcvx^{-4}\Cnrm^2 \log(2d/\delta),\]  \[m \geq C_{\ref{lem:feasibility}}\epcvx^{-2}A^2(\alpha\etarch)^{-2}\log(4dn/\delta).\] %
Then with probability at least $1-\delta$, the algorithm $\ESC(h, \Psi_h, \epcvx,\Cnrm,n,m)$ (\cref{alg:esc}) produces a solution $(\hat\mu_{h+1}^j)_{j=1}^m$ to Program~\cref{eq:program}.

\end{lemma}

\begin{proof}
We need to show that Program~\cref{eq:program} is feasible with high probability. Indeed, we show that for any fixed samples $(x_h^i,a_h^i)_{i=1}^n$, the program is feasible with high probability over the conditional samples $x_{h+1}^i \sim \BP_h(\cdot|x_h^i,a_h^i)$ and the dataset $\MD_h = \{ \tilde x_{h+1}^j \}_{j=1}^m $. 

First, we invoke \cref{lem:randomize-no-trunc} with failure probability $\delta/(2n)$, function class \[\MG = \{ x' \mapsto \phiavg_{h+1}(x')_\ell \ : \ \ell \in [d]\},\] and the dataset $\MD_h$ constructed in \cref{alg:esc}. By \cref{lemma:beta-nu-coverage}, the condition \cref{eq:beta-hyp} of \cref{lem:randomize-no-trunc} is satisfied with $\zeta := \alpha/A$.

Now, \cref{lem:randomize-no-trunc} gives vectors $\psi_{h+1}^1,\dots,\psi_{h+1}^m \in \RR^d$. Applying the lemma with $(x,a)$ equal to each of the $n$ pairs $(x_h^i,a_h^i)_{i=1}^n$, by a union bound over these $n$ pairs, the lemma implies that with probability at least $1-\delta/2$ over the draw of $\MD_h$, the following properties hold:

\begin{itemize}
\item For all $i \in [n]$ and $j \in [m]$,
  \[\lng \phi_h(x_h^i, a_h^i), \psi_{h+1}^j \rng \geq  0.\]
\item $\sum_{j=1}^m \| \psi_{h+1}^j \|_1 \leq \Cnrm$.
\item For all $i\in [n]$ and $\ell \in [d]$,
  \begin{equation}
  \left| \E_{x' \sim \BP_h(x_h^i,a_h^i)}[\phiavg_{h+1}(x')_\ell] - \sum_{j=1}^m \lng \phi_h(x_h^i, a_h^i), \psi_{h+1}^j \rng \cdot \phiavg_{h+1}(\tilde x_{h+1}^j)_\ell \right| \leq  \frac{2A}{\alpha\etarch} \sqrt{\frac{2\log(4dn/\delta)}{m}}\label{eq:p-hatmu-constraint}.
  \end{equation}
\end{itemize}

We claim that the tuple of vectors $(\psi_{h+1}^j)_{j=1}^m$ satisfies Program~\cref{eq:program} with high probability. Indeed, in the above event, it is immediate that \cref{eq:hatmu-cnrm-constraint} and \cref{eq:hatmu-nneg-constraint} are satisfied. It remains to show that \cref{eq:what-wprime-constraint} is satisfied with high probability.
  
  Fix $\ell \in [d]$, and define the vector $\bw_\ell^\st := \sum_{x' \in \MX} \mu_{h+1}(x') \cdot \phiavg_{h+1}(x')_\ell$, so that (by definition of $\BP_h(x,a)$),
  \[\E_{x' \sim \BP_h(x,a)}[\phiavg_{h+1}(x')_\ell] = \lng \bw_\ell^\st, \phi_h(x,a) \rng.\] Note that $\| \bw_\ell^\st \|_1 \leq \sum_{x'\in\MX}\norm{\mu_{h+1}(x')}_1 \leq \Cnrm$. Thus, we can apply \cref{lem:fixed-design-prediction-error} with covariates $X_i := \phi_h(x_h^i, a_h^i)$ ($i \in [n]$), ground truth $\bw_\ell^\st$, and responses $y_i := \phiavg_{h+1}(x_{h+1}^i)_\ell$. By definition of $\hat \bw_\ell$ in \cref{eq:wj-guarantee}, the lemma gives that, for any fixed $(x_h^i, a_h^i)_{i=1}^n$, with probability at least $1-\delta/(2d)$ over the conditional draws $x_{h+1}^i \sim \BP_h(\cdot | x_h^i, a_h^i)$, %
  \begin{align}
\frac 1n \sum_{i=1}^n  \left\lng \phi_h(x_h^i,a_h^i), \hat \bw_\ell - \bw_\ell^\st \right\rng^2 &\leq  C \cdot \Cnrm \cdot \sqrt{\frac{\log(2d^2/\delta)}{n}}\label{eq:wprime-wstar-rch}. 
  \end{align}
By the union bound, this inequality holds for all $\ell \in [d]$ with probability at least $1-\delta/2$. On the other hand, squaring the bound \cref{eq:p-hatmu-constraint} and averaging over $i \in [n]$ gives that with probability at least $1-\delta/2$, for all $\ell \in [d]$,
\begin{align}
\frac 1n \sum_{i=1}^n \left\lng \phi_h(x_h^i, a_h^i), \bw_\ell^\st - \sum_{j=1}^m \psi_{h+1}^j \phiavg_{h+1}(\tilde x_{h+1}^j)_\ell \right\rng^2 \leq \frac{8A^2}{\alpha^2\etarch^2}\frac{\log(4dn/\delta)}{m}.\label{eq:wstar-uhat}
\end{align}

Combining \cref{eq:wprime-wstar-rch} and \cref{eq:wstar-uhat} via the bound $(a+b)^2 \leq 2(a^2+b^2)$ gives that with probability at least $1-\delta$, for all $\ell \in [d]$,
\begin{align*}
\frac 1n \sum_{i=1}^n \left\lng \phi_h(x_h^i, a_h^i), \hat \bw_\ell - \sum_{j=1}^m \psi_{h+1}^j \phiavg_{h+1}(\tilde x_{h+1}^j)_\ell \right\rng^2 &\leq 2C\Cnrm\sqrt{\frac{2\log(2d/\delta)}{n}} + \frac{16A^2}{\alpha^2\etarch^2}\frac{\log(4dn/\delta)}{m}\\ 
&\leq \epcvx^2
\end{align*}
where the final inequality is by choice of $n$ and $m$. In this event, \cref{eq:what-wprime-constraint} is satisfied.
\end{proof}

\subsubsection{Approximate nonnegativity}

In this section we show that with high probability the output $(\hat\mu_{h+1}^j)_{j=1}^m$ of \ESC satisfies \cref{it:emulator-epneg-bound} of \cref{def:not-core-set}, i.e. for any policy $\pi$ and index $j \in [m]$, the expectation $\E^{\pi}[\langle \phi_h(x_h,a_h), \hat\mu_{h+1}^j\rangle]$ is approximately nonnegative. For a fixed policy $\pi$, we are able to show approximate nonnegativity using constraint \cref{eq:hatmu-nneg-constraint} of the convex program, which ensures that \[\langle \phi_h(x_h^i,a_h^i),\hat\mu_{h+1}^j\rangle \geq 0\] for all $j \in [m]$ and all datapoints $(x_h^i,a_h^i)$ from $\MC_h$. Specifically, we can rejection sample $\MC_h$ to generate a subsample $S = S(\pi) \subseteq [n]$ distributed according to $d^\pi_h(x,a)$; by the coverage property (\cref{lemma:beta-nu-coverage}) this subsample will be reasonably large, so a concentration argument shows that the empirical average $\frac{1}{|S|} \sum_{i\in S} \langle \phi_h(x_h^i,a_h^i),\hat\mu_{h+1}^j\rangle$, which is nonnegative by \cref{eq:hatmu-nneg-constraint}, will approximate the true expectation.

However, this only works for a fixed policy $\pi \in \Pi$. To show that approximate nonnegativity holds with high probability uniformly over policies $\pi$, we also need to invoke \cref{lemma:policy-disc}, which shows that there is a small set of policies $\Pidisc$ such that for any linear reward function, there is a near-optimal policy in $\Pidisc$. It then suffices to union bound over $\Pidisc$. Below, we make this argument formal.

\begin{lemma}\label{lemma:hat-mu-apx-nonnegative}
Let $n \in \NN$, $h \in [H]$, and $\epneg,\delta,\zeta \in (0,1)$. Let $(x_h^i, a_h^i)_{i=1}^n$ be i.i.d. samples from a distribution $\nu_h \in \Delta(\MX\times \MA)$ satisfying
\begin{equation}
\nu_h(x,a) \geq \zeta \cdot \max_{\pi \in\Pi} d^\pi_h(x,a) \qquad \forall (x,a) \in \MX\times \MA
\label{eq:nu-hyp}
\end{equation}
Suppose that
\[n \geq \frac{65536\Cnrm^2 H^3}{\epneg^4 \zeta}\log(8AH\Cnrm/\epneg)\log(4d/\delta).\] Then it holds with probability at least $1-\delta$ that for all $\pi \in \Pi$ and all $\theta \in \RR^d$,
\[ \langle \EE^\pi[\phi_h(x_h,a_h)], \theta\rangle \geq \inf_{i \in [n]} \langle \phi_h(x_h^i, a_h^i), \theta\rangle - \epneg\norm{\theta}_1.\]
\end{lemma}

\begin{proof}
Let $\Pidisc$ be the set of policies guaranteed by Lemma~\ref{lemma:policy-disc} with parameter $\epdisc := \epneg/2$. By the bound on $|\Pidisc|$ and choice of $n$, note that $n \geq 32\epneg^{-2}\zeta^{-1}\log(4d|\Pidisc|/\delta)$. Now fix some $\pi \in \Pidisc$. Define a random set $S \subseteq [n]$ by including index $i$ with probability $\frac{\zeta d_h^\pi(x_h^i, a_h^i)}{\nu_h(x_h^i,a_h^i)}.$ By \cref{eq:nu-hyp}, this nonnegative fraction is at most $1$, so the sampling procedure is well-defined. Moreover, in expectation over the randomness in both $(x_h^i,a_h^i)_{i=1}^n$ and $S$,
\[ \EE|S| = \sum_{i=1}^n \EE_{(x_h^i,a_h^i) \sim \nu_h} \frac{\zeta d_h^\pi(x_h^i,a_h^i)}{\nu_h(x_h^i,a_h^i)} = \zeta n.\]
By the Chernoff bound and the fact that $n \geq 8\zeta^{-1}\log(2|\Pidisc|/\delta)$, we have $\Pr[|S| \geq \zeta n/2] \geq 1-\delta/(2|\Pidisc|)$. Condition on $|S|$ and suppose that indeed $|S| \geq \zeta n/2$. Write $S = \{i_1 < \dots < i_{|S|}\}$, and for notational simplicity, let $(\tilde{x}^j,\tilde{a}^j)$ denote $(x^{i_j}_h, a^{i_j}_h)$ for each $1 \leq j \leq |S|$. By construction of the sampling procedure, as $j$ ranges from $1$ to $|S|$, the random variables $(\tilde{x}^j,\tilde{a}^j)$ are independent and identically distributed according to the density $(x,a) \mapsto d^\pi_h(x,a)$. Thus, by Hoeffding's inequality, with probability at least $1-2d\exp(-\epneg^2|S|/8) \geq 1-\delta/(2|\Pidisc|)$ over the draws $(\tilde{x}^j,\tilde{a}^j) \sim d_h^\pi$, we have
\[\norm{\EE^\pi[ \phi_h(x_h,a_h)] - \frac{1}{|S|} \sum_{j=1}^{|S|} \phi_h(\tilde{x}^j,\tilde{a}^j)}_\infty \leq \epneg/2.\]
In this event, we get that for any $\theta \in \RR^d$,
\begin{align*}
\langle \EE^\pi[\phi_h(x_h,a_h)], \theta\rangle
&\geq \frac{1}{|S|}\sum_{j=1}^{|S|} \langle\phi_h(\tilde{x}^j,\tilde{a}^j),\theta\rangle - \epneg \norm{\theta}_1/2 \\
&\geq \inf_{i\in [n]} \langle \phi_h(x_h^i,a_h^i), \theta\rangle - \epneg \norm{\theta}_1/2.
\end{align*}
Removing the conditioning on $|S|$, the above inequality holds with probability at least $1-\delta/|\Pidisc|$ over the randomness in $(x_h^i,a_h^i)_{i=1}^n$. By a union bound over $\pi \in \Pidisc$, with probability at least $1-\delta$, it holds for all $\pi \in \Pidisc$ and $\theta \in \RR^d$ that
\[\langle \EE^\pi[\phi_h(x_h,a_h)],\theta\rangle \geq \inf_{i \in [n]} \langle \phi_h(x_h^i,a_h^i),\theta\rangle - \epneg \norm{\theta}_1/2.\]
Suppose that this event holds. By the guarantee of Lemma~\ref{lemma:policy-disc}, we conclude that for all $\pi \in \Pi$ and $\theta \in \RR^d$,
\[ \langle \EE^\pi[\phi_h(x_h,a_h)],\theta\rangle \geq \inf_{i \in [n]} \langle \phi_h(x_h^i,a_h^i),\theta\rangle - \epneg\norm{\theta}_1,\]
which completes the proof. %
\end{proof}

\subsubsection{All-policy approximation}

It remains to show that the output $(\hat\mu_{h+1}^j)_{j=1}^m$ of \ESC satisfies \cref{it:emulator-epapx-bound} with high probability. This is a consequence of constraint \cref{eq:what-wprime-constraint} in \cref{eq:program}, together with generalization bounds for $\ell_1$-bounded linear predictors and the coverage property of the distribution from which the samples $\MC_h = (x_h^i,a_h^i)_{i=1}^n$ are drawn.

\begin{lemma}
  \label{lem:muhat-approx}
 There is a universal constant $C_{\ref{lem:muhat-approx}}$ so that the following holds. Fix $n,m \in \NN$, $h \in [H]$, and $\epcvx,\alpha,\delta>0$. Let $\Psi_h$ be an $\alpha$-approximate policy cover for step $h$. With probability at least $1-\delta$, the output of $\ESC(h,\Psi_h,\epcvx,\Cnrm,n,m)$ is either $\perp$ or a set of vectors $(\hat \mu_{h+1}^j)_{j=1}^m$ satisfying the following: there are states $(\tilde{x}_{h+1}^j)_{j=1}^m$ so that for any policy $\pi \in \Pi$,
  \begin{align*}
&\left\| \sum_{x \in \MX} \lng \E^\pi[\phi_h(x_h, a_h)], \mu_{h+1}(x) \rng \cdot \phiavg_{h+1}(x) - \sum_{j=1}^m \lng \E^\pi[\phi_h(x_h, a_h)], \hat \mu_{h+1}^j \rng \cdot \phiavg_{h+1}(\tilde x_{h+1}^j) \right\|_\infty^2 \\
&\leq \frac{2A}{\alpha}\epcvx^2 + C_{\ref{lem:muhat-approx}}\frac{\Cnrm^2 A\sqrt{\log(2d/\delta)}}{\alpha\sqrt{n}}.
  \end{align*}
\end{lemma}

\begin{proof}
  As in the proof of \cref{lem:feasibility}, for each $\ell \in [d]$ we define $\bw_\ell^\st := \sum_{x \in \MX} \mu_{h+1}(x) \cdot \phiavg_{h+1}(x)_\ell$, so that $\E_{x' \sim \BP_h(x,a)}[\phiavg_{h+1}(x')_\ell] = \lng \bw_\ell^\st, \phi_h(x,a) \rng$. Observing that $\norm{\bw_\ell^\st}_1 \leq \Cnrm$, we apply \cref{cor:random-design-prediction-error} to compare $\bw^\st_\ell$ with the $\ell_1$-constrained regressor $\hat{\bw}_\ell$ (defined in \cref{eq:wj-guarantee}).

  Specifically, let $\nu_h\in\Delta(\MX\times\MA)$ be the distribution of each sample $(x_h^i,a_h^i)$ in $\MC_h$. By \cref{cor:random-design-prediction-error} and a union bound over $\ell \in [d]$, there is an event $\ME_1$ that occurs with probability at least $1-\delta/2$ over the draw of $\MC_h = \{(x_h^i, a_h^i, x_{h+1}^i)\}_{i \in [n]}$; in the event $\ME_1$, we have that for all $\ell \in [d]$,
  \begin{align}
    \E_{(x, a) \sim \nu_h} \left[ \lng \hat \bw_\ell - \bw_\ell^\st, \phi_h(x,a) \rng^2 \right] \leq C_{\ref{cor:random-design-prediction-error}} \cdot \Cnrm^2 \cdot \sqrt{\frac{\log(2d^2/\delta)}{n}}\label{eq:wprime-wstar-nu}. 
  \end{align}
  By \cref{lem:l1-generalization}, there is an event $\ME_2$ that occurs with probability at least $1-\delta/2$ over the draw of $\MC_h$, and in the event $\ME_2$, we have for all $\theta \in \RR^d$ that
  \begin{equation}
  \left|\EE_{(x,a)\sim\nu_h}[\langle \phi_h(x,a),\theta\rangle^2] - \frac{1}{n}\sum_{i=1}^n \langle \phi_h(x_h^i,a_h^i),\theta\rangle^2\right| \leq \frac{C_{\ref{lem:l1-generalization}}\norm{\theta}_1^2\sqrt{\log(d/\delta)}}{\sqrt{n}}.
  \label{eq:phi-generalization}
  \end{equation}
  for some constant $C$. We claim that in the event $\ME_1 \cap \ME_2$, the output of \ESC{} satisfies the guarantee claimed in the lemma statement. If the output is $\perp$, this is immediate. Otherwise, the output is a solution $(\hat\mu_{h+1}^j)_{j=1}^m$ to Program~\cref{eq:program}. Fix any $\ell \in [d]$. Set $\theta := \hat{\bw}_\ell - \sum_{j=1}^m \hat\mu_{h+1}^j \cdot \phiavg_{h+1}(\tilde x_{h+1}^j)_\ell$. By the constraint $\norm{\hat{\bw}_\ell} \leq \Cnrm$ in \cref{eq:wj-guarantee} together with the program constraint \cref{eq:hatmu-cnrm-constraint}, we have $\norm{\theta}_1 \leq 2\Cnrm$. Thus, combining the program constraint \cref{eq:what-wprime-constraint} with \cref{eq:phi-generalization} gives
  \begin{align}
\E_{(x,a) \sim \nu_h} \left[ \left\lng \phi_h(x,a),\ \hat{\bw}_\ell - \sum_{j=1}^m \hat\mu_{h+1}^j \cdot \phiavg_{h+1}(\tilde x_{h+1}^j)_\ell \right\rng^2 \right] \leq \epcvx^2 + \frac{2C_{\ref{lem:l1-generalization}}\Cnrm^2 \sqrt{\log(2d/\delta)}}{\sqrt{n}}\label{eq:wprime-what-nu}. 
  \end{align}
  Combining \cref{eq:wprime-wstar-nu} and \cref{eq:wprime-what-nu} via the bound $(a+b)^2 \leq 2a^2+2b^2$ gives that
  \begin{align}
\E_{(x_h, a_h) \sim \nu_h} \left[\left\langle \phi_h(x_h,a_h),\ \bw_\ell^\st - \sum_{j=1}^m\hat\mu_{h+1}^j \cdot \phiavg_{h+1}(\tilde x_{h+1}^j)_\ell\right\rangle^2 \right] \leq 2\epcvx^2 + \frac{C_{\ref{lem:muhat-approx}}\Cnrm^2\sqrt{\log(2d/\delta)}}{\sqrt{n}}\nonumber
  \end{align}
  where we take $C_{\ref{lem:muhat-approx}} := 4(C_{\ref{lem:l1-generalization}}+C_{\ref{cor:random-design-prediction-error}})$. Finally, by the assumption that $\Psi_h$ is an $\alpha$-approximate policy cover, \cref{lemma:beta-nu-coverage} gives that $\nu_h(x,a) \geq \frac{\alpha}{A} \cdot \max_{\pi \in \Pi} d_h^\pi(x,a)$ for all $(x,a) \in \MX\times \MA$. Thus, the above inequality implies that for all $\pi \in \Pi$,
  \begin{align}
\E^\pi \left[\left\langle \phi_h(x_h,a_h),\ \bw_\ell^\st - \sum_{j=1}^m\hat\mu_{h+1}^j \cdot \phiavg_{h+1}(\tilde x_{h+1}^j)_\ell\right\rangle^2 \right] \leq \frac{2A}{\alpha}\epcvx^2 + \frac{C_{\ref{lem:muhat-approx}}\Cnrm^2 A\sqrt{\log(2d/\delta)}}{\alpha\sqrt{n}}.\nonumber
  \end{align}
Applying Jensen's inequality and recalling the definition of $\bw_\ell^\st$ yields that
\begin{align*} 
&\left(\sum_{x \in \MX}\langle \EE^\pi[\phi_h(x_h,a_h)],\mu_{h+1}(x)\rangle \cdot \phiavg_{h+1}(x)_\ell - \sum_{j=1}^m \langle \EE^\pi[\phi_h(x_h,a_h)], \hat\mu_{h+1}^j\rangle \cdot \phiavg_{h+1}(\tilde x_{h+1}^j)_\ell\right)^2 \\
&\qquad\leq \frac{2A}{\alpha}\epcvx^2 + \frac{C_{\ref{lem:muhat-approx}}\Cnrm^2 A\sqrt{\log(2d/\delta)}}{\alpha\sqrt{n}}.
\end{align*}
Since $\ell \in [d]$ was arbitrary, this completes the proof.

\end{proof}

\subsubsection{Putting it all together}

The proof of \cref{thm:muhat-coreset} is now essentially immediate from combining \cref{lem:feasibility}, \cref{lemma:hat-mu-apx-nonnegative}, and \cref{lem:muhat-approx}.

\begin{proof}[Proof of \cref{thm:muhat-coreset}]
By \cref{lem:feasibility}, with probability at least $1-\delta/3$ the algorithm \ESC{} produces a solution $(\hat\mu_{h+1}^j)_{j=1}^m$ to Program~\cref{eq:program}.  Next, since $\Psi_h$ is a $\alpha$-approximate policy cover, \cref{lemma:beta-nu-coverage} implies that the distribution of the samples $(x_h^i,a_h^i)_{i=1}^n$ satisfies the precondition \cref{eq:nu-hyp} of \cref{lemma:hat-mu-apx-nonnegative} with parameter $\zeta := \alpha/A$. Thus, \cref{lemma:hat-mu-apx-nonnegative} implies that with probability at least $1-\delta/3$, for all $\pi \in \Pi$ and $x \in \RR^d$,
\begin{equation}
\langle \EE^\pi[\phi_h(x_h,a_h)],x\rangle \geq \inf_{i \in [n]} \langle \phi_h(x_h^i,a_h^i),x\rangle - \epneg\norm{x}_1.
\label{eq:apx-nneg-final}
\end{equation}
Finally, by \cref{lem:muhat-approx}, the choice of $n$, and the fact that $\epcvx \leq \sqrt{\frac{\alpha}{4A}} \cdot \epapx$, 
with probability at least $1-\delta/3$, the output of \ESC{} is either $\perp$ or a set of vectors $(\hat\mu_{h+1}^j)_{j=1}^m$ satisfying that for any policy $\pi \in \Pi$,
\begin{equation}\left\| \sum_{x \in \MX} \lng \E^\pi[\phi_h(x_h, a_h)], \mu_{h+1}(x) \rng \cdot \phiavg_{h+1}(x) - \sum_{j=1}^m \lng \E^\pi[\phi_h(x_h, a_h)], \hat \mu_{h+1}^j \rng \cdot \phiavg_{h+1}(\tilde x_{h+1}^j) \right\|_\infty^2 \leq \epapx^2.
\label{eq:apx-error-final}
\end{equation}

Suppose that all three of the above events hold, which happens with probability at least $1-\delta$. The output of \ESC{} is a solution to Program~\cref{eq:program}. By \cref{eq:hatmu-cnrm-constraint}, \cref{it:emulator-C-bound} of \cref{def:not-core-set} is satisfied. By \cref{eq:apx-nneg-final} together with \cref{eq:hatmu-nneg-constraint}, \cref{it:emulator-epneg-bound} is satisfied. By \cref{eq:apx-error-final}, \cref{it:emulator-epapx-bound} is satisfied. Thus, $(\hat\mu_{h+1}^j)_{j=1}^m$ is an $(\epapx,\epneg,\Cnrm)$-\coreset{} for $M$ at step $h$.
\end{proof}

\subsection{From \coreset to policy cover}
\label{sec:pc-analysis-reachable}

In this section we prove \cref{lemma:pc-coreset-guarantee}, which states that when \PC (\cref{alg:pc}) is given an $(\epapx,\epneg,\Cemp)$-\coreset as input, then the output is a set of policies $\Psiapx_{h+1}$ that forms an ``intermediate'' policy cover \--- in the sense that when each policy is extended by a uniformly random action at step $h+1$, the resulting set of policies $\Psi_{h+2}$ is a true policy cover for step $h+2$. 

\paragraph{Overview.} At a high level, the proof strategy is as follows. For any policy $\pi$ that takes a uniformly random action at step $h+1$, its visitation probability at a state $x' \in \MX$ can be written as
\[\E^\pi[\langle \phiavg_{h+1}(x_{h+1}), \mu_{h+2}(x')\rangle] = \sum_{x\in\MX} \langle \E^\pi[\phi_h(x_h,a_h)], \mu_{h+1}(x)\rangle \cdot g(x)\]
where $g: \MX\to[0,\norm{\mu_{h+2}(x')}_1]$ is the bounded functional defined by $g(x) = \langle \phiavg_{h+1}(x), \mu_{h+2}(x')\rangle$. For any \coreset $(\hat\mu_{h+1}^j)_{j=1}^m$, by \cref{it:emulator-epapx-bound} of \cref{def:not-core-set}, the above quantity can be approximated by
\begin{equation} \sum_{j=1}^m \langle \E^\pi[\phi_h(x_h,a_h)], \hat\mu_{h+1}^j\rangle \cdot g(\tilde{x}_{h+1}^j)
\label{eq:muhat-apx-visitation-overview}
\end{equation} 
for some states $(\tilde{x}_{h+1}^j)_{j=1}^m$. Next, ideally we would show that for \emph{every} $j \in [m]$, there is some policy $\pi'$ in the intermediate policy cover $\Psiapx_{h+1}$ such that
\begin{equation} \sup_{\pi \in\Pi} \langle \E^{\pi}[\phi_h(x_h,a_h)], \hat\mu_{h+1}^j\rangle \leq \gamma \cdot \langle \E^{\pi'}[\phi_h(x_h,a_h)], \hat\mu_{h+1}^j\rangle\label{eq:muhat-coverage-overview}\end{equation}
for some reasonable parameter $\gamma$. Since $\langle \E^\pi[\phi_h(x_h,a_h)],\hat\mu_{h+1}^j\rangle$ is approximately nonnegative for every $j \in [m]$ and $\pi \in \Pi$ (\cref{it:emulator-epneg-bound}), the above coverage property would let us show (hiding some terms in the big-$O$ notation) that
\begin{align}
&\sum_{j=1}^m \langle \E^\pi[\phi_h(x_h,a_h)], \hat\mu_{h+1}^j\rangle \cdot g(\tilde{x}_{h+1}^j) \leq O(\epneg\norm{\mu_{h+2}(x')}_1)  \nonumber \\
&\qquad+ |\Psiapx_{h+1}| \cdot \gamma \cdot \E_{\pi'\sim\unif(\Psiapx_{h+1})} \langle \E^{\pi'}[\phi_h(x_h,a_h)], \hat\mu_{h+1}^j\rangle \cdot g(\tilde{x}_{h+1}^j).\label{eq:muhat-sum-coverage-overview}
\end{align}
By applying the \coreset approximation guarantee (\cref{it:emulator-epapx-bound}) again, the right-hand side can then be related to the expected visitation probability of a random cover policy $\pi' \sim \unif(\Psiapx_{h+1}) \circ_{h+1} \unif(\MA)$ at state $x'$, as desired. Unfortunately, it's not clear how to construct a set of policies $\Psiapx_{h+1}$ such that \cref{eq:muhat-coverage-overview} holds for all $j \in [m]$, because there could be some $j$ for which $\sup_{\pi\in\Pi} \langle \E^\pi[\phi_h(x_h,a_h)],\hat\mu_{h+1}^j\rangle$ may be very close to $0$. However, if we use \PC to construct $\Psiapx_{h+1}$, then we can essentially prove that such $j$ are the only issue: for any given threshold $\xi>0$, due to the guarantee that $\sum_{j=1}^m \norm{\hat\mu_{h+1}^j}_1 \leq \Cemp$ (\cref{it:emulator-C-bound}), \PC produces a set $\Psiapx_{h+1}$ of $O(\Cnrm/\xi)$ policies, and \cref{eq:muhat-coverage-overview} holds for all $j \in [m]$, with $\gamma = O(\Cemp/\xi)$, except for a ``bad'' set $\MB$, which is small in the sense that $\sum_{j\in\MB}\langle\E^\pi[\phi_h(x_h,a_h)],\hat\mu_{h+1}^j\rangle \leq O(\xi)$ for all policies $\pi \in \Pi$.

For every policy $\pi \in \Pi$, the ``bad'' set $\MB$ cannot significantly affect the approximate visitation probability \cref{eq:muhat-apx-visitation-overview}. Thus, \cref{eq:muhat-sum-coverage-overview} still holds, with an additional error term scaling with $\xi$. Ultimately, this implies that for any policy $\pi \in \Pi$ and state $x'\in\MX$ (again hiding some terms),
\begin{align*} d^{\pi\circ_{h+1}\unif(\MA)}_{h+2}(x') 
&\leq O(\epapx + \epneg + \xi) \norm{\mu_{h+2}(x')}_1 \\ 
&\qquad+ \poly(\Cemp,1/\xi) \cdot \E_{\pi' \sim \unif(\Psiapx_{h+1})} d^{\pi'\circ_{h+1}\unif(\MA)}_{h+2}(x').
\end{align*}

This is the point where the reachability assumption is crucial. Suppose that $\pi$ is the policy that maximizes the left-hand side (note that restricting $\pi$ to take a uniform action at step $h+1$ only affects this maximum by at most a factor of $A$); if the MDP is $\etarch$-reachable, then the left-hand side must be at least $(\etarch/A)\norm{\mu_{h+2}(x')}_1$. So long as $O(\epapx+\epneg+\xi) \ll \etarch/A$, the error term on the right-hand side is negligible, and we may conclude that $\Psi_{h+2} := \{\pi \circ_{h+1} \unif(\MA):\pi \in \Psiapx_{h+1}\}$ is a $\poly(\Cemp,1/\xi)$-approximate policy cover for step $h+2$.

\paragraph{Analysis of \PC.} We start by proving \cref{lem:pc-size-bound} and \cref{lem:emp-pc-guarantee-rch}, which state the basic correctness guarantees that \PC satisfies for \emph{any} set of input vectors $\hat\mu_{h+1}^1,\dots,\hat\mu_{h+1}^m$ with bounded total norm $\sum_{j=1}^m \norm{\hat\mu_{h+1}^j}_ 1 \leq \Cemp$, and any given threshold $\xi>0$. For any vector $\theta\in\RR^d$ and policy $\pi\in\Pi$ let us informally refer to $\E\sups{M,\pi}[\langle \phi_h(x_h,a_h), \theta\rangle]$ as the ``reward'' of $\pi$ in direction $\theta$. Then \cref{lem:emp-pc-guarantee-rch} guarantees that \--- aside from a set of ``uncovered'' vectors $\{\hat\mu_{h+1}^j: j \in \MB\}$ for which every policy has total reward at most $O(\xi)$ \--- for every vector $\hat\mu_{h+1}^j$, there is a policy in $\Psiapx_{h+1}$ that approximately optimizes the reward in direction $\hat\mu_{h+1}^j$ (up to a multiplicative factor of $O(\Cemp/\xi)$). Finally, the size of $\Psiapx_{h+1}$ is also only $O(\Cemp/\xi)$.

The proofs use the same ideas as \cref{lem:idealpc-guarantees}; the differences arise because (unlike what is assumed for \IdealPC) we cannot exactly optimize $\E\sups{M,\pi}[\langle \phi_h(x_h,a_h), \theta\rangle]$ over policies $\pi$, and we cannot exactly compute $\E\sups{M,\pi}[\langle \phi_h(x_h,a_h)]$ for a given policy. Thus we need to account for errors in the approximate policy optimization algorithm \PSDP and the algorithm \FE.

\begin{lemma}[Size bound and sample/time complexity for \PC]\label{lem:pc-size-bound}
Let $\xi,\Cemp > 0$. Let $(\hat\mu_{h+1}^j)_{j=1}^m \subseteq \RR^d$ satisfy $\sum_{j=1}^n \norm{\hat\mu_{h+1}^j}_1 \leq \Cemp$. For any $h \in [H]$, policy sets $\Psi_{1:h}$, and $N \in \NN$, we have that \PC (\cref{alg:pc}) with parameters $h, (\hat\mu_{h+1})_{j=1}^m, \xi,\Cemp,\Psi_{1:h},N$ produces a set $\Psiapx_{h+1}$ with $|\Psiapx_{h+1}| \leq 2\Cemp/\xi$. Moreover, the sample complexity of the algorithm is $O(N\Cemp/\xi)$, and the time complexity is $O(\Cemp/\xi) \cdot \poly(H, N, d, A)$.
\end{lemma} 

\begin{proof}
For each policy $\pi^t$ in $\Psiapx_{h+1}$, we have 
\[\xi \leq \langle \hat \phi_h^{t}, \mures{t}\rangle = \langle \hat\phi_h^t, \sum_{j \in \MB_t} \hat\mu_{h+1}^j\rangle\] where the inequality is by \cref{alg:pc:line:break} of \cref{alg:pc}, and the equality is by definition of $\mures{t}$. On the other hand, we also have
\[ \langle \hat \phi_h^{t}, \sum_{j \in \MB_t \setminus \MG_t} \hat \mu_{h+1}^j \rangle \leq \frac{\xi}{2\Cemp} \sum_{j \in \MB_t \setminus \MG_t} \norm{\hat{\mu}_{h+1}^j}_1 \leq \frac{\xi}{2}\]
where the first inequality is by definition of the set $\MG_i$ (\cref{alg:pc:line:jidef}), and the last inequality follows by assumption on $\sum_{j=1}^m \norm{\hat\mu_{h+1}^j}_1$. Comparing the above inequalities, we have \[\sum_{j \in \MG_t} \norm{\hat{\mu}_{h+1}^j}_1 \geq \langle \hat \phi_h^t, \sum_{j \in \MG_t} \hat{\mu}_{h+1}^j\rangle \geq \xi/2\]
where the first inequality uses the guarantee of \FE that $\norm{\hat\phi^t_h}_\infty \leq 1$ (\cref{lem:fe}). Since the sets $\MG_1,\MG_2,\dots,\MG_t,\dots$ are disjoint, summing the above inequality over $t$ gives
\[ \Cemp \geq \sum_{j=1}^m \norm{\hat{\mu}_{h+1}^j}_1 \geq \sum_{t: \pi^t \in \Psiapx_{h+1}} \frac{\xi}{2} = \frac{\xi |\Psiapx_{h+1}|}{2}\]
where the first inequality uses the bound on $\sum_{j=1}^m \norm{\hat\mu_{h+1}^j}_1$ again. Thus, $|\Psiapx_{h+1}| \leq 2\Cemp/\xi$ as claimed. It follows that the algorithm terminates after $O(\Cemp/\xi)$ iterations and therefore makes only $O(\Cemp/\xi)$ calls to \PSDP and \FE, which implies the claimed sample and time complexity bounds by \cref{lem:psdp} and \cref{lem:fe}.
\end{proof}

\begin{lemma}[Guarantee for \PC]
  \label{lem:emp-pc-guarantee-rch}
  Let $\alpha,\xi,\Cemp,\delta>0$ with $\xi \leq \Cemp$. Let $h \in [H]$ and let $(\hat\mu_{h+1}^j)_{j=1}^m \subseteq \RR^d$ be vectors satisfying 
  \begin{equation}
\sum_{j=1}^m \norm{\hat \mu_{h+1}^j}_1 \leq \Cemp.
\label{eq:emp-pc-cnrm-bound}
\end{equation}  Let $\Psi_{1:h}$ be $\alpha$-approximate policy covers for steps $1,\dots,h$ respectively (\cref{defn:pc}).
  Suppose that $N \geq \max\{\nstatrch(\xi/(4\Cemp), \alpha,\delta\xi/(6\Cemp)), \nfe(\xi/(4\Cemp), \delta\xi/(6\Cemp))\}$. Then \PC (\cref{alg:pc}) with parameters $\xi,\Cemp,N$ outputs a set of policies $\Psiapx_{h+1}$ of size $|\Psiapx_{h+1}| \leq 2\Cemp/\xi$ and a subset $\MG \subseteq [m]$ so that with probability at least $1-\delta$, the following conditions hold for all $\pi \in \Pi$:
  \begin{enumerate}
  \item $\sum_{j \in [m]\setminus \MG} \lng \E^\pi[\phi_h(x_h, a_h)], \hat \mu_{h+1}^j \rng \leq 3\xi/2$.
  \item For all $j \in \MG$, there is some $\pi' \in \Psiapx_{h+1}$ so that \[\lng\E^{\pi'}[\phi_h(x_h, a_h)], \hat \mu_{h+1}^j \rng \geq \frac{\xi}{4 \Cemp} \cdot \lng \E^\pi[\phi_h(x_h, a_h)], \hat \mu_{h+1}^j \rng.\]
  \end{enumerate}
\end{lemma}
\begin{proof}
By \cref{lem:pc-size-bound}, we have that $|\Psiapx_{h+1}| \leq 2\Cemp/\xi$. It follows that the algorithm makes at most $1 + 2\Cemp/\xi \leq 3\Cemp/\xi$ calls to each of \PSDP and \FE. Each call has failure probability at most $\delta\xi/(6\Cemp)$, so with probability at least $1-\delta$ all of the calls succeed (i.e. satisfy the guarantees of \cref{lem:psdp} and \cref{lem:fe}). We assume from now on that this event holds. By choice of $N$, the guarantees of \cref{lem:psdp} and \cref{lem:fe} hold with error at most $\epstat := \xi/(4\Cemp)$.

For convenience set $T = |\Psiapx_{h+1}|+1$. Recall that $\pifinal := \pi^T$ is the last policy computed by the procedure (the policy that is not added to $\Psiapx_{h+1}$). Note that the output set $\MG$ of \cref{alg:pc} is given by $\MG := [m] \setminus \MB_T$. 

By the termination condition (\cref{alg:pc:line:break}), we have $\langle \hat \phi_h^T, \sum_{j \in [m]\setminus \MG} \hat{\mu}_{h+1}^j \rangle = \lng \hat \phi_h^T, \sum_{j \in \MB_T} \hat \mu_{h+1}^j \rng  < \xi$. By the guarantee of \FE (\cref{lem:fe}) and (\ref{eq:emp-pc-cnrm-bound}), it follows that
\[\langle \EE^{\pifinal}[\phi_h(x_h,a_h)], \sum_{j \in [m]\setminus \MG} \hat \mu_{h+1}^j\rangle  < \xi + \epstat \Cemp.\]
By the guarantee of \PSDP (\cref{lem:psdp}) and choice of $\pifinal$, we get
\[ \sup_{\pi \in \Pi} \langle \EE^\pi[\phi_h(x_h,a_h)], \sum_{j \in [m]\setminus \MG} \hat{\mu}_{h+1}^j\rangle < \xi + 2\epstat \Cemp \leq 3\xi/2\]
which proves claim (1) of the lemma statement. Next, for each $j \in \MG$ there is some $t$ so that $j \in \MG_t$. But then the estimated feature vector $\hat\phi_h^t$ of policy $\pi^t \in \Psiapx_{h+1}$ satisfies
\[\langle \hat \phi_h^t, \hat \mu_{h+1}^j\rangle \geq (\xi/(2\Cemp)) \cdot \norm{\hat \mu_{h+1}^j}_1,\]
so by the guarantee of \FE (\cref{lem:fe}),
\begin{align*}
\langle \EE^{\pifinal}[\phi_h(x_h,a_h)], \hat\mu_{h+1}^j \rangle
&\geq \left(\frac{\xi}{2\Cemp} - \epstat\right)\cdot\norm{\hat\mu_{h+1}^j}_1 \\
&\geq \left(\frac{\xi}{2\Cemp} - \epstat\right)\cdot \sup_{\pi \in \Pi} \langle \EE^\pi[\phi_h(x_h,a_h)], \hat\mu_{h+1}^j\rangle.
\end{align*}
Since $\epstat \leq \xi/(4\Cemp)$, this completes the proof of claim (2).
\end{proof}

We now formalize \cref{eq:muhat-sum-coverage-overview}, including the error terms due to the ``bad'' set $\MB\subseteq [m]$ and the possibility that some of the inner products $\langle \E^\pi[\phi_h(x_h,a_h)],\hat\mu_{h+1}^j\rangle$ may be slightly negative. %

\begin{lemma}
  \label{lem:aux-pc-coreset}
  In the setting of \cref{lemma:pc-coreset-guarantee}, the following holds with probability at least $1-\delta$. For any $\pi \in \Pi$, $R>0$, and function $g : [m] \ra [0, R]$,
\begin{align} 
\sum_{j=1}^m \lng \E^\pi[\phi_h(x_h, a_h)], \hat \mu_{h+1}^j \rng \cdot g(j)
&\leq \frac{17\Cemp^3 R\epneg}{\xi^2} + 3R\xi/2 \nonumber \\
&+ \frac{4\Cemp}{\xi} \sum_{\pi' \in \Psiapx_{h+1}} \sum_{j =1}^m \langle \EE^{\pi'}[\phi_h(x_h,a_h)],\hat\mu_{h+1}^j\rangle \cdot g(j) \label{eq:aux-pc-coreset-bound}
\end{align}
\end{lemma}

\begin{proof}
By \cref{lem:emp-pc-guarantee-rch}, using the norm bound $\sum_{j=1}^m \norm{\hat\mu_{h+1}^j}_1 \leq \Cemp$ guaranteed by \cref{it:emulator-C-bound} of \cref{def:not-core-set}, the output of $\PC$ is a set of policies $\Psiapx_{h+1}$ of size $|\Psiapx_{h+1}| \leq 2\Cemp/\xi$ and a set $\MG \subset [m]$ satisfying (with probability at least $1-\delta$) the following properties:
  \begin{enumerate}[label=\textbf{(\roman*)}]
  \item $\sum_{j \in [m]\setminus \MG} \lng \E^\pi[\phi_h(x_h, a_h)], \hat \mu_{h+1}^j \rng \leq 3\xi/2$.
  \item For all $j \in \MG$, there is some $\pi' \in \Psiapx_{h+1}$ so that \[\lng\E^{\pi'}[\phi_h(x_h, a_h)], \hat \mu_{h+1}^j \rng \geq \frac{\xi}{4 \Cemp} \cdot \lng \E^\pi[\phi_h(x_h, a_h)], \hat \mu_{h+1}^j \rng.\]
  \end{enumerate}
Additionally, \cref{it:emulator-epneg-bound} of \cref{def:not-core-set} guarantees that for any $\pi \in \Pi$ and $f: [m]\to[0,R]$,
\begin{align}
\sum_{j=1}^m \max\left(0, -\langle \EE^\pi[\phi_h(x_h,a_h)], \hat\mu_{h+1}^j\rangle \cdot f(j)\right)
&\leq \sum_{j=1}^m R\epneg \norm{\hat\mu_{h+1}^j}_1 \nonumber \\
&\leq R\epneg \Cemp.
\label{eq:hatmu-neg-error}
\end{align}
We use the above properties to prove the claimed bound \cref{eq:aux-pc-coreset-bound}. Fix $\pi \in \Pi$. We separate the LHS of \cref{eq:aux-pc-coreset-bound} into two terms:
\[\sum_{j=1}^m \lng \E^\pi[\phi_h(x_h, a_h)], \hat \mu_{h+1}^j \rng \cdot g(j) = \sum_{j \in [m] \backslash \MG} \lng \E^\pi[\phi_h(x_h, a_h)], \hat \mu_{h+1}^j \rng \cdot g(j) + \sum_{j\in \MG} \lng \E^\pi[\phi_h(x_h, a_h)], \hat \mu_{h+1}^j \rng \cdot g(j).\]
By \cref{eq:hatmu-neg-error} and the bound $g(j) \in [0,R]$ for all $j$, the first term can be bounded as %
\begin{align*}
\sum_{j \in [m] \backslash \MG} \lng \E^\pi[\phi_h(x_h, a_h)], \hat \mu_{h+1}^j \rng \cdot g(j)
&\leq R\epneg\Cemp + R \cdot \sum_{j \in [m] \backslash \MG} \lng \E^\pi[\phi_h(x_h, a_h)], \hat \mu_{h+1}^j \rng\\
&\leq R\epneg\Cemp + 3R\xi/2
\end{align*}
where the last inequality uses property \textbf{(i)}. Next,
\begin{align*}
&\sum_{j\in \MG} \lng \E^\pi[\phi_h(x_h, a_h)], \hat \mu_{h+1}^j \rng \cdot g(j) \\
&\leq \sum_{j \in \MG} \frac{4\Cemp}{\xi}\max_{\pi' \in \Psiapx_{h+1}} \langle \EE^{\pi'}[\phi_h(x_h,a_h)],\hat\mu_{h+1}^j\rangle \cdot g(j) \\
&\leq \sum_{j \in \MG} \frac{4\Cemp}{\xi} \left(R|\Psiapx_{h+1}|\epneg\norm{\hat\mu_{h+1}^j}_1 + \sum_{\pi' \in \Psiapx_{h+1}}\langle \EE^{\pi'}[\phi_h(x_h,a_h)],\hat\mu_{h+1}^j\rangle \cdot g(j)\right) \\
&\leq \frac{4\Cemp^2 R|\Psiapx_{h+1}|\epneg}{\xi} + \frac{4\Cemp}{\xi} \sum_{\pi' \in \Psiapx_{h+1}} \sum_{j \in \MG} \langle \EE^{\pi'}[\phi_h(x_h,a_h)],\hat\mu_{h+1}^j\rangle \cdot g(j) \\
&\leq \frac{8\Cemp^3 R\epneg}{\xi^2} + \frac{4\Cemp}{\xi} \sum_{\pi' \in \Psiapx_{h+1}} \sum_{j \in \MG} \langle \EE^{\pi'}[\phi_h(x_h,a_h)],\hat\mu_{h+1}^j\rangle \cdot g(j) \\
&\leq \frac{16\Cemp^3 R\epneg}{\xi^2} + \frac{4\Cemp}{\xi} \sum_{\pi' \in \Psiapx_{h+1}} \sum_{j =1}^m \langle \EE^{\pi'}[\phi_h(x_h,a_h)],\hat\mu_{h+1}^j\rangle \cdot g(j)
\end{align*}
where the first inequality uses property \textbf{(ii)}, the second inequality uses \cref{it:emulator-epneg-bound} of \cref{def:not-core-set}, the third inequality uses \cref{it:emulator-C-bound} of \cref{def:not-core-set}, the fourth inequality uses the bound $|\Psiapx_{h+1}| \leq 2\Cemp/\xi$, and the fifth inequality uses \cref{eq:hatmu-neg-error} together with the bound $|\Psiapx_{h+1}| \leq 2\Cemp/\xi$.

Combining the two bounds and using that $\Cemp/\xi \geq 1$, we get that
\begin{align*} 
\sum_{j=1}^m \lng \E^\pi[\phi_h(x_h, a_h)], \hat \mu_{h+1}^j \rng \cdot g(j)
&\leq \frac{17\Cemp^3 R\epneg}{\xi^2} + 3R\xi/2 \\
&+ \frac{4\Cemp}{\xi} \sum_{\pi' \in \Psiapx_{h+1}} \sum_{j =1}^m \langle \EE^{\pi'}[\phi_h(x_h,a_h)],\hat\mu_{h+1}^j\rangle \cdot g(j)
\end{align*}
as claimed. 
\end{proof}
\allowdisplaybreaks

Finally, we formally prove \cref{lemma:pc-coreset-guarantee}, combining \cref{lem:aux-pc-coreset} with the \coreset approximation guarantee (\cref{it:emulator-epapx-bound}) as sketched at the beginning of the section.

\begin{namedproof}{\cref{lemma:pc-coreset-guarantee}}
We condition on the event that the guarantee of \cref{lem:aux-pc-coreset} holds (which occurs with probability at least $1-\delta$); in this event, we show that the set $\Psi_{h+2}$, defined in the lemma statement, is a $\xi^2/(16\Cemp^2 A)$-approximate policy cover for step $h+2$.

By \cref{it:emulator-epapx-bound} of \cref{def:not-core-set}, there are states $x_{h+1}^1,\dots,x_{h+1}^m \in \MX$ such that for any $\pi \in \Pi$ and $\mu_{h+2} \in \BR^d$, we have
  \begin{align}
&\left| \sum_{x \in \MX} \lng \E^\pi[\phi_h(x_h, a_h)], \mu_{h+1}(x) \rng \cdot \lng \phiavg_{h+1}(x), \mu_{h+2} \rng -\sum_{j=1}^m \lng \E^\pi[\phi_h(x_h, a_h)], \hat \mu_{h+1}^j \rng \cdot \lng  \phiavg_{h+1}(x_{h+1}^j), \mu_{h+2} \rng \right| \nonumber\\ 
&\leq \epapx \| \mu_{h+2} \|_1\label{eq:hatmu-approx-epapx}
  \end{align}
  Fix any $\pi \in \Pi$ and $x' \in \MX$. Then we may compute
  \begin{align}
    & \lng \E^\pi[\phiavg_{h+1}(x_{h+1})], \mu_{h+2}(x') \rng \nonumber\\
    &= \sum_{x \in \MX} \lng \E^\pi[\lng \phi_h(x_h, a_h)], \mu_{h+1}(x) \rng \cdot \lng \phiavg_{h+1}(x), \mu_{h+2}(x') \rng \nonumber\\
    &\leq  \epapx \cdot \| \mu_{h+2}(x') \|_1 + \sum_{j=1}^m \lng \E^\pi[\phi_h(x_h, a_h)], \hat \mu_{h+1}^j \rng \cdot \lng \phiavg_{h+1}(x_{h+1}^j), \mu_{h+2}(x') \rng \nonumber\\
    &\leq  \epapx \cdot \| \mu_{h+2}(x') \|_1 + \frac{17\Cemp^3 \norm{\mu_{h+2}(x')}_1 \epneg}{\xi^2} + \frac{3}{2}\xi \cdot \| \mu_{h+2}(x') \|_1 \nonumber\\
    &\qquad+ \frac{4\Cemp}{\xi} \cdot  \sum_{\pi' \in \Psiapx_{h+1}} \sum_{j=1}^m  \lng \E^{\pi'}[\phi_h(x_h, a_h)], \hat \mu_{h+1}^j \rng \cdot \lng \phiavg_{h+1}(x_{h+1}^j), \mu_{h+2}(x') \rng  \nonumber\\
    &\leq \epapx \cdot \| \mu_{h+2}(x') \|_1 + \frac{17\Cemp^3 \norm{\mu_{h+2}(x')}_1 \epneg}{\xi^2} + \frac{3}{2}\xi \cdot \| \mu_{h+2}(x') \|_1 \nonumber\\
    &\qquad + \frac{4\Cemp}{\xi} \cdot \sum_{\pi' \in \Psiapx_{h+1}} \left( \epapx \cdot \| \mu_{h+2}(x') \|_1 +  \sum_{x \in \MX} \lng \E^{\pi'}[\phi_h(x_h, a_h)], \mu_{h+1}(x) \rng \cdot \lng \phiavg_{h+1}(x), \mu_{h+2}(x') \rng\right)\nonumber\\
    &\leq \left( \frac{9\Cemp^2}{\xi^2}\epapx + \frac{17\Cemp^3}{\xi^2} \epneg + \frac{3}{2}\xi \right) \cdot \| \mu_{h+2}(x') \|_1 \nonumber\\
    &\qquad+  \frac{4\Cemp}{\xi} \cdot \sum_{\pi' \in \Psiapx_{h+1}} \lng \E^{\pi'}[\phiavg_{h+1}(x_{h+1})], \mu_{h+2}(x') \rng \label{eq:cover-bound-0},
  \end{align}
  where the first inequality uses \cref{eq:hatmu-approx-epapx}; the second inequality uses \cref{lem:aux-pc-coreset} with the function $g(j) := \langle \phiavg_{h+1}(x_{h+1}^j),\mu_{h+2}(x')\rangle$ (note that $0 \leq g(j) \leq \norm{\mu_{h+2}(x')}_1$ for all $j \in [m]$ and $x' \in \MX$); the third inequality uses \cref{eq:hatmu-approx-epapx}; and the fourth inequality collects terms and uses the bounds $|\Psiapx_{h+1}| \leq 2\Cemp/\xi$ and $1 \leq \Cemp/\xi$.

  Next, since the MDP is $\etarch$-reachable (\cref{defn:rch}), we have that for any $x' \in \MX$,
  \begin{align}
\max_{\pi \in \Pi} \lng \E^\pi[\phiavg_{h+1}(x_{h+1})], \mu_{h+2}(x') \rng \geq \frac{1}{A}\max_{\pi \in \Pi} \lng \E^\pi[\phi_{h+1}(x_{h+1},a_{h+1})], \mu_{h+2}(x') \rng 
 \geq \frac{\etarch}{A} \cdot \| \mu_{h+2}(x') \|_1\nonumber.
  \end{align}
  Thus, the additive error term in \cref{eq:cover-bound-0} is bounded by
  \begin{align*}
  \left( \frac{9\Cemp^2}{\xi^2}\epapx + \frac{17\Cemp^3 }{\xi^2} \epneg + \frac{3}{2}\xi \right) \cdot \| \mu_{h+2}(x') \|_1 
&\leq \frac{\etarch}{2A} \norm{\mu_{h+2}(x')}_1 \\
&\leq \frac{1}{2} \max_{\pi \in \Pi} \lng \E^\pi[\phiavg_{h+1}(x_{h+1})], \mu_{h+2}(x') \rng
  \end{align*} 
  where the first inequality uses \cref{eq:rch-param-constraint}. Substituting the above bound into \cref{eq:cover-bound-0}, we get that for any $\pi' \in \Pi$ and $x' \in \MX$,
  \begin{align*} 
  \langle\EE^{\pi'}[\phiavg_{h+1}(x_{h+1})],\mu_{h+2}(x')\rangle 
  &\leq \frac{1}{2} \max_{\pi \in \Pi} \langle\EE^{\pi}[\phiavg_{h+1}(x_{h+1})],\mu_{h+2}(x')\rangle \\
  &\qquad+ \frac{4\Cemp}{\xi} \cdot \sum_{\bar\pi \in \Psiapx_{h+1}} \lng \E^{\bar\pi}[\phiavg_{h+1}(x_{h+1})], \mu_{h+2}(x') \rng
  \end{align*}
  and therefore for any $x' \in \MX$,
  \begin{align*}
  \max_{\pi \in \Pi} \lng \E^\pi[\phiavg_{h+1}(x_{h+1})], \mu_{h+2}(x')\rng
&\leq \frac{8\Cemp}{\xi} \sum_{\bar\pi \in \Psiapx_{h+1}} \lng \E^{\bar\pi}[\phiavg_{h+1}(x_{h+1})], \mu_{h+2}(x') \rng \\
&= \frac{8\Cemp}{\xi} \sum_{\bar\pi \in \Psiapx_{h+1}} \lng \E^{\bar\pi\circ_{h+1}\unif(\MA)}[\phi_{h+1}(x_{h+1},a_{h+1})], \mu_{h+2}(x') \rng \\
&= \frac{8\Cemp}{\xi} \sum_{\pi' \in \Psi_{h+2}} \lng \E^{\pi'}[\phi_{h+1}(x_{h+1},a_{h+1})], \mu_{h+2}(x') \rng \\
&\leq \frac{16\Cemp^2}{\xi^2} \frac{1}{|\Psi_{h+2}|}\sum_{\pi' \in \Psi_{h+2}} \lng \E^{\pi'}[\phi_{h+1}(x_{h+1},a_{h+1})], \mu_{h+2}(x') \rng
  \end{align*}
where the first equality uses the definition of $\phiavg_{h+1}$, the second equality uses the definition of $\Psi_{h+2}$, and the final inequality uses that $|\Psi_{h+2}| = |\Psiapx_{h+1}| \leq 2\Cemp/\xi$. Finally, for any $x' \in \MX$ and $\pi \in \Pi$ we know that
\[\lng \E^\pi[\phiavg_{h+1}(x_{h+1})], \mu_{h+2}(x')\rng \geq \frac{1}{A} \lng \E^\pi[\phi_{h+1}(x_{h+1},a_{h+1})], \mu_{h+2}(x')\rng.\]
It follows that
\[ 
\max_{\pi \in \Pi} \lng \E^\pi[\phi_{h+1}(x_{h+1},a_{h+1})], \mu_{h+2}(x')\rng \leq \frac{16\Cemp^2 A}{\xi^2} \frac{1}{|\Psi_{h+2}|}\sum_{\pi' \in \Psi_{h+2}} \lng \E^{\pi'}[\phi_{h+1}(x_{h+1},a_{h+1})], \mu_{h+2}(x') \rng 
\] as claimed.
\end{namedproof}

\subsection{Analysis of \PSDP and \FE}

In this section we provide sample and time complexity analysis for \PSDP (\cref{alg:psdp}), the algorithm that we use in \PC to optimize in a given direction at step $h$, given policy covers for steps $1,\dots,h$. We also analyze the simple algorithm \FE (\cref{alg:fe}) that estimates $\E^\pi[\phi_h(x_h,a_h)]$ for a given policy $\pi$.

\subsubsection{Policy Search by Dynamic Programming}

The \PSDP algorithm is given in \cref{alg:psdp}. The main guarantee of \PSDP is stated below. 

\begin{algorithm}[t]
	\caption{$\PSDP(p, k, \Cnrm,\theta, \Psi_{1:k-1},N)$: Policy Search by Dynamic Programming (variant of \cite{bagnell2003policy})}
	\label{alg:psdp}
	\begin{algorithmic}[1]\onehalfspacing
		\Require Dimension $p \in \NN$; target layer $k\in[H]$; norm parameter $\Cnrm \in \BR$; target direction $\theta\in \BR^p$; policy covers $\Psi_1, \ldots, \Psi_{k-1}$; number of samples $N\in \mathbb{N}$.
        \State Set $\hat \bw_k \gets \theta$ and define $\hat\pi_k: \MX \to \MA$ by \[ \hat\pi_k(x) := \argmax_{a\in\MA} \langle\phi_k(x,a),\hat \bw_k\rangle.\] 
		\For{$h=k-1, \dots, 1$} 
		\State $\MD_h \gets\emptyset$. 
		\For{$N$ times}
		\State Sample trajectory $(x_1, a_1, \dots, x_k, a_k)\sim
		\unif(\Psi_h)\circ_h \unif(\MA) \circ_{h+1} \hat \pi^{h+1:k}$.
        \State Set $r_k \gets \langle \phi_k(x_k,a_k),\theta\rangle$.
		\State Update dataset: $\MD_h \gets \MD_h \cup \{ (x_h, a_h, r_k) \}$.
		\EndFor
		\State Solve regression:
		\[\hat \bw_h \gets\argmin_{w \in \BR^p: \| w \|_1 \leq \Cnrm}  \sum_{(x, a, r)\in\MD_h} (\lng \phi_h(x, a), w \rng  -r)^2.\] \label{eq:psdp-mistake}
		\State Define $\hat \pi_h : \MX \ra \MA$ by
		\[
		\hat \pi_h(x)  := 
			\argmax_{a\in \MA} \lng \phi_h(x,a), \hat \bw_h \rng,
          \]
          \quad\, and write $\hat \pi^{h:k} = (\hat \pi_h, \ldots, \hat \pi_k)$. 
		\EndFor
		\State \textbf{Return:} Policy $\hat \pi^{1:k} \in \Pi$. 
	\end{algorithmic}
\end{algorithm}

\begin{lemma}[PSDP]
  \label{lem:psdp}
  There is a constant $C_{\ref{lem:psdp}} > 0$ so that the following holds. Fix $k \in [H]$ and $\alpha, \epsilon,\delta \in (0,1)$, and suppose that $\Psi_{1:k-1}$ are $\alpha$-approximate policy covers (\cref{defn:pc}) for steps $1,\dots,k-1$ respectively. Fix any $\theta \in \RR^d$ and $N \in \NN$ such that
  \[N \geq \nstatrch(\epsilon, \alpha, \delta) := \frac{C_{\ref{lem:psdp}} H^4A^2 \Cnrm^4 \log(Hd/\delta)}{\alpha^4 \epsilon^4}.\]
  Then the output of the algorithm $\PSDP(d,k,\Cnrm,\theta,\Psi_{1:k-1},N)$ is a policy $\hat \pi := \hat \pi^{1:k}$ that, with probability at least $1-\delta$, satisfies
  \begin{align}
\lng \E^{\hat \pi} [\phi_k(x_k, a_k)], \theta \rng \geq \max_{\pi \in \Pi} \lng \E^\pi[\phi_k(x_k, a_k)], \theta \rng - \epsilon\norm{\theta}_1\nonumber.
  \end{align} Moreover, the sample complexity of \PSDP with this input is $(k-1)N$, and the time complexity is $\poly(k,N,d,A)$.
\end{lemma}

\begin{proof}[Proof of Lemma~\ref{lem:psdp}]
  Define reward function $\bfr = (\bfr_1,\dots,\bfr_k)$ by 
  \[\bfr_h(x,a) = \begin{cases} \langle \phi_h(x,a),\theta\rangle & \text{ if } h = k \\ 0 & \text{ otherwise } \end{cases}.\] 
  Let $\pi^\st \in \argmax_{\pi\in\Pi}\langle\E^\pi[\phi_k(x_k,a_k)],\theta\rangle$. The performance difference lemma (\cref{lemma:perf-diff}) gives
  \[\lng \E^{\pi^\st}[\phi_k(x_k, a_k)], \theta \rng - \lng \E^{\hat \pi}[\phi_k(x_k, a_k)], \theta \rng = \sum_{h=1}^k \E^{\pi^\st} \left[ Q_h\sups{M,\hat \pi,\bfr}(x_h, \pi_h^\st(x_h)) - Q_h\sups{M,\hat \pi,\bfr}(x_h, \hat \pi_h(x_h))\right].\]
  When $h=k$ we have (by definition of $\pi^\st$ and $\hat\pi$) that $Q\sups{M,\hat\pi,\bfr}_k(x,\pi^\st_k(x)) = Q\sups{M,\hat\pi,\bfr}_k(x,\hat\pi_k(x)) = \max_{a\in\MA} \langle\phi_k(x,a),\theta\rangle$ for all $x\in\MX$, so the final summand of the above summation is $0$. 
  
  Now fix $h \in [k-1]$. By definition we have $\hat\pi_h(x) \in \argmax_{a\in\MA} \langle\phi_h(x,a),\hat\bw_h\rangle$. Additionally, by \cref{lemma:linear-q}, there is some vector $\bw_h^\st \in \BR^d$ (with $\| \bw_h^\st \|_1 \leq \| \theta \|_1 \cdot \Cnrm$) so that $Q_h\sups{M, \wh \pi^{h+1:k}, \bfr}(x,a) = \lng \phi_h(x,a), \bw_h^\st\rng$. For each $x \in \MX$, let us define $\Delta_h(x) := \max_{a \in \MA} | \lng \phi_h(x,a), \bw_h^\st - \hat \bw_h \rng |$. Then for any $x \in \MX$, we have
  \begin{align}
    Q_h\sups{M,\wh \pi,\bfr}(x, \pi_h^\st(x)) - Q_h\sups{M,\wh \pi,\bfr}(x, \wh \pi_h(x)) 
    &= \lng \phi_h(x, \pi_h^\st(x)), \bw_h^\st \rng - \lng \phi_h(x, \wh \pi_h(x)), \bw_h^\st \rng\nonumber\\
    &\leq  \lng \phi_h(x, \pi_h^\st(x)), \hat \bw_h \rng - \lng \phi_h(x, \wh \pi_h(x)), \bw_h^\st \rng + \Delta_h(x)\nonumber\\
    &\leq  \lng \phi_h(x, \wh \pi_h(x)), \hat \bw_h \rng - \lng \phi_h(x, \wh \pi_h(x)), \bw_h^\st \rng + \Delta_h(x)\nonumber\\
    &\leq  2\Delta_h(x)\nonumber,
  \end{align}
  where the first and third inequalities use the definition of $\Delta_h(x)$ and the second inequality uses the fact that $\wh \pi_h(x) \in \argmax_{a \in \MA} \lng \phi_h(x,a), \hat \bw_h \rng$. It follows that
  \begin{equation} 
  \lng \E^{\pi^\st}[\phi_k(x_k, a_k)], \theta \rng - \lng \E^{\hat \pi}[\phi_k(x_k, a_k)], \theta \rng
  \leq \sum_{h=1}^{k-1} \E^{\pi^\st}[2\Delta_h(x_h)],
  \label{eq:psdp-perf-diff}
  \end{equation}
  and it only remains to upper bound each term $\E^{\pi^\st}[\Delta_h(x_h)]$. Once more, fix $h \in [k-1]$. The dataset $\MD_h$ consists of independent samples $(x,a,r)$ with 
  \[\E[r|x,a] = \E^{\wh \pi^{h+1:k}}[\lng \phi_k(x_k, a_k), \theta \rng | (x_h, a_h) = (x,a)] = Q_h\sups{M,\wh \pi^{h+1:k},\bfr}(x,a).\]
   Thus, we can apply \cref{cor:random-design-prediction-error} with covariates $(\phi_h(x,a): (x,a,r) \in \MD_h)$, ground truth $\bw^\st_h \in \RR^d$, and responses $(r: (x,a,r) \in \MD_h)$. 
  Observe that for any sample $(x,a,r) \in \MD_h$, by definition there are some $x_k\in\MX,a_k\in\MA$ such that $r = \langle \phi_k(x_k,a_k),\theta\rangle$, so $|r| \leq \norm{\theta}_1$; thus, \[|r - Q_h^{\wh\pi^{h+1:k}}(x,a)| \leq \norm{\theta}_1 + \norm{\bw_h^\st}_1 \leq 2\Cnrm \norm{\theta}_1\]
  where the last inequality uses that $\Cnrm\geq 1$.
  Recalling the definition of the regressor $\hat\bw_h$ (\cref{eq:psdp-mistake} of \cref{alg:psdp}), \cref{cor:random-design-prediction-error} gives some event $\ME_h$ that holds with probability at least $1-\delta/H$, under which
  \begin{align}
\E_{\pi \sim \unif(\Psi_h)} \E^{\pi \circ_h \unif(\MA)} \left[ \lng \phi_h(x_h, a_h), \bw_h^\st - \hat \bw_h \rng^2 \right] \leq \frac{3C_{\ref{cor:random-design-prediction-error}}\Cnrm^2\norm{\theta}_1^2 \sqrt{\log(dH/\delta)}}{\sqrt{N}} =: \vep_0^2\label{eq:good-psdp-rch-event}.
  \end{align}
   It follows from \cref{eq:good-psdp-rch-event} and the definition of $\Delta_h(x)$ that
  \begin{align}
\E_{\pi \sim \unif(\Psi_h)} \E^{\pi}[\Delta_h(x)^2] \leq \E_{\pi \sim \unif(\Psi_h)} \E^\pi\left[ \sum_{a \in \MA} \lng \phi_h(x_h, a_h), \bw_h^\st - \hat \bw_h \rng^2 \right] \leq A \cdot\vep_0^2\nonumber,
  \end{align}
  which yields, via Jensen's inequality, that $\E_{\pi \sim \unif(\Psi_h)} \E^\pi[\Delta_h(x_h)] \leq \sqrt{A} \cdot \vep_0$. Using the assumption that $\Psi_h$ is an $\alpha$-approximate policy cover for step $h$ (\cref{defn:pc}) together with non-negativity of $\Delta_h(x)$, it follows that under event $\ME_h$,
  \begin{align}
    \max_{\pi \in \Pi} \E^\pi[\Delta_h(x_h)] =& \max_{\pi \in \Pi} \sum_{x \in \MX} d_h^\pi(x) \cdot \Delta_h(x)\nonumber\\
    &\leq  \frac{1}{|\Psi_h|}\sum_{\pi' \in \Psi_h} \sum_{x \in \MX}\frac{d_h(\pi')(x)}{\alpha} \cdot \Delta_h(x) \nonumber\\
    &\leq \frac{1}{\alpha} \cdot \E_{\pi \sim \unif(\Psi_h)} \E^\pi[\Delta_h(x_h)]\nonumber\\
    &\leq \frac{\sqrt{A} \cdot \vep_0}{\alpha} \nonumber.
  \end{align}
Substituting into \cref{eq:psdp-perf-diff}, we conclude that, under the event $\bigcap_{h=1}^{k-1} \ME_h$ (which occurs with probability at least $1-\delta$), 
  \[
    \lng \E^{\pi^\st}[\phi_k(x_k, a_k)], \theta \rng - \lng \E^{\hat \pi}[\phi_k(x_k, a_k)], \theta \rng \leq \frac{2 (k-1)\sqrt{A} \vep_0}{\alpha} \leq \epsilon \norm{\theta}_1
  \]
  where the last inequality uses the definition of $\vep_0$ (from \cref{eq:good-psdp-rch-event}) and the assumption that $N \geq \nstatrch(\epsilon,\alpha,\delta)$ (where we take $C_{\ref{lem:psdp}} := 2\sqrt{3C_{\ref{cor:random-design-prediction-error}}}$).
\end{proof}

\subsubsection{Feature Estimation}
\begin{algorithm}[t]
\caption{$\FE(k,\pi,N)$: Feature Estimation}
	\label{alg:fe}
	\begin{algorithmic}[1]\onehalfspacing
		\Require Target layer $k\in[H]$; policy $\pi \in \Pi$; number of samples $N\in \mathbb{N}$.
		\For{$1 \leq i \leq N$}
		\State Sample $(x^i_1,a^i_1,\dots,x^i_k,a^i_k) \sim \pi$.
		\EndFor
		\State \textbf{Return:} $\frac{1}{N}\sum_{i=1}^N \phi_k(x_k^i,a_k^i)$. 
	\end{algorithmic}
  \end{algorithm}

  \begin{lemma}[Feature estimation]
  \label{lem:fe}
  Fix $k \in [H]$, $\pi \in \Pi$, $N \in \NN$, and $\epstat, \delta>0$. Suppose that 
  \[N \geq \nfe(\epstat,\delta) := 2\epstat^{-2}\log(2d/\delta).\]
  Then the output of $\FE(k,\pi,N)$ (\cref{alg:fe}) is a vector $\wh \phi_k \in \BR^d$ such that $\norm{\hat \phi_k}_\infty \leq 1$, and with probability at least $1-\delta$,
  \begin{align}
\left\| \E\sups{M, \pi}[\phi_k(x_k, a_k)] - \wh\phi_k \right\|_\infty \leq \epstat\nonumber.
  \end{align}
  The sample complexity of \FE with this input is $N$, and the time complexity is $\poly(N,d)$.
\end{lemma}
\begin{proof}
  Immediate from Hoeffding's inequality and a union bound over $[d]$.
\end{proof}

\section{Beyond reachability: the truncated MDP}\label{sec:trunc-mdps}
For the remainder of the paper, we drop the reachability condition and assume that we have sample access to an arbitrary $\ell_1$-bounded $d$-dimensional linear MDP $M = (H,\MX,\MA,\BP_1,(\phiorig_h)_h,(\muorig_{h+1})_h,\allowbreak(\thetaorig_h)_h)$, with known $\ell_1$ bound $\Cnrm \geq 1$. We fix a global parameter $\epfinal>0$; our ultimate goal is to show that \OPT (\cref{alg:opt}) finds (with high probability) a policy with suboptimality at most $\epfinal$.

\subsection{Global truncation parameters \& extended feature mappings}\label{section:extended-overview}

As discussed in \cref{section:trunc-overview}, the analysis will involve $(d+1)$-dimensional ``truncated'' MDPs. These are formally defined in \cref{sec:trunc-mdp-def} in terms of the parameter of $M$ as well as two global parameters $\trunc$ and $\tsmall$, which we set as $\trunc := \frac{\epfinal}{8\Cnrm^3 H^4}$ and $\tsmall := \frac{\trunc^7}{2^{48} A^{13} H^7 \Cnrm^9}$ ($\epfinal$ denotes an overall error parameter which will be introduced in \cref{sec:learning}). The first parameter governs the threshold at which nearly-unreachable states are truncated, and the second parameter governs how well the backup policy cover $\Gamma$ needs to visit a state for it to avoid truncation.

For notational convenience, throughout \cref{sec:trunc-mdps,sec:policy-cover-unreachable}, we extend $M$ to a $(d+1)$-dimensional linear MDP $(H, \bar\MX,\MA,\BP_1,(\phi_h)_h,(\mu_{h+1})_h,(\theta_h)_h)$ on state space $\bar\MX := \MX\cup\{\term\}$, where $\term$ is a special terminal state with $\BP_1(\term) = 0$. For any $h \in [H]$, we define $\phi_h: \bar\MX\times\MA \to \RR^{d+1}$, $\mu_{h+1}:\bar\MX\to\RR^{d+1}$, and $\theta_h : \bar\MX \times \MA \ra \BR^{d+1}$ as follows:
\[\phi_h(x,a) = \begin{cases}
[\phi^\mathsf{orig}_h(x,a), 0] & \text{ if } x \in \MX \\
e_{d+1} & \text{ if } x = \term
\end{cases}.
\]

\[\mu_{h+1}(x) = \begin{cases} 
[\mu^\mathsf{orig}_{h+1}(x), 0] & \text{ if } x \in \MX \\ 
e_{d+1} & \text{ if } x = \term
\end{cases}.
\]

\[\theta_h(x,a) = \begin{cases}
[\theta^\mathsf{orig}_h(x,a), 0] & \text{ if } x \in \MX \\
0 & \text{ if } x = \term
\end{cases}.
\]

Note that this change is without loss of generality: the state $\term$ is never visited, so interaction with $(H,\MX,\MA,\BP_1,(\phiorig_h)_h,(\muorig_{h+1})_h,(\thetaorig_h)_h)$ is identical to interaction with $(H, \bar\MX,\MA,\BP_1,(\phi_h)_h,(\mu_{h+1})_h,\allowbreak(\theta_h)_h)$. We make this change just to simplify notation in the analysis when we relate $M$ with the truncated MDPs defined in \cref{sec:trunc-mdp-def}. Since $\theta_h(\term, a) = 0$ for all $a$, and the terminal state is absorbing (by definitions of $\phi_h(\term, a), \mu_h(\term)$), it follows that $Q_h\sups{M, \pi}(\term, a) = V_h\sups{M, \pi}(\term) = 0$ for all $\pi \in \Pi,\ a \in \MA$, and $h \in [H]$. As per \cref{def:phi-avg}, we will continue to write $\phiavg_h(x) := \frac 1A \sum_{a \in \MA} \phi_h(x,a)$. At various points, it will be convenient to write $\mu^{\mathrm{orig}}_{H+1}(x) = 0 \in \BR^d$, $\mu_{H+1}(x) = 0 \in \BR^{d+1}$ as a matter of convention. %

\subsection{The truncated MDP}\label{sec:trunc-mdp-def}
In this section we define the truncated MDPs $\Mbar(\Gamma)$ that will be used in the analysis of \cref{alg:slm-trunc}. Roughly, these MDPs ``truncate'' transitions to states of $M$ that are hard to reach. We emphasize that the truncated MDPs are purely a tool for the analysis. Like $M$, their parameters are unknown to the learning algorithm. We \emph{do not} assume that the algorithm has sample access to any of the truncated MDPs.

\cref{def:r-truncation} below defines a generic notion of truncation of the linear MDP $M$, where transitions to some specified subsets $\Xreach_h \subseteq \MX$ ($h \in [H]$) of states are truncated. Subsequently we will discuss how precisely we define the sets $\Xreach_h$. We leave the environmental reward vectors of the truncated MDPs unspecified, since they do not matter for the analysis.
\begin{definition}
  \label{def:r-truncation}
Let $\Xreach = (\Xreach_2,\dots,\Xreach_H)$ for some given sets $\Xreach_2,\dots,\Xreach_{H} \subseteq \MX$. The $\Xreach$-\emph{truncation} of $M$ is the  $(d+1)$-dimensional linear MDP $\Mbar = (H,\bar\MX,\MA,\BP_1,(\phi\sups{\Mbar}_h)_h,(\mu\sups{\Mbar}_{h+1})_h)$ with state space $\bar\MX$, action space $\MA$, and horizon $H$, whose feature vectors $\phi_h\sups{\Mbar}(x,a), \mu_{h+1}\sups{\Mbar}(x)$, are defined as follows, for each $h \in [H]$: %
\begin{enumerate}
    \item For every $x \in \MX$, $a \in \MA$,%
      \[ \phi_h\sups{\Mbar}(x,a) := [\phi^{\mathsf{orig}}_h(x,a), 1- \preach_h(x,a)] \text{ and } \mu_{h+1}\sups{\Mbar}(x) := \begin{cases} \mu_{h+1}(x) & \text{ if } x \in \Xreach_{h+1} \\ 0 & \text{ otherwise}\end{cases},\]
      where $\Xreach_{H+1} := \emptyset$ and
      \begin{align}
\preach_h(x,a) = \sum_{x' \in \Xreach_{h+1}} \lng \phi_h(x,a), \mu_{h+1}(x') \rng\nonumber.
      \end{align}
    \item For every $a \in \MA$,
    \[ \phi_h\sups{\Mbar}(\term,a) = e_{d+1} = \mu_{h+1}\sups{\Mbar}(\term).\]
\end{enumerate}
\end{definition}

Essentially, whenever a state outside $\Xreach_h$ would have been visited under the dynamics of $M$ at some step $h$, instead, under the dynamics of $\Mbar$, the terminal state $\term$ is visited; moreover, for the remainder of the trajectory all states are $\term$.

\paragraph{Overview of the construction.} Suppose we are given a finite subset of policies $\Gamma \subset \Pi$. We will define a \emph{truncated MDP} $\Mbar(\Gamma)$, which will be the $\Xreach(\Gamma)$-truncation of $M$ (according to \cref{def:r-truncation}) for an appropriate choice of $\Xreach(\Gamma) = (\Xreach_2(\Gamma), \ldots, \Xreach_H(\Gamma))$. The choice of $\Xreach(\Gamma)$ (and thus $\Mbar(\Gamma)$) depends on $\trunc, \tsmall$ in addition to $\Gamma$; since these are fixed global parameters (defined in \cref{section:extended-overview}) we omit this dependence. Roughly speaking, the sets $\Xreach_2(\Gamma), \ldots, \Xreach_H(\Gamma)$ are constructed so as to satisfy the following properties:
\begin{itemize}
\item Each set $\Xreach_h(\Gamma)$ includes all states $x$ that can be visited in the MDP $\Mbar(\emptyset)$, under \emph{some} policy in $\Pi$, with probability at least $\trunc \cdot \| \mu_h(x) \|_1$. %
\item Each set $\Xreach_h(\Gamma)$ also includes states $x$ that are visited in the MDP $M$, under a \emph{uniformly random} policy drawn from $\Gamma$,  with probability at least $\tsmall \cdot \| \mu_h(x) \|_1$.
\item Each set $\Xreach_h(\Gamma)$ does not include any other states. 
\end{itemize}
For technical reasons, $\Xreach_h(\Gamma)$ is specified as above only for \emph{odd} $h \in [H]$; for even $h \in [H]$ we simply take $\Xreach_h(\Gamma) := \MX$. Below we proceed to formally define the sets  $\Xreach_h(\Gamma)$. 

\paragraph{Defining $\Xreach_h(\Gamma)$.} Next we inductively define $\Xreach_h(\Gamma)$, for each $h \in [H]$. To aid in doing so, we define, for each $\Gamma \subset \Pi$, $\Mbar_h(\Gamma)$ to be the $(\Xreach_2(\Gamma), \ldots, \Xreach_h(\Gamma), \MX, \ldots, \MX)$-truncation of $M$ (per \cref{def:r-truncation}). Note that $\Mbar_1(\Gamma) = M$.  Given $h \in [H]$ with $h \geq 2$, suppose that $\Xreach_g(\Gamma)$, has been defined for all $\Gamma \subset \Pi$ and all $2 \leq g < h$. It follows that $\Mbar_g(\Gamma)$ has been defined, as well, for all $\Gamma \subset \Pi$ and $1 \leq g < h$. We define $\Xreach_h(\Gamma)$ for each $\Gamma \subset \Pi$ as follows: first, define 
\begin{align}
 \til \MX_{h}^{\mathrm{rch}} :=  \begin{cases}
    \MX &: \mbox{ $h$ is even} \\
    \left\{ x \in \MX \ : \ \max_{\pi \in \Pi} \lng \mu_{h}(x), \E\sups{\Mbar_{h-1}(\emptyset), \pi}[\phi_{h-1}(x_{h-1}, a_{h-1})] \rng \geq \trunc \cdot \| \mu_{h}(x) \|_1 \right\} &: \mbox{ $h$ is odd,}\end{cases}\nonumber
\end{align}
and then %
\begin{align}
 \Xreach_{h}(\Gamma) := \tilXreach_{h} \cup \left\{ x \in \MX \ : \ \E_{\pi' \sim \unif(\Gamma)}[d_{h}\sups{M,\pi'}(x)] \geq \tsmall \cdot \| \mu_h(x) \|_1 \right\}\label{eq:define-xreach-gamma}.
\end{align}
If $\Gamma = \emptyset$, we use the convention that $\E_{\pi' \sim \unif(\Gamma)}[d_h\sups{M, \pi'}(x)] := 0$ for all $x \in \MX$, so that $\Xreach_h(\emptyset) = \tilXreach_h$. Having defined $\Xreach_h(\Gamma)$, it follows that $\Mbar_h(\Gamma)$ is defined (i.e., as the $(\Xreach_2(\Gamma), \ldots, \Xreach_h(\Gamma), \MX, \ldots, \MX)$-truncation of $M$). Finally, we define $\Mbar(\Gamma) = \Mbar_H(\Gamma)$.

 By the construction in \cref{def:r-truncation}, $\Mbar_h(\Gamma)$ is truncated up to the transition into state $x_h$, and in particular the transitions of $\Mbar_h(\Gamma)$ and $\Mbar_H(\Gamma)$ are identical up to step $h$; thus, for any choice of policy $\pi$,  $\Mbar_h(\Gamma)$ has the same distribution on trajectory prefixes $(x_1,a_1,\dots,x_{h-1}, a_{h-1},x_h)$ as $\Mbar_H(\Gamma) = \Mbar(\Gamma)$. In particular, we have the following fact:
\begin{fact}\label{fact:trunc-intermediate-dist}
For any $\pi \in \Pi$, $h \in [H]$, $g \leq h$, $\Gamma \subset \Pi$, and $x \in \bar\MX$, it holds that $d\sups{\Mbar_h(\Gamma), \pi}_g(x) = d\sups{\Mbar(\Gamma), \pi}_g(x)$.
\end{fact}

On the other hand, the transition distribution of $\Mbar_h(\Gamma)$ from step $h$ to step $h+1$ is the same as that of $M$:

\begin{fact}\label{fact:trunc-intermediate-trans}
For any $h \in [H]$, $\Gamma\subset \Pi$, $x \in \bar\MX$, and $a \in \MA$, it holds that $\BP\sups{\Mbar_h(\Gamma)}_h(\cdot|x,a) = \BP\sups{M}_h(\cdot|x,a)$.
\end{fact}

\subsection{Properties of truncated MDPs}

The following lemma is an immediate consequence of the construction of $\Mbar(\Gamma)$.
\begin{lemma}
  \label{lem:reachability-in-trunc}
 Fix $\Gamma \subset \Pi$ and odd $h \in [H]$ with $h \geq 2$. For all $x \in \MX$, if $\mu_h\sups{\Mbar(\Gamma)}(x) \neq 0$, then either
  \begin{align}
\max_{\pi \in \Pi} d_h\sups{\Mbar(\emptyset), \pi}(x) = \max_{\pi \in \Pi} \lng \mu_h(x), \E\sups{\Mbar(\emptyset), \pi}[\phi_{h-1}(x_{h-1}, a_{h-1})] \rng \geq \trunc \cdot \| \mu_h(x) \|_1\label{eq:reachability-in-trunc}
  \end{align}
  or
  \begin{align}
\E_{\pi' \sim \unif(\Gamma)}[d_{h}\sups{M, \pi'}(x)]\geq \tsmall \cdot \| \mu_h(x) \|_1.\label{eq:gamma-good}
  \end{align}
\end{lemma}
\begin{proof}
  If $\mu_h\sups{\Mbar(\Gamma)}(x) \neq 0$ for $x \in \MX$, then by definition of $\mu_h\sups{\Mbar(\Gamma)}(x)$ (\cref{def:r-truncation}) we must have that $x \in \MX_h^{\mathrm{rch}}(\Gamma)$. 
Hence, either $x \in \tilXreach_h$, in which case \cref{eq:reachability-in-trunc} holds (the equality is by the fact that $\mu_h(x) = \mu\sups{\Mbar(\emptyset)}_h(x)$ and $d_{h-1}\sups{\Mbar(\emptyset),\pi} = d_{h-1}\sups{\Mbar_{h-1}(\emptyset),\pi}$, per \cref{fact:trunc-intermediate-dist}; the inequality is by definition of $\tilXreach_h$); or otherwise $x \in \Xreach_h(\Gamma) \setminus \tilXreach$, in which case \cref{eq:gamma-good} holds, by \cref{eq:define-xreach-gamma}. 
\end{proof}

\begin{lemma}\label{lem:gamma-monotonicity}
Fix $\Gamma \subset \Pi$. For all $x \in \MX$ and $\pi \in \Pi$ and $h \in [H]$, it holds that $d\sups{\Mbar(\emptyset), \pi}_h(x) \leq d\sups{\Mbar(\Gamma),\pi}_h(x) \leq d\sups{M,\pi}_h(x)$. %
\end{lemma}

\begin{proof}
  We start by proving that $d\sups{\Mbar(\Gamma),\pi}_h(x) \leq d\sups{M,\pi}_h(x)$ for all $h\in[H]$, $\pi \in \Pi$, $x \in \MX$. We proceed by induction on $h$, noting that the base case $h=1$ is immediate since $d_1\sups{\Mbar(\Gamma), \pi}(x) = d_1\sups{M, \pi}(x) = \BP_1(x)$ for all $x \in \MX, \pi \in \Pi$. To establish the inductive step, we first note that if $\mu_h\sups{\Mbar(\Gamma)}(x) = 0$, then $d_h\sups{\Mbar(\Gamma), \pi}(x) = 0$ whereas $d_h\sups{M, \pi}(x) \geq 0$. Otherwise (i.e., if $\mu_h\sups{\Mbar(\Gamma)}(x) \neq 0$), we have $\mu_h\sups{\Mbar(\Gamma)}(x) = \mu_h(x)$. Hence, %
  \begin{align}
    d_h\sups{\Mbar(\Gamma), \pi}(x) &=  \lng \mu_h(x), \E\sups{\Mbar(\Gamma), \pi}[\phi_{h-1}(x_{h-1}, a_{h-1})] \rng \nonumber\\
                                    & = \sum_{x_{h-1} \in \MX} d_h\sups{\Mbar(\Gamma), \pi}(x_{h-1}) \cdot \lng \mu_h(x), \phi_{h-1}(x_{h-1}, \pi_{h-1}(x_{h-1}))\rng \nonumber\\
    &\leq  \sum_{x_{h-1} \in \MX} d_h\sups{M, \pi}(x_{h-1}) \cdot \lng \mu_h(x), \phi_{h-1}(x_{h-1}, \pi_{h-1}(x_{h-1}))\rng \nonumber\\
    &\leq d_h\sups{M, \pi}(x)\label{eq:monotonicity-inductive-step},
  \end{align}
  where the first inequality uses the inductive hypothesis and the fact that $\lng \mu_h(x), \phi_{h-1}(x_{h-1}, a_{h-1}) \rng \geq 0$ for all $x_{h-1}, a_{h-1}, x$.

  The proof that $d_h\sups{\Mbar(\emptyset), \pi}(x) \leq d_h\sups{\Mbar(\Gamma), \pi}(x)$ for all $h \in [H]$, $\pi \in \Pi$, $x \in \MX$ is similar. The base case $h=1$ is again immediate. For the induction step, we first note that if $\mu_h\sups{\Mbar(\Gamma)}(x) = 0$, then $\mu_h\sups{\Mbar(\emptyset)}(x) = 0$ (since $\Xreach_h(\emptyset) \subseteq \Xreach_h(\Gamma)$), and so $d_h\sups{\Mbar(\emptyset), \pi}(x) = d_h\sups{\Mbar(\Gamma), \pi}(x) = 0$ for all $\pi$. Otherwise, we have that $\mu_h\sups{\Mbar(\Gamma)}(x) = \mu_h(x)$, 
  and an analogous computation to \cref{eq:monotonicity-inductive-step} establishes that $d_h\sups{\Mbar(\emptyset), \pi} \leq d_h\sups{\Mbar(\Gamma), \pi}$. 
\end{proof}

\subsection{Closeness between truncated MDPs and $M$}

Next, we show that the truncated MDP $\Mbar(\emptyset)$ is a good approximation to the true MDP $M$.
\begin{lemma}
  \label{lem:trunc-loss-empty}
Write  $\Mbar = \Mbar(\emptyset)$.   Fix $B > 0$, and consider any vectors $\bw_1, \ldots, \bw_H \in \BR^d \times \{0\}$ with $\| \bw_h \|_1 \leq B$ for each $h \in [H]$. For each $h \in [H]$, $x \in \bar\MX$, and $a \in \MA$, define $r_h(x,a) := \lng \phi_h(x,a), \bw_h\rng$. Then for any policy $\pi \in \Pi$, 
  \begin{align}
 \E\sups{\Mbar, \pi} \left[\sum_{h=1}^H r_h(x_h,a_h) \right] \geq  \E\sups{M, \pi} \left[ \sum_{h=1}^H r_h(x_h,a_h)\right] - B \trunc  \Cnrm H^2 \nonumber.
  \end{align}
\end{lemma}
\begin{proof}
  Recall from above that, for $\Gamma \subset \Pi$ and each $h \in [H]$, $\Mbar_h(\Gamma)$ is defined to be the \\ $(\Xreach_2(\Gamma), \ldots, \Xreach_{h}(\Gamma), \MX, \ldots, \MX)$-truncation of $M$. 
  Throughout the proof of the lemma, we write $\Mbar_h = \Mbar_h(\emptyset)$ for each $h \in [H]$, $\Mbar = \Mbar(\emptyset)$, and $\Xreach_h = \Xreach_h(\emptyset) = \tilXreach_h$ for each $h \geq 2$. 

  Consider any policy $\pi \in \Pi$ and any $1 \leq h \leq H-1$. %
  Since the transitions of $\Mbar_{h+1}$ and $\Mbar_{h}$ are the same up to step $h$ (\cref{fact:trunc-intermediate-dist}), we have
  \begin{align}
\E\sups{\Mbar_{h+1}, \pi} \left[ \sum_{g=1}^h r_g(x_g, a_g) \right] = \E\sups{\Mbar_{h}, \pi} \left[ \sum_{g=1}^h r_g(x_g, a_g)\right]\label{eq:first-h-rewards}.
  \end{align}
 Since the transitions of $\Mbar_{h+1}$ and $\Mbar_{h}$ at steps $\{h+1,\dots,H\}$ are identical to those of $M$ (\cref{fact:trunc-intermediate-trans}), it also holds that for all $(x,a) \in \MX\times \MA$,
  \begin{align}
    &\E\sups{\Mbar_{h+1}, \pi}\left[ \sum_{g=h+1}^H r_g(x_g, a_g) \middle| (x_{h+1}, a_{h+1}) = (x,a) \right]\nonumber\\
    =\,&  \E\sups{\Mbar_{h}, \pi}\left[ \sum_{g=h+1}^H r_g(x_g, a_g) \middle| (x_{h+1}, a_{h+1}) = (x,a) \right] =: f_{h+1}(x,a).\label{eq:theta-rewards}
  \end{align}
Note that $|f_{h+1}(x,a)| \leq BH$ for all $(x,a) \in\MX\times\MA$.
We now compute
  \begin{align}
     \E\sups{\Mbar_{h}, \pi} \left[ \sum_{g=h+1}^H r_g(x_g, a_g) \right]
    &= \sum_{x \in \MX} \left\lng \mu_{h+1}\sups{\Mbar_{h}}(x), \E\sups{\Mbar_{h}, \pi} [\phi_h\sups{\Mbar_{h}}(x_h, a_h)] \right\rng \cdot f_{h+1}(x,\pi_{h+1}(x))\nonumber\\
    &= \sum_{x \in \MX} \left\lng \mu_{h+1}(x), \E\sups{\Mbar_{h+1}, \pi} [\phi_h(x_h, a_h)] \right\rng \cdot f_{h+1}(x,\pi_{h+1}(x))\nonumber\\
    &= \sum_{x \in \Xreach_{h+1}}\left\lng \mu_{h+1}\sups{\Mbar_{h+1}}(x), \E\sups{\Mbar_{h+1}, \pi} [\phi\sups{\Mbar_{h+1}}_h(x_h, a_h)] \right\rng \cdot f_{h+1}(x,\pi_{h+1}(x)) \nonumber\\
    &\quad + \sum_{x \in \MX \setminus \Xreach_{h+1}} \left\lng \mu_{h+1}(x), \E\sups{\Mbar_{h}, \pi}[\phi_h(x_h, a_h)] \right\rng \cdot f_{h+1}(x,\pi_{h+1}(x))\nonumber\\
    &\leq  \sum_{x \in \MX_{h+1}^{\mathrm{rch}}}\left\lng \mu_{h+1}\sups{\Mbar_{h+1}}(x), \E\sups{\Mbar_{h+1}, \pi} [\phi\sups{\Mbar_{h+1}}_h(x_h, a_h)] \right\rng \cdot f_{h+1}(x,\pi_{h+1}(x)) \nonumber\\
    &\quad+ \sum_{x \in \MX \backslash \MX_{h+1}^{\mathrm{rch}}}\trunc \cdot \| \mu_{h+1}(x) \|_1 \cdot  B H\nonumber\\
    &\leq  \E\sups{\Mbar_{h+1}, \pi} \left[ \sum_{g=h+1}^H r_g(x_g, a_g) \right] + \trunc \cdot \Cnrm B H\label{eq:hybrid-trunc-bound},
  \end{align}
  where the first equality uses \cref{eq:theta-rewards} together with the fact that the reward at all future states is 0 if $\term$ is ever reached; the second equality uses the definition of the feature vectors $ \mu\sups{\Mbar_h}_{h+1}(x), \phi\sups{\Mbar_h}_h(x,a)$ as well as the fact that $d\sups{\Mbar_h,\pi}_h = d\sups{\Mbar_{h+1},\pi}_h$; the third equality splits the sum and uses that $\mu_{h+1}\sups{\Mbar_{h+1}}(x) = \mu_{h+1}(x)$ for $x \in \Xreach_{h+1}$; the first inequality uses %
  the definition of $\MX^{\mathrm{rch}}_{h+1} = \Xreach_{h+1}(\emptyset)$ and
  the previously established bound on $|f_{h+1}(x,a)|$; and the final inequality uses the fact that $\mu\sups{\Mbar_{h+1}}_{h+1}(x) = 0$ for all $x \in \MX\setminus \Xreach_{h+1}$ together with another application of \cref{eq:theta-rewards}. Summarizing, we have shown (by combining \cref{eq:first-h-rewards} and \cref{eq:hybrid-trunc-bound}) that for each $h \leq H-1$, it holds that
  \begin{align}
\E\sups{\Mbar_{h}, \pi}\left[ \sum_{g=1}^H r_g(x_g, a_g)\right] \leq \E\sups{\Mbar_{h+1}, \pi} \left[ \sum_{g=1}^H r_g(x_g, a_g) \right] + \trunc \Cnrm B H \nonumber.
  \end{align}
  Telescoping the above inequality over $1 \leq h \leq H-1$ and using $\Mbar_1 = M$ and $\Mbar_{H} = \Mbar$ yields the desired result. %
\end{proof}

An immediate corollary of \cref{lem:trunc-loss-empty} is the following result bounding the probability of reaching $\MX \backslash \Xreach_h(\Gamma)$ under any policy, for each $h \in [H]$. 
\begin{lemma}
  \label{cor:not-reach-bound}
  For any $\Gamma \subset \Pi$ and $h \in [H]$ with $h \geq 2$, it holds that
   \begin{align}
 \max_{\pi \in \Pi} d_h\sups{M, \pi}(\MX \backslash \Xreach_h(\Gamma)) \leq 2 \trunc \Cnrm^2 H^2\nonumber,
   \end{align}
   Hence, for any $\Gamma \subset \Pi$ and $h \in [H]$, it holds that $\min_{\pi \in \Pi} d_h\sups{\Mbar(\Gamma), \pi}(\MX) \geq 1-2\trunc \Cnrm^2 H^3 \geq 3/4$. 
 \end{lemma}
 \begin{proof}
   It suffices to prove the result for the case $\Gamma = \emptyset$, since $\MX \backslash \Xreach_h(\Gamma) \subseteq \MX \backslash \Xreach_h(\emptyset)$ for all $h \in [H]$ and $\Gamma \subset \Pi$. Consider any $h \geq 2$, and any $x \in \MX \backslash \Xreach_h(\emptyset)$. For any $\pi \in \Pi$, 
   \begin{align}
     d_h\sups{M, \pi}(x) 
     &= \E\sups{M, \pi} \left[ \lng \mu_h(x), \phi_{h-1}(x_{h-1}, a_{h-1}) \rng \right] \nonumber\\
     &\leq \E\sups{\Mbar(\emptyset), \pi} \left[ \lng \mu_h(x), \phi_{h-1}(x_{h-1}, a_{h-1}) \rng \right] +  \| \mu_h(x) \|_1 \cdot \trunc \Cnrm H^2\nonumber\\
     &\leq  \trunc \cdot \| \mu_h(x) \|_1 + \| \mu_h(x) \|_1 \cdot \trunc \Cnrm H^2\nonumber\\
     &\leq  \| \mu_h(x) \|_1  \cdot 2 \trunc \Cnrm H^2\nonumber,
   \end{align}
   where the first inequality uses \cref{lem:trunc-loss-empty} with $\bw_{h-1} = \mu_h$ and $\bw_g = 0$ for all $g \neq h-1$, and the second inequality uses that $x \not \in \Xreach_h(\emptyset) = \tilXreach_h$ and the definition of $\tilXreach_h$. Summing over all $x \in \MX \backslash \Xreach_h(\Gamma)$, we obtain the first claim of the lemma. To prove the second claim, fix $\pi \in \Pi$. Note that $d\sups{\Mbar(\Gamma),\pi}_1(\term) = 0$, and for any $h \geq 2$,
   \[ d\sups{\Mbar(\Gamma),\pi}_h(\term) = d\sups{\Mbar(\Gamma),\pi}_{h-1}(\term) + d\sups{\Mbar_{h-1}(\Gamma),\pi}_h(\MX\setminus\Xreach_h(\Gamma)) \leq d\sups{\Mbar(\Gamma),\pi}_{h-1}(\term) + d\sups{M,\pi}_h(\MX\setminus\Xreach_h(\Gamma)).\]
   Applying the first claim of the lemma statement and telescoping yields $d_h\sups{\Mbar(\Gamma),\pi}(\term) \leq 2\trunc\Cnrm^2 H^3$ for all $h \in [H]$. Finally, the bound $2\trunc\Cnrm^2 H^3 \leq 1/4$ holds by choice of $\trunc$ (\cref{section:extended-overview}).
 \end{proof}

 The next lemma upper bounds the distance between the expected feature vectors under $M$ and $\Mbar(\Gamma)$ at any fixed step $k$  by the probability of visiting the (approximately) unreachable set $\MX \backslash \Xreach_g(\Gamma)$, at steps $g \in [k]$. At a high level, the statement follows from the fact that transitions under $M$ and $\Mbar(\Gamma)$ only differ when a state in $\MX \backslash \Xreach_g(\Gamma)$ is reached in $M$, for some $g \in [H]$. %
\begin{lemma}
  \label{lem:m-mbar-delta}
  Consider any $k \in [H]$ and vector $w \in \BR^d \times \{0\}$, and suppose we are given $\Gamma \subset \Pi$. Then for all $\pi \in \Pi$, 
  \begin{align}
\left| \lng w,\E\sups{M, \pi}[\phi_k(x_k, a_k)]\rng - \lng w, \E\sups{\Mbar(\Gamma), \pi}[\phi_k(x_k, a_k)] \rng \right| &\leq \| w \|_1 \cdot \sum_{g=1}^k d_g\sups{M, \pi}(\MX \backslash \Xreach_g(\Gamma))\nonumber.
  \end{align}
\end{lemma}
\begin{proof}
  Let $\bfr = (\bfr_1,\dots,\bfr_H)$ denote the reward function given by $\bfr_k(x,a) := \lng w, \phi_k(x,a) \rng$ for all $(x,a) \in \bar\MX\times\MA$, and $\bfr_h(x,a) := 0$ for all $h \neq k$ and $(x,a) \in \bar\MX\times\MA$. For each $x \in \bar \MX$ and $\pi \in \Pi$, we have $\E_{a \sim \pi_h(x)}[Q_h\sups{M, \pi}(x,a)] = V_h\sups{M, \pi}(x)$. Then %
  \begin{align}
    &\left| \lng w,\E\sups{M, \pi}[\phi_k(x_k, a_k)]\rng - \lng w, \E\sups{\Mbar(\Gamma), \pi}[\phi_k(x_k, a_k)] \rng\right| \nonumber\\
    &= \left|\sum_{h=1}^{k-1} \E\sups{\Mbar(\Gamma), \pi} [ Q_h\sups{M, \pi,\bfr}(x_h, a_h) - V_{h+1}\sups{M, \pi,\bfr}(x_{h+1})]\right|\nonumber\\
    &= \left|\sum_{h=1}^{k-1} \E\sups{\Mbar(\Gamma), \pi}%
       \left[ \E_{x_{h+1} \sim \BP_h\sups{M}(x_h, a_h)}[V_{h+1}\sups{M, \pi,\bfr}(x_{h+1})] - \E_{x_{h+1} \sim \BP_h\sups{\Mbar(\Gamma)}(x_h, a_h)}[V_{h+1}\sups{M, \pi,\bfr}(x_{h+1})]\right]\right|\nonumber\\
    &\leq  \sum_{h=1}^{k-1} \E\sups{\Mbar(\Gamma), \pi}%
           \left[ \E_{x_{h+1} \sim \BP_h\sups{M}(x_h, a_h)}[\| w_k \|_1 \cdot \One{x_{h+1} \in \MX \backslash \Xreach_{h+1}(\Gamma)}]\right]\nonumber\\
    &\leq  \| w \|_1 \cdot \sum_{h=1}^k d_h\sups{M, \pi}(\MX \backslash \Xreach_h(\Gamma))\nonumber,
  \end{align}
  where the first equality uses \cref{lem:simulation}, the first inequality uses that $| V_{h+1}\sups{M, \pi,\bfr}(x)| \leq \| w \|_1$ for all $h, \pi, x$ as well as the bound $$\tvd{\BP_h\sups{M}(x, a)}{\BP_h\sups{\Mbar(\Gamma)}(x, a)} \leq \E_{x' \sim \BP_h\sups{M}(x,a)}\left[\One{x' \in \MX \backslash \Xreach_{h+1}(\Gamma)}\right],$$
  and the second inequality uses the fact that $d_h\sups{M, \pi}(x) \geq d_h\sups{\Mbar(\Gamma)}(x)$ for all $h, \pi, x$ (\cref{lem:gamma-monotonicity}). %
\end{proof}

\section{Constructing a policy cover in a general LMDP}
\label{sec:policy-cover-unreachable}

In this section, we prove \cref{theorem:main-pc-trunc}, which states that \SLMt (\cref{alg:slm-trunc}) succeeds in constructing a truncated policy cover of $M$ at steps $1,\dots,H$. The proof is structured similarly to that of \cref{theorem:main-pc} (for the reachable case), with the additional wrinkle raised in the technical overview (\cref{section:trunc-overview}) that necessitates the multiple phases of \SLMt. In particular, the proof is organized as follows:

\begin{itemize}
\item In \cref{sec:convex-program-unreachable}, we formally define the notion of a \emph{truncated emulator}, and we show that the algorithm \ESCt (\cref{alg:esc-trunc}), given a truncated policy cover at step $h$, constructs a truncated emulator at step $h$. This is analogous to the analysis of \ESC in \cref{sec:convex-program}.

\item In \cref{sec:emulator-pc-unreachable}, we analyze \PC (\cref{alg:pc}), and show that given truncated policy covers at steps $1,\dots,h$, and a truncated emulator at step $h$, either the algorithm constructs truncated policy covers for steps $h+1$ and $h+2$, or else it finds a set of policies with large \emph{extraneous visitation probability}. This generalizes the analysis of \PC in \cref{sec:pc-analysis-reachable}, which did not require considering this second outcome.

\item In \cref{sec:slmt-analysis}, we finally prove \cref{theorem:main-pc-trunc} by analyzing the full algorithm \SLMt, which proceeds in several phases. Each phase is analogous to the entire execution of the reachable-case algorithm \SLM, but builds on the progress made in previous phases, using a \emph{backup policy cover} $\Gamma$ consisting of certain policies discovered in previous phases. The key lemma is that there is some phase where none of the newly discovered policies have large extraneous visitation probability (i.e. the second outcome in the previous bullet doesn't happen), and thus the truncated policy cover construction in that phase succeeds.
\end{itemize}

Before proceeding, we expand upon the discussion in \cref{section:trunc-overview} and explain in greater detail why the guarantee of \PC on a truncated emulator has two potential outcomes (unlike in the reachable case), necessitating the multiple phases in \SLMt. At a high level, the challenge originates from the fact that the output vectors $\hat \mu_{h+1}^j$ of \ESCt are not exactly equal to $\mu_{h+1}(x)$ for some $x \in \MX$, and so it will not be the case that $\lng \E\sups{M, \pi}[\phi_h(x_h,a_h)], \hat \mu_{h+1}^j \rng \geq 0$ for all $\pi$. We remark that this non-negativity condition is necessary in various parts of the analysis, including in the inductive step of extending a policy cover. %

  The convex program (\ref{eq:program-trunc}) attempts to ensure that such a non-negativity condition holds by adding in the constraints  \cref{eq:hatmu-nneg-constraint-trunc}. These constraints suffice to guarantee, as per \cref{eq:core-nonneg-empty,eq:core-nonneg-gamma} of \cref{def:trunc-core}, that an approximate non-negativity statement holds for all $\pi$ in the MDPs $\Mbar(\Gamma), \Mbar(\emptyset)$. When analyzing \PSDP, though, such a guarantee is not quite sufficient: it turns out that we will need a stronger statement, which holds for the true MDP $M$.  However, there could be states $x \in \MX$ which are visited with probability up to $\trunc$ in $M$ under some policy $\pi$, but which are not reachable in $\Mbar(\Gamma)$ or $\Mbar(\emptyset)$ (i.e., they are truncated). In particular, conditions such as \cref{eq:core-nonneg-empty,eq:core-nonneg-gamma}, which hold with respect to the MDPs $\Mbar(\Gamma), \Mbar(\emptyset)$, only imply analogous non-negativity conditions for the MDP $M$ with error term growing as $O(\trunc)$, in particular, of the form:
  \begin{align}
\E\sups{M, \pi} \left[ \max \left\{0, \max_{a \in \MA} - \lng \phi_h(x_h, a), \hat \mu^j_h \rng \right\} \right] \lesssim \trunc \cdot \| \hat \mu^j \|_1\label{eq:sigmatrunc-nonneg}.
  \end{align}
  Unfortunately, \cref{eq:sigmatrunc-nonneg} is insufficient for our purposes: we need all error terms in the inductive step to be $\ll \trunc$, as otherwise, we will need to increase the truncation parameter $\trunc$ by at least a constant factor at each step, which will lead to exponential dependence on $H$ in our sample complexity.

  To overcome this issue, we make the following key observation: if the $\hat \mu^j_h$ vectors are chosen so that \cref{eq:sigmatrunc-nonneg} is approximately tight for some policy $\tilde\pi$, then in the course of the calls to \PSDP in \PC, we will actually find such a policy $\tilde\pi$. Moreover, we can show that for such $\tilde\pi$, there must be $g \leq h$ and $x \in \MX$ so that $\tilde\pi$ visits $x$ at step $g$ in $M$, but not in $\Mbar(\Gamma)$ (i.e. $x$ was truncated at step $g$). We can then add $\tilde\pi$ to the backup policy cover $\Gamma$, and in all subsequent phases of the algorithm, the state $x$ will no longer cause such a problem. We can bound the number of policies added to $\Gamma$ over the course of the algorithm using the condition that $\sum_{x \in \MX} \| \mu_h(x) \|_1 \leq \Cnrm$. Thus, after sufficiently many phases, the error terms caused by non-negativity violations as discussed above will be sufficiently small (in at least one of the phases) for our purposes.

\subsection{Convex program}
\label{sec:convex-program-unreachable}

In this section, we introduce $\ESCt$ (\cref{alg:esc-trunc}), which is a modification of the algorithm $\ESC$ for the setting when $M$ may not satisfy reachability. $\ESCt$ proceeds similarly to $\ESC$, taking as input a policy cover $\Psi_h$ for step $h$ as well as a ``backup'' policy cover $\Gamma \subset \Pi$. It then uses the sets $\Psi_h, \Gamma$ to construct datasets $\MC_h = \{ (x_h^i, a_h^i, x_{h+1}^i) \}_{i=1}^n$ and $\MD_h = \{ \tilde x_{h+1}^j \}_{j=1}^m$; note that these datasets are constructed using a mixture of policies from $\Psi_h$ and $\Gamma$, which is different from $\ESC$ (which only uses $\Psi_h$). Then, $\ESCt$ solves a convex program (\ref{eq:program-trunc}) to compute a \emph{truncated} version of an emulator, defined formally in \cref{def:trunc-core} below. We note that the program (\ref{eq:program-trunc}) is slightly different from the analogous program (\ref{eq:program}) used in $\ESC$, in that the non-negativity constraint \cref{eq:hatmu-nneg-constraint-trunc} must hold for all $a \in \MA$, whereas its analogue in $\ESC$, \cref{eq:hatmu-nneg-constraint}, only needs to hold for $a = a_h^i$. This stronger constraint is needed since, in the definition of a truncated emulator, the non-negativity condition (in \cref{eq:core-nonneg-empty,eq:core-nonneg-gamma}) involves a maximum over $a \in \MA$. 

\begin{algorithm}[t]
	\caption{$\ESCt(h, \Psi_h, \Gamma , \epcvx, \Cnrm, n, m)$}
	\label{alg:esc-trunc}
	\begin{algorithmic}[1]\onehalfspacing
		\Require Step $h \in [H]$; policy cover $\Psi_h$ and backup policy cover $\Gamma$; tolerance $\epcvx>0$; norm parameter $\Cnrm \in \RR$; sample complexity $n$; size $m$ of output \coreset{}.
  
  \State $\MC_h, \MD_h \gets \DrawDataTrunc(h,\Psi_h,\Gamma,n,m)$ \label{line:ch-dh} \Comment{\cref{alg:ddtrunc}}
  \LineComment{{We write $\MC_h = \{ (x_h^i, a_h^i, x_{h+1}^i) \}_{i=1}^n$, and $\MD_h = \{ \tilde x_{h+1}^j \}_{j=1}^m$. }}
        \State For each $\ell \in [d]$, solve regression 
        \begin{equation} \hat \bw_{\ell} := \argmin_{w \in \RR^d\times\{0\}: \norm{w}_1 \leq \Cnrm} \sum_{i=1}^n \left(\langle \phi_h(x_h^i, a_h^i), w\rangle - \phiavg_{h+1}(x_{h+1}^i)_\ell\right)^2
        \label{eq:wj-guarantee-trunc}
        \end{equation}
        \State Find $\hat\mu_{h+1}^1,\dots,\hat\mu_{h+1}^m \in \RR^d\times\{0\}$ satisfying the following convex program \cref{eq:program-trunc}:\footnotemark%
        
        \begin{subequations}
        \label{eq:program-trunc}
        \begin{align}
        \frac{1}{n} \sum_{i=1}^n \left(\langle \phi_h(x_h^i,a_h^i), \hat \bw_\ell\rangle - \sum_{j=1}^m \left\langle \phi_h(x_h^i,a_h^i), \hat\mu_{h+1}^j \right \rangle \phiavg_{h+1}(\tilde x_{h+1}^j)_\ell \right)^2 &\leq \epcvx^2 &\forall \ell \in [d] \label{eq:what-wprime-constraint-trunc}\\
        \sum_{j=1}^m \norm{\hat\mu_{h+1}^j}_1 &\leq \Cnrm & \label{eq:hatmu-cnrm-constraint-trunc}\\
        \langle \phi_h(x_h^i,a), \hat\mu_{h+1}^j\rangle &\geq 0 &\forall i \in [n], j \in [m], a \in \MA \label{eq:hatmu-nneg-constraint-trunc}
        \end{align}
        \end{subequations}
		\State \textbf{Return:} $(\hat\mu_{h+1}^j)_{j=1}^m$ if \cref{eq:program-trunc} is feasible, otherwise $\perp$
	\end{algorithmic}
\end{algorithm}

\cref{thm:mu-coreset-trunc}, whose proof is the main objective of this section, gives our guarantee for \\  $\ESCt$: it states that with high probability, $\ESCt$ will output a \emph{truncated emulator}, as defined below:
\begin{definition}\label{def:trunc-core}
    Fix $\Gamma \subset \Pi$ and $h \in [H]$. For any $m \in \NN$ and $\epapx, \epneg, C > 0$, a set of vectors $(\hat\mu^j)_{j=1}^m \subset \RR^d \times \{0\}$ is a \emph{$(\epapx, \epneg, C; \Gamma )$-\trunccore} for the MDP at step $h$ if the following conditions hold:
\begin{enumerate}
    \item\label{it:norm-bound} $\sum_{j=1}^m \norm{\hat\mu^j}_1 \leq C$
    \item \label{it:approx-nonneg} For any policy $\pi \in \Pi$ and $j \in [m]$, %
      \begin{align} \EE\sups{\Mbar(\emptyset), \pi}\left[ \max \left\{ 0, \max_{a \in \MA}- \langle \phi_h(x_h,a), \hat\mu^j\rangle \right\} \right] &\leq   \epneg\norm{\hat\mu^j}_1, \label{eq:core-nonneg-empty}\\
         \E\sups{\Mbar(\Gamma), \pi} \left[\max\left\{0,\max_{a \in \MA}-\lng\phi_h(x_h, a), \hat \mu^j \rng\right\}\right] &\leq   \epneg \norm{\hat\mu^j}_1.\label{eq:core-nonneg-gamma}
        \end{align}
    \item \label{it:pol-approx} There are states $(\tilde x^j)_{j=1}^m \subseteq \MX$ so that for any policy $\pi \in \Pi$,%
      \begin{align}
      \max_{\ell \in [d]}  \abs{\sum_{x \in \MX} \langle \EE\sups{\Mbar(\Gamma), \pi}[\phi_h(x_h,a_h)], \mu_{h+1}(x)\rangle \cdot \phiavg_{h+1}(x) - \sum_{j = 1}^m \langle \EE\sups{\Mbar(\Gamma), \pi}[\phi_h(x_h,a_h)], \hat\mu^j\rangle \cdot \phiavg_{h+1}(\tilde{x}^j)} &\leq  \epapx,\nonumber\\
                \max_{\ell \in [d]}\abs{\sum_{x \in \MX} \langle \EE\sups{\Mbar(\emptyset), \pi}[\phi_h(x_h,a_h)], \mu_{h+1}(x)\rangle \cdot \phiavg_{h+1}(x) - \sum_{j = 1}^m \langle \EE\sups{\Mbar(\emptyset), \pi}[\phi_h(x_h,a_h)], \hat\mu^j\rangle \cdot \phiavg_{h+1}(\tilde{x}^j)} &\leq  \epapx\nonumber.
      \end{align}
\end{enumerate}
\end{definition}

\footnotetext{As in the reachable case (\cref{alg:esc}), we assume for now that we can efficiently find an exact solution to the stated convex program. In \cref{sec:opt-details} we argue that an approximate solution suffices (and moreover that it can efficiently found via the ellipsoid algorithm).}

\cref{def:trunc-core} is similar to its non-truncated analogue, \cref{def:not-core-set}, with the difference that the conditions in  \cref{it:approx-nonneg,it:pol-approx} are required to hold in the two MDPs $\Mbar(\emptyset), \Mbar(\Gamma)$ (as opposed to in $M$). Moreover, the non-negativity constraint in \cref{it:approx-nonneg} is slightly stronger in the truncated case, due to the maximum with 0 inside the expectation in \cref{eq:core-nonneg-empty,eq:core-nonneg-gamma}; this is needed in order to bound the error of \PSDP with one of the vectors $\hat\mu^j$ as a target reward, when the algorithm is only given truncated policy covers (see \cref{sec:psdp-unreachable} and in particular \cref{def:nnnt}).

The remainder of the section is analyzed as follows. In \cref{section:cov-guarantee}, we prove the key coverage guarantee satisfied by the datasets $\MC_h, \MD_h$ under the induction hypothesis that $\Psi_{1:h}$ are $\alpha$-truncated covers for steps $1,\dots,h$. In \cref{sec:feasibility-trunc}, we show that the convex program in $\ESCt$ is feasible with high probability. In \cref{sec:approx-nneg-trunc}, we show how the the non-negativity constraint \cref{eq:hatmu-nneg-constraint-trunc} ensures that the output of $\ESCt$ will satisfy \cref{it:approx-nonneg} of \cref{def:trunc-core}. In \cref{sec:all-policy-trunc}, we show how the constraint \cref{eq:what-wprime-constraint-trunc} of $\ESCt$ will satisfy \cref{it:pol-approx} of \cref{def:trunc-core}. In \cref{sec:emulator-guarantee-trunc}, we put these pieces together and prove \cref{thm:mu-coreset-trunc}.

\subsubsection{Coverage guarantee}\label{section:cov-guarantee}

\begin{algorithm}[t]
    \caption{$\DrawDataTrunc(h,\Psi_h,\Gamma,n,m)$}
    \label{alg:ddtrunc}
    \begin{algorithmic}[1]\onehalfspacing
            \State $\MC_h, \MD_h \gets \emptyset$
            \For{$1 \leq i \leq n$}
                \State Sample $(x_1^i,a_1^i,\dots,x_h^i,a_h^i ,x_{h+1}^i) \sim \frac{1}{2}(\unif(\Psi_h) \circ_h \unif(\MA) + \unif(\Gamma) \circ_h \unif(\MA))$\label{line:define-ch-trunc}
                \State Update dataset: $\MC_h \gets \MC_h \cup \{(x_h^i,a_h^i,x_{h+1}^i)\}$
            \EndFor
            \For{$1 \leq j \leq m$}
                \State Sample $(\tilde{x}_1^j,\tilde{a}_1^j,\dots,\tilde x_h^j, \tilde a_h^j, \tilde x_{h+1}^j) \sim \frac{1}{2}(\unif(\Psi_h) \circ_h \unif(\MA) + \unif(\Gamma) \circ_h \unif(\MA))$
                \State Update dataset: $\MD_h \gets \MD_h \cup \{\tilde x_{h+1}^j\}$
            \EndFor
            \State \textbf{Return:} $\MC_h$, $\MD_h$
    \end{algorithmic}
\end{algorithm}

Suppose we are given $\Psi_h \subset \Pi$ so that $\Psi_h$ satisfies \cref{eq:pc-trunc}, as well as a subset $\Gamma \subset \Pi$. %
In \cref{lem:nu-beta-indhyp} below, we define distributions $\beta_{h+1} \in \Delta(\MX)$, $\iota_h \in \Delta(\MX \times \MA \times \MX)$ to be the distributions of $\MC_h, \MD_h$ as produced by \DrawDataTrunc (\cref{alg:ddtrunc}). Explicitly, $\iota_h$ is the distribution of $(x_h, a_h, x_{h+1})$ in the MDP $M$ under a policy $\pi \sim \frac 12 \cdot \left( \unif(\Gamma) \circ_h \unif(\MA) + \unif(\Psi_h) \circ_h \unif(\MA) \right)$, and $\beta_{h+1}$ is the marginal distribution of $x_{h+1}$ under $\nu_h$. %
Moreover, we recall the parameters $\tsmall, \trunc > 0$ which are used to define $\Mbar(\Gamma)$. 

  \begin{lemma}
    \label{lem:nu-beta-indhyp}
    Consider any odd $h \in [H]$ and finite subsets $\Gamma,\Psi_h \subset \Pi$. Suppose that $\Psi_h$ satisfies \cref{eq:pc-trunc} for some $\alpha \geq \frac{\tsmall}{\trunc}$. %
    For any $n,m \in \NN$, let $\MC_h,\MD_h$ be the outputs of $\DrawDataTrunc(h,\Psi_h,\Gamma,n,m)$ (\cref{alg:ddtrunc}). Let $\iota_h \in \Delta(\MX\times\MA\times\MX)$ be the distribution of the i.i.d. samples $(x_h^i,a_h^i,x_{h+1}^i) \in \MC_h$, and let $\beta_{h+1} \in \Delta(\MX)$ be the distribution of the i.i.d. samples $\tilde{x}_{h+1}^j \in \MD_h$. Let $\nu_h(x,a)$ denote the marginal distribution of $(x,a)$ under $(x,a,x') \sim \iota_h$, and $\nu_h(x)$ denote the marginal distribution of $x$. Then the following inequalities hold, for all $x \in \MX, a \in \MA$:
\begin{align}
  \nu_h(x,a) \geq \frac{\alpha}{2A} \cdot \max_{\pi \in \Pi} d_h\sups{\Mbar(\emptyset), \pi}(x,a), \qquad  \beta_{h+1}(x) \geq \frac{\alpha}{2A} \cdot \max_{\pi \in \Pi} d_{h+1}\sups{\Mbar(\emptyset),\pi}(x)\label{eq:nu-beta-indhyp-trunc}
\end{align}
\begin{gather}
\nu_h(x) \geq \frac{\tsmall}{2} \cdot \max_{\pi \in \Pi} d_h\sups{\Mbar(\Gamma), \pi}(x), \quad \nu_h(x,a) \geq \frac{\tsmall}{2A} \cdot \max_{\pi \in \Pi} d_h\sups{\Mbar(\Gamma), \pi}(x,a),\nonumber\\ 
\beta_{h+1}(x) \geq \frac{\tsmall}{2A} \cdot \max_{\pi \in \Pi} d_{h+1}\sups{\Mbar(\Gamma). \pi}(x)\label{eq:nu-beta-mtil-trunc}
\end{gather}
\end{lemma}
\begin{proof}
  We treat the case $h=1$ separately. In this case $\nu_1(x) = \BP_1(x) = \max_{\pi \in \Pi} d\sups{M,\pi}_1(x)$. Also $\nu_1(x,a) = \frac{1}{A}\BP_1(x) \geq \frac{1}{A}\max_{\pi\in\Pi} d\sups{M,\pi}_1(x,a)$. Similarly $\beta_2(x) \geq \frac{1}{A}\max_{\pi\in\Pi} d\sups{M,\pi}_2(x)$. Then \cref{eq:nu-beta-indhyp-trunc} and \cref{eq:nu-beta-mtil-trunc} both follow from applications of \cref{lem:gamma-monotonicity} and the fact that $\alpha \geq \tsmall$. 
  
  From now on suppose that $h \geq 2$, so that \cref{lem:reachability-in-trunc} applies. We begin with the proof of \cref{eq:nu-beta-indhyp-trunc}: from \cref{eq:pc-trunc} and the fact that for $(x_h,a_h) \sim \nu_h$, the action $a_h$ is chosen uniformly in $\MA$, we have, for all $x \in \MX, a \in \MA$,
  \begin{align}
\nu_h(x,a) \geq \frac{1}{2|\Psi_h|}\sum_{\pi' \in \Psi_h} d_h\sups{M, \pi'}(x,a) \geq \frac{1}{2A|\Psi_h|} \sum_{\pi'\in\Psi_h} d_h\sups{M, \pi'}(x) \stackrel{\cref{eq:pc-trunc}}{\geq} \frac{\alpha}{2A} \cdot \max_{\pi \in \Pi} d_h\sups{\Mbar(\emptyset), \pi}(x) \geq \frac{\alpha}{2A} \max_{\pi \in \Pi} d_h\sups{\Mbar(\emptyset), \pi}(x,a)\nonumber.
  \end{align}
  Since the distribution of $x' \sim \beta_{h+1}$ is the distribution of $x' \sim \BP_h\sups{M}(\cdot | x,a)$ for $(x,a) \sim \nu_h$ and since $\BP_h\sups{\Mbar(\emptyset)}(x' | x,a) \leq \BP_h\sups{M}(x' | x,a)$ for all $x,x'\in\MX$ and $a\in\MA$, it follows that for all $x \in \MX$, $\beta_{h+1}(x) \geq \frac{\alpha}{2A} \max_{\pi \in \Pi} d_{h+1}\sups{\Mbar(\emptyset), \pi}(x)$.

  Next we prove \cref{eq:nu-beta-mtil-trunc}. Fix any $x \in \MX,\ a \in \MA$.
Using the assumption that $h$ is odd, we will apply \cref{lem:reachability-in-trunc} to the set $\Gamma$, which yields that either \cref{eq:reachability-in-trunc} or \cref{eq:gamma-good} holds. In the event that \cref{eq:reachability-in-trunc} holds, we have
  \begin{align}
    \nu_h(x,a) = \frac{1}{A} \nu_h(x) \stackrel{\cref{eq:pc-trunc}}{\geq} & \frac{\alpha}{2A} \cdot \max_{\pi \in \Pi} d_h\sups{\Mbar(\emptyset), \pi}(x) \stackrel{\cref{eq:reachability-in-trunc}}{\geq}  \frac{\alpha \trunc}{2A} \cdot \| \mu_h(x) \|_1\nonumber\\
    \geq & \frac{\alpha \trunc}{2A} \max_{\pi \in \Pi} d_h\sups{\Mbar(\Gamma), \pi}(x) \geq \frac{\alpha \trunc}{2A} \max_{\pi \in \Pi} d_h\sups{\Mbar(\Gamma), \pi}(x,a)\nonumber.
  \end{align}
  In the event that \cref{eq:gamma-good} holds, we have
  \begin{align}
    \nu_h(x,a) = \frac{1}{A} \nu_h(x) \geq \frac{1}{2|\Gamma|A} \sum_{\pi' \in \Gamma} d_h\sups{M, \pi'}(x) &\stackrel{\cref{eq:gamma-good}}{\geq} \frac{\tsmall}{2A} \cdot \| \mu_h(x) \|_1\nonumber\\ &\geq  \frac{\tsmall}{2A} \cdot \max_{\pi \in \Pi} d_h\sups{\Mbar(\Gamma), \pi}(x) \geq \frac{\tsmall}{2A} \cdot \max_{\pi \in \Pi} d_h\sups{\Mbar(\Gamma), \pi}(x,a)\nonumber.
  \end{align}
  In either case, since ${\alpha\trunc} \geq \tsmall$, we have shown that the first and second inequalities in \cref{eq:nu-beta-mtil-trunc} hold. The third inequality holds since $\BP_h\sups{\Mbar(\Gamma)}(x' | x,a) \leq \BP_h\sups{M}(x' | x,a)$ for all $x,x'\in\MX$ and $a\in\MA$. 
\end{proof}

\subsubsection{Feasibility of the convex program}
\label{sec:feasibility-trunc}
\cref{lem:randomize-trunc} below uses an importance sampling argument to ensure that the datasets $\MC_h, \MD_h$ in \cref{alg:esc-trunc} have certain properties which will ensure that (\ref{eq:program-trunc}) is feasible. Its proof is similar to its counterpart \cref{lem:randomize-no-trunc} in the reachable setting, but we can no longer rely on \cref{eq:beta-rch-lb} due to the lack of reachability. Instead, we aim for a weaker guarantee: namely, \cref{it:gapprox-trunc} of \cref{lem:randomize-trunc} only considers state-action pairs drawn from a certain roll-in distribution, for which hard-to-reach states cannot significantly contribute. We remark that there is a cost to not relying on reachability: \cref{lem:randomize-trunc} gets worse dependence on $m$ than \cref{lem:randomize-no-trunc} (namely, $m^{-1/4}$ as opposed to $m^{-1/2}$).
\begin{lemma}
  \label{lem:randomize-trunc}
  Let $\delta \in (0,1)$ and $m \in \NN$. Let $\MG$ be a non-empty set of functions $g : \MX \ra [-1,1]$. Suppose that $\delta \leq \sqrt{\frac{2\log(4|\MG|/\delta^2)}{m}}$.
  Fix $h \in [H]$, and a distribution $\nu_{h} \in \Delta(\MX\times\MA)$. Let $\beta_{h+1}$ denote the distribution of $x' \sim \BP_h(\cdot | x,a)$ for $(x,a) \sim \nu_h$.  Let $\MD_h = \{ \tilde x_{h+1}^j \}_{j=1}^m$ consist of i.i.d. draws $\tilde x_{h+1}^j \sim \beta_{h+1}$.
  Then there are vectors $\psi_{h+1}^1, \ldots, \psi_{h+1}^m \in \BR^d\times\{0\}$, depending on the dataset $\MD_h$, with the following property. %
  With probability at least $1-\delta$ over the draw of the dataset $\MD_h$ (and the ensuing $\psi_{h+1}^j$), we have that:
  \begin{enumerate}
  \item \label{it:phimu-nonneg-trunc} For all $i \in [m]$ and $(x,a) \in \MX \times \MA$,  $\lng \phi_h(x,a), \psi_{h+1}^i \rng \geq 0$. %
  \item \label{it:gapprox-trunc} It holds that
    \begin{align}
\E_{(x, a) \sim \nu_h} \left[ \sup_{g \in \MG} \left| \E_{x' \sim \BP_h(x, a)} [g(x')] - \frac{1}{m} \sum_{i=1}^m \lng \phi_h(x, a), \psi_{h+1}^i \rng \cdot g(\tilde x_{h+1}^i) \right|\right] \leq  \frac{9\Cnrm^{1/2} \log^{1/4}(4|\MG|/\delta^2)}{m^{1/4}}\nonumber.
    \end{align}
  \item \label{it:muhat-norm-trunc} It holds that $\frac{1}{m} \sum_{j=1}^m \| \psi_{h+1}^j \|_1 \leq \Cnrm$. 
  \end{enumerate}
\end{lemma}
\begin{proof}
  Fix some $R > 0$, to be specified below, and set $\epsilon := 4R \sqrt{\frac{2\log(4|\MG|/\delta^2)}{m}}$. %
  Let us define $\ul \MX \subseteq \MX$ by
  \begin{align}
    \ul \MX :=&  \left\{ x' \in \MX \ : \ \frac{\| \mu_{h+1}(x') \|_1}{\beta_{h+1}(x')} \leq R \right\}.\nonumber
  \end{align}
  For all $x,x' \in \MX$ and $a \in \MA$, define
  \begin{align}
    \ul \BP_h(x'|x,a) := \BP_h(x' | x,a) \cdot \One{x' \in \ul \MX}.\nonumber
  \end{align}
  Note that $\ul \BP_h(\cdot | x,a)$ is a sub-distribution supported on $\ul\MX$, for any $x \in \MX$ and $a \in \MA$. We will write $\E_{x' \sim \ul\BP_h(x,a)}[g(x')]$ to denote $\sum_{x' \in \MX} \ul\BP_h(x'|x,a) \cdot g(x')$.

  For each $j \in [m]$, we define the (random) vector $\psi_{h+1}^j := \One{\tilde x_{h+1}^j \in \ul \MX} \cdot (1-\ep/2)\cdot  \frac{\mu_{h+1}(\tilde x_{h+1}^j)}{\beta_{h+1}(\tilde x_{h+1}^j)}$. Note that $\tilde x_{h+1}^j \in \MX$ so $\mu_{h+1}(\tilde x_{h+1}^j) \in \RR^d\times\{0\}$ (as defined in \cref{section:extended-overview}).
  
  \paragraph{Proof of \cref{it:phimu-nonneg-trunc}.} Since $\psi_{h+1}^j$ is proportional to $\mu_{h+1}(\tilde x_{h+1}^j)$ (with non-negative constant of proportionality) for each $j \in [m]$, the first claimed statement holds (with probability 1).

\paragraph{Proof of \cref{it:gapprox-trunc}.} For any $(x,a) \in \MX\times \MA$, $j \in [m]$ and $g \in \MG$, we have with probability $1$ that
  \begin{align}
    \left|\lng \phi_h(x,a), \psi_{h+1}^j \rng \cdot g(\tilde x_{h+1}^j) \right|
    &\leq \One{\tilde x_{h+1}^j \in \ul \MX} \cdot  \left\lng \phi_h(x,a), \frac{\mu_{h+1}(\tilde x_{h+1}^j)}{\beta_{h+1}(\tilde x_{h+1}^j)} \right\rng\nonumber\\
    &\leq \One{\tilde x_{h+1}^j \in \ul \MX} \cdot \frac{\norm{\mu_{h+1}(\tilde x_{h+1}^j)}_1}{\beta_{h+1}(\tilde x_{h+1}^j)} \nonumber\\
    &\leq R\label{eq:g-as-bound},
  \end{align}
  by the bounds $\| g \|_\infty \leq 1$, $\| \phi_h(x,a) \|_\infty \leq 1$, and the definition of $\ul \MX$. Also, for any $g \in \MG$ we have that in expectation over the draw of $\tilde x_{h+1}^j \sim \beta_{h+1}$ (which determines $\psi_{h+1}^j$),
  \begin{align}
    \E \left[ \lng \phi_h(x,a), \psi_{h+1}^j \rng \cdot g(\tilde x_{h+1}^j) \right] &= (1 - \ep/2) \cdot \sum_{x' \in \MX} \One{x' \in \ul \MX} \cdot \beta_{h+1}(x') \cdot \left \lng \phi_h(x,a), \frac{\mu_{h+1}(x')}{\beta_{h+1}(x')} \right\rng \cdot g(x') \nonumber\\
    &= (1-\ep/2) \cdot \E_{x' \sim \ul\BP_h(x,a)} [ g(x')]\nonumber.
  \end{align}
  By Hoeffding's inequality and a union bound, it follows that for any fixed $(x,a) \in \MX\times\MA$, with probability at least $1-\delta^2/4$ over the draw of $\MD_{h}$, for all $g \in \MG$,%
  \begin{align}
\left| (1-\ep/2) \cdot \E_{x' \sim \ul\BP_h(x,a)}[g(x')] - \frac{1}{m} \sum_{j=1}^m \lng \phi_h(x,a), \psi_{h+1}^j \rng \cdot g(\tilde x_{h+1}^j) \right| \leq R\cdot \sqrt{ \frac{2\log (4|\MG|/\delta^2)}{m}} \leq \frac{\ep}{4}\label{eq:chernoff-mu-g-trunc}.
  \end{align}
  by choice of $\epsilon$. %

Next, we have to bound the distance between $\BP_h(\cdot | x,  a)$ and $\ul \BP_h(\cdot |  x,  a)$, for $(x, a) \sim \nu_h$. To do so, define $\Delta_h: \MX\times\MA \to [0,1]$ by
  \begin{align}
\Delta_h(x,a) := \BP_h (\MX\setminus \ul \MX | x,a).\nonumber
  \end{align}
  Then for any $(x,a)\in\MX\times\MA$ and $g \in \MG$, we have that
  \begin{align}
\left| \E_{x' \sim \BP_h(x,a)} [g(x')] - \E_{x' \sim \ul \BP_h(x,a)}[g(x')]\right| = \left|\sum_{x'\in\MX} \One{x' \in \MX\setminus\ul\MX} \cdot \BP_h(x'|x,a) \cdot g(x')\right| \leq \Delta_h(x,a)\nonumber.
  \end{align}
  Thus, for any $(x,a)$, in the event that \cref{eq:chernoff-mu-g-trunc} holds for the given $(x,a)$, it follows that
  \begin{align}
    & \left| \E_{x' \sim \BP_h(x,a)}[g(x')] - \frac 1m \sum_{j=1}^m \lng \phi_h(x,a), \psi_{h+1}^j \rng \cdot g(\tilde x_{h+1}^j) \right|\nonumber\\
    &\leq \Delta_h(x,a) + \frac{\ep}{2} + \frac{\ep}{4} < \Delta_h(x,a) + \ep\label{eq:bound-by-deltah}.
  \end{align}
  Moreover, using that $\beta_{h+1}(x') = \sum_{(x,a) \in \MX \times \MA} \nu_h(x,a) \cdot \BP_h(x' | x,a)$, 
  \begin{align}
    \E_{( x,  a) \sim \nu_h}[\Delta_h( x,  a)] =&  \sum_{( x, a) \in \MX \times \MA} \nu_h( x, a) \sum_{x' \in \MX} \BP_h(x' |  x, a) \cdot \One{ \frac{ \| \mu_{h+1}(x')\|_1}{\beta_{h+1}(x')} > R } \nonumber\\
    =& \sum_{x' \in \MX} \beta_{h+1}(x') \cdot \One{ \frac{ \| \mu_{h+1}(x')\|_1}{\beta_{h+1}(x')} > R } \leq \frac{\Cnrm}{R}\label{eq:exn-deltah},
  \end{align}
  where the final inequality is by Markov's inequality and the fact that $\E_{x' \sim \beta_{h+1}}[\frac{\| \mu_{h+1}(x') \|_1}{\beta_{h+1}(x')}] \leq \Cnrm$. 
  
  For each $(x,a) \in \MX \times \MA$, let  $\BI_{x,a} \in \{0,1\}$ be equal to 1 if \cref{eq:chernoff-mu-g-trunc} does not hold, and 0 otherwise. For any fixed instantiation of $\MD_h$, we may now compute
  \begin{align}
    &  \E_{( x,  a) \sim \nu_h} \left[ \sup_{g \in \MG} \left| \E_{x' \sim \BP_h(\cdot |  x,  a)} [g(x')]-\frac 1m \sum_{i=1}^m \lng \phi_h( x,  a), \psi_{h+1}^i \rng \cdot g(\til x_{h+1}^i) \right| \right]\nonumber\\
    &\leq  \E_{( x,  a) \sim \nu_h} \left[ \Delta_h( x,  a) + \ep + (R+1)\BI_{ x,  a} \right]\nonumber\\
    &\leq  \frac{\Cnrm}{R} + \ep + (R+1)\E_{( x,  a) \sim \nu_h}[\BI_{ x,  a}]\nonumber,
  \end{align}
  where the first inequality uses \cref{eq:bound-by-deltah} for $(x,a)$ such that $\BI_{x,a}=0$, and otherwise uses \cref{eq:g-as-bound}, the bound $\| g \|_\infty\leq 1$, and the triangle inequality; and the second inequality uses \cref{eq:exn-deltah}.
  For any fixed $(x,a)$, we have previously shown that $\E_{\MD_h}[\BI_{x,a}] \leq \delta^2/4$, meaning that $\E_{\MD_h} \E_{( x,  a) \sim \nu_h}[\BI_{ x,  a}] \leq \delta^2/4$. In particular, with probability at least $1-\delta/2$ over the draw of $\MD_h$, we have $\E_{( x,  a) \sim \nu_h}[\BI_{ x, a}] \leq \delta/2$. Finally, by choosing $R = \Cnrm^{1/2} m^{1/4} \log^{-1/4}(4|\MG|/\delta^2)$, we have 
\begin{align}
\frac{\Cnrm}{R} + \ep + (R+1) \delta/2 
&\leq \frac{\Cnrm}{R} + 4R \sqrt{\frac{2\log(4|\MG|/\delta^2)}{m}} + R\delta \nonumber \\
&\leq \frac{\Cnrm}{R} + 5R \sqrt{\frac{2\log(4|\MG|/\delta^2)}{m}} \nonumber \\
&\leq \frac{9\Cnrm^{1/2} \log^{1/4}(4|\MG|/\delta^2)}{m^{1/4}}\nonumber
\end{align}
where the first two inequalities use that $1 \leq R$ and $\delta \leq \sqrt{\frac{2\log(4|\MG|/\delta^2)}{m}}$ respectively. Thus, on an event that occurs with probability at least $1-\delta/2$, we have established that the second claimed statement of the lemma holds. 

\paragraph{Proof of \cref{it:muhat-norm-trunc}.}  For each $j \in [m]$, we compute that in expectation over the draw of $\tilde x_{h+1}^j \sim \beta_{h+1}$,
  \begin{align}
\E \left[ \| \psi_{h+1}^j \|_1 \right] = & (1-\ep/2) \cdot \sum_{x' \in \MX}\beta_{h+1}(x')  \cdot \frac{\norm{\mu_{h+1}(x')}_1}{\beta_{h+1}(x')} \cdot \One{x' \in \bar \MX} \leq (1-\ep/2) \cdot \Cnrm\nonumber.
  \end{align}
  Additionally, $\| \psi_{h+1}^j \|_1 \leq R$ for all $j \in [m]$, so Hoeffding's inequality gives that with probability at least $1-\delta/2$ over the draw of $\MD_{h+1}$,
  \begin{align}
\frac{1}{m} \sum_{j=1}^m \| \psi_{h+1}^j \|_1 \leq (1-\ep/2) \cdot \Cnrm + R \cdot \sqrt{\frac{2\log 2/\delta}{m}} \leq \Cnrm\nonumber,
  \end{align}
  by choice of $\epsilon$ and the fact that $\Cnrm \geq 1$.

  By a union bound, all 3 claimed statements of the lemma hold on an event that occurs with probability at least $1-\delta$. 
\end{proof}

Using \cref{lem:randomize-trunc}, \cref{lem:feasibility-trunc} below establishes that the program (\ref{eq:program-trunc}) in $\ESCt$ (\cref{alg:esc-trunc}) is feasible with high probability over the draws of $\MC_h, \MD_h$. 
\begin{lemma}\label{lem:feasibility-trunc}
There is a universal constant $C_{\ref{lem:feasibility-trunc}}$ so that the following holds. Let $\epcvx, \delta,\alpha > 0$ and $n,m \in \NN$. Fix $h \in [H-1]$ and suppose that $\Psi_h, \Gamma \subset \Pi$ are given. Suppose that the following bounds hold:  \[n \geq C_{\ref{lem:feasibility-trunc}}\epcvx^{-4}\Cnrm^4 \log(d/\delta),\]  \[m \geq C_{\ref{lem:feasibility-trunc}}\epcvx^{-8}\Cnrm^6\log(d/\delta).\] 
\[\delta \leq \sqrt{\frac{2\log(4d/\delta^2)}{m}}.\]
Then with probability at least $1-\delta$, the algorithm $\ESCt(h, \Psi_h, \Gamma, \epcvx,\Cnrm,n,m)$ (\cref{alg:esc-trunc}) produces a feasible solution $(\hat\mu_{h+1}^j)_{j=1}^m$ to Program~\cref{eq:program-trunc}.
\end{lemma}

\begin{proof}
  We need to show that Program~\cref{eq:program-trunc} is feasible with high probability over the draws of $\MC_h = \{ (x_h^i, a_h^i, x_{h+1}^i) \}_{i=1}^n$ and $\MD_h = \{ \tilde x_{h+1}^j \}_{j=1}^m$. 

First, we invoke \cref{lem:randomize-trunc} with failure probability $\delta/3$, function class $\MG = \{ x' \mapsto \phiavg_{h+1}(x')_\ell \ : \ \ell \in [d]\}$, and the dataset $\MD_h$ constructed in \cref{alg:esc-trunc}. By construction, the elements $\tilde x_{h+1}^j$ of $\MD_h$ are i.i.d. draws from the following process: sample a policy $\pi' \sim \frac 12 \cdot (\unif(\Psi_h)\circ_h \unif(\MA) + \unif(\Gamma)\left.\circ_h\unif(\MA))\right.$, then sample a trajectory $(\tilde x_1,\tilde a_1, \dots,\tilde x_h, \tilde a_h) \sim \pi'$, and finally a state $\tilde x_{h+1} \sim \BP_h(\cdot|\tilde x_h,\tilde a_h)$. We henceforth refer to the resulting distribution of $(\tilde x_h, \tilde a_h)$ as $\nu_h$ and to the resulting distribution of $\tilde x_{h+1}$ as $\beta_{h+1}$. Thus, we may apply \cref{lem:randomize-trunc} with these choices of of $\nu_h$ and $\beta_{h+1}$. In particular, the dataset $\MD_h = (\tilde x_{h+1}^j)_{j=1}^m$ constructed in \cref{alg:esc-trunc} is drawn exactly as required by \cref{lem:randomize-trunc}. 

Now, \cref{lem:randomize-trunc} gives that, under some event $\ME_1$ that occurs with probability at least $1-\delta/3$ over the draw of $\MD_h$ (regardless of the choice of $\MC_h$), there are vectors  $\psi_{h+1}^1,\dots,\psi_{h+1}^m \in \RR^d\times\{0\}$ so that %
\begin{align}
  \lng \phi_h(x, a), \psi_{h+1}^j \rng  \geq  0 \qquad &   \forall x \in \MX, a \in \MA\nonumber\\
  \frac{1}{m}\sum_{j=1}^m \| \psi_{h+1}^j \|_1  \leq \Cnrm &   \nonumber\\
  \E_{(\bar x, \bar a) \sim \nu_h} \left[ \max_{\ell \in [d]} \left| \E_{x' \sim \BP_h(\bar x, \bar a)}[ \phiavg_{h+1}(x')_\ell] - \frac{1}{m}\sum_{j=1}^m \lng \phi_h(\bar x, \bar a), \psi_{h+1}^j \rng \cdot  \phiavg_{h+1}(\tilde x_{h+1}^j)_\ell \right|\right] & \nonumber\\
   \leq \frac{9\Cnrm^{1/2} \log^{1/4}(36d/\delta^2)}{m^{1/4}}\qquad &  \forall i \in [n], \ell \in [d]\label{eq:p-hatmu-constraint-trunc}.
\end{align}

Define $\hat\mu_{h+1}^j := \psi_{h+1}^j / m$ for each $j \in [m]$. We claim that this choice of $(\hat\mu_{h+1}^j)_{j=1}^m$ satisfies Program~\cref{eq:program-trunc} with high probability. Indeed, under the event $\ME_1$, it is immediate that \cref{eq:hatmu-cnrm-constraint-trunc} and \cref{eq:hatmu-nneg-constraint-trunc} are satisfied. It remains to show that \cref{eq:what-wprime-constraint-trunc} is satisfied with high probability.
  
  Fix $\ell \in [d]$, and define $w_\ell^\st := \sum_{x \in \MX} \mu_{h+1}(x) \cdot \phiavg_{h+1}(x)_\ell$, so that $\E_{x' \sim \BP_h(x,a)}[\phiavg_{h+1}(x')_\ell] = \lng w_\ell^\st, \phi_h(x,a) \rng$ for all $x \in \MX$, $a \in \MA$. Note that $\| w_\ell^\st \|_1 \leq \Cnrm$. Thus, we can apply \cref{lem:fixed-design-prediction-error} with covariates $X_i := \phi_h(x_h^i, a_h^i)$ (for $i \in [n]$), ground truth $w_\ell^\st$, and responses $y_i := \phiavg_{h+1}(x_{h+1}^i)_\ell$. By definition of $\hat w_\ell$ in \cref{eq:wj-guarantee-trunc}, the lemma gives that for some constant $C > 0$, for any fixed $(x_h^i, a_h^i)_{i=1}^n$, with probability at least $1-\delta/(3d)$ over the conditional draws $x_{h+1}^i \sim \BP_h(\cdot | x_h^i, a_h^i)$, %
  \begin{align}
\frac 1n \sum_{i=1}^n  \left\lng \phi_h(x_h^i,a_h^i), \hat w_\ell - w_\ell^\st \right\rng^2 \leq & C \cdot \Cnrm \cdot \sqrt{\frac{\log(3d^2/\delta)}{n}}\label{eq:wprime-wstar}. 
  \end{align}
  By the union bound, this inequality holds for all $\ell \in [d]$ with probability at least $1-\delta/3$; we let this event be denoted by $\ME_2$. %
  
  Next, fix any choice of $\MD_h$ (which determines $(\hat\mu_{h+1}^j)_{j=1}^m$) for which $\ME_1$ holds. Since the draw of $(x_h^i, a_h^i)_{i=1}^n$ from $\nu_h$ is independent of the choice of $\MD_h$, Hoeffding's inequality gives that with probability at least $1-\delta/3$ over the choice of $\MC_h$,
  \begin{align}
    & \frac 1n \sum_{i=1}^n \max_{\ell \in [d]} \left| \left\lng \phi_h(x_h^i, a_h^i), w_\ell^\st -  \sum_{j=1}^m \phiavg_{h+1}(\tilde x_{h+1}^j)_\ell \cdot \hat \mu_{h+1}^j \right\rng \right| \nonumber\\
    &\leq \E_{(\bar x, \bar a) \sim \nu_h} \left[ \max_{\ell \in [d]} \left| \left\lng \phi_h(\bar x, \bar a), w_\ell^\st -  \sum_{j=1}^m \phiavg_{h+1}(\tilde x_{h+1}^j)_\ell \cdot \hat \mu_{h+1}^j \right\rng \right| \right] + 4\Cnrm \sqrt{\frac{\log 3/\delta}{n}}\nonumber\\
    &\leq   \frac{9\Cnrm^{1/2} \log(36d/\delta^2)^{1/4}}{m^{1/4}} +  4\Cnrm \sqrt{\frac{\log 3/\delta}{n}}\label{eq:hoeffding-nu-n},
  \end{align}
  where the final inequality uses that \cref{eq:p-hatmu-constraint-trunc} holds under $\ME_1$. By integrating over $\MD_h$, there is some event $\ME_3$ that occurs with probability at least $1-\delta/3$, so that under the event $\ME_1 \cap \ME_3$, the above bound holds. %

  By Jensen's inequality (to take the maximum over $\ell \in [d]$ outside the average over $i \in [n]$) and the fact that $\left| \lng \phi_h(x_h^i, a_h^i), w_\ell^\st - \sum_{j=1}^m \phiavg_{h+1}(\tilde x_{h+1}^j)_\ell \cdot \hat \mu_{h+1}^j \rng \right| \leq 2\Cnrm$ for all $i \in [n]$, we have from \cref{eq:hoeffding-nu-n} that, under the event $\ME_1 \cap \ME_3$,
  \begin{align}
\max_{\ell \in [d]} \frac 1n \sum_{i=1}^n \left( \left\lng \phi_h(x_h^i, a_h^i), w_\ell^\st - \sum_{j=1}^m \phiavg_{h+1}(\tilde x_{h+1}^j)_\ell \cdot \hat \mu_{h+1}^j \right\rng \right)^2  \leq & \frac{18\Cnrm^{3/2} \log^{1/4}(36d/\delta^2)}{m^{1/4}} + 8\Cnrm^2 \sqrt{\frac{\log 3/\delta}{n}}\label{eq:wstar-uhat-trunc}.
  \end{align}

Combining \cref{eq:wprime-wstar} and \cref{eq:wstar-uhat-trunc} via the bound $(a+b)^2 \leq 2(a^2+b^2)$ gives that for a sufficiently large constant $C > 0$, under the event $\ME_1 \cap \ME_2 \cap \ME_3$ (which occurs with probability at least $1-\delta$), for all $\ell \in [d]$,
\begin{align*}
\frac 1n \sum_{i=1}^n \left\lng \phi_h(x_h^i, a_h^i), \hat w_\ell - \frac{1}{m}\sum_{j=1}^m \psi_{h+1}^j \phiavg_{h+1}(\tilde x_{h+1}^j)_\ell \right\rng^2 &\leq C \cdot \left(\Cnrm^2 \sqrt{\frac{\log(3d^2/\delta)}{n}} + \frac{\Cnrm^{3/2} \log^{1/4}(36d/\delta^2)}{m^{1/4}}\right) \\ 
&\leq \epcvx^2
\end{align*}
where the final inequality is by choice of $m,n$ and holds as long as $C_{\ref{lem:feasibility-trunc}}$ is chosen sufficiently large. In this event, \cref{eq:what-wprime-constraint-trunc} is satisfied.
\end{proof}

\subsubsection{Approximate non-negativity}
\label{sec:approx-nneg-trunc}

Using \cref{lemma:policy-disc} together with a union bound over the resulting discretized set of policies, we next show that the desired non-negativity property in \cref{it:approx-nonneg} of \cref{def:trunc-core} is satisfied by the output of $\ESCt$. 
\begin{lemma}\label{lemma:hat-mu-apx-nonnegative-trunc}
There is a constant $C_{\ref{lemma:hat-mu-apx-nonnegative-trunc}}$ so that the following holds. Let $\Gamma \subset \Pi$, $n \in \NN$, $h \in [H]$, and $\epneg,\delta,\zeta \in (0,1)$ be given. Let $(x_h^i)_{i=1}^n$ be i.i.d. samples from a distribution $\nu_h \in \Delta(\MX)$ satisfying
\begin{equation}
\nu_h(x) \geq \zeta \cdot \max_{\pi \in\Pi} d\sups{\Mbar(\Gamma), \pi}_h(x) \qquad \forall x \in \MX
\label{eq:nu-hyp-trunc}
\end{equation}
Suppose that
\[n \geq \frac{C_{\ref{lemma:hat-mu-apx-nonnegative-trunc}}A^2\Cnrm^2 H^3}{\epneg^4 \zeta}\log(8AH\Cnrm/\epneg)\log(4d/\delta).\] Then it holds with probability at least $1-\delta$ that for all $\pi \in \Pi$ and all $\theta \in \RR^d\times\{0\}$,
\begin{align} \EE\sups{\Mbar(\Gamma), \pi} \left[ \min\left\{0, \min_{a \in \MA} \langle \phi_h(x_h,a), \theta\rangle\right\}\right] \geq & \min\left\{ 0, \min_{i \in [n], a \in \MA} \langle \phi_h(x_h^i, a), \theta\rangle \right\}- \epneg\norm{\theta}_1.\nonumber\\
  \end{align}
\end{lemma}

\begin{proof} 
  Let $\Pidisc$ be the set of policies guaranteed by \cref{lemma:policy-disc} applied to the ($(d+1)$-dimensional linear) MDP $\Mbar(\Gamma)$ with error parameter $\epdisc := \epneg/2$. For any $\pi \in \Pi$ and $\theta \in \RR^d\times\{0\}$, because $\phi_h\sups{\Mbar(\Gamma)}(x,a)_{1:d} = \phi_h(x,a)_{1:d}$ for all $x\in\bar\MX$, $a \in \MA$, it holds that $\lng \phi_h(x,a), \theta \rng = \lng \phi_h\sups{\Mbar(\Gamma)}(x,a), \theta \rng$ for all $h \in [H], x \in \MX, a \in \MA$. 
Thus the guarantee of \cref{lemma:policy-disc} implies that there is some $\pidisc \in \Pidisc$ so that
  \begin{equation} \langle \E\sups{\Mbar(\Gamma),\pidisc}[\phi_h(x_h,a_h)],\theta\rangle \geq \langle \E\sups{\Mbar(\Gamma),\pi}[\phi_h(x_h,a_h)],\theta\rangle - \frac{\epneg}{2} \norm{\theta}_1.
  \label{eq:pol-disc-app}
  \end{equation}

  By the bound on $|\Pidisc|$ and choice of $n$, as long as $C_{\ref{lemma:hat-mu-apx-nonnegative-trunc}}$ is sufficiently large, we have
  \begin{align}
    \label{eq:n-nonneg-lb}
    n \geq 2304A^2 \epneg^{-2}\zeta^{-1}\log(16d|\Pidisc|/\delta).
  \end{align}
  Fix some $\pi \in \Pidisc$.
  Define a random set $S \subseteq [n]$ by including index $i \in [n]$ in $S$ with probability $\frac{\zeta d_h\sups{\Mbar(\Gamma), \pi}(x_h^i)}{\nu_h(x_h^i)}.$ By \cref{eq:nu-hyp-trunc}, this nonnegative fraction is at most $1$; thus the sampling procedure is well-defined. Moreover, in expectation over the randomness in both $(x_h^i)_{i=1}^n$ and $S$,
  \[ \EE|S| = \sum_{i=1}^n \EE_{x_h^i \sim \nu_h} \frac{\zeta
      d_h\sups{\Mbar(\Gamma), \pi}(x_h^i)}{\nu_h(x_h^i)} = \zeta n \cdot d\sups{\Mbar(\Gamma),\pi}_h(\MX) \geq \frac{3\zeta n}{4}\]
  where the last inequality is by \cref{cor:not-reach-bound}. By the Chernoff bound and the fact that
  $n \geq 64\zeta^{-1}\log(4|\Pidisc|/\delta)$, we have
  $\Pr[|S| \geq \zeta n/2] \geq 1-\delta/(4|\Pidisc|)$. Condition on
  $|S| = n'$ for some $n' \geq \zeta n/2$. Write
  $S = \{i_1 < \dots < i_{n'}\}$, and for notational simplicity, let
  $\tilde{x}^j$ denote $x^{i_j}_h$ for each $1 \leq j \leq n'$. %

  By
  construction of the sampling procedure, as $j$ ranges from $1$ to
  $|S|$, the random variables $\tilde{x}^j$ are independent and
  identically distributed on $\MX$ according to the density
  $x \mapsto \frac{d\sups{\Mbar(\Gamma), \pi}_h(x)}{d\sups{\Mbar(\Gamma),\pi}_h(\MX)}$. We now need to show that the empirical average $\frac{1}{n'}\sum_{j=1}^{n'}\min\left\{0,\min_{a'\in\MA}\langle\phi_h(\tilde x^j,a'),\theta\rangle\right\}$ concentrates around its expectation under this conditional density, uniformly over $\theta$. To do so we invoke generalization bounds via Rademacher complexity. Consider the class $\MF$ of functions mapping $\MX$ to $\BR$, defined by
  \begin{align}
\MF := \left\{ x \mapsto \min\left\{ \min_{a' \in \MA} \lng \phi_h(x, a'), \theta \rng, 0 \right\} \ : \ \theta \in \BR^d\times\{0\},\ \| \theta\|_1 \leq 1 \right\}\nonumber.
  \end{align}
  By \cref{lem:min-feature-rc} (as well as the guarantee on $n$ in \cref{eq:n-nonneg-lb}) we get that \[\MR_{n'}(\MF) \leq 2A \sqrt{\frac{\pi \log d}{n'}} \leq 6A \sqrt{\frac{ \log d}{\zeta n}} \leq \frac{\epneg}{8},\] and so \cref{lem:unif-conv} with $B=1$ (together with the guarantee on $n$ in \cref{eq:n-nonneg-lb}) gives that with probability at least $1-\delta/(4|\Pidisc|)$ over the draws $\tilde x^j \sim \frac{\One{\cdot \in \MX} \cdot d_h\sups{\Mbar(\Gamma), \pi}(\cdot)}{d_h\sups{\Mbar(\Gamma), \pi}(\MX)}$, for all $\theta \in \BR^d\times\{0\}$,
  \begin{align}
    &\E\sups{\Mbar(\Gamma), \pi} \left[ \min\left\{ 0, \min_{a' \in \MA} \lng \phi_h(x_h, a'), \theta \rng \right\}\middle| x_h \in \MX \right] \nonumber\\
    &\geq \frac{1}{n'} \sum_{j=1}^{n'} \min \left\{ 0, \min_{a' \in \MA} \lng \phi_h(\tilde x^j, a'), \theta \rng \right\} -2\| \theta\|_1  \MR_{n'}(\MF)- \frac{\epneg \| \theta \|_1 }{4}\nonumber\\
    &\geq \min\left\{ 0, \min_{i \in [n], a' \in \MA} \lng \phi_h(x_h^i, a'), \theta \rng \right\} - \frac{\epneg \| \theta \|_1}{2}\nonumber.
  \end{align}
Recalling our convention that $\phi_h(\term,a) = e_{d+1}$ for all $a\in\MA$ (\cref{section:extended-overview}), so that $\langle \phi_h(\term,a),\theta\rangle = 0$, we have
\begin{align*}
\E\sups{\Mbar(\Gamma), \pi} \left[ \min\left\{ 0, \min_{a' \in \MA} \lng \phi_h(x_h, a'), \theta \rng \right\} \right] 
&= d\sups{\Mbar(\Gamma),\pi}(\MX) \cdot \E\sups{\Mbar(\Gamma), \pi} \left[ \min\left\{ 0, \min_{a' \in \MA} \lng \phi_h(x_h, a'), \theta \rng \right\}\middle| x_h \in \MX \right] \\
&\geq \E\sups{\Mbar(\Gamma), \pi} \left[ \min\left\{ 0, \min_{a' \in \MA} \lng \phi_h(x_h, a'), \theta \rng \right\}\middle| x_h \in \MX \right] \\
&\geq \min\left\{ 0, \min_{i \in [n], a' \in \MA} \lng \phi_h(x_h^i, a'), \theta \rng \right\} - \frac{\epneg \| \theta \|_1}{2}\nonumber
\end{align*}
as well. Removing the conditioning on $|S|$, the above inequality holds with probability at least $1-\delta/(2|\Pidisc|)$ over the randomness in $(x_h^i)_{i=1}^n$. By a union bound over $\pi \in \Pidisc$, with probability at least $1-\delta$, it holds for all $\pi \in \Pidisc$ and $\theta \in \RR^d \times \{0\}$ that
\[\langle \EE\sups{\Mbar(\Gamma), \pi}[\phi_h(x_h,a_h)],\theta\rangle \geq \min \left\{0, \min_{i \in [n],a \in \MA} \langle \phi_h(x_h^i,a),\theta\rangle\right\} - \epneg \norm{\theta}_1/2.\]
Suppose that this event holds. By the guarantee \cref{eq:pol-disc-app}, we conclude that for all $\pi \in \Pi$ and $\theta \in \RR^d \times \{0\}$,
\[ \langle \EE\sups{\Mbar(\Gamma),\pi}[\phi_h(x_h,a_h)],\theta\rangle \geq \min \left\{0, \min_{i \in [n],a \in \MA} \langle \phi_h(x_h^i,a),\theta\rangle\right\} - \epneg\norm{\theta}_1,\]
which completes the proof. %
\end{proof}

\subsubsection{All-policy approximation}
\label{sec:all-policy-trunc}
\begin{lemma}
  \label{lem:muhat-approx-trunc}
  There is a constant $C_{\ref{lem:muhat-approx-trunc}}$ so that the following holds. Fix $n,m \in \BN$, an odd $h \in [H]$, and $\epcvx, \alpha, \delta > 0$. Consider $\Gamma, \Psi_h \subset \Pi$ so that $\Psi_h$ is an $\alpha$-truncated cover at step $h$. With probability at least $1-\delta$, the output of $\ESCt(h, \Psi_h, \Gamma, \epcvx, \Cnrm, n, m)$ (\cref{alg:esc-trunc}) is either $\perp$ or a set of vectors $(\hat \mu_{h+1}^j)_{j=1}^m \subset \RR^d\times\{0\}$ satisfying the following: there are states $(\tilde{x}_{h+1}^j)_{j=1}^m \subseteq \MX$ so that for any policy $\pi \in \Pi$, the following inequalities hold:
  \begin{align}
&\max_{\ell \in [d]} \abs{ \sum_{x \in \MX} \lng \E\sups{\bar M(\emptyset), \pi}[\phi_h(x_h, a_h)], \mu_{h+1}(x) \rng \cdot \phiavg_{h+1}(x) - \sum_{j=1}^m \lng \E\sups{\Mbar(\emptyset), \pi}[\phi_h(x_h, a_h)], \hat \mu_{h+1}^j \rng \cdot \phiavg_{h+1}(\tilde x_{h+1}^j) }^2 \nonumber\\
    &\leq  \frac{2A}{\alpha} \cdot \left(2\epcvx^2 + C_{\ref{lem:muhat-approx-trunc}}\cdot \frac{\Cnrm^2 \sqrt{\log(2d/\delta)}}{\sqrt{n}}\right).\label{eq:mbar-empty-allpi}\\
    &\max_{\ell \in [d]} \abs{ \sum_{x \in \MX} \lng \E\sups{\bar M(\Gamma), \pi}[\phi_h(x_h, a_h)], \mu_{h+1}(x) \rng \cdot \phiavg_{h+1}(x) - \sum_{j=1}^m \lng \E\sups{\Mbar(\Gamma), \pi}[\phi_h(x_h, a_h)], \hat \mu_{h+1}^j \rng \cdot \phiavg_{h+1}(\tilde x_{h+1}^j) }^2 \nonumber\\
    &\leq  \frac{2A}{\tsmall} \cdot \left(2\epcvx^2 + C_{\ref{lem:muhat-approx-trunc}}\cdot \frac{\Cnrm^2 \sqrt{\log(2d/\delta)}}{\sqrt{n}}\right).\label{eq:mbar-gamma-allpi}
  \end{align}
\end{lemma}
\begin{proof}
  As in the proof of \cref{lem:feasibility-trunc}, for $\ell \in [d]$, we define $\bw_\ell^\st := \sum_{x \in \MX} \mu_{h+1}(x) \cdot \phiavg_{h+1}(x)_\ell$, so that $\E_{x' \sim \BP_h(x,a)}[\phiavg_{h+1}(x')_\ell] = \lng \bw_\ell^\st, \phi_h(x,a) \rng$ and $\| \bw_\ell^\st \|_1 \leq \Cnrm$. 
  
  Let $\nu_h$ be the distribution of $(x_h, a_h) \sim \frac 12 \cdot \left(\unif(\Psi_h) \circ_h \unif(\MA) + \unif(\Gamma)\circ_h \unif(\MA)\right)$, which is the marginal distribution of $(x_h^i, a_h^i)$ for each $i \in [n]$ indexing the dataset $\MC_h$ (defined in \cref{line:define-ch-trunc} of \cref{alg:ddtrunc}). 
 By \cref{cor:random-design-prediction-error} and a union bound over $\ell \in [d]$, there is an event $\ME_1$ that occurs with probability at least $1-\delta/2$ over the draw of $\MC_h = \{(x_h^i, a_h^i, x_{h+1}^i)\}_{i \in [n]}$, so that in the event $\ME_1$, we have that for all $\ell \in [d]$,
  \begin{align}
    \E_{(x, a) \sim \nu_h} \left[ \lng \hat \bw_\ell - \bw_\ell^\st, \phi_h(x,a) \rng^2 \right] \leq 3C_{\ref{cor:random-design-prediction-error}} \cdot \Cnrm^2 \cdot \sqrt{\frac{\log(2d^2/\delta)}{n}}\label{eq:wprime-wstar-nu-trunc}.
  \end{align}
  By \cref{lem:l1-generalization}, there is an event $\ME_2$ that occurs with probability at least $1-\delta/2$ over the draw of $\MC_h$, so that in the event $\ME_2$, we have for all $\theta \in \RR^d\times\{0\}$ that
  \begin{equation}
  \left|\EE_{(x,a)\sim\nu_h}[\langle \phi_h(x,a),\theta\rangle^2] - \frac{1}{n}\sum_{i=1}^n \langle \phi_h(x_h^i,a_h^i),\theta\rangle^2\right| \leq \frac{C_{\ref{lem:l1-generalization}}\norm{\theta}_1^2\sqrt{\log(d/\delta)}}{\sqrt{n}}.
  \label{eq:phi-generalization-trunc}
  \end{equation}
  We claim that in the event $\ME_1 \cap \ME_2$, the output of \ESCt{} satisfies the guarantee claimed in the lemma statement, where the states $(\tilde x_{h+1}^j)_{j=1}^m$ are precisely as computed in \ESCt{}. If the output is $\perp$, this is immediate. Otherwise, the output is a solution $(\hat\mu_{h+1}^j)_{j=1}^m$ to Program~\cref{eq:program-trunc}. Fix any $\ell \in [d]$, and recall the definition of $\hat \bw_\ell$ in \cref{eq:wj-guarantee-trunc}. Set $\theta := \hat{\bw}_\ell - \sum_{j=1}^m \hat\mu_{h+1}^j \cdot \phiavg_{h+1}(\tilde x_{h+1}^j)_\ell$. By the constraint $\norm{\hat{\bw}_\ell} \leq \Cnrm$ in \cref{eq:wj-guarantee-trunc} together with the program constraint \cref{eq:hatmu-cnrm-constraint-trunc}, we have $\norm{\theta}_1 \leq 2\Cnrm$. Thus, combining the program constraint \cref{eq:what-wprime-constraint-trunc} with \cref{eq:phi-generalization-trunc} for this value of $\theta$ gives
  \begin{align}
\E_{(x,a) \sim \nu_h} \left[ \left\lng \phi_h(x,a),\ \hat{\bw}_\ell - \sum_{j=1}^m \hat\mu_{h+1}^j \cdot \phiavg_{h+1}(\tilde x_{h+1}^j)_\ell \right\rng^2 \right] \leq \epcvx^2 + \frac{4C_{\ref{lem:l1-generalization}}\Cnrm^2 \sqrt{\log(2d/\delta)}}{\sqrt{n}}\label{eq:wprime-what-nu-trunc}. 
  \end{align}
  Combining \cref{eq:wprime-wstar-nu-trunc} and \cref{eq:wprime-what-nu-trunc} via the bound $(a+b)^2 \leq 2a^2+2b^2$ gives that
  \begin{align}
\E_{(x_h, a_h) \sim \nu_h} \left[\left\langle \phi_h(x_h,a_h),\ \bw_\ell^\st - \sum_{j=1}^m\hat\mu_{h+1}^j \cdot \phiavg_{h+1}(\tilde x_{h+1}^j)_\ell\right\rangle^2 \right] \leq 2\epcvx^2 + \frac{C_{\ref{lem:muhat-approx-trunc}}\Cnrm^2\sqrt{\log(2d/\delta)}}{\sqrt{n}}\nonumber
  \end{align}
  where we define $C_{\ref{lem:muhat-approx-trunc}} := 6C_{\ref{cor:random-design-prediction-error}} + 8C_{\ref{lem:l1-generalization}}$. Since $\Psi_h$ is assumed to be an $\alpha$-truncated policy cover (\cref{def:ih-trunc}) and $h$ is odd, by \cref{eq:nu-beta-indhyp-trunc} of  \cref{lem:nu-beta-indhyp} and definition of $\nu_h$, we have that $\nu_h(x,a) \geq \frac{\alpha}{2A} \cdot \max_{\pi \in \Pi} d_h\sups{\Mbar(\emptyset), \pi}(x,a)$ for all $(x,a) \in \MX\times \MA$. Thus, a change-of-measure gives that for all $\pi \in \Pi$,
  \begin{align}
\E\sups{\Mbar(\emptyset), \pi} \left[\left\langle \phi_h(x_h,a_h),\ \bw_\ell^\st - \sum_{j=1}^m\hat\mu_{h+1}^j \cdot \phiavg_{h+1}(\tilde x_{h+1}^j)_\ell\right\rangle^2 \right] \leq \frac{2A}{\alpha} \cdot \left(2\epcvx^2 + \frac{C_{\ref{lem:muhat-approx-trunc}}\Cnrm^2 \sqrt{\log(2d/\delta)}}{\sqrt{n}}\right).\label{eq:mbar-empty-squareinside}
  \end{align}
In a similar manner, by \cref{eq:nu-beta-mtil-trunc} of \cref{lem:nu-beta-indhyp} and definition of $\nu_h$, we have that $\nu_h(x,a) \geq \frac{\tsmall}{2A} \cdot \max_{\pi \in \Pi} d_h\sups{\Mbar(\Gamma), \pi}(x,a)$ for all $(x,a) \in \MX \times \MA$. Thus, for all $\pi \in \Pi$, 
    \begin{align}
\E\sups{\Mbar(\Gamma), \pi} \left[\left\langle \phi_h(x_h,a_h),\ \bw_\ell^\st - \sum_{j=1}^m\hat\mu_{h+1}^j \cdot \phiavg_{h+1}(\tilde x_{h+1}^j)_\ell\right\rangle^2 \right] \leq \frac{2A}{\tsmall} \cdot \left(2\epcvx^2 + \frac{C_{\ref{lem:muhat-approx-trunc}}\Cnrm^2 \sqrt{\log(2d/\delta)}}{\sqrt{n}}\right).\label{eq:mbar-gamma-squareinside}
    \end{align}
Applying Jensen's inequality to \cref{eq:mbar-empty-squareinside} and recalling the definition of $\bw_\ell^\st$ yields that
\begin{align*} 
&\left(\sum_{x \in \MX}\langle \EE\sups{\Mbar(\emptyset), \pi}[\phi_h(x_h,a_h)],\mu_{h+1}(x)\rangle \cdot \phiavg_{h+1}(x)_\ell - \sum_{j=1}^m \langle \EE\sups{\Mbar(\emptyset), \pi}[\phi_h(x_h,a_h)], \hat\mu_{h+1}^j\rangle \cdot \phiavg_{h+1}(\tilde x_{h+1}^j)_\ell\right)^2 \\
&\qquad\leq \frac{2A}{\alpha} \cdot \left(2\epcvx^2 + \frac{C_{\ref{lem:muhat-approx-trunc}}\Cnrm^2 \sqrt{\log(2d/\delta)}}{\sqrt{n}}\right),
\end{align*}
which verifies \cref{eq:mbar-empty-allpi} since $\ell \in [d]$ is arbitrary. 
Similarly, applying Jensen's inequality to \cref{eq:mbar-gamma-squareinside} yields \cref{eq:mbar-gamma-allpi}. 
\end{proof}

\subsubsection{Guarantee for emulator construction}
\label{sec:emulator-guarantee-trunc}
\cref{thm:mu-coreset-trunc} combines the results proven earlier in this section to establish that the output of \\ \ESCt is a \trunccore (per \cref{def:trunc-core}) with high probability.
\begin{theorem}
  \label{thm:mu-coreset-trunc}
  There is a constant $C_{\ref{thm:mu-coreset-trunc}}$ so that the following holds. Fix $n,m \in \BN$, $h \in [H]$ with $h$ odd, and $\epneg, \epapx, \epcvx, \alpha,\delta > 0$. Let $\Psi_h, \Gamma \subset \Pi$ be given so that $\Psi_h$ is an $\alpha$-truncated cover at step $h$ (\cref{def:ih-trunc}). Suppose that the following bounds hold:
  \begin{align}
    n &\geq  C_{\ref{thm:mu-coreset-trunc}} \cdot \max \left\{\epcvx^{-4} \Cnrm^4 \log(d/\delta) ,\epneg^{-4} \tsmall^{-1} A^2 \Cnrm^2 H^3 \log(AH\Cnrm/\epneg) \log(d/\delta),\right. \nonumber\\
    & \qquad \qquad \quad \quad  \left. \epapx^{-4}\Cnrm^4 (\tsmall^{-2}+\alpha^{-2}) A^2 \log(d/\delta) \right\}\nonumber\\
    m &\geq  C_{\ref{thm:mu-coreset-trunc}} \cdot  \epcvx^{-8} \Cnrm^6 \log(d/\delta), \nonumber\\
    \epcvx &\leq \epapx \cdot  \min \left\{ \sqrt{\alpha/A}/4, \sqrt{\tsmall/A}/4 \right\}  \nonumber\\
    \alpha &\geq  \frac{\tsmall}{\trunc} \nonumber\\
    \delta &\leq 1/\sqrt{m}.\nonumber
  \end{align}
  Then with probability at least $1-\delta$, the output of $\ESCt(h, \Psi_h, \Gamma, \epcvx, \Cnrm, n,m)$ (\cref{alg:esc-trunc}) is a $(\epapx, \epneg, \Cnrm; \Gamma)$-\trunccore at step $h$. Moreover, the sample complexity of $\ESCt(h, \Psi_h, \Gamma, \epcvx, \Cnrm, n, m)$ is $n+m$, and the time complexity is \\  $\poly(n,m,d, \log(1/(\epneg\epapx\epcvx)))$. 
\end{theorem}
\begin{remark}[Sample and computational costs; analogue of \cref{rmk:convex-program-efficient}]
  \label{rmk:esc-sample-comp-trunc}
  Note that the sample complexity of \ESCt{} in the context of \cref{thm:mu-coreset-trunc} is $n+m$. As in the reachable setting, we have assumed for simplicity that the convex program \cref{eq:program-trunc} can be solved exactly in time $\poly(n,m,d)$. As this is not strictly speaking known to be true, we once again apply the argument in \cref{sec:opt-details} to obtain that we can implement \ESCt{}  in time $\poly(n,m,d, \log(\Cnrm/(\epcvx\epapx\epneg))$,  where a relaxation of \cref{eq:program-trunc} is solved, which is still sufficient to compute a $(\epapx, \epneg,\Cnrm; \Gamma)$-\trunccore (after decreasing the values of $\epapx, \epneg$ passed to \ESCt{} by a constant factor). 
\end{remark}
\begin{proof}[Proof of \cref{thm:mu-coreset-trunc}]
  By \cref{lem:feasibility-trunc} and our choices of $n,m,\epcvx,\delta$, with probability at least $1-\delta/3$, $\ESCt(h, \Psi_h, \Gamma, \epcvx, \Cnrm, n, m)$ produces a solution $(\hat \mu_{h=1}^j)_{j=1}^m \subset \RR^d\times\{0\}$ to Program \ref{eq:program-trunc} (i.e., the program is feasible). Certainly any such solution satisfies \cref{it:norm-bound} of \cref{def:trunc-core}, with norm bound $C := \Cnrm$. 

  Next, by the assumptions that $\Psi_h$ is an $\alpha$-truncated cover at step $h$, that $h$ is odd, and that $\alpha \geq \frac{\tsmall}{\trunc}$, \cref{eq:nu-beta-mtil-trunc} of \cref{lem:nu-beta-indhyp} gives the following: letting $\nu_h \in \Delta(\MX)$ denote the distribution of $x_h$ under a policy $\pi \sim \frac 12 \cdot (\unif(\Gamma)\circ_h \unif(\MA) + \unif(\Psi_h) \circ_h \unif(\MA))$, we have that for all $x \in \MX$, \[\nu_h(x) \geq \frac{\tsmall}{2} \cdot \max_{\pi \in \Pi} d_h\sups{\Mbar(\Gamma), \pi}(x) \geq \frac{\tsmall}{2} \cdot \max_{\pi \in \Pi} d_h\sups{\Mbar(\emptyset), \pi}(x).\] (The second inequality above uses \cref{lem:gamma-monotonicity}.) In particular, the condition \cref{eq:nu-hyp-trunc} of \cref{lemma:hat-mu-apx-nonnegative-trunc} is satisfied with $\zeta = \tsmall/2$ and any $\Gamma'\in\{\emptyset,\Gamma\}$. Thus, since the marginal of the points $x_h^i$ in the dataset $\MC_h = \{(x_h^i, a_h^i, x_{h+1}^i)\}_{i=1}^n$ is according to $\nu_h$, it follows from \cref{lemma:hat-mu-apx-nonnegative-trunc} that with probability at least $1-\delta/3$, the output $(\hat \mu_{h+1}^j)_{j=1}^m$ produced by \cref{alg:esc-trunc} (if not $\perp$) satisfies, for all $\pi \in \Pi$, $j \in [m]$, and $\Gamma' \in \{ \emptyset, \Gamma\}$, 
  \begin{align}
    \E\sups{\Mbar(\Gamma'), \pi}\left[ \max\left\{ 0, \max_{a \in \MA}-\lng \phi_h(x_h, a), \hat \mu_{h+1}^j \rng \right\} \right] 
    &\leq \max \left\{0, \max_{i \in [n], a \in \MA} -\lng \phi_h(x_h^i, a), \hat \mu_{h+1}^j \rng \right\} + \epneg \| \hat \mu_{h+1}^j \|_1\nonumber\\
    &\leq \epneg \| \hat \mu_{h+1}^j \|_1\nonumber,
  \end{align}
  where the second inequality uses the constraint \cref{eq:hatmu-nneg-constraint-trunc}. The above inequality establishes that $(\hat \mu_{h+1}^j)_{j=1}^m$ satisfies \cref{it:approx-nonneg} of \cref{def:trunc-core}.

  Finally, by \cref{lem:muhat-approx-trunc} and our choice of $n, \epcvx$, with probability at least $1-\delta/3$, if the output of \cref{alg:esc-trunc} is not $\perp$, then the output $(\hat \mu_{h+1}^j)_{j=1}^m$ satisfies, for all $\pi \in \Pi, \Gamma' \in \{\emptyset, \Gamma\}$,
  \begin{align}
\max_{\ell \in [d]}\abs{ \sum_{x \in \MX} \lng \E\sups{\Mbar(\Gamma'), \pi}[\phi_h(x_h, a_h)], \mu_{h+1}(x) \rng \cdot \phiavg_{h+1}(x) - \sum_{j=1}^m \lng \E\sups{\Mbar(\Gamma'), \pi}[\phi_h(x_h, a_h)], \hat \mu_{h+1}^j \rng \cdot \phiavg_{h+1}(\tilde x_{h+1}^j) } &\leq \epapx\nonumber,
  \end{align}
  which verifies \cref{it:pol-approx} of \cref{def:trunc-core}.

\end{proof}

\subsection{From emulator to policy cover}
\label{sec:emulator-pc-unreachable}
In this section, the main technical result is \cref{lem:pc-induction-trunc}, which uses the properties of a \trunccore to show that the \PC algorithm can extend a truncated policy cover for steps $1, \ldots, h$ to a truncated policy cover for steps $1, \ldots, h+2$, unless certain policies have large \emph{extraneous visitation probability}. The following definition formalizes the set of policies that need to be considered, when the truncated policy covers given to \PC are $\Psi_{1:h}$, the backup policy cover is $\Gamma$, and the output of \PC is $\Psiapx_{h+1}$. These are also the policies that must be added to the backup policy cover in the next phase in order to make progress:

\begin{definition}[{Extraneous visitation probability}]\label{def:extra-vis-prob}
  Let $\Psi_1,\dots,\Psi_h,\Psiapx_{h+1},\Gamma \subset \Pi$ %
  be sets of policies and let $\pifinal \in \Pi$. Let $h \in [H]$. Define 
  \begin{align}
    \Sigma_h(\Psi_{1:h},\Psiapx_{h+1}, \pifinal) :=& \Psiapx_{h+1} \cup \bigcup_{\substack{0 \leq h_0 \leq h \\ \pi \in \Psi_{h_0}}} \pi \circ_{h_0} \unif(\MA) \circ_{h_0+1} \pifinal \label{eq:sigma-gamma-psi}\\
    \Delta_h(\Psi_{1:h},\Psiapx_{h+1}, \pifinal;\Gamma) :=& \max_{\substack{\pi \in \Sigma_h(\Psi_{1:h}, \Psiapx_{h+1}, \pifinal) \\ 1 \leq g \leq h}} d_g\sups{M,\pi}(\MX \backslash \Xreach_g(\Gamma))  %
    \label{eq:delta-gamma-psi}
  \end{align}
  where for notational convenience we are defining $\Psi_0$ to contain a single arbitrary policy (recall that $\pi \circ_1 \pi' = \pi'$ for all policies $\pi,\pi'$, so the choice does not matter). We call the quantities $\Delta_h(\Psi_{1:h}, \Psiapx_{h+1}, \pifinal; \Gamma)$ \emph{extraneous visitation probabilities}. 
\end{definition}

\cref{lem:emp-pc-guarantee} provides an analogue of \cref{lem:emp-pc-guarantee-rch} for the unreachable setting, stating that with high probability, when given a \trunccore $(\hat \mu_{h+1}^j)_{j=1}^m$, \PC will output a set of policies that cover most of the emulator vectors $\hat \mu_{h+1}^j$. %
\begin{lemma}[Guarantee for \PC;  unreachable setting]
  \label{lem:emp-pc-guarantee}
  Let $\alpha,\epapx,\epneg,\xi,\Cemp,\delta>0$ with $\xi \leq \Cemp$. Let $h \in [H]$, $N \in \NN$, $\Gamma \subset \Pi$, and let $(\hat\mu_{h+1}^j)_{j=1}^m \subseteq \RR^d\times\{0\}$ be an $(\epapx, \epneg, \Cemp; \Gamma)$-\trunccore at step $h$ (\cref{def:trunc-core}) so that $\epneg \leq \xi/(4\Cemp H)$. 
  Let $\Psi_{1:h}$ be $\alpha$-truncated policy covers for steps $1,\dots,h$ respectively (\cref{def:ih-trunc}).
  
  Suppose that $N \geq \max(\nstat(\xi/(8\Cemp), \alpha, \delta\xi/(6\Cemp)), \nfe(\xi/(8\Cemp), \delta\xi/(6\Cemp)))$. Then \PC (\cref{alg:pc}) with parameters $\xi,\Cemp,N$ outputs a set of policies $\Psiapx_{h+1}$ of size $|\Psiapx_{h+1}| \leq 2\Cemp/\xi$, a subset $\MG \subset [m]$ and a policy $\pifinal \in \Pi$ satisfying the following property.  %
With probability at least $1-\delta$, for all $\pi \in \Pi$:
  \begin{enumerate}
  \item \label{it:bad-inds-bound-trunc} $\sum_{j \in [m]\setminus \MG} \lng \E\sups{\Mbar(\emptyset), \pi}[\phi_h(x_h, a_h)], \hat \mu_{h+1}^j \rng \leq 3\xi/2 + \frac{2AH^2\Cemp}{\alpha} \cdot \Delta_h(\Psi_{1:h}, \Psiapx_{h+1}, \pifinal; \Gamma)$. 
  \item     \label{it:good-inds-bound-trunc} In the event that $\Delta_h(\Psi_{1:h}, \Psiapx_{h+1}, \pifinal; \Gamma) \leq \frac{\xi}{4\Cemp}$, the following holds: for all $j \in \MG$, there is some $\pi' \in \Psiapx_{h+1}$ so that
    \begin{align}
      \lng\E\sups{\Mbar(\Gamma), \pi'}[\phi_h(x_h, a_h)], \hat \mu_{h+1}^j \rng \geq  \left( \frac{\xi}{4\Cemp} - \Delta_h(\Psi_{1:h}, \Psiapx_{h+1}, \pifinal; \Gamma) \right) \cdot \lng \E\sups{\Mbar(\emptyset), \pi}[\phi_h(x_h, a_h)], \hat \mu_{h+1}^j \rng .\nonumber
    \end{align}
  \end{enumerate}
\end{lemma}
\begin{proof}
 By \cref{lem:pc-size-bound} and the guarantee $\sum_{j=1}^m \norm{\hat\mu_{h+1}^j}_1 \leq \Cemp$ (\cref{it:norm-bound} of \cref{def:trunc-core}), we have $|\Psiapx_{h+1}| \leq 2\Cemp/\xi$. It follows that the algorithm makes at most $1 + 2\Cemp/\xi \leq 3\Cemp/\xi$ calls to each of \PSDP and \VI. By the choice of $N$, each call has failure probability at most $\delta\xi/(6\Cemp)$, so with probability at least $1-\delta$ all of the calls succeed (i.e. satisfy the guarantees of \cref{lem:psdp-trunc,lem:fe}). We assume from now on that this event holds. Also by the choice of $N$, the guarantee of \cref{lem:fe} holds with $\ell_\infty$ error at most $\epstat := \xi/(8\Cemp)$.

\paragraph{Proof of \cref{it:bad-inds-bound-trunc}.} For convenience set $T = |\Psiapx_{h+1}|+1$. Recall that $\pifinal = \pi^T$ is the last policy computed by \cref{alg:pc} (the policy which is not added to $\Psiapx_{h+1}$). The output set $\MG$ of \cref{alg:pc} is given by $\MG := [m] \setminus \MB_T$. %

By the termination condition (\cref{alg:pc:line:break}), we have $\langle \hat \phi_h^T, \sum_{j \in [m]\setminus \MG} \hat{\mu}_{h+1}^j \rangle = \lng \hat \phi_h^T, \sum_{j \in \MB_T} \hat \mu_{h+1}^j \rng  < \xi$. By the guarantee of \FE (\cref{lem:fe}) and the fact that $\left\| \sum_{j \in [m] \backslash \MJ} \hat \mu_{h+1}^j \right\|_1 \leq \Cemp$ (per \cref{it:norm-bound} of \cref{def:trunc-core}), it follows that %
\begin{align}\label{eq:fe-piprime}\left\langle \EE\sups{M,\pifinal}[\phi_h(x_h,a_h)], \sum_{j \in [m]\setminus \MG} \hat \mu_{h+1}^j\right\rangle  < \xi + \epstat \Cemp .\end{align}

Next we apply the guarantee of \PSDP (\cref{lem:psdp-trunc}); we first verify its preconditions. Note that $\pifinal$ is the output of \PSDP at step $k:=h$ with input vector $\theta := \mures{T} = \sum_{j \in [m]\setminus \MG} \hat\mu_{h+1}^j$. By \cref{it:norm-bound} of \cref{def:trunc-core} we have $\norm{\theta}_1 \leq \Cemp$. For any $x,h$, note that the function $w \mapsto \max \{ 0, \max_{a \in \MA} - \lng \phi_h(x,a), w \rng \}$ is subadditive.  It follows by \cref{it:approx-nonneg} (in particular, \cref{eq:core-nonneg-gamma}) that $\theta$ is an $\epneg \Cemp$-nearly nonnegative target at step $h$ with respect to $\Gamma$ (\cref{def:nnnt}). By assumption, $\Psi_{1:h-1}$ are $\alpha$-truncated covers at steps $1,\dots,h-1$. Thus, by choice of $N \geq \nstat(\epstat, \alpha, \delta\xi/(6\Cemp))$, we have that for all $\pi \in \Pi$,

\begin{align*} 
\langle \E\sups{M,\pifinal}[\phi_h(x_h,a_h)],\theta\rangle 
&\geq \langle \E\sups{\Mbar(\emptyset),\pi}[\phi_h(x_h,a_h)],\theta\rangle \\
&\quad- \epstat\Cemp - \epneg\Cemp H - \frac{2A\Cemp H^2}{\alpha} \Delta_h(\Psi_{1:h},\Psiapx_{h+1},\pifinal;\Gamma)
\end{align*}
where we obtained the third term by observing that for all $h' \leq g \leq h$ and $\pi' \in \Psi_{h'-1}$, we have $d_g^{M, \pi' \circ_{h'-1} \unif(\MA) \circ_{h'} \pifinal}(\MX\setminus \Xreach_g(\Gamma)) \leq \Delta_h(\Psi_{1:h},\Psiapx_{h+1},\pifinal;\Gamma)$.

Rearranging and combining with \cref{eq:fe-piprime}, we get that for all $\pi \in \Pi$, %
\begin{align}
  & \left\langle \EE\sups{\Mbar(\emptyset), \pi}[\phi_h(x_h,a_h)], \sum_{j \in [m]\setminus \MG} \hat{\mu}_{h+1}^j\right\rangle \nonumber\\
  &\leq \left\lng \E\sups{M, \pifinal}[\phi_h(x_h, a_h)], \sum_{j \in [m] \backslash \MG} \hat \mu_{h+1}^j \right\rng + \left(\epstat + \epneg  H + \frac{2AH^2}{\alpha} \cdot \Delta_h(\Psi_{1:h},\Psiapx_{h+1},\pifinal;\Gamma)\right)\cdot\Cemp \nonumber\\
  &<  \xi + \left(2\epstat + \epneg H + \frac{2AH^2}{\alpha} \cdot \Delta_h(\Psi_{1:h},\Psiapx_{h+1},\pifinal;\Gamma)\right)\cdot\Cemp\nonumber\\
  &\leq 3\xi/2 + \frac{2AH^2\Cemp}{\alpha} \cdot \Delta_h(\Psi_{1:h},\Psiapx_{h+1},\pifinal;\Gamma)\nonumber,
\end{align}
where the first inequality uses the previously-derived \PSDP guarantee; the second inequality uses \cref{eq:fe-piprime}; and the third inequality uses the choice of $\epstat = \xi/(8\Cemp)$ and the assumption that $\epneg \leq \xi/(4\Cemp H)$. 
The above display establishes the first claim of the lemma statement.

\paragraph{Proof of \cref{it:good-inds-bound-trunc}.} To establish the second claim, we note that for each $j \in \MG$ there is some $t < T$ so that $j \in \MG_t$. But then the estimated feature vector $\hat\phi_h^t$ of policy $\pi^t \in \Psiapx_{h+1}$ satisfies
\[\langle \hat \phi_h^t, \hat \mu_{h+1}^j\rangle \geq (\xi/(2\Cemp)) \cdot \norm{\hat \mu_{h+1}^j}_1,\]
so by the guarantee of \FE (\cref{lem:fe}),
\begin{align}
\langle \EE\sups{M, \pi^t}[\phi_h(x_h,a_h)], \hat\mu_{h+1}^j \rangle
  &\geq \left(\frac{\xi}{2\Cemp} - \epstat\right)\cdot\norm{\hat\mu_{h+1}^j}_1 .\label{eq:mpi-hatmu}
\end{align}
Moreover, by \cref{lem:m-mbar-delta},
\begin{align}
  \lng \E\sups{M, \pi^t}[\phi_h(x_h, a_h)], \hat \mu_{h+1}^j \rng &\leq  \lng \E\sups{\Mbar(\Gamma), \pi^t}[\phi_h(x_h, a_h)], \hat \mu_{h+1}^j \rng +  \| \hat \mu_{h+1}^j \|_1 \cdot \sum_{g=1}^h d_g\sups{M, \pi^t}(\MX \backslash \Xreach_g(\Gamma))\nonumber\\
  &\leq  \lng \E\sups{\Mbar(\Gamma), \pi^t}[\phi_h(x_h, a_h)], \hat \mu_{h+1}^j \rng + \| \hat \mu_{h+1}^j \|_1 \cdot \Delta_h(\Psi_{1:h},\Psiapx_{h+1},\pifinal;\Gamma)\label{eq:mpi-delta},
\end{align}
where the second inequality uses that $\pi^t \in \Psiapx_{h+1}$. Combining \cref{eq:mpi-hatmu} and \cref{eq:mpi-delta} and using the fact that $\epstat \leq \xi / (4\Cemp)$, we see that, in the event that $\Delta_h(\Psi_{1:h}, \Psiapx_{h+1}, \pifinal; \Gamma) \leq \frac{\xi}{4\Cemp}$, 
\begin{align}
  \lng \E\sups{\Mbar(\Gamma), \pi^t}[\phi_h(x_h, a_h)], \hat \mu_{h+1}^j \rng \geq &\left( \frac{\xi}{4\Cemp} - \Delta_h(\Psi_{1:h},\Psiapx_{h+1},\pifinal;\Gamma) \right) \cdot \norm{ \hat \mu_{h+1}^j}_1 \nonumber\\
  \geq & \left( \frac{\xi}{4\Cemp} - \Delta_h(\Psi_{1:h},\Psiapx_{h+1},\pifinal;\Gamma) \right) \cdot \sup_{\pi \in \Pi} \langle \EE\sups{\Mbar(\emptyset), \pi}[\phi_h(x_h,a_h)], \hat\mu_{h+1}^j\rangle\nonumber,
\end{align}
which completes the proof of \cref{it:good-inds-bound-trunc}. 
\end{proof}

The next lemma, an analogue of \cref{lem:aux-pc-coreset} in the unreachable setting, is a consequence of \cref{lem:emp-pc-guarantee} and the approximate non-negativity property of a \trunccore (\cref{def:trunc-core}). Unlike \cref{lem:aux-pc-coreset}, there is a failure event (\cref{it:case-large-vis-gamma}) in which the guarantee \cref{eq:aux-coreset-trunc} may fail to hold. In this failure event, some policy found by \PC (namely, one in the set $\Sigma_h(\Psi_{1:h}, \Psiapx_{h+1}, \pifinal)$) visits a truncated set $\MX \backslash \Xreach_g(\Gamma)$, for some $g \in [H]$. 
\begin{lemma}
  \label{lem:newgamma-or-psigood}
  In the setting of \cref{lem:emp-pc-guarantee}, with probability at least $1-\delta$ (over the randomness in \PC), \textbf{at least one} of the following statements holds:
  \begin{enumerate}[label=(\alph*)]
  \item\label{it:case-large-vis-gamma} The output $(\Psiapx_{h+1}, \pifinal)$ of \PC (\cref{alg:pc}) satisfies $\Delta_h(\Psi_{1:h},\Psiapx_{h+1},\pifinal;\Gamma) \geq \frac{\xi \alpha}{8AH^2 \Cemp}$ (recall that $\Delta_h(\Psi_{1:h},\Psiapx_{h+1},\pifinal;\Gamma)$ is defined in \cref{eq:delta-gamma-psi}).
  \item\label{it:case-small-vis-gamma} For any $\pi \in \Pi$, $B > 0$, and function $g : [m] \ra [0,B]$,
    \begin{align}
      \sum_{j=1}^m \lng \E\sups{\Mbar(\emptyset), \pi}[\phi_h(x_h, a_h)], \hat \mu_{h+1}^j \rng \cdot g(j) 
      &\leq  \frac{33 \Cemp^3 B \epneg}{\xi^2} + 2B\xi \nonumber\\
      & + \frac{8 \Cemp}{\xi} \sum_{\pi' \in \Psiapx_{h+1}} \sum_{j=1}^m \lng \E\sups{\Mbar(\Gamma), \pi'}[\phi_h(x_h, a_h)], \hat \mu_{h+1}^j \rng \cdot g(j)\label{eq:aux-coreset-trunc}.
    \end{align}
  \end{enumerate}
\end{lemma}
\begin{proof}
  Recall that the output of $\PC$ is a tuple $(\Psiapx_{h+1}, \MG, \pifinal)$, where $\Psiapx_{h+1} \subset \Pi, \MG \subset [m], \pifinal \in \Pi$. 
By \cref{lem:emp-pc-guarantee}, $|\Psiapx_{h+1}| \leq 2\Cemp/\xi$. Let us assume henceforth that \cref{it:case-large-vis-gamma} does not hold, i.e. $\Delta_h(\Psi_{1:h},\Psiapx_{h+1},\pifinal;\Gamma) < \frac{\xi \alpha}{8AH^2 \Cemp} < \frac{\xi}{4\Cemp}$. We then will prove \cref{it:case-small-vis-gamma}.

Combining the guarantees of \cref{lem:emp-pc-guarantee} with the assumed bound on $\Delta_h(\Psi_{1:h},\Psiapx_{h+1},\pifinal;\Gamma)$, we have that $\Psiapx_{h+1}, \MG, \pifinal$ satisfy (with probability at least $1-\delta$) the following properties for all $\pi \in \Pi$:
  \begin{enumerate}%
  \item \label{it:missed-inds-trunc} $\sum_{j \in [m]\setminus \MG} \lng \E\sups{\Mbar(\emptyset), \pi}[\phi_h(x_h, a_h)], \hat \mu_{h+1}^j \rng \leq 2\xi$.
  \item \label{it:hit-inds-trunc} For all $j \in \MG$, there is some $\pi' \in \Psiapx_{h+1}$ so that $\lng\E\sups{\Mbar(\Gamma),\pi'}[\phi_h(x_h, a_h)], \hat \mu_{h+1}^j \rng \geq \frac{\xi}{8 \Cemp} \cdot \lng \E\sups{\Mbar(\emptyset), \pi}[\phi_h(x_h, a_h)], \hat \mu_{h+1}^j \rng$. 
  \end{enumerate}
  Additionally, the fact that $(\hat \mu_{h+1}^j)_{j=1}^m$ is assumed to be an $(\epapx, \epneg, \Cemp; \Gamma)$-truncated coreset at step $h$ (\cref{def:trunc-core}) guarantees that for any $\pi \in \Pi$, $\Gamma' \in \{\emptyset, \Gamma\}$, and $f: [m]\to[0,B]$,
\begin{align}
  \sum_{j=1}^m \max\left(0, -\langle \EE\sups{\Mbar(\Gamma'), \pi}[\phi_h(x_h,a_h)], \hat\mu_{h+1}^j\rangle \cdot f(j)\right)
  &\leq \sum_{j=1}^m f(j) \cdot \max \left\{ 0, \E\sups{\Mbar(\Gamma'), \pi} \left[ \max_{a' \in \MA} - \lng \phi_h(x_h, a'), \hat \mu_{h+1}^j\rng \right] \right\}\nonumber\\
  &\leq\sum_{j=1}^m f(j) \cdot \E\sups{\Mbar(\Gamma'), \pi} \left[ \max\left\{ 0, \max_{a' \in \MA} -\lng \phi_h(x_h, a'), \hat \mu_{h+1}^j \rng \right\} \right] \nonumber\\
  & \leq  \sum_{j=1}^m B\epneg \norm{\hat\mu_{h+1}^j}_1 \nonumber \\
&\leq B\epneg \Cemp,
\label{eq:hatmu-neg-error-trunc}
\end{align}
where the second inequality uses Jensen's inequality.

We use the above properties to prove the claimed bound \cref{eq:aux-coreset-trunc}. Fix $\pi \in \Pi$ and $g: [m]\to[0,B]$. We separate the left-hand side of \cref{eq:aux-coreset-trunc} into two terms:
\begin{align*} 
&\sum_{j=1}^m \lng \E\sups{\Mbar(\emptyset), \pi}[\phi_h(x_h, a_h)], \hat \mu_{h+1}^j \rng \cdot g(j) \\
&= \underbrace{\sum_{j \in [m] \setminus \MG} \lng \E\sups{\Mbar(\emptyset), \pi}[\phi_h(x_h, a_h)], \hat \mu_{h+1}^j \rng \cdot g(j)}_{\S} + \underbrace{\sum_{j\in \MG} \lng \E\sups{\Mbar(\emptyset), \pi}[\phi_h(x_h, a_h)], \hat \mu_{h+1}^j \rng \cdot g(j)}_{\dag}.
\end{align*}

\paragraph{Bounding $\S$.} By \cref{eq:hatmu-neg-error-trunc} with $\Gamma' = \emptyset$ and the bound $g(j) \in [0,B]$ for all $j$, defining $f(j) := B - g(j) \in [0,B]$, the first term $\S$ can be bounded as %
\begin{align*}
  & \sum_{j \in [m] \backslash \MG} \lng \E\sups{\Mbar(\emptyset), \pi}[\phi_h(x_h, a_h)], \hat \mu_{h+1}^j \rng \cdot g(j)\\
  &= B \cdot \sum_{j \in [m] \backslash \MG} \lng \E\sups{\Mbar(\emptyset), \pi}[\phi_h(x_h, a_h)], \hat \mu_{h+1}^j \rng  - \sum_{j \in [m] \backslash \MG}\E\sups{\Mbar(\emptyset), \pi}[\phi_h(x_h, a_h)], \hat \mu_{h+1}^j \rng \cdot f(j) \\
&\leq  B \cdot \sum_{j \in [m] \backslash \MG} \lng \E\sups{\Mbar(\emptyset), \pi}[\phi_h(x_h, a_h)], \hat \mu_{h+1}^j \rng + B\epneg \Cemp\\
&\leq   2B\xi + B \epneg \Cemp,
\end{align*}
where the last inequality uses \cref{it:missed-inds-trunc} above. 

\paragraph{Bounding $\dag$.} Next,
\begin{align*}
& \sum_{j\in \MG} \lng \E\sups{\Mbar(\emptyset), \pi}[\phi_h(x_h, a_h)], \hat \mu_{h+1}^j \rng \cdot g(j)\\
&\leq \sum_{j \in \MG} \frac{8\Cemp}{\xi}\max_{\pi' \in \Psiapx_{h+1}} \langle \EE\sups{\pi', \Mbar(\Gamma)}[\phi_h(x_h,a_h)],\hat\mu_{h+1}^j\rangle \cdot g(j) \\
&\leq \sum_{j \in \MG} \frac{8\Cemp}{\xi} \left(B|\Psiapx_{h+1}|\epneg\norm{\hat\mu_{h+1}^j}_1 + \sum_{\pi' \in \Psiapx_{h+1}}\langle \EE\sups{\Mbar(\Gamma),\pi'}[\phi_h(x_h,a_h)],\hat\mu_{h+1}^j\rangle \cdot g(j)\right) \\
&\leq \frac{8\Cemp^2 B|\Psiapx_{h+1}|\epneg}{\xi} + \frac{8\Cemp}{\xi} \sum_{\pi' \in \Psiapx_{h+1}} \sum_{j \in \MG} \langle \EE\sups{\Mbar(\Gamma),\pi'}[\phi_h(x_h,a_h)],\hat\mu_{h+1}^j\rangle \cdot g(j) \\
&\leq \frac{16\Cemp^3 B\epneg}{\xi^2} + \frac{8\Cemp}{\xi} \sum_{\pi' \in \Psiapx_{h+1}} \sum_{j \in \MG} \langle \EE\sups{\Mbar(\Gamma),\pi'}[\phi_h(x_h,a_h)],\hat\mu_{h+1}^j\rangle \cdot g(j) \\
&\leq \frac{32\Cemp^3 B\epneg}{\xi^2} + \frac{8\Cemp}{\xi} \sum_{\pi' \in \Psiapx_{h+1}} \sum_{j =1}^m \langle \EE\sups{\Mbar(\Gamma), \pi'}[\phi_h(x_h,a_h)],\hat\mu_{h+1}^j\rangle \cdot g(j)
\end{align*}
where the first inequality uses the property of \cref{it:hit-inds-trunc} above, the second inequality uses the property \cref{eq:core-nonneg-gamma} of \cref{def:trunc-core}, the third inequality uses \cref{it:norm-bound} of \cref{def:trunc-core}, the fourth inequality uses the bound $|\Psiapx_{h+1}| \leq 2\Cemp/\xi$, and the fifth inequality uses \cref{eq:hatmu-neg-error-trunc} with $\Gamma' = \Gamma$ together with the bound $|\Psiapx_{h+1}| \leq 2\Cemp/\xi$.

\paragraph{Putting everything together.} Combining the two bounds $\S$ and $\dag$, and using that $\Cemp/\xi \geq 1$, we get that
\begin{align*} 
\sum_{j=1}^m \lng \E\sups{\Mbar(\emptyset),\pi}[\phi_h(x_h, a_h)], \hat \mu_{h+1}^j \rng \cdot g(j)
&\leq \frac{33\Cemp^3 B\epneg}{\xi^2} + 2B\xi \\
&+ \frac{8\Cemp}{\xi} \sum_{\pi' \in \Psiapx_{h+1}} \sum_{j =1}^m \langle \EE\sups{\Mbar(\Gamma),\pi'}[\phi_h(x_h,a_h)],\hat\mu_{h+1}^j\rangle \cdot g(j)
\end{align*}
as claimed. 
\end{proof}

\cref{lem:pc-induction-trunc} below is an analogue of \cref{lemma:pc-coreset-guarantee} in the unreachable setting: it states that, with high probability, \PC will either return a truncated policy cover for the next step (\cref{it:found-trunc-pc}), or otherwise will find some policy (in the set $\Sigma_h(\Psi_{1:h}, \Psiapx_{h+1}, \pifinal \Gamma)$) that explores states that were truncated in $\Mbar(\Gamma)$  (\cref{it:found-gamma-violation}).
\begin{lemma}
  \label{lem:pc-induction-trunc}
  Suppose $h \in [H]$ is odd, $\epapx, \epneg, \alpha, \Cemp, \delta > 0$, and $\xi \in (0, \Cemp)$. Suppose $\epneg \leq \xi/(4\Cemp H)$. %
  Let $\Psi_1, \ldots, \Psi_h \subset \Pi$ denote $\alpha$-truncated policy covers for steps $1,\dots,h$, and $\Gamma \subset \Pi$ be given. Let $(\hat \mu_{h+1}^j)_{j=1}^m$ denote an $(\epapx, \epneg, \Cemp; \Gamma)$-\trunccore (\cref{def:trunc-core}) at step $h$.

  Suppose that $N \geq \max(\nstat(\xi/(8\Cemp), \alpha, \delta\xi/(6\Cemp)), \nfe(\xi/(8\Cemp), \delta\xi/(6\Cemp)))$, and that
  \begin{align}
 \frac{17\Cemp^2}{\xi^2}\epapx + \frac{33\Cemp^3}{\xi^2} \epneg + {2}\xi \leq \frac{\trunc}{2A}.\label{eq:trunc-param-constraint}
  \end{align}
  Let $(\Psiapx_{h+1}, \pifinal)$ denote the output of \PC (\cref{alg:pc}) with parameters $\xi, \Cemp, N$. Define
  \begin{align}
\Psi_{h+2} := \{ \pi \circ_{h+1} \unif(\MA) \ : \ \pi \in \Psiapx_{h+1} \}.\label{eq:define-psih2-lemma}
  \end{align}
  Then with probability at least $1-\delta$, \textbf{at least one} of the following two statements holds:
  \begin{enumerate}[label=(\alph*)]
  \item \label{it:found-trunc-pc} $\Psi_{h+2}$ is a $\xi^2/(32\Cemp^2 A)$-truncated policy cover at step $h+2$ (\cref{def:ih-trunc}). 
  \item \label{it:found-gamma-violation} $\Delta_h(\Psi_{1:h},\Psiapx_{h+1},\pifinal; \Gamma) \geq \frac{\xi \alpha}{8AH^2 \Cemp}$

  \end{enumerate}
\end{lemma}
\begin{proof}
Let us choose states $x_{h+1}^1, \ldots, x_{h+1}^m \in \MX$ per \cref{it:pol-approx} of \cref{def:trunc-core} so that for any $\pi \in \Pi$ and $w \in \BR^d \times \{ 0\}$,
  \begin{align}
   &  \left| \sum_{x \in \MX} \lng \E\sups{\Mbar(\emptyset), \pi}[\phi_h(x_h, a_h)], \mu_{h+1}(x) \rng  \lng \phiavg_{h+1}(x), w \rng -\sum_{j=1}^m \lng \E\sups{\Mbar(\emptyset), \pi}[\phi_h(x_h, a_h)], \hat \mu_{h+1}^j \rng  \lng  \phiavg_{h+1}(x_{h+1}^j), w \rng \right| \nonumber\\
    &\leq  \epapx \| w \|_1\label{eq:hatmu-approx-ep0-trunc0}\\
    &  \left| \sum_{x \in \MX} \lng \E\sups{\Mbar(\Gamma), \pi}[\phi_h(x_h, a_h)], \mu_{h+1}(x) \rng  \lng \phiavg_{h+1}(x), w \rng -\sum_{j=1}^m \lng \E\sups{\Mbar(\Gamma), \pi}[\phi_h(x_h, a_h)], \hat \mu_{h+1}^j \rng  \lng  \phiavg_{h+1}(x_{h+1}^j), w \rng \right|\nonumber\\
    &\leq \epapx \| w \|_1\label{eq:hatmu-approx-ep0-truncg}.
  \end{align}
  By choice of $\xi,\epneg, N$, the assumption that $\Psi_{1:h}$ are $\alpha$-truncated policy covers for steps $1,\dots,h$, and the fact that $(\hat\mu_{h+1}^j)_{j=1}^m$ is a \trunccore, we can invoke \cref{lem:newgamma-or-psigood}. By \cref{lem:newgamma-or-psigood}, there is an event $\ME$ that occurs with probability at least $1-\delta$ (over the randomness in \PC) so that, under $\ME$, either the output $(\Psiapx_{h+1}, \pifinal)$ of \PC satisfies $\Delta_h(\Psi_{1:h}, \Psiapx_{h+1}, \pifinal;\Gamma) \geq \frac{\xi\alpha}{8AH^2 \Cemp}$, or else \cref{eq:aux-coreset-trunc} holds. The former case is exactly \cref{it:found-gamma-violation} of the lemma statement.
  So let us assume from here on that $\Psiapx_{h+1}$ satisfies \cref{eq:aux-coreset-trunc}.

  Fix any $\pi \in \Pi$ and $x' \in \MX$. Then we may compute
  \begin{align}
    & \lng \E\sups{\Mbar(\emptyset), \pi}[\phiavg_{h+1}(x_{h+1})], \mu_{h+2}(x') \rng \nonumber\\
    &= \sum_{x \in \MX} \lng \E\sups{\Mbar(\emptyset), \pi}[ \phi_h\sups{\Mbar(\emptyset)}(x_h, a_h)], \mu_{h+1}\sups{\Mbar(\emptyset)}(x) \rng \cdot \lng \phiavg_{h+1}(x), \mu_{h+2}(x') \rng \nonumber\\
    &= \sum_{x \in \MX} \lng \E\sups{\Mbar(\emptyset), \pi}[ \phi_h(x_h, a_h)], \mu_{h+1}(x) \rng \cdot \lng \phiavg_{h+1}(x), \mu_{h+2}(x') \rng \nonumber\\
    &\leq  \epapx \cdot \| \mu_{h+2}(x') \|_1 + \sum_{j=1}^m \lng \E\sups{\Mbar(\emptyset),\pi}[\phi_h(x_h, a_h)], \hat \mu_{h+1}^j \rng \cdot \lng  \phiavg_{h+1}(x_{h+1}^j), \mu_{h+2}(x') \rng \nonumber\\
    &\leq  \epapx \cdot \| \mu_{h+2}(x') \|_1 + \frac{33\Cemp^3 \norm{\mu_{h+2}(x')}_1 \epneg}{\xi^2} + 2\xi \cdot \| \mu_{h+2}(x') \|_1 \nonumber\\
    &\qquad+ \frac{8\Cemp}{\xi} \cdot  \sum_{\pi' \in \Psiapx_{h+1}} \sum_{j=1}^m  \lng \E\sups{\Mbar(\Gamma),\pi'}[\phi_h(x_h, a_h)], \hat \mu_{h+1}^j \rng \cdot \lng  \phiavg_{h+1}(x_{h+1}^j), \mu_{h+2}(x') \rng \nonumber\\
    &\leq \epapx \cdot \| \mu_{h+2}(x') \|_1 + \frac{33\Cemp^3 \norm{\mu_{h+2}(x')}_1 \epneg}{\xi^2} + {2}\xi \cdot \| \mu_{h+2}(x') \|_1 \nonumber\\
    &\qquad + \frac{8\Cemp}{\xi} \sum_{\pi' \in \Psiapx_{h+1}} \left( \epapx \cdot \| \mu_{h+2}(x') \|_1 + \sum_{x \in \MX} \lng \E\sups{\Mbar(\Gamma),\pi'}[\phi_h(x_h, a_h)], \mu_{h+1}(x) \rng \cdot \lng \phiavg_{h+1}(x), \mu_{h+2}(x') \rng\right)\nonumber\\
    &\leq \left( \frac{17\Cemp^2}{\xi^2}\epapx+ \frac{33\Cemp^3 \epneg}{\xi^2} + 2\xi \right) \cdot \| \mu_{h+2}(x') \|_1 \nonumber\\
    &\qquad+ \frac{8\Cemp}{\xi} \sum_{\pi' \in \Psiapx_{h+1}} \sum_{x \in \MX} \lng \E\sups{\Mbar(\Gamma),\pi'}[\phi_h(x_h, a_h)], \mu_{h+1}(x) \rng \cdot \lng \phiavg_{h+1}(x), \mu_{h+2}(x') \nonumber \\
    &= \left( \frac{17\Cemp^2}{\xi^2}\epapx+ \frac{33\Cemp^3 \epneg}{\xi^2} + 2\xi \right) \cdot \| \mu_{h+2}(x') \|_1 \nonumber\\
    &\qquad+  \frac{8\Cemp}{\xi} \cdot \sum_{\pi' \in \Psiapx_{h+1}} \lng \E\sups{\Mbar(\Gamma),\pi'}[\phiavg_{h+1}(x_{h+1})], \mu_{h+2}(x') \rng \label{eq:vis-bound-1},
  \end{align}
  where the first equality uses the definition of the transition dynamics of $\Mbar(\emptyset)$ as well as the fact that $\langle \phiavg_{h+1}(\term), \mu_{h+2}(x')\rangle = 0$ for all $x'\in\MX$; the second equality uses that $h+1$ is even, so $\Xreach_{h+1}(\emptyset) = \MX$ and hence $\langle \phi\sups{\Mbar(\emptyset)}_h(x,a), \mu\sups{\Mbar(\emptyset)}_{h+1}(x')\rangle = \langle \phi_h(x,a), \mu_{h+1}(x')\rangle$ for all $x,x'\in\MX$, $a \in \MA$; the first inequality uses \cref{eq:hatmu-approx-ep0-trunc0} and the fact that $\mu_{h+2}(x') \in \BR^d \times \{ 0 \}$ as $x' \in \MX$; the second inequality uses \cref{eq:aux-coreset-trunc} with the function $g(j) := \langle \phiavg_{h+1}(x_{h+1}^j),\mu_{h+2}(x')\rangle$ (note that $0 \leq g(j) \leq \norm{\mu_{h+2}(x')}_1$ for all $j \in [m]$ and $x' \in \MX$); the third inequality uses \cref{eq:hatmu-approx-ep0-truncg}; the fourth inequality collects terms and uses the bounds $|\Psiapx_{h+1}| \leq 2\Cemp/\xi$ and $1 \leq \Cemp/\xi$; and the final equality uses the definition of the transition dynamics of $\Mbar(\Gamma)$ as well as the facts that $\Xreach_{h+1}(\Gamma)=\MX$ and $\langle\phiavg_{h+1}(\term),\mu_{h+2}(x')\rangle=0$.

Next, since $h+2$ is odd (and at least $2$), by \cref{lem:reachability-in-trunc} with $\Gamma = \emptyset$, we have that for any $x' \in \Xreach_{h+2}(\emptyset)$,
  \begin{align}
   \max_{\pi \in \Pi} \lng \E\sups{\Mbar(\emptyset), \pi}[\phiavg_{h+1}(x_{h+1})], \mu_{h+2}(x') \rng 
   &=  \max_{\pi \in \Pi} d^{\Mbar(\emptyset),\pi\circ_{h+1}\unif(\MA)}_{h+2}(x') \nonumber\\
    &\geq   \frac{1}{A}\max_{\pi \in \Pi} d\sups{\Mbar(\emptyset),\pi}_{h+2}(x') \nonumber\\
    &\geq   \frac{\trunc}{A} \cdot \| \mu_{h+2}(x') \|_1\nonumber.
  \end{align}
  Thus, for any $x' \in \Xreach_{h+2}(\emptyset)$, the additive error term in \cref{eq:vis-bound-1} is bounded as follows:
  \begin{align}
    \left( \frac{17\Cemp^2}{\xi^2}\epapx + \frac{33\Cemp^3 \epneg}{\xi^2} + {2}\xi \right) \cdot \| \mu_{h+2}(x') \|_1 
    &\leq  \frac{\trunc}{2A} \cdot \| \mu_{h+2}(x') \|_1\nonumber\\
    &\leq \frac 12 \cdot  \max_{\pi \in \Pi} \lng \E\sups{\Mbar(\emptyset), \pi}[\phiavg_{h+1}(x_{h+1})], \mu_{h+2}(x') \rng\nonumber,
  \end{align}
  where the first inequality uses \cref{eq:trunc-param-constraint}. Using the above display in \cref{eq:vis-bound-1} and rearranging  terms, we obtain that for all $x' \in \Xreach_{h+2}(\emptyset)$,
  \begin{align}
    \max_{\pi \in \Pi} \lng \E\sups{\Mbar(\emptyset), \pi} [\phiavg_{h+1}(x_{h+1})], \mu_{h+2}(x') \rng 
    &\leq  \frac{16\Cemp}{\xi} \cdot \sum_{\pi' \in \Psiapx_{h+1}} \lng \E\sups{\Mbar(\Gamma), \pi'}[\phiavg_{h+1}(x_{h+1})], \mu_{h+2}(x')\rng \nonumber\\
    &=  \frac{16\Cemp}{\xi} \cdot \sum_{\pi' \in \Psiapx_{h+1}} \lng \E\sups{\Mbar(\Gamma), \pi' \circ_{h+1} \unif(\MA)}[\phi_{h+1}(x_{h+1}, a_{h+1})], \mu_{h+2}(x')\rng \nonumber\\
    &= \frac{16\Cemp}{\xi} \cdot \sum_{\pi' \in \Psi_{h+2}} \lng \E\sups{\Mbar(\Gamma), \pi'}[\phi_{h+1}(x_{h+1}, a_{h+1})], \mu_{h+2}(x')\rng \nonumber\\
    &\leq  \frac{32\Cemp^2}{\xi^2} \cdot \frac{1}{|\Psi_{h+2}|} \cdot  \sum_{\pi' \in \Psi_{h+2}} \lng \E\sups{\Mbar(\Gamma), \pi'}[\phi_{h+1}(x_{h+1}, a_{h+1})], \mu_{h+2}(x')\rng\nonumber,
  \end{align}
where the first equality uses the definition of $\phiavg_{h+1}$, the second equality uses the definition of $\Psi_{h+2}$ in \cref{eq:define-psih2-lemma}, and the final inequality uses that $|\Psi_{h+2}| = |\Psiapx_{h+1}| \leq 2\Cemp/\xi$ (\cref{lem:emp-pc-guarantee}). Finally, for any $x' \in \MX$ and $\pi \in \Pi$ we know that
\[\lng \E\sups{\Mbar(\emptyset), \pi}[ \phiavg_{h+1}(x_{h+1})], \mu_{h+2}(x')\rng \geq \frac{1}{A} \lng \E\sups{\Mbar(\emptyset), \pi}[\phi_{h+1}(x_{h+1},a_{h+1})], \mu_{h+2}(x')\rng.\]
It follows that, for all $x' \in \Xreach_{h+2}(\emptyset)$, 
\begin{align}
  &\max_{\pi \in \Pi} \lng \E\sups{\Mbar(\emptyset), \pi}[\phi_{h+1}(x_{h+1},a_{h+1})], \mu_{h+2}(x')\rng \nonumber\\
  &\leq \frac{32\Cemp^2 A}{\xi^2} \frac{1}{|\Psi_{h+2}|}\sum_{\pi' \in \Psi_{h+2}} \lng \E\sups{\Mbar(\Gamma), \pi'}[\phi_{h+1}(x_{h+1},a_{h+1})], \mu_{h+2}(x') \rng \nonumber\\
  &\leq \frac{32\Cemp^2 A}{\xi^2} \frac{1}{|\Psi_{h+2}|}\sum_{\pi' \in \Psi_{h+2}} \lng \E\sups{M, \pi'}[\phi_{h+1}(x_{h+1},a_{h+1})], \mu_{h+2}(x') \rng \nonumber
\end{align}
where the last inequality uses that $d_{h+1}\sups{\Mbar(\Gamma),\pi'}(x) \leq d_{h+1}\sups{M,\pi'}(x)$ for all $x\in\MX$ (\cref{lem:gamma-monotonicity}) as well as the fact that $\lng \phi_{h+1}(x,a), \mu_{h+2}(x') \rng \geq 0$ for all $x,a,x'$. Thus, $\Psi_{h+2}$ satisfies the coverage condition \cref{eq:pc-trunc} of \cref{def:ih-trunc} for all states $x \in \Xreach_{h+2}(\emptyset)$. It remains to observe that \cref{eq:pc-trunc} holds trivially when $x \in \MX \setminus \Xreach_{h+2}(\emptyset)$, since $d\sups{\Mbar(\emptyset),\pi}_{h+2}(x) = 0$ for such $x$. This verifies \cref{it:found-trunc-pc} of the lemma statement.
\end{proof}

\subsection{Analysis of the full algorithm: \SLMt}
\label{sec:slmt-analysis}
We now analyze the full algorithm \SLMt{} (\cref{alg:slm-trunc}), which consists of $T$ phases. Each phase starts with a backup policy cover $\Gamma^t$ (which grows as $t$ increases) and seeks to construct truncated policy covers for steps $1,\dots,H$ as per \cref{def:ih-trunc}. As a consequence of \cref{lem:pc-induction-trunc}, for any given phase $t$, where the backup policy cover is $\Gamma^t$, for any step $h$ where we have already constructed policy covers $\Psi_{1:h}$ for steps $1,\dots,h$, if the extraneous visitation probability $\Delta_h(\Psi_{1:h},\Psiapx_{h+1},\pifinal;\Gamma^t)$ is small, then the induction step succeeds and we can construct a policy cover at step $h+2$. It remains to show that there is some phase $t$ during which the extraneous visitation probabilities are small at every step $h$; this is the content of \cref{lem:gamma-stop} below.
  \begin{lemma}
    \label{lem:gamma-stop}
Let $\epfinal,\delta>0$, and consider the execution of $\SLMt(\epfinal,\delta)$. Recall the definitions of $T$, $\tsmall$, and $\xi$ (\cref{line:mess}). Consider the event in which \ESCt never outputs $\perp$ in phases $1,\dots,T$. Then there is some $t \in [T]$ so that for all $h \in [H-2]$ with $h$ odd, the values of $\Gamma^t, \Psi_{1:h}^t, \Psiapxt{t}_{h+1}, \pifinals{t,h}$ maintained by $\SLMt(\epfinal,\delta)$ satisfy 
\[\Delta_h(\Psi_{1:h}^t, \Psiapxt{t}_{h+1}, \pifinals{t,h}; \Gamma^t) < \rho := \max\left(\frac{4H\Cnrm}{T}, \sqrt{\frac{16\Cnrm^3 H^3 \tsmall}{\xi}}\right).\] 
\end{lemma}

We next discuss the main idea behind the proof of \cref{lem:gamma-stop}. At any phase $t$, if some policy $\pi$ (belonging to any of the policy sets $\Sigma_h(\Psi_{1:h}^t,\Psiapxt{t}_{h+1},\pifinal)$) visits a state $x \in \MX\setminus\Xreach_g(\Gamma^t)$ with at least some probability $p \cdot \norm{\mu_g(x)}_1$, then since $\pi$ gets added to subsequent backup covers (\cref{line:append-gammat}), at every subsequent phase $t'$, a uniformly random policy from $\Gamma^{t'}$ will visit $x$ with probability at least $\frac{p}{|\Gamma^{t'}|}\cdot \norm{\mu_g(x)}_1$. If $p/|\Gamma^{t'}| \geq \tsmall$, then by definition we will have $x \in \Xreach_g(\Gamma^{t'})$, so $x$ cannot contribute to any subsequent extraneous visitation probabilities. Of course, since $|\Gamma^{t'}|$ typically grows linearly in $t'$, the inequality $p / |\Gamma^{t'}| \geq \tsmall$ cannot hold indefinitely. However, if we restrict attention to the first $O(1/\delta)$ phases and assume that $\tsmall = O(\delta^2)$ (ignoring factors of $\xi$,$\Cnrm$,$H$, etc.) we can show that any state $x$ contributes extraneous visitation probability $\Omega(\delta) \cdot \norm{\mu_g(x)}_1$ to at most one phase (for each step $g$). The sum of all other contributions at each phase can be bounded by $O(\delta \cdot \Cnrm)$. As long as $\delta$ is sufficiently small, it follows that there must be some phase in the first $O(1/\delta)$ phases which has sufficiently small extraneous visitation probability of roughly $O(\delta)$ (ignoring factors of $\xi, \Cnrm, H$, etc.). 
We emphasize that, as long as the algorithm \SLMt{} does not 
prematurely exit (which happens if the convex program in \ESCt is infeasible at some step), this guarantee holds deterministically. %

\begin{algorithm}
  \caption{$\Paramst(A,H,\Cnrm,\delta,\epfinal,d)$: Para\textbf{ME}ter \textbf{S}etting\textbf{S} for \SLMt} 
  \label{alg:slm-params}
  \begin{algorithmic}[1]\onehalfspacing
    \Require action count $A$, horizon $H$, norm param.~$\Cnrm$, failure prob.~$\delta$, error $\epfinal$, dimension $d$.%
    \State $\trunc \gets \frac{\epfinal}{8\Cnrm^3 H^4}$.
    \State $\xi \gets \frac{\trunc}{16A}$.
    \State $\alpha \gets \frac{\xi^2}{32\Cnrm^2 A^2}$. 
    \State $T \gets \frac{32AH^3\Cnrm^2}{\xi\alpha}$. 
    \State $\tsmall \gets \frac{\xi^3\alpha^2}{1024A^2H^7\Cnrm^5}$. %
    \State $\epapx \gets \frac{\trunc\xi^2}{136\Cnrm^2A}$.
    \State $\epneg \gets \frac{\trunc\xi^2}{264\Cnrm^2HA}$.
    \State $\epcvx \gets \epapx \cdot \min \left\{ \sqrt{\alpha/A}/4, \sqrt{\tsmall/A}/4 \right\}$. 
    \State $m \gets 2\max\{C_{\ref{thm:mu-coreset-trunc}},C_{\ref{lem:feasibility-trunc}}\} \cdot \frac{\Cnrm^6}{\epcvx^8}\cdot\log\left(\max\{C_{\ref{thm:mu-coreset-trunc}},C_{\ref{lem:feasibility-trunc}}\} \cdot \frac{2dTH}{\delta} \cdot \frac{\Cnrm^6}{\epcvx^8}\right)$.
    \State $n \gets \max\{C_{\ref{thm:mu-coreset-trunc}},C_{\ref{lem:feasibility-trunc}}\} \cdot \max\left\{\frac{A^2 \Cnrm^2 H^3 \log(AH\Cnrm/\epneg) \log(2dTH\sqrt{m}/\delta)}{\epneg^4 \tsmall}, \frac{16\Cnrm^2 A^2 (\tsmall^{-2} + \alpha^{-2}) \log(2dTH\sqrt{m}/\delta)}{\epapx^4} \right\}$.
    \State $N \gets \max(\nstat(\xi/(8\Cnrm), \alpha, \delta\xi/(12TH\Cnrm)), \nfe(\xi/(8\Cnrm), \delta\xi/(12TH\Cnrm)))$.\label{line:define-N-trunc}
    \State \Return $(T, \trunc, \tsmall, \alpha, \xi, \epapx, \epneg, \epcvx, n,m,N)$.
    \end{algorithmic}

\end{algorithm}

\begin{algorithm}
	\caption{$\SLMt(\epfinal, \delta)$: Explore $\ell_1$-Bounded Linear MDP}
	\label{alg:slm-trunc}
	\begin{algorithmic}[1]\onehalfspacing
		\Require Error tolerance $\epfinal$; failure probability $\delta$
		\State Let $\pi_\unif = \unif(\MA) \circ \unif(\MA) \circ \dots \circ \unif(\MA)$ be the uniform policy
        \State $(T, \trunc, \tsmall, \alpha, \xi, \epapx, \epneg, \epcvx, n,m,N) \gets \Paramst(A, H, \Cnrm, \delta, \epfinal, d)$ \label{line:mess}\Comment{\cref{alg:slm-params}}.
                \State $\Gamma^1 \gets \emptyset$
        
        \For{$1 \leq t \leq T$}\Comment{Phase $t$}
        \State Set $\Psi_1^t = \Psi_2^t = \{\pi_\unif\}$
        \For{$1 \leq h \leq H-2$ with $h$ odd}
            \State $(\hat\mu_{h+1}^{j,t})_{j=1}^m \gets \ESCt(h,\Psi_h^t,\Gamma^t,\epapx\sqrt{\alpha/(4A)},\Cnrm,n,m)$ \label{line:trunc-emulator}
            \Statex\Comment{\cref{alg:esc-trunc}}
            \If{$(\hat \mu_{h+1}^{j,t})_{j=1}^m = \perp$}\Comment{\emph{$\ESCt$ returned $\perp$}}
            \State \Return $\perp$ 
            \EndIf
            \State $(\Psiapxt{t}_{h+1}, \pifinals{t,h}) \gets \PC(h,(\hat\mu_{h+1}^{j,t})_{j=1}^m,\xi,\Cnrm,\Psi_{1:h-1}^t,N)$ \label{line:trunc-greedycover}\Comment{\cref{alg:pc}}
            \State $\Psi_{h+2}^t \gets \{\pi \circ_{h+1} \unif(\MA): \pi \in \Psiapxt{t}_{h+1}\}$ \label{line:odd-psi}
            \State $\Psi_{h+3}^t \gets  \{\pi \circ_{h+2} \unif(\MA): \pi \in \Psi_{h+2}^t\}$.\label{line:even-psi} %
            \EndFor
            \State Set $\Gamma^{t+1} \gets \Gamma^t \cup \bigcup_{h\in [H-2], h \ \mathrm{ odd}} \Sigma_h(\Psi_{1:h}^t, \Psiapxt{t}_{h+1}, \pifinals{t,h})$  \label{line:append-gammat}\Comment{(\cref{def:extra-vis-prob})}
         \EndFor
        \State \textbf{Return:} $\{\Psi_{1:H}^t\}_{1 \leq t \leq T}$, $\alpha$, $T$
	\end{algorithmic}
  \end{algorithm}

\begin{proof}[Proof of \cref{lem:gamma-stop}]
  Set $T_0 := 4H\Cnrm/\rho$ and $W := \frac{2T_0 \Cnrm H^2}{\xi}$, so that $T \geq T_0$. Recall the definition of $\Sigma_h(\Psi_{1:h}^t, \Psiapxt{t}_{h+1}, \pifinals{t,h})$ from  \cref{def:extra-vis-prob}. %
  Since
  \begin{align}
| \Sigma_h(\Psi_{1:h}^t, \Psiapxt{t}_{h+1}, \pifinals{t,h})| \leq |\Psiapxt{t}_{h+1}| + 1 + h \cdot \max_{1\leq h_0 \leq h} | \Psi_{h_0}^t| \leq \frac{2\Cnrm H}{\xi}\nonumber,
  \end{align}
  for each odd $h \in [H-2]$ and $t \in [T]$, 
  (where we have used \cref{lem:pc-size-bound} with $\Cemp = \Cnrm$, the fact that the output $(\hat\mu_{h+1}^{j,t})_{j=1}^m$ of \ESCt, if not $\perp$, always satisfies $\sum_{j=1}^m \norm{\hat\mu_{h+1}^{j,t}}_1 \leq \Cnrm$, and the fact that $|\Psi_{h_0}| = |\Psiapx_{h_0-b}|$ for some $b \in \{1,2\}$ to bound the size of $\Psiapx_{h+1}$ and $\Psi_{h_0}$ by $2\Cnrm/\xi$), it follows that $|\Gamma^t| \leq \frac{2t\Cnrm H^2}{\xi} \leq W$ for all $t \in [T_0]$.  
  
  Now suppose that the lemma statement is false, i.e. for each $t \in [T]$, there is some odd $h^t \in [H-2]$ and policy $\pi^t \in \Sigma_{h^t}(\Psi_{1:h^t}^t, \Psiapxt{t}_{h^t+1}, \pifinals{t,h^t})$ so that, for some $g^t \in [h^t] \subset [H]$, $d_{g^t}\sups{M, \pi^t}(\MX \setminus \Xreach_{g^t}(\Gamma^t)) \geq \rho$. Then for each $t$, define $\MZ^t \subset \MX \times [H]$ by 
  \[\MZ^t := \{ (x,g^t) \ : \ x \in \MX \setminus \Xreach_{g^t}(\Gamma^t) \text{ and } d_{g^t}\sups{M, \pi^t}(x) \geq W \tsmall \cdot \| \mu_{g^t}(x) \|_1 \}.\] Further, write $\til \MZ^t := \{ x \ : \ (x,g^t) \in \MZ^t\}$.  We claim that for all $t \in [T]$,
  \begin{align}
  \sum_{(x,h) \in \MZ^t} \norm{\mu_h(x)}_1 = \sum_{x \in \til\MZ^t} \| \mu_{g^t}(x) \|_1 \geq \frac{\rho}{2}.\label{eq:tilzt-p}
  \end{align}
  To see that \cref{eq:tilzt-p} holds, we write
  \begin{align}
   \rho \leq  d_{g^t}\sups{M, \pi^t}(\MX \setminus \Xreach_{g^t}(\Gamma^t)) 
   &\leq  d_{g^t}\sups{M, \pi^t}(\til \MZ^t) + d_{g^t}\sups{M, \pi^t}((\MX\setminus \Xreach_{g^t}(\Gamma^t)) \setminus \til \MZ^t) \nonumber\\
    &\leq \sum_{x \in \til \MZ^t} \| \mu_{g^t}(x) \|_1  + W\tsmall \cdot \sum_{x \in \MX} \| \mu_{g^t}(x) \|_1 \nonumber\\
    &\leq  \sum_{x \in \til \MZ^t} \| \mu_{g^t}(x) \|_1 + W\tsmall \Cnrm \nonumber\\
    &\leq \sum_{x \in \til \MZ^t} \| \mu_{g^t}(x) \|_1 + \frac{\rho}{2}\nonumber,
  \end{align}
  where the final inequality holds by our assumption that $\rho \geq \sqrt{\frac{16\Cnrm^3 H^3 \tsmall}{\xi}}$ (which implies that $W\tsmall = \frac{2T_0\Cnrm\tsmall H^2}{\xi} \leq \frac{\rho}{2\Cnrm}$). Rearranging gives \cref{eq:tilzt-p}. 
  
  Moreover, for any $1 \leq t < t' \leq T_0$, we have (by construction of $\Gamma^{t'}$) that \[\pi^t \in \Sigma_{h^t}(\Psi^t_{1:h^t},\Psiapxt{t}_{h^t+1}, \pifinals{t,h^t}) \subseteq \Gamma^{t'}.\] As a result, for any $x \in \til\MZ^t$,
  \begin{align}
\E_{\pi' \sim \unif(\Gamma^{t'})}[d_{g^t}\sups{M, \pi'}(x)] \geq \frac{1}{W} \cdot d_{g^t}\sups{M, \pi^t}(x) \geq \tsmall \cdot \| \mu_{g^t}(x) \|_1\nonumber,
  \end{align}
  which implies that $x \in \Xreach_{g^t}(\Gamma^{t'})$ (by the definition of $\Xreach_{g^t}(\Gamma^{t'})$ in \cref{eq:define-xreach-gamma}), and hence $(x, g^t) \not\in \MZ^{t'}$. %
  It follows that the sets $(\MZ^t: t \in [T_0])$ are disjoint, and hence $\sum_{t \in [T_0]} \sum_{(x,h) \in \MZ^t} \| \mu_h(x) \|_1 \leq H \Cnrm$ (where this sum is well-defined because $T_0 \leq T$).  Combining this inequality with \cref{eq:tilzt-p} gives that $T_0 \cdot(\rho/2) \leq H\Cnrm$, which contradicts the definition of $T_0$. 
\end{proof}

We also observe a few non-obvious consequences of the parameter settings of \Paramst, which will be needed when we invoke \cref{thm:mu-coreset-trunc} in the proof of \cref{theorem:main-pc-trunc}:

\begin{lemma}\label{lemma:mess-facts}
Fix $\epsilon,\delta\in(0,1)$. Consider the parameters $(T,\trunc,\tsmall,\alpha,\xi,\epapx,\epneg,\epcvx,n,m,N)$ produced by $\Paramst(A,H,\Cnrm,\delta,\epsilon,d)$ (\cref{alg:slm-params}). Then the following inequalities hold:
\[ \alpha \geq \tsmall / \trunc.\]
\[ m \geq \max\{C_{\ref{thm:mu-coreset-trunc}},C_{\ref{lem:feasibility-trunc}}\} \cdot \frac{\Cnrm^6 \log(2dTH\sqrt{m}/\delta)}{\epcvx^8}.\]
As a result, the parameters $n,m,\epapx,\epneg,\epcvx,\alpha,\delta/(2TH\sqrt{m})$ satisfy the preconditions of \cref{thm:mu-coreset-trunc,lem:feasibility-trunc}.
\end{lemma}

\begin{proof}
  For the first inequality we note that, since $\xi \leq \trunc$, we have $\tsmall/\trunc \leq \tsmall/\xi \leq \alpha$ (since $A,H,\Cnrm,T \geq 1$ and $\xi, \alpha \leq 1$). %
  For the second inequality, we use the fact that for any $x,y,C>0$, if $y = 2x\log (Cx)$ then
\[ x\log (Cy) = x\log(2Cx\log(Cx)) = x\log(Cx) + x\log(2\log(Cx)) \leq 2x\log(Cx) = y\]
where the last inequality uses that $2\log(Cx) \leq Cx$ if $Cx >0$. %
Applying this bound with $x := \max\{C_{\ref{thm:mu-coreset-trunc}}, C_{\ref{lem:feasibility-trunc}}\} \cdot \Cnrm^6/\epcvx^8$ and $C := 2dTH/\delta$ and $y := m = 2x\log(Cx)$ (by definition of $m$), we get that $m \geq x\log(Cm) \geq x\log(C\sqrt{m})$, which implies the second inequality of the lemma statement.

The remaining preconditions of \cref{thm:mu-coreset-trunc,lem:feasibility-trunc} are immediate from the definitions of $n$ and $\epcvx$, as well as the fact that we are taking the failure probability to be $\delta/(2TH\sqrt{m})$ (which is therefore trivially at most $1/\sqrt{m}$).
\end{proof}

  \begin{theorem}\label{theorem:main-pc-trunc}
  Let $\epfinal,\delta\in(0,1)$, and consider the execution of $\SLMt(\epfinal,\delta)$ (\cref{alg:slm-trunc}). Recall the definitions of $T$, $\alpha$ (\cref{line:mess}). Then with probability at least $1-\delta$, there is some $t \in [T]$ so that the output $\Psi_{1:H}^t$ of $\SLMt(\epfinal,\delta)$  satisfies that $\Psi_h^t$ is an $\alpha$-truncated policy cover for all $h \in [H]$ (\cref{def:ih-trunc}). Moreover, the sample complexity of $\SLMt(\delta)$ is $\poly(\Cnrm,A,H,\epfinal^{-1}, \log(d/\delta))$, and the time complexity is $\poly(d,\Cnrm,A,H,\epfinal^{-1},\log(1/\delta))$.
\end{theorem}

\begin{proof}
  We start by proving the following claim, which shows that every phase either produces a collection of truncated policy covers (the desired outcome) or has large extraneous visitation probability at some timestep. 
  \begin{claim}
    \label{clm:slmdp-inductive-step}
    Fix any $t \in [T]$. Conditioned on $\Gamma^t$, with probability at least $1-\delta/(2T)$ over the randomness in phase $t$, \textbf{at least one} of the following statements holds: %
  \begin{itemize}
  \item For every $h \in [H]$, $\Psi_h^t$ is an $\alpha$-truncated policy cover at step $h$. %
  \item There is some odd $h \in [H-2]$ so that $\Delta_h(\Psi_{1:h}^t, \Psiapxt{t}_{h+1}, \pifinals{t,h};\Gamma^t) \geq \frac{\xi\alpha}{8AH^2 \Cnrm}$, where $\xi$ is as defined in \cref{line:mess}.
  \end{itemize}
\end{claim}
\begin{proof}[Proof of \cref{clm:slmdp-inductive-step}]
  Fix any $t \in [T]$ and condition on $\Gamma^t$. We prove by induction that for all odd $h \in [H-2]$, the following statement (that we denote by $\mathfrak{S}(h)$) holds: with probability at least $1-(h+1)\delta/(2TH)$ over the randomness of the first $h$ steps of phase $t$, either \textbf{(i)} $\Psi_g^t$ is an $\alpha$-truncated policy cover at step $g$, for all $g \leq \min(h+3,H)$, or \textbf{(ii)} there is some odd $g \leq h$ so that $\Delta_g(\Psi_{1:g}^t,\Psiapxt{t}_{g+1},\pifinals{t,g};\Gamma^t) \geq \frac{\xi\alpha}{8AH^2\Cnrm}$.
  
  Since $\Psi_1^t$ contains the policy that takes uniformly random actions, it is a $(1/A)$-truncated policy cover for step $1$. Similarly, $\Psi_2^t$ is a $(1/A^2)$-truncated policy cover for step $2$. Since $\alpha \leq 1/A^2$, this proves the statement $\mathfrak{S}(-1)$.

  Now fix an odd value of $h \in [H-2]$ and suppose that the statement $\mathfrak{S}(h-2)$ holds. Let us condition on the first $h$ steps of phase $t$. On the one hand, if event \textbf{(ii)} holds for $h-2$, then we immediately get that \textbf{(ii)} holds for $h$. On the other hand, suppose that event \textbf{(i)} holds for $h-2$, i.e. suppose that $\Psi_g^t$ is an $\alpha$-truncated policy cover for all steps $g \leq h+1$. By \cref{thm:mu-coreset-trunc} (which is applicable by \cref{lemma:mess-facts}), the set of vectors $(\hat\mu_{h+1}^{j,t})_{j=1}^m$ is a $(\epapx,\epneg,\Cnrm; \Gamma)$-\trunccore{} for step $h$ with probability at least $1-\delta/(2TH\sqrt{m}) \geq 1-\delta/(2TH)$. Suppose that this event occurs. We can then apply \cref{lem:pc-induction-trunc} with the sets $\Psi_1^t, \ldots, \Psi_h^t, \Gamma^t \subset \Pi$ and vectors $(\hat\mu_{h+1}^{j,t})_{j=1}^m$. By the fact that $(\hat\mu_{h+1}^{j,t})_{j=1}^m$ is a $(\epapx,\epneg,\Cnrm; \Gamma)$-\trunccore{}; the choice of $N$; and the choices of $\epapx,\epneg,\xi$ (so that in particular \cref{eq:trunc-param-constraint} and the bound $\epneg \leq \xi/(4\Cnrm H)$ both hold), the preconditions of \cref{lem:pc-induction-trunc} are satisfied with parameters $\epapx,\epneg,\alpha,\Cnrm,\delta/(2TH),\xi$. Thus we get that, with probability at least $1-\delta/(2TH)$, either \textbf{(a)} $\Psi_{h+2}^t$ is a $\xi^2/(32\Cnrm^2 A)$-truncated policy cover for step $h+2$, or \textbf{(b)} $\Delta_h(\Psi_{1:h}^t, \Psiapxt{t}_{h+1}, \pifinals{t,h};\Gamma^t) \geq \frac{\xi\alpha}{8AH^2 \Cnrm}$. In the former case, since $\alpha = \frac{\xi^2}{32\Cnrm^2 A^2}$, we get that $\Psi_{h+2}^t$ and $\Psi_{h+3}^t$ (if $h+3 \leq H$) are $\alpha$-truncated policy covers for steps $h+2$ and $h+3$ respectively. Hence, \textbf{(i)} holds for $h$. In the latter case, \textbf{(ii)} holds for $h$. A union bound (over the bad event of \cref{thm:mu-coreset-trunc}, the bad event of \cref{lem:pc-induction-trunc}, and the bad event of the statement $\mathfrak{S}(h-2)$) proves the statement $\mathfrak{S}(h)$, completing the inductive step.
\end{proof}

  For each $t \in [T]$, let $\ME_t$ be the good event of \cref{clm:slmdp-inductive-step}. By \cref{clm:slmdp-inductive-step} and a union bound over $t \in [T]$, we have that $\ME_1\cap\dots\cap\ME_T$ holds with probability at least $1-\delta/2$. Next, for each $t \in [T]$ and odd $h \in [H-2]$, let $\MF_{t,h}$ be the event that \ESCt does not output $\perp$ in iteration $t$ at step $h$. For each $t \in [T]$ and odd $h \in [H-2]$, we apply \cref{lem:feasibility-trunc} with the sets $\Psi_h^t,\Gamma^t$ and parameters $\epcvx, \delta/(2TH\sqrt{m}),\alpha,n,m$. By \cref{lemma:mess-facts}, the preconditions of \cref{lem:feasibility-trunc} are satisfied. Thus, by \cref{lem:feasibility-trunc} and a union bound, we have that $\cap_{t,h} \MF_{t,h}$ holds with probability at least $1-\delta/(2\sqrt{m}) \geq 1-\delta/2$. 
  
  Consider the event in which both $\ME_1\cap\dots\cap\ME_T$ and $\cap_{t,h} \MF_{t,h}$ hold, which occurs with probability at least $1-\delta$. By \cref{lem:gamma-stop}, since \ESCt never outputs $\perp$, there is some $t^\st \in [T]$ so that for all $h \in [H-2]$ with $h$ odd, it holds that \[\Delta_h(\Psi_{1:h}^{t^\st}, \Psiapxt{t^\st}_{h+1}, \pifinals{t^\st,h};\Gamma^{t^\st}) < \max\left(\frac{4H\Cnrm}{T},\sqrt{\frac{16\Cnrm^3 H^3\tsmall}{\xi}}\right) \leq \frac{\xi\alpha}{8AH^2 \Cnrm}\]
  where the last inequality is by choice of $T$ and $\tsmall$. But since we have assumed that event $\ME_{t^\st}$ holds, it must then be that $\Psi_h^{t^\st}$ is an $\alpha$-truncated policy cover at step $h$, for all $h \in [H]$.

  It remains to analyze the sample and time complexity of \SLMt{}, which are dominated by the (at most $HT$) calls to each of \ESCt{} and \PC. By \cref{rmk:esc-sample-comp-trunc}, the sample complexity of each call to \ESCt{} is $n+m$, and by \cref{lem:pc-size-bound}, the sample complexity of each call to \PC is $O(N \Cnrm/\xi)$, where $N,n,m$ are defined in  \cref{alg:slm-params}. Altogether, the sample complexity is $O(TH \cdot (n + m + N\Cnrm/\xi))$, which is bounded above by $\poly(\Cnrm, A, H, \epfinal^{-1}, \log(d/\delta))$, where we have used the definitions of $\nstat$ and $\nfe$ in \cref{lem:psdp-trunc,lem:fe} respectively. A similar analysis shows that the overall time complexity is $\poly(d, \Cnrm, A, H, \epfinal^{-1}, \log(1/\delta))$, where we use \cref{rmk:esc-sample-comp-trunc} to bound the running time of each call to \ESCt{}. 
\end{proof}

\subsection{Analysis of \PSDP on truncated covers}

We now re-analyze \PSDP (\cref{alg:psdp}) under the weaker assumption (compared to \cref{lem:psdp}) that the inputs $\Psi_{1:k-1}$ are \emph{truncated} policy covers. As previously discussed in \cref{section:trunc-overview}, this leads to additional error terms in the suboptimality of the output policy, stemming from a generalized performance difference lemma (\cref{lem:pd-trunc}) that compares the value of a policy $\pi$ in the original MDP $M$ with the value of the optimal policy $\pi^\st$ (with respect to $M$) in the \emph{truncated} MDP $\Mbar(\emptyset)$.

This is made simpler when the rewards are non-negative, since then the value of any policy under $\Mbar(\emptyset)$ is no more than its value under $M$ (by definition of truncation, see e.g. \cref{lem:gamma-monotonicity}). In our applications of \PSDP we were not able to exactly ensure this, but we could ensure that the target reward vector satisfied the following approximate non-negativity property:

\label{sec:psdp-unreachable}
\begin{defn}[Nearly non-negative target]
  \label{def:nnnt}
  Consider $\Gamma \subset \Pi$ and $h \in [H]$. For $\ep > 0$, we say that a vector $w \in \BR^d \times \{0\}$ is a \emph{$\ep$-nearly non-negative target with respect to $\Gamma$} at step $h$ if
  \begin{align}
\max_{\pi \in \Pi} \E\sups{\Mbar(\Gamma), \pi} \left[ \max\left\{ 0, \max_{a \in \MA} \left\{ - \lng \phi_h(x_h, a), w \rng \right\} \right\} \right] \leq \ep\nonumber.
  \end{align}
\end{defn}

\begin{lemma}[\PSDP; unreachable setting]
  \label{lem:psdp-trunc}
  There is a constant $C_{\ref{lem:psdp-trunc}}>0$ so that the following holds. Fix $k \in [H]$ and $\alpha, \epnnnt,\epstat, \delta \in (0,1)$, and some $\theta \in \BR^d\times \{0\}$, as well as the following:
  \begin{itemize}
  \item $\Psi_1, \ldots, \Psi_{k-1} \subset \Pi$, so that each $\Psi_h$ is an $\alpha$-truncated cover at step $h$ (per \cref{eq:pc-trunc}).
  \item $\Gamma \subset \Pi$, so that $\theta$ is an $\epnnnt$-nearly non-negative target at step $k$ with respect to $\Gamma$ (per \cref{def:nnnt}).
  \end{itemize}
  Fix $N \in \NN$ such that 
  \begin{align}
    N \geq \nstat(\epstat,\alpha,\delta) := \frac{C_{\ref{lem:psdp-trunc}}H^4 A^2\Cnrm^4 \log(Hd/\delta)}{\alpha^4\epstat^4}.\label{eq:psdp-N-choice}
  \end{align}
  Then the output of the algorithm $\PSDP(d+1,k,\Cnrm,\theta,\Psi_{1:k-1},N)$ (\cref{alg:psdp}) is a policy $\hat \pi := \hat \pi^{1:k}$ that, with probability at least $1-\delta$, satisfies
  \begin{align}
    \lng \E\sups{M, \hat \pi} [\phi_k(x_k, a_k)], \theta \rng \geq &  \max_{\pi \in \Pi} \lng \E\sups{\Mbar(\emptyset), \pi}[\phi_k(x_k, a_k)], \theta \rng - \epstat\norm{\theta}_1 - \epnnnt H \nonumber\\
    & - \frac{2A \| \theta \|_1}{\alpha} \sum_{h=1}^k \sum_{g=h}^k \E_{\pi' \sim \unif(\Psi_{h-1})} d_g\sups{M, \pi' \circ_{h-1}\unif(\MA)\circ_h \hat\pi}(\MX \backslash \Xreach_g(\Gamma))\nonumber.
  \end{align}
\end{lemma}

While the guarantee of \cref{lem:psdp-trunc} depends on the backup cover $\Gamma$, notice that the algorithm \PSDP does not. 

\begin{proof}
  Let $\bfr = (\bfr_1,\dots,\bfr_H)$ be the reward function which is given by $(x,a) \mapsto \lng \theta, \phi_k(x,a) \rng$ at step $k$ and $0$ at all other steps.

  Fix $h \in [k]$. By \cref{lemma:linear-q}, there is some vector $\bw_h^\st \in \BR^d\times\{0\}$ with $\| \bw_h^\st \|_1 \leq \| \theta \|_1 \cdot \Cnrm$ so that $Q_h\sups{M, \hat \pi^{h+1:k},\bfr}(x,a) = \lng \phi_h(x,a), \bw_h^\st\rng$ for all $x \in \MX$ and $a \in \MA$. Define $\vep_0 \in (0,1)$ by $\vep_0^2 := \frac{12C_{\ref{cor:random-design-prediction-error}} \cdot \Cnrm^2 \| \theta \|_1^2 \sqrt{\log(Hd/\delta)}}{\sqrt N}$. Since $\E\sups{M, \hat \pi^{h+1:k}}[\lng \phi_k(x_k, a_k), \theta \rng | (x_h, a_h) = (x,a)] = Q_h\sups{M, \hat \pi^{h+1:k},\bfr}(x,a)$, it follows
  by \cref{cor:random-design-prediction-error} that, for some event $\ME_h$ that holds with probability at least $1-\delta/H$,
  \begin{align}
\E_{\pi' \sim \unif(\Psi_h)} \E\sups{M, \pi'\circ_h \unif(\MA)} \left[ \lng \phi_h(x_h, a_h), \bw_h^\st - \hat \bw_h \rng^2 \right] \leq \vep_0^2/4\label{eq:good-psdp-event}.
  \end{align}
  For each $x \in \MX$ and $h \in [k]$, let us write $\Delta_h(x) := \max_{a \in \MA} | \lng \phi_h(x,a), \bw_h^\st - \hat \bw_h \rng |$. It follows from \cref{eq:good-psdp-event} that
  \begin{align}
\E_{\pi' \sim \unif(\Psi_h)} \E\sups{M, \pi'}[\Delta_h(x_h)^2] \leq \E_{\pi' \sim \unif(\Psi_h)} \E\sups{M, \pi'}\left[ \sum_{a \in \MA} \lng \phi_h(x_h, a), \bw_h^\st - \hat \bw_h \rng^2 \right] \leq A \cdot\vep_0^2/4\nonumber,
  \end{align}
  which yields, via Jensen's inequality, that $\E_{\pi' \sim \unif(\Psi_h)} \E\sups{M, \pi'}[\Delta_h(x_h)] \leq \sqrt{A} \cdot \vep_0/2$. Using the assumption that each $\Psi_h$ satisfies the condition \cref{eq:pc-trunc} together with non-negativity of $\Delta_h(x)$, it follows that under the event $\ME_h$,
  \begin{align}
    \max_{\pi \in \Pi} \E\sups{\Mbar(\emptyset), \pi}[\Delta_h(x_h)] 
    &= \max_{\pi \in \Pi} \sum_{x \in \MX} d_h\sups{\Mbar(\emptyset), \pi}(x) \cdot \Delta_h(x)\nonumber\\
    &\leq \E_{\pi' \sim \unif(\Psi_h)} \sum_{x \in \MX} \frac{1}{\alpha} \cdot d_h\sups{M, \pi'}(x)\cdot \Delta_h(x)\nonumber\\
    &= \frac{1}{\alpha} \E_{\pi' \sim \unif(\Psi_h)} \E\sups{M, \pi'} [\Delta_h(x_h)] \leq \frac{\sqrt{A} \cdot \vep_0}{2\alpha}\label{eq:delta-psdp-bound}.
  \end{align}
  where we have defined $\Delta_h(\term) := 0$. Define $\pi^\st \in \Pi$ by $\pi^\st := \argmax_{\pi \in \Pi} \lng \E\sups{\Mbar(\emptyset), \pi} [\phi_k(x_k, a_k)], \theta \rng$. Then for any $x \in \MX$, we can bound
  \begin{align}
    Q_h\sups{M, \hat \pi,\bfr}(x, \pi_h^\st(x)) - Q_h\sups{M, \hat \pi,\bfr}(x, \hat \pi_h(x)) 
    &= \lng \phi_h(x, \pi_h^\st(x)), \bw_h^\st \rng - \lng \phi_h(x, \hat \pi_h(x)), \bw_h^\st \rng\nonumber\\
    &\leq  \lng \phi_h(x, \pi_h^\st(x)), \hat \bw_h \rng - \lng \phi_h(x, \hat \pi_h(x)), \bw_h^\st \rng + \Delta_h(x)\nonumber\\
    &\leq  \lng \phi_h(x, \hat \pi_h(x)), \hat \bw_h \rng - \lng \phi_h(x, \hat \pi_h(x)), \bw_h^\st \rng + \Delta_h(x)\nonumber\\
    &\leq  2\Delta_h(x)\nonumber,
  \end{align}
  where the first and third inequalities use the definition of $\Delta_h(x)$ and the second inequality uses the definition $\hat \pi_h(x) := \argmax_{a \in \MA} \lng \phi_h(x,a), \hat \bw_h \rng$. When $x=\term$ we have $Q\sups{M,\hat\pi,\bfr}(\term,a) = 0$ for all $a\in\MA$, so the above inequality still holds. Then by \cref{lem:pd-trunc}, we conclude that, under the event $\bigcap_{h=1}^k \ME_h$ (which occurs with probability at least $1-\delta$), 
  \begin{align}
    & \lng \E\sups{\Mbar(\emptyset), \pi^\st}[\phi_k(x_k, a_k)], \theta \rng - \lng \E\sups{M, \hat \pi}[\phi_k(x_k, a_k)], \theta \rng \nonumber\\
    &\leq  \epnnnt k +  \sum_{h=1}^{k} \E\sups{\Mbar(\emptyset), \pi^\st}[Q_h\sups{M, \hat\pi,\bfr}(x_h, a_h) - V_h\sups{M, \hat\pi,\bfr}(x_h)] \nonumber\\
    &\qquad+ \frac{2A \| \theta \|_1}{\alpha} \sum_{h=1}^k \sum_{g=h}^k \E_{\pi' \sim \unif(\Psi_{h-1})} d_g\sups{M, \pi' \circ_{h-1}\unif(\MA)\circ_h \hat\pi}(\MX \backslash \Xreach_g(\Gamma))\nonumber\\
    &=  \epnnnt k +  \sum_{h=1}^{k} \E\sups{\Mbar(\emptyset), \pi^\st}[Q_h\sups{M, \hat\pi,\bfr}(x_h, \pi_h^\st(x_h)) - Q_h\sups{M, \hat\pi,\bfr}(x_h, \hat \pi_h(x_h))] \nonumber\\
    &\qquad+ \frac{\norm{\theta}_1 2A}{\alpha} \sum_{h=1}^k \sum_{g=h}^k \E_{\pi' \sim \unif(\Psi_{h-1})} d_g\sups{M, \pi' \circ_{h-1}\unif(\MA)\circ_h \hat\pi}(\MX \backslash \Xreach_g(\Gamma))\nonumber\\
    &\leq   \epnnnt k +  \sum_{h=1}^{k} \E\sups{\Mbar(\emptyset), \pi^\st}[2\Delta_h(x_h)] + \frac{\norm{\theta}_1 2A}{\alpha} \sum_{h=1}^k \sum_{g=h}^k \E_{\pi' \sim \unif(\Psi_{h-1})} d_g\sups{M, \pi'\circ_{h-1}\unif(\MA) \circ_h \hat\pi}(\MX \backslash \Xreach_g(\Gamma))\nonumber\\
    &\leq  \epnnnt k + \frac{\vep_0 H \sqrt{A}}{\alpha} +\frac{\norm{\theta}_1 2A}{\alpha} \sum_{h=1}^k \sum_{g=h}^k \E_{\pi' \sim \unif(\Psi_{h-1})} d_g\sups{M, \pi'\circ_{h-1}\unif(\MA) \circ_h \hat\pi}(\MX \backslash \Xreach_g(\Gamma)) \nonumber
  \end{align}
  where the last inequality is by \cref{eq:delta-psdp-bound}. Substituting in the definition of $\vep_0$ and the definition of $N$ in \cref{eq:psdp-N-choice} completes the proof, as long as the constant $C_{\ref{lem:psdp-trunc}}$ is chosen sufficiently large. 
\end{proof}

The below lemma is used in the proof of \cref{lem:psdp-trunc}, and can be viewed as a variant of the performance difference lemma. 
\begin{lemma}
  \label{lem:pd-trunc}
  Consider any $k \in [H]$ and vector $w \in \BR^d\times\{0\}$. Also fix the following:
  \begin{itemize}
  \item $\Psi_1, \ldots, \Psi_{k-1} \subset \Pi$ such that each $\Psi_h$ is an $\alpha$-truncated cover at step $h$ (per \cref{eq:pc-trunc}). 
  \item $\Gamma \subset \Pi$ such that $w$ is an $\epnnnt$-nearly non-negative target at step $k$ with respect to $\Gamma$ (per \cref{def:nnnt}).
  \end{itemize}
  Let $\bfr = (\bfr_1,\dots,\bfr_H)$ denote the reward function defined by $(x,a) \mapsto \lng w, \phi_k(x,a) \rng$ at step $k$ and $0$ at all other steps. Then for any $\pi,\pi^\st \in \Pi$,
  \begin{align}
    & \lng w, \E\sups{\Mbar(\emptyset), \pi^\st}[\phi_k(x_k, a_k)] \rng - \lng w, \E\sups{M, \pi}[\phi_k(x_k, a_k)]\rng \nonumber\\
    &\leq \epnnnt k +  \sum_{h=1}^{k} \E\sups{\Mbar(\emptyset), \pi^\st}[Q_h\sups{M, \pi,\bfr}(x_h, a_h) - V_h\sups{M, \pi,\bfr}(x_h)] \nonumber\\
    &\qquad+ \frac{2A \| w \|_1}{\alpha} \sum_{h=1}^k \sum_{g=h}^k \E_{\pi' \sim \unif(\Psi_{h-1})} d_g\sups{M, \pi' \circ_{h-1}\unif(\MA)\circ_h \pi}(\MX \backslash \Xreach_g(\Gamma))\nonumber.
  \end{align}
\end{lemma}
\begin{proof}[Proof of \cref{lem:pd-trunc}]
 Note that, for each $h < k$, each $x \in \bar\MX$, $a \in \MA$, and each $\pi \in \Pi$,
  \begin{align}
Q_h\sups{M, \pi,\bfr}(x,a) = \E_{x' \sim \BP_h\sups{M}(x,a)} [V_{h+1}\sups{M, \pi,\bfr}(x')] = \E_{x'\sim \BP_h\sups{\Mbar_h}(x,a)}[V_{h+1}\sups{M,\pi,\bfr}(x')]\label{eq:qv-eq}.
  \end{align}
  where the last equality is by \cref{fact:trunc-intermediate-trans}.

  Consider policies $\pi, \pi^\st$ as in the lemma statement. For $1 \leq h \leq H$, write $\Mbar_h := \Mbar_h(\emptyset)$. Then $\Mbar_1 = M$ and $\Mbar_{H} = \Mbar$. As a matter of convention, we further write $\Mbar_0 := M$. By telescoping, we have
  \begin{align}
    & \lng w, \E\sups{\Mbar, \pi^\st}[\phi_k(x_k, a_k)] \rng - \lng w, \E\sups{M, \pi}[\phi_k(x_k, a_k)]\rng\nonumber\\
    &= \E\sups{\Mbar, \pi^\st}[Q_k\sups{M, \pi,\bfr}(x_k, a_k)] - V_1\sups{M, \pi,\bfr}(x_1)\nonumber\\
    &= \E\sups{\Mbar, \pi^\st}[Q_k\sups{M, \pi,\bfr}(x_k, a_k)] - V_1\sups{M, \pi,\bfr}(x_1) + \sum_{h=1}^{k-1} \E\sups{\Mbar_{h}, \pi^\st}[Q_h\sups{M, \pi,\bfr}(x_h, a_h)] - \E\sups{\Mbar_{h}, \pi^\st}[V_{h+1}\sups{M, \pi,\bfr}(x_{h+1})]  \nonumber\\
    &= \sum_{h=1}^k \E\sups{\Mbar_{h}, \pi^\st}[Q_h\sups{M, \pi,\bfr}(x_h, a_h)] - \E\sups{\Mbar_{h-1}, \pi^\st}[V_h\sups{M, \pi,\bfr}(x_h)]\nonumber\\
    &= \sum_{h=1}^{k} \left( \E\sups{\Mbar_{h}, \pi^\st}[V_h\sups{M, \pi,\bfr}(x_h)] - \E\sups{\Mbar_{h-1}, \pi^\st}[V_h\sups{M, \pi,\bfr}(x_h)] \right) + \sum_{h=1}^{k} \E\sups{\Mbar_{h}, \pi^\st}[Q_h\sups{M, \pi,\bfr}(x_h, a_h) - V_h\sups{M, \pi,\bfr}(x_h)]\label{eq:pd-decomposition}
  \end{align}
  where the second equality is by \cref{eq:qv-eq}. By \cref{fact:trunc-intermediate-dist}, the second summation above is precisely the second term in the claimed bound of the lemma statement. We proceed to bound each term of the first summation above. If $\bfr$ were exactly non-negative, then every term in the first summation would be non-positive (since $\Mbar_h$ is ``more truncated'' than $\Mbar_{h-1}$), but since we only assume approximate non-negativity (in the sense specified by the second lemma hypothesis), bounding these terms requires more work and leads to the additional error terms in the lemma statement. For each $1 \leq h \leq k$, we use the following triangle inequality:

  \begin{align*}
  &\E\sups{\Mbar_{h}, \pi^\st}[V_h\sups{M, \pi}(x_h)] - \E\sups{\Mbar_{h-1}, \pi^\st}[V_h\sups{M, \pi}(x_h)]\\
  &\leq \underbrace{\E\sups{\Mbar_h,\pi^\st}[V\sups{\Mbar(\Gamma),\pi}_h(x_h)] - \E\sups{\Mbar_{h-1},\pi^\st}[V\sups{\Mbar(\Gamma),\pi}_h(x_h)]}_{\S} \\  
  &+\underbrace{\left|\E\sups{\Mbar_h,\pi^\st}[V\sups{M,\pi}_h(x_h)] - \E\sups{\Mbar_h,\pi^\st}[V\sups{\Mbar(\Gamma),\pi}_h(x_h)]\right|}_{\ddag}
  + \underbrace{\left|\E\sups{\Mbar_{h-1},\pi^\st}[V\sups{M,\pi}_h(x_h)] - \E\sups{\Mbar_{h-1},\pi^\st}[V\sups{\Mbar(\Gamma),\pi}_h(x_h)]\right|}_\dag
  \end{align*}
 where $V_h\sups{\Mbar(\Gamma),\pi}(x)$ denotes the $V$-value function for the MDP $\Mbar(\Gamma)$, policy $\pi$, and reward given by $(x,a) \mapsto \langle w,\phi_k(x,a)\rangle$ for all $(x,a) \in \MX\times\MA$ at step $k$, reward $0$ for $x=\term$, and reward $0$ for all other steps. (In particular, in the above inequality and in the remainder of the proof of this lemma, we omit the superscript $\bfr$ in the value functions.) We proceed to bound terms $\S$, $\ddag$, and $\dag$ individually.
 
 \paragraph{Bounding $\ddag$ and $\dag$.} For any $h \in [H]$ and $x \in \MX$, and all $\pi \in \Pi$, we have that
  \begin{align}
| V_h\sups{M, \pi}(x) - V_h\sups{\Mbar(\Gamma), \pi}(x) | &\leq  \| w \|_1 \sum_{g=h+1}^k \BP\sups{M, \pi}(x_g \in \MX \backslash \Xreach_g(\Gamma) \ | \ x_h = x)\label{eq:vf-trunc-error},
  \end{align}
  where we have used the fact that $|\lng w, \phi_k(x_k, a_k)\rng| \leq \| w \|_1$ (see the proof of \cref{lem:m-mbar-delta}). Moreover $V_h\sups{M,\pi}(\term) = V_h\sups{\Mbar(\Gamma),\pi}(\term) = 0$. Thus,
  \begin{align}
    & \left| \E\sups{\Mbar_{h-1}, \pi^\st}[V_h\sups{M, \pi}(x_h)] - \E\sups{\Mbar_{h-1}, \pi^\st}[V_h\sups{\Mbar(\Gamma), \pi}(x_h)]\right| \nonumber\\
    &\leq \| w \|_1 \cdot \sum_{x \in \MX}d_h\sups{\Mbar_{h-1}, \pi^\st}(x) \sum_{g=h+1}^k   \BP\sups{M, \pi}(x_g \in \MX \backslash \Xreach_g(\Gamma) \ | \ x_h = x)\nonumber\\
    &\leq \frac{\| w \|_1 A}{\alpha}\cdot \EE_{\pi'\sim\unif(\Psi_{h-1})}\sum_{x \in \MX} d_h\sups{M,\pi'\circ_{h-1}\unif(\MA)}(x) \sum_{g=h+1}^k   \BP\sups{M, \pi}(x_g \in \MX \backslash \Xreach_g(\Gamma) \ | \ x_h = x)\nonumber\\
    &= \frac{\| w \|_1 A}{\alpha} \cdot \E_{\pi' \sim \unif(\Psi_{h-1})} \sum_{g=h+1}^k d_g\sups{M, \pi' \circ_{h-1} \unif(\MA) \circ_h \pi}(\MX \backslash \Xreach_g(\Gamma))\label{eq:vm-mtil-2},
  \end{align}
  where the second inequality uses \cref{lemma:indhyp-shifted-trunc} (note that $\Mbar_0  =\Mbar_1$, so when $h=1$ the desired inequality still holds).
  By an identical argument, we also have that
  \begin{align}
 \left| \E\sups{\Mbar_{h}, \pi^\st}[V_h\sups{M, \pi}(x_h)] - \E\sups{\Mbar_{h}, \pi^\st}[V_h\sups{\Mbar(\Gamma), \pi}(x_h)]\right|  & \leq \frac{A \| w \|_1}{\alpha} \cdot \E_{\pi' \sim \unif(\Psi_{h-1})} \sum_{g=h+1}^k d_g\sups{M, \pi' \circ_{h-1}\unif(\MA) \circ_h \pi}(\MX \backslash \Xreach_g(\Gamma))\label{eq:vm-mtil-1}.
  \end{align}
  \paragraph{Bounding $\S$.} We finally bound $\S$ using the assumption that $w$ is an $\epnnnt$-nearly non-negative target at step $k$ with respect to $\Gamma$ (\cref{def:nnnt}). For any $1 \leq h \leq k$, since $V\sups{\Mbar(\Gamma),\pi}_h(\term) = 0$ and $d\sups{\Mbar_h,\pi^\st}_h(x) = d\sups{\Mbar_{h-1},\pi^\st}_h(x)$ for all $x \in \Xreach_h(\emptyset)$ (and on the other hand $d\sups{\Mbar_h,\pi^\st}_h(x) = 0$ for all $x \in \MX\setminus\Xreach_h(\emptyset)$), we have that
  \begin{align}
    & \E\sups{\Mbar_{h}, \pi^\st}[V_h\sups{\Mbar(\Gamma), \pi}(x_h)] - \E\sups{\Mbar_{h-1}, \pi^\st}[V_h\sups{\Mbar(\Gamma), \pi}(x_h)] \nonumber\\
    &= \sum_{x \in \MX\backslash \Xreach_h(\emptyset)} -d_h\sups{\Mbar_{h-1}, \pi^\st}(x) \cdot V_h\sups{\Mbar(\Gamma), \pi}(x)\nonumber\\
    &= \sum_{x \in \MX\backslash \Xreach_h(\Gamma)} -d_h\sups{\Mbar_{h-1}, \pi^\st}(x) \cdot V_h\sups{\Mbar(\Gamma), \pi}(x) + \sum_{x \in \Xreach_h(\Gamma) \backslash \Xreach_h(\emptyset)} -d_h\sups{\Mbar_{h-1}, \pi^\st}(x) \cdot V_h\sups{\Mbar(\Gamma), \pi}(x)\nonumber
  \end{align}
  
  To bound the first term, we simply note that $|V_h\sups{\Mbar(\Gamma),\pi}(x)| \leq \norm{w}_1$ for all $x \in \MX$ and $\pi \in \Pi$, so that
  \begin{align*}
  \sum_{x \in \MX\backslash \Xreach_h(\Gamma)} -d_h\sups{\Mbar_{h-1}, \pi^\st}(x) \cdot V_h\sups{\Mbar(\Gamma), \pi}(x) 
  &\leq \| w \|_1 \cdot d_h\sups{\Mbar_{h-1}, \pi^\st}(\MX\backslash \Xreach_h(\Gamma))  \\
  &\leq \frac{A \| w \|_1}{\alpha} \cdot \E_{\pi' \sim \unif(\Psi_{h-1})} d_h\sups{M, \pi'\circ_{h-1}\unif(\MA)}(\MX \backslash \Xreach_h(\Gamma))
  \end{align*}
  where the final inequality uses \cref{lemma:indhyp-shifted-trunc}. %

  To bound the second term, note that for any $x \in \MX$, 
  \begin{align}
-V_h\sups{\Mbar(\Gamma), \pi}(x) \leq \E\sups{\Mbar(\Gamma), \pi} \left[ \max\left\{ 0, \max_{a \in \MA} -\lng w, \phi_k(x_k, a) \rng \right\} \ \middle| \ x_h = x \right]\nonumber.
  \end{align}
  Hence,
  \begin{align}
  &\sum_{x \in \Xreach_h(\Gamma) \backslash \Xreach_h(\emptyset)} -d_h\sups{\Mbar_{h-1}, \pi^\st}(x) \cdot V_h\sups{\Mbar(\Gamma), \pi}(x)\nonumber \\
    &\leq \sum_{x \in \Xreach_h(\Gamma) \backslash \Xreach_h(\emptyset)} d_h\sups{\Mbar_{h-1}, \pi^\st}(x) \cdot \E\sups{\Mbar(\Gamma), \pi}\left[ \max\left\{ 0, \max_{a \in \MA} -\lng w, \phi_k(x_k,a) \rng \right\}  \ \middle| \ x_h = x\right]\nonumber\\
    &\leq \sum_{x \in \Xreach_h(\Gamma) \backslash \Xreach_h(\emptyset)} d_h\sups{\Mbar(\Gamma), \pi^\st}(x) \cdot \E\sups{\Mbar(\Gamma), \pi}\left[ \max\left\{ 0, \max_{a \in \MA} -\lng w, \phi_k(x_k,a) \rng \right\}  \ \middle| \ x_h = x\right]\nonumber\\
    &\leq \E\sups{\Mbar(\Gamma), \pi^\st \circ_h \pi}\left[ \max\left\{ 0, \max_{a\in \MA} -\lng w, \phi_k(x_k, a) \rng \right\} \right]\nonumber\nonumber\\
    &\leq \epnnnt k \nonumber
  \end{align}
  where the second inequality uses that $d_h\sups{\Mbar_{h-1}, \pi^\st}(x) \leq d_h\sups{\Mbar_{h-1}(\Gamma), \pi^\st}(x) = d_h\sups{\Mbar(\Gamma), \pi^\st}(x)$ for all $x \in \Xreach_h(\Gamma)$; and
  the final inequality uses the assumption that $w$ is an $\epnnnt$-nearly non-negative target at step $k$ with respect to $\Gamma$. We conclude that

  \begin{align}
  &\E\sups{\Mbar_{h}, \pi^\st}[V_h\sups{\Mbar(\Gamma), \pi}(x_h)] - \E\sups{\Mbar_{h-1}, \pi^\st}[V_h\sups{\Mbar(\Gamma), \pi}(x_h)] \nonumber\\
  &\leq \epnnnt k + \frac{A \| w \|_1}{\alpha} \cdot \E_{\pi' \sim \unif(\Psi_{h-1})} d_h\sups{M, \pi'\circ_{h-1}\unif(\MA)}(\MX \backslash \Xreach_h(\Gamma)). \label{eq:mtil-nonneg}
  \end{align}

  \paragraph{Putting everything together.} Combining the bounds \cref{eq:vm-mtil-2}, \cref{eq:vm-mtil-1}, and \cref{eq:mtil-nonneg} on $\dag$, $\ddag$, and $\S$, and substituting into \cref{eq:pd-decomposition}, we have that
  \begin{align}
    & \lng w, \E\sups{\Mbar, \pi^\st}[\phi_k(x_k, a_k)] \rng - \lng w, \E\sups{M, \pi}[\phi_k(x_k, a_k)]\rng\nonumber\\
    &\leq  \sum_{h=1}^{k} \E\sups{\Mbar_{h}, \pi^\st}[Q_h\sups{M, \pi}(x_h, a_h) - V_h\sups{M, \pi}(x_h)] + \sum_{h=1}^k  \left| \E\sups{\Mbar_{h-1}, \pi^\st}[V_h\sups{M, \pi}(x_h)] - \E\sups{\Mbar_{h-1}, \pi^\st}[V_h\sups{\Mbar(\Gamma), \pi}(x_h)]\right|\nonumber\\
    & \quad+ \sum_{h=1}^k  \left| \E\sups{\Mbar_{h}, \pi^\st}[V_h\sups{M, \pi}(x_h)] - \E\sups{\Mbar_{h}, \pi^\st}[V_h\sups{\Mbar(\Gamma), \pi}(x_h)]\right| + \sum_{h=1}^k \E\sups{\Mbar_{h}, \pi^\st}[V_h\sups{\Mbar(\Gamma), \pi}(x_h)] - \E\sups{\Mbar_{h-1}, \pi^\st}[V_h\sups{\Mbar(\Gamma), \pi}(x_h)] \nonumber\\
    &\leq \epnnnt k +  \sum_{h=1}^{k} \E\sups{\Mbar_{h}, \pi^\st}[Q_h\sups{M, \pi}(x_h, a_h) - V_h\sups{M, \pi}(x_h)] \nonumber \\ 
    &\qquad+ \frac{2A \| w \|_1}{\alpha} \sum_{h=1}^k \sum_{g=h}^k \E_{\pi' \sim \unif(\Psi_{h-1})} d_g\sups{M, \pi' \circ_{h-1}\unif(\MA) \circ_h \pi}(\MX \backslash \Xreach_g(\Gamma)) \nonumber.
  \end{align}
  The proof is completed by noting that $\E\sups{\Mbar_{h}, \pi^\st}[Q_h\sups{M, \pi}(x_h, a_h) - V_h\sups{M, \pi}(x_h)] = \E\sups{\Mbar, \pi^\st}[Q_h\sups{M, \pi}(x_h, a_h) - V_h\sups{M, \pi}(x_h)]$ for each $h$ (using \cref{fact:trunc-intermediate-dist} with $\Gamma=\emptyset$). 
\end{proof}
We finally prove the following helper lemma, which relates the state visitation probabilities under any policy $\pi$ in $\Mbar_h(\emptyset)$ to the visitation probabilities under a uniform policy from $\Psi_{h-1}$ in $M$, and was used in the proof of \cref{lem:pd-trunc}.

\begin{lemma}\label{lemma:indhyp-shifted-trunc}
Let $\alpha>0$ and $2 \leq h \leq H$. Suppose that $\Psi_{h-1}$ is an $\alpha$-truncated cover at step $h-1$. Then for all $x \in \MX$ and $\pi \in \Pi$, it holds that
\[ d\sups{\Mbar_h(\emptyset),\pi}_h(x) \leq d\sups{\Mbar_{h-1}(\emptyset),\pi}_{h}(x) \leq \frac{A}{\alpha} \cdot \EE_{\pi' \sim \unif(\Psi_{h-1})} d\sups{M,\pi'\circ_{h-1}\unif(\MA)}_h(x).\]
\end{lemma}

\begin{proof}
Throughout the proof of the lemma, we write $\Mbar_h := \Mbar_h(\emptyset)$ for $h \in [H]$. 
The first inequality in the lemma statement is by construction of the truncated MDPs: if $x \in \Xreach_h(\emptyset)$ then $d\sups{\Mbar_h,\pi}_h(x) = d\sups{\Mbar_{h-1},\pi}_{h}(x)$, and otherwise $d\sups{\Mbar_h,\pi}_h(x) = 0$. It remains to prove the second inequality. We have that
  \begin{align}
  d_h\sups{\Mbar_{h-1},\pi}(x) 
  &\leq A \cdot d_h\sups{\Mbar_{h-1}, \pi \circ_{h-1} \unif(\MA)}(x) \nonumber\\
  &= \sum_{z \in \MX} d_{h-1}\sups{\Mbar_{h-1},\pi}(z) \sum_{a \in \MA} \BP\sups{\Mbar_{h-1}}[x_h=x|x_{h-1}=z,a_{h-1}=a]\nonumber\\
  &\leq \frac{1}{\alpha}\sum_{z \in \MX} \EE_{\pi' \sim\unif(\Psi_{h-1})} d_{h-1}\sups{M,\pi'}(z) \sum_{a \in \MA} \BP\sups{\Mbar_{h-1}}[x_h=x|x_{h-1}=z,a_{h-1}=a]\nonumber\\
  &\leq \frac{1}{\alpha}\sum_{z \in \MX} \EE_{\pi' \sim\unif(\Psi_{h-1})} d_{h-1}\sups{M,\pi'}(z) \sum_{a \in \MA} \BP\sups{M}[x_h=x|x_{h-1}=z,a_{h-1}=a]\nonumber\\
  &= \frac{A}{\alpha} \cdot \EE_{\pi' \sim \unif(\Psi_{h-1})} d_h\sups{M,\pi'\circ_{h-1}\unif(\MA)}(x)\label{eq:indhyp-shifted-trunc}
  \end{align}
  where the second inequality uses the fact that $d_{h-1}\sups{\Mbar_{h-1},\pi} = d_{h-1}\sups{\Mbar(\emptyset),\pi}$ (\cref{fact:trunc-intermediate-dist}) together with the assumption that $\Psi_{h-1}$ is an $\alpha$-truncated cover at step $h-1$ (per \cref{eq:pc-trunc}), and the final equality uses the definition of $\pi'\circ_{h-1}\unif(\MA)$.
\end{proof}

\section{Learning a near-optimal policy}\label{sec:learning}

In this section we formally present and analyze \OPT{} (\cref{alg:opt}), our algorithm for learning a near-optimal policy in a $\ell_1$-bounded linear MDP. This will yield a proof of the following theorem, which implies \cref{thm:main-informal}. %
\begin{theorem}[Efficient learning of $\ell_1$-bounded linear MDPs]\label{thm:main}
  Let $d,\Cnrm,A,H \in \NN$ and $\epfinal,\delta>0$. Let $M$ be an $\ell_1$-bounded $d$-dimensional linear MDP (\cref{defn:l1lmdp}) with $A$ actions, planning horizon $H$, and norm bound $\Cnrm$. With probability at least $1-\delta$, the algorithm $\OPT(\epfinal,\delta)$ outputs a policy with suboptimality at most $\epfinal$. Moreover, the sample complexity of the algorithm is $\poly(\Cnrm,A,H,\epfinal^{-1},\log(d/\delta))$ and the time complexity is $\poly(d,A,H,\epfinal^{-1},\log(1/\delta))$.
\end{theorem}
The bulk of the work was in proving \cref{theorem:main-pc-trunc}, where we showed that \SLMt{} produces truncated policy covers (\cref{def:ih-trunc}) for all $h \in [H]$. It remains to show that truncated policy covers suffice for policy optimization.

\begin{algorithm}[t]
\caption{$\OPT(\epfinal,\delta)$: \textbf{Po}licy Learning with \textbf{Em}ulator-driven Exploration}
	\label{alg:opt}
	\begin{algorithmic}[1]\onehalfspacing
		\Require Error tolerance $\epfinal$; failure probability $\delta$
		\State $\{\Psi_{1:H}^t\}_{1\leq t \leq T}, \alpha, T \gets \SLMt(\epfinal,\delta/2)$.\Comment{\cref{alg:slm-trunc}}
        \State $N \gets \npsdprew(\epfinal/2, \alpha/T, \delta/2)$.\label{line:poem-N}
        \State $\hat\pi \gets \PSDPrew((\cup_{1\leq t \leq T} \Psi_h^t)_{1 \leq h \leq H}, N)$.\Comment{\cref{alg:psdp-rew}}
        \State \textbf{Return:} $\hat\pi$.
	\end{algorithmic}
  \end{algorithm}

\begin{algorithm}[t]
	\caption{$\PSDPrew(p, \Psi_{1:H},N)$: Policy Search by Dynamic Programming (variant of \cite{bagnell2003policy})}
	\label{alg:psdp-rew}
	\begin{algorithmic}[1]\onehalfspacing
		\Require Dimension $p \in \NN$; norm parameter $\Cnrm \in \BR$; policy covers $\Psi_1, \ldots, \Psi_{H}$; number of samples $N\in \mathbb{N}$.
  
		\For{$h=H, \dots, 1$} 
		\State $\MD_h \gets\emptyset$. 
		\For{$n$ times}
		\State Sample $(x_h, a_h, r_{h:H})\sim
		\unif(\Psi_h)\circ_h \unif(\MA) \circ_{h+1} \hat \pi^{h+1:H}$.
		\State Update dataset: $\MD_h \gets \MD_h \cup \{ (x_h, a_h, \sum_{g=h}^H r_g) \}$.
		\EndFor
		\State Solve regression:
		\[\hat \bw_h \gets\argmin_{w \in \BR^p: \| w \|_1 \leq \Cnrm}  \sum_{(x, a, r)\in\MD} (\lng \phi_h(x, a), w \rng  -r)^2.\] \label{eq:psdp-rew-mistake}
		\State Define $\hat \pi_h : \MX \ra \MA$ via
		\[
		\hat \pi_h(x)  := 
			\argmax_{a\in \MA} \lng \phi_h(x,a), \hat \bw_h \rng,
          \]
          and write $\hat \pi^{h:H} = (\hat \pi_h, \ldots, \hat \pi_H)$. 
		\EndFor
		\State \textbf{Return:} Policy $\hat \pi^{1:H} \in \Pi$. 
	\end{algorithmic}
\end{algorithm}

We accomplish this by introducing \PSDPrew{} (\cref{alg:psdp-rew}), a slight variant of the \PSDP algorithm (\cref{alg:psdp}) discussed previously. Algorithmically, the only difference with \PSDP is that \PSDPrew aims to find an optimal policy for the value function given by the \emph{environmental (true) rewards} of the MDP, rather than the rewards induced by an input vector at a particular step. As such, it requires policy covers at all steps, namely $\Psi_1, \ldots, \Psi_H$. 

Analytically, the proof is a slight generalization of \cref{lem:psdp}, which assumed access to true policy covers (see \cref{defn:pc}, a stronger condition than that of a truncated policy cover). However, it is significantly simpler than the proof of \cref{lem:psdp-trunc} (which also only assumed access to truncated policy covers), because the parameter $\trunc$ is allowed to appear in the error below, whereas \cref{lem:psdp-trunc} required more fine-grained control \--- to be useful for inductively constructing the truncated policy covers that we now get to use.

\begin{lemma}[\PSDP with environmental rewards]\label{lem:psdp-rew}
 There is a constant $C_{\ref{lem:psdp-rew}} > 0$ so that the following holds. Fix $\alpha, \ep, \delta \in (0,1)$, and suppose that $\Psi_{1:H}$ are $\alpha$-truncated policy covers (\cref{def:ih-trunc}) at step $1,\dots,H$ respectively. Fix any $N \in \NN$ such that 
 \begin{equation} N \geq \npsdprew(\epsilon,\alpha,\delta) := \frac{C_{\ref{lem:psdp-rew}} H^8 A^2 \Cnrm^4 \log(Hd/\delta)}{\alpha^4 \epsilon^4}.\label{eq:npsdp-rew}\end{equation}
 Then the output of the algorithm $\PSDPrew(\Psi_{1:H}, N)$ (\cref{alg:psdp-rew}) is a policy $\hat \pi$ that, with probability at least $1-\delta$, satisfies
  \begin{align}
\E\sups{M,\hat\pi}\left[\sum_{h=1}^H r_h\right] \geq \max_{\pi \in \Pi} \E\sups{M, \pi}\left[\sum_{h=1}^H r_h\right] - \epsilon - 4 \trunc \Cnrm^3 H^4\label{eq:rew-psdp-bound}
  \end{align}
  where $\trunc$ %
  is the truncation parameter used in defining $\Mbar(\emptyset)$ (\cref{section:extended-overview}), which in turn is used in the definition of an $\alpha$-truncated policy cover.
\end{lemma}
\begin{proof}
Let $\bfr=(\bfr_1,\dots,\bfr_H)$ denote the environmental reward function of $M$, specified by $\bfr_h(x,a) = \langle\phi_h(x,a),\theta_h\rangle$; recall that $\E^M[r_h|x_h=x,a_h=a] = \bfr_h(x,a)$ and $r_h \in [0,1]$ almost surely, for all $(x,a) \in \MX\times\MA$. For each $h \in [H]$, by \cref{lemma:linear-q-env}, there is some vector $\bw_h^\st \in \BR^d$ with $\| \bw_h^\st \|_1 \leq \Cnrm H$ so that, for all $(x,a) \in \MX \times \MA$, $Q_h\sups{M, \hat \pi}(x,a) = \lng \phi_h(x,a), \bw_h^\st \rng$.

Let $\pi^\st \in \argmax_{\pi \in\Pi} \E^{M,\pi}\left[\sum_{h=1}^H r_h\right]$. The performance difference lemma (\cref{lemma:perf-diff}) gives
\[\E\sups{M,\pi^\st}\left[\sum_{h=1}^H r_h\right] - \E\sups{M,\hat\pi} \left[\sum_{h=1}^H r_h\right] = \sum_{h=1}^H \E\sups{M,\pi^\st}[Q\sups{M,\hat\pi}_h(x_h,\pi^\st_h(x_h)) - Q\sups{M,\hat\pi}_h(x_h,\hat\pi_h(x_h))].\]

However, we can only bound the RHS difference in expectation under $\Mbar(\emptyset)$, not under $M$ (unlike in \cref{lem:psdp}), so we must bound the discrepancy. By \cref{lem:m-mbar-delta} with $k := h$ and policy $\pi^\st$, we have
\begin{align}
  \left| \E\sups{M,\pi^\st}[Q_h\sups{M, \hat \pi}(x_h, a_h)] - \E\sups{\Mbar(\emptyset), \pi^\st}[Q_h\sups{M, \hat \pi}(x_h, a_h)] \right|  
  &= \left| \E\sups{M, \pi^\st}[\lng \phi_h(x_h, a_h), \bw_h^\st \rng] - \E\sups{\Mbar(\emptyset), \pi^\st}[\lng \phi_h(x_h, a_h), \bw_h^\st \rng] \right|\nonumber\\
  &\leq  \| \bw_h^\st \|_1 \cdot \sum_{g=1}^h d_g\sups{M, \pi^\st}(\MX \backslash \Xreach_g(\emptyset)) \nonumber\\
  &\leq 2\trunc \Cnrm^3 H^4\nonumber
\end{align}
where the final inequality uses the bound $\norm{\bw^\st_h}_1 \leq \Cnrm H$ together with \cref{cor:not-reach-bound}. Similarly, applying \cref{lem:m-mbar-delta} with $k := h$ and policy $\pi^\st \circ_h \hat\pi$ yields
\begin{align}
\left| \E\sups{M, \pi^\st}[Q_h\sups{M, \hat \pi}(x_h, \hat \pi_h(x_h))] - \E\sups{\Mbar(\emptyset), \pi^\st}[Q_h\sups{M, \hat \pi}(x_h, \hat \pi_h(x_h))] \right| &\leq 2\trunc \Cnrm^3 H^4\nonumber.
\end{align}
Therefore the suboptimality of $\hat\pi$ can be bounded as
\begin{align} 
\E\sups{M,\pi^\st}\left[\sum_{h=1}^H r_h\right] - \E\sups{M,\hat\pi} \left[\sum_{h=1}^H r_h\right] 
&\leq \sum_{h=1}^H \E^{\Mbar(\emptyset),\pi^\st}[Q\sups{M,\hat\pi}_h(x_h,\pi^\st_h(x_h)) - Q\sups{M,\hat\pi}_h(x_h,\hat\pi_h(x_h))] \nonumber \\ 
&\qquad+ 4\trunc \Cnrm^3 H^4.\label{eq:m-mbar-pd}
\end{align}

The remainder of the proof proceeds akin to that of \cref{lem:psdp}. Fix $h \in [H]$. By definition we have $\hat\pi_h(x) \in \argmax_{a \in \MA} \langle \phi_h(x,a),\hat\bw_h\rangle$ for all $x \in \MX$. For each $x \in \MX$ and $h \in [H]$, let us define $\Delta_h(x) := \max_{a \in \MA} | \lng \phi_h(x,a), \bw_h^\st - \hat \bw_h \rng |$. Then for any $x \in \MX$, we have
  \begin{align}
    Q_h\sups{M, \wh \pi}(x, \pi_h^\st(x)) - Q_h\sups{M, \wh \pi}(x, \wh \pi_h(x)) 
    &= \lng \phi_h(x, \pi_h^\st(x)), \bw_h^\st \rng - \lng \phi_h(x, \wh \pi_h(x)), \bw_h^\st \rng\nonumber\\
    &\leq  \lng \phi_h(x, \pi_h^\st(x)), \hat \bw_h \rng - \lng \phi_h(x, \wh \pi_h(x)), \bw_h^\st \rng + \Delta_h(x)\nonumber\\
    &\leq  \lng \phi_h(x, \wh \pi_h(x)), \hat \bw_h \rng - \lng \phi_h(x, \wh \pi_h(x)), \bw_h^\st \rng + \Delta_h(x)\nonumber\\
    &\leq  2\Delta_h(x)\label{eq:q-delta-ub},
  \end{align}
  where the first and third inequalities use the definition of $\Delta_h(x)$ and the second inequality uses the fact that $\wh \pi_h(x) \in \argmax_{a \in \MA} \lng \phi_h(x,a), \hat w_h \rng$. Substituting into \cref{eq:m-mbar-pd}, we get
  \begin{equation}
    \E\sups{M,\pi^\st}\left[\sum_{h=1}^H r_h\right] - \E\sups{M,\hat\pi} \left[\sum_{h=1}^H r_h\right]
    \leq \sum_{h=1}^H \E^{\Mbar(\emptyset),\pi^\st}[2\Delta_h(x_h)] + 4\trunc \Cnrm^3 H^4,\label{eq:m-mbar-pd-2}
  \end{equation}
  so it remains to bound $\E^{\Mbar(\emptyset),\pi^\st}[2\Delta_h(x_h)]$ for all $h$. Once more fix $h \in [H]$. The dataset $\MD_h$ consists of $N$ independent samples $(x,a,r)$ with 
  \[ \E[r|x,a] = \E\sups{M,\hat\pi} \sum_{g=h}^H r_g = Q\sups{M,\hat\pi}_h(x,a).\] 
  Thus, we can apply \cref{cor:random-design-prediction-error} with covariates $(\phi_h(x,a): (x,a,r) \in \MD_h)$, ground truth $\bw^\st_h$, and responses $(r: (x,a,r) \in \MD_h)$. For any sample $(x,a,r) \in \MD_h$, it holds that $|r| \leq H$, so $|r - Q\sups{M,\hat\pi}_h(x,a)| \leq 2H$. Also, we have seen that $\norm{\bw^\st_h}_1 \leq \Cnrm H$. Recalling the definition of $\hat\bw_h$, \cref{cor:random-design-prediction-error} gives some event $\ME_h$ that holds with probability at least $1-\delta/H$, under which
    \begin{align}
\E_{\pi' \sim \unif(\Psi_h)\circ_h \unif(\MA)} \E\sups{M, \pi'} \left[ \lng \phi_h(x_h, a_h), \bw_h^\st - \hat \bw_h \rng^2 \right] \leq \frac{3C_{\ref{cor:random-design-prediction-error}} \cdot \Cnrm^2 H^2 \sqrt{\log(dH/\delta)}}{\sqrt{N}} =: \vep_0^2. \label{eq:good-psdp-rew-event}
  \end{align}
  It follows from \cref{eq:good-psdp-rew-event}, and the fact that the action $a_h$ above is uniformly random, that
  \begin{align}
\E_{\pi' \sim \unif(\Psi_h)} \E\sups{M, \pi'}[\Delta_h(x_h)^2] \leq \E_{\pi \sim \unif(\Psi_h)} \E\sups{M, \pi'}\left[ \sum_{a \in \MA} \lng \phi_h(x_h, a), \bw_h^\st - \hat \bw_h \rng^2 \right] \leq A \cdot\vep_0^2\nonumber,
  \end{align}
  which yields, via Jensen's inequality, that $\E_{\pi' \sim \unif(\Psi_h)} \E\sups{M, \pi'}[\Delta_h(x_h)] \leq \sqrt{A} \cdot \vep_0$. By the assumption that each $\Psi_h$ is an $\alpha$-truncated policy cover (\cref{def:ih-trunc}) together with non-negativity of $\Delta_h(x)$, it follows that
  \begin{align}
    \max_{\pi \in \Pi} \E\sups{\Mbar(\emptyset), \pi}[\Delta_h(x_h)] 
    &= \max_{\pi \in \Pi} \sum_{x \in \MX} d_h\sups{\Mbar(\emptyset), \pi}(x) \cdot \Delta_h(x)\nonumber\\
    &\leq  \E_{\pi' \sim \unif(\Psi_h)} \sum_{x \in \MX} \frac{1}{\alpha} \cdot d_h\sups{M, \pi'}(x)\cdot \Delta_h(x)\nonumber\\
    &= \frac{1}{\alpha} \E_{\pi' \sim \unif(\Psi_h)} \E\sups{M, \pi'} [\Delta_h(x_h)] \nonumber\\
    &\leq \frac{\sqrt{A} \cdot \vep_0}{\alpha}\label{eq:delta-alpha-ub}. 
  \end{align}
  Substituting into \cref{eq:m-mbar-pd-2} yields
  \[\E\sups{M,\pi^\st}\left[\sum_{h=1}^H r_h\right] - \E\sups{M,\hat\pi} \left[\sum_{h=1}^H r_h\right]
    \leq \frac{2\vep_0H\sqrt{A}}{\alpha} + 4\trunc \Cnrm^3 H^4\]
 which completes the proof by definition of $\vep_0$ and choice of $N$ in \cref{eq:npsdp-rew}, with $C_{\ref{lem:psdp-rew}} := 144C_{\ref{cor:random-design-prediction-error}}^2$.
\end{proof}

\begin{proof}[Proof of \cref{thm:main}]
    By \cref{theorem:main-pc-trunc}, with probability at least $1-\delta/2$, there is some $t \in [T]$ so that $\Psi_h^t$ is an $\alpha$-truncated policy cover at step $h$, for all $h \in [H]$. In this event, it is immediate (from \cref{def:ih-trunc}) that for all $h \in [H]$, the set $\bigcup_{t\in [T]} \Psi_h^t$ is an $\alpha/T$-truncated policy cover at step $h$. By \cref{lem:psdp-rew} (with $\ep$ set to $\epfinal/2$, $\delta$ set to $\delta/2$, and $\alpha$ set to the value defined in \cref{alg:slm-trunc}) and choice of $N$ on \cref{line:poem-N} of \cref{alg:opt}, it follows that with probability at least $1-\delta/2$, the output policy $\hat\pi$ computed by \PSDPrew satisfies 
    \[ \E\sups{M,\hat\pi}\left[\sum_{h=1}^H r_h\right] \geq \max_{\pi\in\Pi} \E\sups{M,\pi}\left[\sum_{h=1}^H r_h\right] - \frac{\epfinal}{2} - 4\trunc\Cnrm^3 H^4.\] 
    But now recall that we defined $\trunc := \frac{\epfinal}{8\Cnrm^3 H^4}$ (\cref{section:extended-overview}). Substituting in, we get
    \[ \E\sups{M,\hat\pi}\left[\sum_{h=1}^H r_h\right] \geq \max_{\pi\in\Pi} \E\sups{M,\pi}\left[\sum_{h=1}^H r_h\right] - \epfinal.\]
    By the union bound, this occurs with probability at least $1-\delta$, as desired. The claimed sample complexity bound is immediate from \cref{theorem:main-pc-trunc} and the definition of $N = \npsdprew(\epfinal/2,\alpha/T,\delta/2)$ (\cref{eq:npsdp-rew}), noting that $T/\alpha \leq \poly(\Cnrm,A,H,\epfinal^{-1})$ (as defined in \cref{alg:slm-params}). The claimed time complexity bound is immediate from \cref{theorem:main-pc-trunc} and inspection of \PSDPrew.
\end{proof}

\subsection*{Acknowledgments}

We thank Zakaria Mhammedi, Dylan Foster, and Sasha Rakhlin for helpful early discussions on this problem, and for pointing out the relevance of the concentrability coefficient \cite{xie2022role}.

\bibliographystyle{amsalpha}
\bibliography{bib}

\appendix

\section{Technical lemmas}

\subsection{Lemmas for sparse regression}\label{sec:sparse-regression-lemmas}
\begin{lemma}
  \label{lem:l1-generalization}
  There is a constant $C_{\ref{lem:l1-generalization}}$ with the following property. Fix $n,d \in \BN$, $k > 0$, and consider a distribution $\nu$ on $\BR^d$ which is supported on $\{ x \in \BR^d \ : \ \| x \|_\infty \leq 1\}$. Then, with probability $1-\delta$ over $X_1, \ldots, X_n \sim \nu$, it holds that
  \begin{align}
\sup_{w \in \BR^d :\ \| w \|_1\leq k} \left| \E_{X \sim \nu}[\lng w, X \rng^2] - \frac 1n \sum_{i=1}^n \lng w, X_i \rng^2\right| \leq \frac{C_{\ref{lem:l1-generalization}}k^2 \sqrt{\log (d/\delta)}}{\sqrt n}\nonumber.
  \end{align}
\end{lemma}
\begin{proof}
Define $\Sigma = \E_{X\sim \nu} XX^\t$ and $\hat\Sigma = \frac{1}{n}\sum_{i=1}^n X_i X_i^\t$. For any $a,b \in [d]$ by Hoeffding's inequality we have with probability at least $1-\delta/d^2$ that $|\Sigma_{ab} - \hat\Sigma_{ab}| \leq \frac{C\sqrt{\log(d/\delta)}}{\sqrt{n}}$ for an absolute constant $C$. By the union bound we get that $\norm{\Sigma - \hat\Sigma}_\infty \leq \frac{C\sqrt{\log(d/\delta)}}{\sqrt{n}}$ with probability at least $1-\delta$. In this event, we have $|w^\t (\Sigma-\hat\Sigma) w| = \left|\sum_{a,b=1}^d (\Sigma-\hat\Sigma)_{ab}w_aw_b\right| \leq \frac{C\norm{w}_1^2\sqrt{\log(d/\delta)}}{\sqrt{n}}$ for any $w \in \RR^d$.
\end{proof}

\begin{lemma}[Fixed-design error, see e.g. {\cite[Theorem 7.20]{wainwright2019high}}]\label{lem:fixed-design-prediction-error}
There is a constant $C_{\ref{lem:fixed-design-prediction-error}}$ with the following property. Fix $n,d \in \NN$ and $\sigma,k > 0$ and let $X \in [-1,1]^{n \times d}$. Fix some $w^* \in \RR^n$ with $\norm{w^*}_1 \leq k$. Define $y = Xw^* + \xi$ where $\xi_1,\dots,\xi_n$ are independent random variables satisfying $\EE \xi_i = 0$ and $|\xi_i| \leq \sigma$ almost surely. Define the estimator
\[ \hat{w} \in \argmin_{w \in \RR^d: \norm{w}_1 \leq k} \sum_{i=1}^m (\langle x_i,w\rangle - y_i)^2.\]
Then with probability $1-\delta$, it holds that
\[\frac{1}{n}\norm{X(\hat{w}-w^*)}_2^2 \leq \frac{C_{\ref{lem:fixed-design-prediction-error}}\sigma k\sqrt{\log(d/\delta)}}{\sqrt{n}}.\]
\end{lemma}

The following bound on out-of-sample prediction error of the constrained Lasso is immediate from combining \cref{lem:l1-generalization,lem:fixed-design-prediction-error}.

\begin{corollary}[Random-design error]
    \label{cor:random-design-prediction-error}
There is some constant $C_{\ref{cor:random-design-prediction-error}}$ with the following property. Fix $n,d\in\NN$ and $k,\sigma,\delta>0$ and let $\nu$ be a distribution on $[-1,1]^d$. Fix some $w^* \in \RR^n$ with $\norm{w^*}_1 \leq k$. Consider i.i.d. samples $(x_i,y_i)_{i=1}^n$ where $x_i \sim \nu$ and $y_i = \langle x_i,w^*\rangle + \xi_i$ where $\xi_i$ is independent noise satisfying $\EE \xi_i = 0$ and $|\xi_i| \leq \sigma$ almost surely. Define the estimator
\[ \hat{w} \in \argmin_{w \in \RR^d: \norm{w}_1 \leq k} \sum_{i=1}^n (\langle x_i,w\rangle - y_i)^2.\]
Then with probability $1-\delta$, it holds that
\[\EE_{x\sim \nu} [\langle x,w^* - \hat{w}\rangle^2] \leq \frac{C_{\ref{cor:random-design-prediction-error}}(k+\sigma) k\sqrt{\log(d/\delta)}}{\sqrt{n}}.\]
\end{corollary}

\subsection{Lemmas for Rademacher complexity}

The main result of this section is \cref{lem:min-feature-rc}, which bounds the Rademacher complexity of a certain class of ``min-linear'' functions. We also state a standard uniform convergence bound based on Rademacher complexity (\cref{lem:unif-conv}).

\begin{definition}
For a set $\MX$ and a class $\MF$ of functions $f : \MX \ra \BR$, and $n \in \BN$, the \emph{Rademacher complexity} of $\MF$ with respect to samples $x_1, \ldots, x_n \in \MX$ is
\begin{align}
\MR_n(\MF; x_{1:n}) := \frac{1}{n} \E_{\ep_{1:n} \sim \unif(\{ \pm 1 \})} \left[ \sup_{f \in \MF} \sum_{i=1}^n \ep_i f(x_i) \right]\nonumber.
\end{align}
We further write $\MR_n(\MF) := \sup_{x_{1:n}} \MR_n(\MF; x_{1:n})$.

The \emph{Gaussian complexity} of $\MF$ with respect to samples $x_1,\dots,x_n \in \MX$ is
\[\MG_n(\MF; x_{1:n}) := \frac{1}{n}\E_{\xi_{1:n} \sim N(0,1)}\left[\sup_{f\in\MF} \sum_{i=1}^n \xi_i f(x_i) \right].\] 
We write $\MG_n(\MF) = \sup_{x_{1:n}} \MG_n(\MF;x_{1:n})$.
\end{definition}

It will be more convenient to work with the Gaussian complexity, which upper bounds the Rademacher complexity as stated below:

\begin{lemma}[see e.g. {\cite[Exercise 5.5]{wainwright2019high}}]\label{lemma:rademacher-gaussian}
Let $n \in \NN$. For any set $\MX$ and class $\MF$ of functions $f: \MX \to \RR$, it holds that $\MR_n(\MF) \leq \sqrt{\frac{\pi}{2}} \MG_n(\MF)$.
\end{lemma}

In particular, Gaussian complexity has the following composition property. The proof essentially follows that of \cite[Theorem 14]{bartlett2002rademacher}.\footnote{In \cite{bartlett2002rademacher}, the Rademacher complexity and Gaussian complexity are defined slightly differently (with absolute values around the sum). This introduces a slight technical flaw in the proof of \cite[Theorem 14]{bartlett2002rademacher}, which can be avoided in several ways. Perhaps the simplest is to drop the absolute values. This matches the definition in e.g. \cite{shalev2014understanding}, and still suffices for uniform convergence as stated in \cref{lem:unif-conv}.}

\begin{lemma}
  \label{lem:rc-composition}
  Let $\MX$ be a set. Fix $A,L \in \BN$ and let $\MF_1, \ldots, \MF_A$ be classes of functions mapping $\MX$ to $\BR$. Let $\phi : \BR^A \ra \BR$ be $L$-Lipschitz with respect to the Euclidean distance on $\BR^A$. Let $\MF$ be the class of real-valued functions on $\MX$ defined as follows:
  \begin{align}
\MF := \left\{ x \mapsto \phi(f_1(x), \ldots, f_A(x)) \ : \ f_1 \in \MF_1, \ldots, f_A \in \MF_A \right\}\nonumber.
  \end{align}
  Then for all $n \in \BN$, 
  \begin{align}
\MG_n(\MF) \leq L \sum_{a=1}^A \MG_n(\MF_a)\nonumber.
  \end{align}
\end{lemma}

\begin{proof}
Fix $x_1,\dots,x_n \in \MX$. Let $Z_1,\dots,Z_n \sim N(0,1)$ be independent standard normal random variables. For each $\ff = f_{1:A} \in \MF_1 \times \dots \times \MF_A$, define the random variable
\[X_{\ff} := \sum_{i=1}^n Z_i \phi(f_1(x_i),\dots,f_A(x_i)).\]
Also, for each $a \in [A]$, let $Z'_{a1},\dots,Z'_{an} \sim N(0,1)$ be independent standard normal random variables. For each $\ff = f_{1:A} \in \MF_1 \times \dots \MF_A$, define the random variable
\[Y_{\ff} := \sum_{a = 1}^A \sum_{i=1}^n Z'_{ai} f_a(x_i).\]
On the one hand, we have $n\MG_n(\MF;x_{1:n}) = \E \sup_{\ff} X_{\ff}$. On the other hand, \[\E\sup_{\ff}Y_{\ff} = \E \sup_{\ff} \sum_{a=1}^A \sum_{i=1}^n Z'_{ai} f_a(x_i) = \sum_{a=1}^A \E \sup_{f_a \in \MF_a} \sum_{i=1}^n Z'_{ai} f_a(x_i) = \sum_{a=1}^A n\MG_n(\MF_a; x_{1:n}).\]
It remains to show that $\E \sup_{\ff} X_\ff \leq L\E \sup_\ff Y_\ff$. By the Sudakov-Fernique inequality (\cref{lemma:sf-ineq}) applied to the centered Gaussian processes $\{X_\ff\}_\ff$ and $\{Y_\ff\}_\ff$, it suffices to show that $\E(X_\ff - X_{\ff'})^2 \leq L^2\E(Y_\ff - Y_{\ff'})^2$ for all $\ff,\ff' \in \MF_1\times\dots\times\MF_A$. Fix $\ff = f_{1:A}$ and $\ff' = f'_{1:A}$. Then we have
\begin{align*}
\E(X_\ff - X_{\ff'})^2
&= \sum_{i=1}^n \left(\phi(f_1(x_i),\dots,f_A(x_i)) - \phi(f'_1(x_i),\dots,f'_A(x_i))\right)^2 \\
&\leq L^2 \sum_{i=1}^n \sum_{a=1}^A (f_a(x_i) - f'_a(x_i))^2 \\ 
&= L^2 \E(Y_\ff - Y_{\ff'})^2
\end{align*}
where the inequality uses the assumption that $\phi: (\RR^A,\norm{\cdot}_2) \to (\RR,|\cdot|)$ is $L$-Lipschitz. The lemma follows.
\end{proof}

\begin{lemma}[Sudakov-Fernique inequality; see e.g. {\cite[Theorem 2.2.3]{adler2007random}}]\label{lemma:sf-ineq}
Let $T$ be an index set, and let $(X_t)_{t \in T}$ and $(Y_t)_{t \in T}$ be centered Gaussian processes. If $\E(X_t - X_s)^2 \leq \E(Y_t - Y_s)^2$ for all $s,t \in T$, then $\E \sup_{t\in T} X_t \leq \E \sup_{t \in T} Y_t$.
\end{lemma}

\begin{lemma}
  \label{lem:min-feature-rc}
  Fix $d, A, B \in \BN$ and write $\MX \subset \BR^{d \times A}$ to denote the space of $d \times A$ real-valued matrices whose columns have $\ell_\infty$ norm at most 1. For $\theta \in \BR^d$ and $X \in \MX$, write
  \begin{align}
f_\theta(X) := \min \left\{0, \min_{a \in [A]} \lng X_a, \theta \rng \right\}\nonumber,
  \end{align}
  where $X_a \in \BR^d$ denotes the $a$th column of $X$. Write $\MF_{d,B} := \{ f_\theta \ : \ \theta \in \BR^d,\ \| \theta \|_1 \leq B \}$. Then for $n \in \BN$,
  \begin{align}
\MR_n(\MF_{d,B}) \leq AB \sqrt{\frac{\pi \log (2d)}{n}} \nonumber.
  \end{align}
\end{lemma}
\begin{proof}
  Write $\MB_{d, \infty} := \{ x \in \BR^d \ : \ \| x \|_\infty \leq 1\}$, and let $\MW_{d,B}$ denote the class of functions from $\MB_{d, \infty}$ to $\BR$ defined as:
  \begin{align}
\MW_{d,B} := \left\{ x \mapsto \lng \theta, x \rng \ : \ \theta \in \BR^d,\ \| \theta\|_1 \leq B \right\}\nonumber.
  \end{align}
  For any $x_1,\dots,x_n \in \MB_{d,\infty}$ we have $\MG_n(\MW_{d,B}; x_{1:n}) = \frac{B}{n} \E_{\xi \sim N(0,I_n)} \norm{\sum_{i=1}^n \xi_i x_i}_\infty$. For any $j \in [d]$, the $j$-th entry of $\sum_{i=1}^n \xi_i x_i$ is a mean-$0$ Gaussian with variance $\sum_{i=1}^n x_{ij}^2 \leq n$, so by the Gaussian maximal inequality, it holds that $\E_{\xi \sim N(0,I_n)} \norm{\sum_{i=1}^n \xi_i x_i}_\infty \leq \sqrt{2n\log(2d)}$ and thus $\MG_n(\MW_{d,B};x_{1:n}) \leq B \sqrt{\frac{2 \log (2d)}{n}}$. Since $x_1,\dots,x_n \in \MB_{d,\infty}$ were arbitrary it follows that $\MG_n(\MW_{d,B}) \leq B\sqrt{\frac{2\log(2d)}{n}}$ as well. 
  
  We can now bound $\MR_n(\MF_{d,B})$. For the (1-Lipschitz) function $\phi : \BR^A \ra \BR$ defined by $\phi(z_1, \ldots, z_A):= \min\{0, \min_{a \in [A]} z_a \}$, the function class $\MF_{d,B}$ defined in the lemma statement can be rewritten as
  \begin{align}
\MF_{d,B} = \left\{ X \mapsto \phi(g_1(X_1), \ldots, g_A(x_A)) \ : \ g_1 \in \MG_{d,B}, \ldots, g_A \in \MG_{d,B} \right\}\nonumber.
  \end{align}
  It follows from \cref{lemma:rademacher-gaussian} and \cref{lem:rc-composition} that \[\MR_n(\MF_{d,B}) \leq \sqrt{\frac{\pi}{2}} \MG_n(\MF_{d,B}) \leq \sqrt{\frac{\pi}{2}}\sum_{a=1}^A \MG_n(\MW_{d,B}) \leq AB \sqrt{\frac{\pi \log (2d)}{n}}\]
  as claimed.
\end{proof}

\begin{lemma}[{\cite[Theorem 26.5]{shalev2014understanding}}] \label{lem:unif-conv}
  Suppose $\MX$ is a set and $\MF$ is a class of functions $f : \MX \ra [-B, B]$ for some $B > 0$. Suppose $P$ is a distribution on $\MX$. Then for any $n \in \BN$ and $\delta \in (0,1)$, with probability at least $1-\delta$ over an i.i.d.~sample $X_1, \ldots, X_n \sim P$, it holds that
  \begin{align}
\sup_{f \in \MF} \left| \E_{X \sim P}[f(X)] - \frac 1n \sum_{i=1}^n f(X_i) \right| \leq 2\MR_n(\MF) + 4B\sqrt{\frac{2 \log(4/\delta)}{n}}.\nonumber
  \end{align}
\end{lemma}

\subsection{Miscellaneous lemmas}
\label{sec:lemmas-misc}

\begin{lemma}[Performance difference lemma \cite{kakade2002approximately}]\label{lemma:perf-diff}
For any MDP $M$, policies $\pi,\pi'\in\Pi$, and collection of reward functions $\bfr = (\bfr_1,\dots,\bfr_H)$, it holds that
\[\E\sups{M,\pi}\left[\sum_{h=1}^H \bfr_h(x_h,a_h)\right] - \E\sups{M,\pi'}\left[\sum_{h=1}^H \bfr_h(x_h,a_h)\right] = \sum_{h=1}^H \E\sups{M,\pi'}\left[V\sups{M,\pi,\bfr}_h(x_h) -  Q\sups{M,\pi,\bfr}_h(x_h,a_h)\right].\]
\end{lemma}

\begin{lemma}[Simulation lemma; Lemma F.3 of \cite{foster2021statistical}]
  \label{lem:simulation}
  Let $M, \Mbar$ denote two horizon-$H$ MDPs with state space $\MX$, action space $\MA$, and reward functions $\bfr, \wb{\bfr}$, respectively. Then, for any $\pi \in \Pi$, 
  \begin{align}
V_1\sups{M, \pi, \bfr}(x_1) - V_1\sups{\Mbar, \pi, \wb\bfr}(x_1) &= \sum_{h=1}^H \E\sups{\Mbar, \pi} \left[ Q_h\sups{M, \pi, \bfr}(x_h, a_h) - V_{h+1}\sups{M, \pi, \bfr}(x_{h+1}) \right] + \sum_{h=1}^H \E\sups{\Mbar, \pi} \left[ \bfr_h(x_h, a_h) - \wb\bfr_h(x_h, a_h) \right]\nonumber.
  \end{align}
\end{lemma}

\subsection{Policy discretization for $\ell_1$-bounded linear MDPs}
\cref{lemma:policy-disc} below shows that in a $d$-dimensional $\ell_1$-bounded linear MDP, there is a discretization $\Pidisc$ of $\Pi$, so that for any $\pi \in \Pi$ and $\theta \in \BR^d$ specifying a linear objective, some policy $\pidisc \in \Pidisc$ optimizes the linear objective induced by $\theta$ nearly as well as $\pi$. Importantly, $\log |\Pidisc|$ depends  only logarithmically on $d$, which allows us to take a union bound over all policies in $\Pidisc$ without incurring $\poly(d)$ factors in the sample complexity.

\begin{lemma}\label{lemma:policy-disc}
  Let $\epdisc, \Cnrm > 0$, $d \in \BN$. Fix any $\ell_1$-bounded $d$-dimensional linear MDP $M = (H,\MX,\MA,\BP_1,(\phi_h)_h,(\mu_{h+1})_h,(\theta_h)_h)$ with norm bound $\Cnrm$. Then there is a set of policies $\Pidisc \subseteq \Pi$ with 
  \[ \log |\Pidisc| \leq \frac{64 \Cnrm^2 H^3}{\epdisc^2} \cdot \log(4AH\Cnrm/\epdisc)\log(2d)
  \] 
  such that the following holds: for all $\theta \in \BR^d$, $h \in [H]$, and $\pi \in \Pi$, there is some $\pidisc \in \Pidisc$ so that
  \begin{align}
\lng \E\sups{M, \pidisc}[\phi_h(x_h, a_h)], \theta \rng \geq \lng \E\sups{M,\pi}[\phi_h(x_h, a_h)], \theta \rng - \epdisc \norm{\theta}_1.\nonumber
  \end{align}
\end{lemma}

We next prove \cref{lemma:policy-disc}. We first prove the following basic lemma that allow us to sparsify a vector of small $\ell_1$ norm. 
\begin{lemma}
  \label{lem:sample-w}
  Consider any vector $w \in\BR^d$. Define $P \in \Delta([d])$ by $P(i) = \frac{|w_i|}{\| w \|_1}$. Given $n \in \BN$, consider the random vector $\hat w \in \BR^d$ defined as follows: take i.i.d.~samples $i_1, \ldots, i_n \sim P$, and write $\hat w = \frac{\|w \|_1}{n} \sum_{j=1}^n \mathrm{sign}(w_{i_j}) \cdot e_{i_j}$. Then for any fixed $v \in \BR^d$ with $\| v \|_\infty \leq 1$, it holds that
  \begin{align}
\Pr \left( | \lng w - \hat w, v \rng | > \frac{2\| w \|_1 \cdot \sqrt{\log 1/\delta}}{\sqrt{n}} \right) \leq \delta\nonumber.
  \end{align}
\end{lemma}
\begin{proof}
For each $j \in [n]$, the random variable $\| w \|_1 \cdot \lng \mathrm{sign}(w_{i_j}) \cdot  e_{i_j}, v \rng$ has absolute value at most $\| w \|_1$ and has expectation equal to $\lng w, v \rng$. The result then follows from Hoeffding's inequality. 
\end{proof}

The proof of \cref{lemma:policy-disc} proceeds by letting $\Pidisc$ be a set of \emph{linear policies}, as defined below.
\begin{defn}[Linear policy]
  Given a sequence of vectors $\bw = (\bw_1, \ldots, \bw_H) \in (\BR^d)^H$ and features $\hat\phi = (\hat \phi_1, \ldots, \hat \phi_H)$, with $\hat \phi_h : \MX \times \MA \ra \BR$, we define the (deterministic) \emph{linear policy} corresponding to $\bw$ and $\hat\phi$, denoted $\pilin(\bw;\hat\phi)$, as follows:
  \begin{align}
\pilin(\bw;\hat\phi)_h(x) := \argmax_{a \in \MA} \lng \hat\phi_h(x,a), \bw_h \rng \qquad \forall h \in [H], x \in \MX\nonumber.
  \end{align}
\end{defn}

We remark that a naive attempt to proving \cref{lemma:policy-disc} would let $\Pidisc$ be the set of all $\pilin(\bw; \phi)$, where $\bw$ ranges over an $\epdisc$-net of $(\BR^d)^H$. However, the resulting set $\Pidisc$ would have $\log |\Pidisc| \geq \Omega(d)$, which does not obtain the desired $\log(d)$ scaling of $\log |\Pidisc|$. To overcome this issue, we only include policies $\pilin(\bw; \hat \phi\sups{\hat M})$ where $\bw$ is sufficiently sparse. 
\begin{proof}[Proof of \cref{lemma:policy-disc}]
  Write $\delta := \epdisc/(4AH\Cnrm)$ and $s := \frac{64\Cnrm^2 H^2 \cdot \log 1/\delta}{\epdisc^2}$. Define \[\Pidisc := \{ \pilin((\bw_1, \ldots, \bw_H); \phi) \ : \ \| \bw_h \|_1 \leq 1 \text{ and } (\bw_h)_i \in (1/s)\mathbb{Z} \quad \forall h \in [H], i \in [d]\}.\] It's easy to see that $|\Pidisc| \leq (2d)^{sH}$, which implies the claimed bound on $\log|\Pidisc|$. It remains to show that for all $\pi \in \Pi$, $h \in [H]$, and $\theta \in \BR^d$, there is some $\pidisc \in \Pidisc$ so that %
  \begin{align}
    \label{eq:pidisc0}
    \lng \E\sups{M,\pidisc}[\phi_h(x_h, a_h)], \theta \rng \geq \lng \E\sups{M,\pi}[\phi_h(x_h, a_h)], \theta \rng - \epdisc \norm{\theta}_1.
  \end{align}

  Fix $\theta \in \RR^d$ and $h \in [H]$. To establish this claim, we first note that without loss of generality, $\pi \in \argmax_{\pi' \in \Pi} \lng \E\sups{M,\pi'}[\phi_h(x_h, a_h)], \theta \rng$ (as otherwise we may replace $\pi$ with such a maximizer). By \cref{lemma:linear-q}, for each $k \in [h]$, there is a vector $\bv^\pi_k \in \BR^d$ with $\| \bv_k^\pi \|_1 \leq \Cnrm \cdot \norm{\theta}_1$ so that $Q\sups{M,\pi,\bfr}_k(x, a) = \lng \phi_k(x,a), \bv_k^\pi\rng$ for all $x,a$, where $\bfr = (\bfr_1,\dots,\bfr_h)$ is the reward function defined by
  \[ \bfr_k(x,a) = \begin{cases} \lng \phi_k(x,a), \theta \rng & \text{ if } k = h \\ 0 & \text{ otherwise} \end{cases}.\]
  Without loss of generality, $\pi$ is deterministic, with $\pi_k(x) \in \argmax_{a \in \MA} Q^{M,\pi,\bfr}_k(x,a)$ for all $x\in\MX$ and $k \in [h]$. By \cref{lemma:perf-diff}, for any deterministic policy $\pi'$, we have
  \begin{align}
\lng \theta, \E\sups{M,\pi}[\phi_h(x_h, a_h)] - \E\sups{M,\pi'}[\phi_h(x_h, a_h)] \rng 
&= \sum_{k=1}^h \E\sups{M,\pi'} \left[ \lng \bv_k^\pi, \phi_k(x_k, \pi_k(x_k)) - \phi_k(x_k, \pi_k'(x_k)) \rng \right]\label{eq:pd-disc}.
  \end{align}
  We define $\pidisc = \pilin(\bw_1,\dots,\bw_H)$ where we will iteratively define $\bw_1,\dots,\bw_h$ (to establish \cref{eq:pidisc0}, it does not matter how $\bw_{h+1},\dots,\bw_H$ are picked). Specifically, fix $k \in [h]$ and suppose that we have already picked $\bw_1,\dots,\bw_{k-1}$. We will show that there must exist a choice of $\bw_k$ such that
  \begin{align}
 \E\sups{M,\pidisc} \left[ \lng \bv_k^\pi, \phi_k(x_k, \pi_k(x_k)) - \phi_k(x_k, (\pidisc)_k(x_k)) \rng \right] \leq \epdisc \norm{\theta}_1/H.\label{eq:wk-apx-error}
  \end{align}
  Once we have proven this, we can iteratively apply it for $k=1,\dots,h$ to construct $\pidisc$. By applying \cref{eq:pd-disc} with $\pi' = \pidisc$, $\pidisc$ will satisfy \cref{eq:pidisc0}. It remains to prove the existence of $\bw_k$ so that the induced policy $\pidisc$ satisfies \cref{eq:wk-apx-error}, which we do so by the probabilistic method.

  Suppose that we randomly pick $\bw_k = \hat{w}/\norm{\bv_k^\pi}_1$ where $\hat{w}$ is the random vector generated by the procedure in \cref{lem:sample-w} with $n := s$ and $w := \bv_k^\pi$. It follows from \cref{lem:sample-w} that $\norm{\bw_k}_1 \leq 1$ and $(\bw_k)_i \in (1/s)\mathbb{Z}$ for all $i \in [d]$, and moreover that for any $x_k \in \MX$,
  \begin{align}
\Pr \left( | \lng \bv_k^\pi - \norm{\bv_k^\pi}_1 \bw_k, \phi_k(x_k,a) \rng | \leq \frac{\epdisc}{4\Cnrm H} \norm{\bv_k^\pi}_1 \ \ \forall a \in \MA \right) \geq 1-\delta A\nonumber.
  \end{align}
  For the above inequality we have also used the choice of parameters $s, \delta$ and a union bound over $a \in \MA$.
  As a result, for any fixed $x_k \in \MX$, we have with probability at least $1-\delta A$ (over the choice of $\bw_k$) that%
  \begin{align}
    \lng \bv_k^\pi, \phi_k(x_k, (\pidisc)_k(x_k)) \rng &\geq  \norm{\bv_k^\pi}_1\lng \bw_k, \phi_k(x_k, (\pidisc)_k(x_k)) \rng - \epdisc\norm{\bv_k^\pi}_1/(4\Cnrm H) \nonumber\\
    &\geq  \norm{\bv_k^\pi}_1\lng \bw_k, \phi_k(x_k, \pi_k(x_k)) \rng - \epdisc\norm{\bv_k^\pi}_1/(4\Cnrm H)\nonumber\\
    &\geq  \lng \bv_k^\pi, \phi_k(x_k, \pi_k(x_k)) \rng - \epdisc\norm{\bv_k^\pi}_1/(2\Cnrm H)\label{eq:whp-apx-error}
  \end{align}
  where the second inequality uses the definition that $(\pidisc)_k(x) \in \argmax_{a \in \MA} \langle \bw_k, \phi_k(x_k,a)\rangle$. Also, for any $x_k \in \MX$ we have with probability $1$ that
  \begin{equation} \langle \bv_k^\pi, \phi_k(x_k,(\pidisc)_k(x_k))\rangle \geq \langle \bv_k^\pi, \phi_k(x_k,\pi_k(x_k))\rangle - 2\norm{\bv_k^\pi}_1. \label{eq:as-apx-error}
  \end{equation}
  Taking expectation over the randomness in choosing $\bw_k$, it therefore holds that for any fixed $x_k \in \MX$,
  \[ \EE_{\bw_k}[\langle \bv_k^\pi, \phi_k(x_k,\pi_k(x_k)) - \phi_k(x_k, (\pidisc)_k(x_k))\rangle] \leq \frac{\epdisc}{2\Cnrm H}\norm{\bv_k^\pi}_1 + 2\delta A \norm{\bv_k^\pi}_1 \leq \frac{\epdisc}{\Cnrm H}\norm{\bv_k^\pi}_1\]
  where the last inequality is by choice of $\delta$. Now note that the prior choice of $\bw_1,\dots,\bw_{k-1}$ fully determines $(\pidisc)_{1:k-1}$ and therefore the distribution of $x_k \sim \pidisc$. Thus, we get
  \[ \EE\sups{M,\pidisc}\left[\EE_{\bw_k}[\langle \bv_k^\pi, \phi_k(x_k,\pi_k(x_k)) - \phi_k(x_k, (\pidisc)_k(x_k))\rangle] \right]\leq \frac{\epdisc}{\Cnrm H}\norm{\bv_k^\pi}_1.\]
  and thus, since $\bw_k$ is independent of $x_k$,
  \begin{align}
\E_{\bw_k} \left[ \E\sups{M,\pidisc} \left[ \lng \bv_k^\pi, \phi_k(x_k, \pi_k(x_k)) - \phi_k(x_k, (\pidisc)_k(x_k)) \rng \right]\right] &\leq  \frac{\epdisc}{\Cnrm H}\norm{\bv_k^\pi}_1\nonumber.
  \end{align}
  In particular, by the probabilistic method, there exists a choice of $\bw_k$ so that the induced policy $\pidisc$ satisfies 
  \begin{equation*}
 \E\sups{M,\pidisc} \left[ \lng \bv_k^\pi, \phi_k(x_k, \pi_k(x_k)) - \phi_k(x_k, (\pidisc)_k(x_k)) \rng \right] \leq  \frac{\epdisc}{\Cnrm H}\norm{\bv_k^\pi}_1.
  \end{equation*}
  Since $\norm{\bv_k^\pi}_1 \leq \Cnrm \cdot \norm{\theta}_1$, we get that $\bw_k$ satisfies \cref{eq:wk-apx-error}, which completes the proof.
\end{proof}

\section{Prior approaches for constructing policy covers}
\label{sec:prior-failures}
In this section, we discuss why existing computationally efficient approaches from prior work do not suffice to efficiently construct a policy cover in sparse linear MDPs (let alone the more general $\ell_1$-bounded linear MDPs), without incurring $\poly(d)$ sample complexity. In particular, we consider the following question: if we are given $\alpha$-approximate policy covers $\Psi_1, \ldots, \Psi_h$, how can we construct an $\alpha$-approximate policy cover $\Psi_{h+1}$? 
\paragraph{Exploration via basis vectors.} A naive approach is to iterate over standard basis vectors $\theta = e_i$, and, at each step $i$, to use the covers $\Psi_1, \ldots, \Psi_h$ together with \PSDP to compute a policy $\pi\^{e_i}$ that approximately maximizes $\E\sups{M, \pi\^{e_i}}\lng \phi_h(x_h, a_h), e_i \rng$. Then we add the policy $\pi\^{e_i}$  to $\Psi_{h+1}$ if $\E\sups{M, \pi\^{e_i}} \lng \phi_h(x_h, a_h), e_i \rng$ is significantly greater than $\max_{\pi \in \Psi_{h+1}} \E\sups{M, \pi} \lng \phi_h(x_h,a_h), e_i \rng$. Unfortunately, this approach runs into issues regarding cancellations in the expected feature vectors: even if $\pi^{(e_1)}$ maximizes $\E\sups{M,\pi}\langle \phi_{h}(x_{h},a_{h}),e_1\rangle$ and $\pi^{(2)}$ maximizes $\E\sups{M,\pi}\langle \phi_{h}(x_{h},a_{h}),e_2\rangle$, due to cancellations there may be a policy $\pi^\st$ so that $\E\sups{M,\pi^\st}\langle \phi_{h}(x_{h},a_{h}),e_1+e_2\rangle$ is much larger than both $\E\sups{M,\pi^{(e_1)}}\langle \phi_{h}(x_{h},a_{h}),e_1+e_2\rangle$ and $\E\sups{M,\pi^{(e_2)}}\langle \phi_{h}(x_{h},a_{h}),e_1+e_2\rangle$.

\paragraph{Barycentric spanners.} To deal with the issue of cancellation, one can try to choose the directions $\theta$ adaptively by constructing an (approximate) barycentric spanner \cite{awerbuch2008online} of the $d$-dimensional polytope $\{ \E\sups{M, \pi}[\phi_h(x_h , a_h)] \ : \ \pi \in \Pi \}$. While this approach has been successful for many reinforcement learning and bandit problems \cite{awerbuch2008online,dani2007price, kakade2007playing, lattimore2017end, foster2021statistical, golowich2022learning,mhammedi2023efficient} (in particular, \cite{mhammedi2023efficient} gave a new algorithm for linear MDPs based on a generalization of the barycentric spanner algorithm from \cite{awerbuch2008online}), it fails in our setting since a barycentric spanner is typically of size $|\Psi_{h+1}| = d$, and the parameter $\alpha$ in \cref{eq:approx-pc} scales inversely proportional to $|\Psi_{h+1}|$. Thus, the sample complexity of \PSDP in future steps would scale polynomially in $d$.

\paragraph{Representation learning.} In the the case that the MDP $M$ is $k$-sparse (\cref{def:sparse-lmdp}, which is a special case of $\ell_1$-bounded linear MDPs), one could attempt to circumvent the issue in the previous bullet point by applying one of several approaches from prior work on representation learning in RL \cite{modi2021model, zhang2022efficient, mhammedi2023efficient} to learn a (size-$k$) set $S \subset [d]$ of features that well-approximate the transitions. One could then construct a barycentric spanner of the feature vectors restricted to coordinates in $S$, which is guaranteed to have size at most $|S| \leq k$. Unfortunately, this approach runs into computational issues, as discussed in \cref{sec:related}: in order to implement the oracles in such existing algorithms, there is no clear way to avoid iterating over all subsets $S$ with $|S| \leq k$, which takes time $d^k$. Moreover, there is some evidence that such brute-force approaches cannot be improved in general \cite{gupte2020fine, zhang2014lower}. Learning a larger set $S\subset [d]$ of $k/\epsilon$ features for some accuracy parameter $\epsilon>0$ could plausibly avoid these intractability results. However, doing so would still seem to require implementing a max-min optimization oracle that finds some discriminator function (roughly, from the class of value functions) that maximizes the representation error of the current features. It is unclear how to do this in a computationally efficient manner.

\paragraph{Existence of a small policy cover.} Finally, we remark that in the more general $\ell_1$-bounded setting (\cref{defn:l1lmdp}) that we consider, none of the above approaches could even plausibly show \emph{existence} of a policy cover of size $\poly(k)$, when the $\ell_1$ norm bound is $k$. This results from the fact that the above approaches are linear algebraic in nature, relying on the fact that in $k$-sparse linear MDPs, it is only necessary to explore a $k$-dimensional subspace of features. Such a fact fails to hold in the $\ell_1$-bounded setting.

\section{Convex optimization details}\label{sec:opt-details}

The algorithms \ESC (\cref{alg:esc}) and \ESCt (\cref{alg:esc-trunc}) both require solving a (feasibility) convex program. For simplicity, we assumed in the main body of the paper that we could efficiently compute solutions that exactly satisfy the constraints of the respective programs whenever they were feasible. In this section we remove that assumption, by showing two facts:

\begin{enumerate}
    \item First, we can solve a relaxation of each program in polynomial time via the ellipsoid algorithm. 
    \item Second, the guarantees of \ESC and \ESCt still hold (up to constant factors) when they compute solutions to the relaxed programs rather than solutions to the original programs. 
\end{enumerate}
We focus on \cref{alg:esc} here; the details for \cref{alg:esc-trunc} are essentially the same.

\paragraph{Solving a relaxed program via the ellipsoid algorithm.} Instead of solving \cref{eq:program} in \cref{alg:esc}, we use the ellipsoid algorithm to solve the following program \cref{eq:program-relaxed}, where we take $\eprelax = \min( \epneg,\epcvx,\epneg\epapx/(2m), \epapx/(2m\Cnrm))$. %
        
        \begin{subequations}
        \label{eq:program-relaxed}
        \begin{align}
        \sum_{j=1}^m \norm{\hat\mu_{h+1}^j}_1 &\leq \Cnrm + \eprelax & \label{eq:hatmu-cnrm-constraint-relaxed}\\
        \langle \phi_h(x_h^i,a_h^i), \hat\mu_{h+1}^j\rangle &\geq -\eprelax &\forall i \in [n], j \in [m] \label{eq:hatmu-nneg-constraint-relaxed} \\
        \frac{1}{n} \sum_{i=1}^n \left(\langle \phi_h(x_h^i,a_h^i), \hat \bw_\ell\rangle - \sum_{j=1}^m \left\langle \phi_h(x_h^i,a_h^i), \hat\mu_{h+1}^j \right \rangle \phiavg_{h+1}(\tilde x_{h+1}^j)_\ell \right)^2 &\leq \epcvx^2 + \eprelax^2 &\forall \ell \in [d] \label{eq:what-wprime-constraint-relaxed}
        \end{align}
        \end{subequations}

Since each constraint of the program \cref{eq:program-relaxed} is either an $\ell_1$-norm constraint, or linear, or of the form $\norm{Ax-b}_2^2 \leq c$, it's clear that we can implement a separating hyperplane oracle in time $\poly(d,m,n)$. Additionally, by \cref{eq:hatmu-cnrm-constraint-relaxed} and the fact that $\eprelax \leq 1 \leq \Cnrm$, the feasible region of \cref{eq:program-relaxed} is contained in an $\ell_2$-ball of radius $2\Cnrm$ (centered at the origin) within $\RR^{dm}$. Finally, whenever the original program \cref{eq:program} is feasible, say realized by $(\hat \mu_{h+1}^j)_{j=1}^m \in (\BR^d)^m$, we claim that the feasible region of the relaxed program contains an $\ell_2$-ball of radius $\eprelax^2 / (5m\Cnrm \sqrt{d}) \leq \eprelax / (m\sqrt{d})$. Indeed, consider any $(\hat \nu_{h+1}^j)_{j=1}^m \in (\BR^d)^m$ which has $\ell_2$-distance (in $\BR^{dm}$) at most $\eprelax^2/(5m\Cnrm\sqrt{d})$ from $(\hat \mu_{h+1}^j)_{j=1}^m$. Then $(\hat \nu_{h+1}^j)_{j=1}^m$ satisfies the relaxed program \cref{eq:program-relaxed}:
\begin{itemize}
\item To see that \cref{eq:hatmu-cnrm-constraint-relaxed} holds, we compute
  \begin{align}
\sum_{j=1}^m \| \hat \nu_{h+1}^j \|_1 \leq \sum_{j=1}^m \| \hat \mu_{h+1}^j \|_1 + \sqrt{d} \sum_{j=1}^m \| \hat \nu_{h+1}^j - \hat \mu_{h+1}^j \|_2 \leq \Cnrm + \eprelax\nonumber.
  \end{align}
\item To see that \cref{eq:hatmu-nneg-constraint-relaxed} holds, note that for all $i \in [n], j \in [m]$, we have \[\lng \phi_h(x_h^i, a_h^i), \hat \nu_{h+1}^j \rng \geq \lng \phi_h(x_h^i, a_h^i), \hat \mu_{h+1}^j \rng - \| \phi_h(x_h^i, a_h^i) \|_2 \cdot \| \hat \nu_{h+1}^j - \hat \mu_{h+1}^j \|_2 \geq -\eprelax.\]
\item To see that \cref{eq:what-wprime-constraint-relaxed} holds, note that for any $\ell \in [d]$ and $i \in [n]$,
\begin{align*}
&\left| \sum_{j=1}^m \left\lng \phi_h(x_h^i, a_h^i), \hat \mu_{h+1}^j \right\rng \phiavg_{h+1}(\tilde x_{h+1}^j)_\ell - \sum_{j=1}^m \left\lng \phi_h(x_h^i, a_h^i), \hat \nu_{h+1}^j \right\rng \phiavg_{h+1}(\tilde x_{h+1}^j)_\ell \right| \\ 
&\leq m\sqrt{d} \max_{j \in [m]} \| \hat \mu_{h+1}^j - \hat \nu_{h+1}^j \|_2
\end{align*}
and also 
\[\left|\langle \phi_h(x_h^i,a_h^i), \hat \bw_\ell\rangle - \sum_{j=1}^m \left\langle \phi_h(x_h^i,a_h^i), \hat\mu_{h+1}^j \right \rangle \phiavg_{h+1}(\tilde x_{h+1}^j)_\ell\right| \leq \norm{\hat{\bw}_\ell}_1 + \sum_{j=1}^m \norm{\hat\mu_{h+1}^j}_1 \leq 2\Cnrm\]
and similarly
\[\left|\langle \phi_h(x_h^i,a_h^i), \hat \bw_\ell\rangle - \sum_{j=1}^m \left\langle \phi_h(x_h^i,a_h^i), \hat\nu_{h+1}^j \right \rangle \phiavg_{h+1}(\tilde x_{h+1}^j)_\ell\right| \leq 3\Cnrm.\]
Thus, since $(\hat\mu_{h+1}^j)_{j=1}^m)$ satisfies \cref{eq:what-wprime-constraint}, we see that
  \begin{align}
    & \frac{1}{n} \sum_{i=1}^n \left(\langle \phi_h(x_h^i,a_h^i), \hat \bw_\ell\rangle - \sum_{j=1}^m \left\langle \phi_h(x_h^i,a_h^i), \hat\nu_{h+1}^j \right \rangle \phiavg_{h+1}(\tilde x_{h+1}^j)_\ell \right)^2\nonumber\\
    &\leq   \epcvx^2 + m\sqrt{d} \max_{j \in [m]} \| \hat \mu_{h+1}^j - \hat \nu_{h+1}^j \|_2 \cdot 5\Cnrm \leq \epcvx^2 + \eprelax^2\nonumber,
  \end{align}
  ensuring that $(\hat\nu_{h+1}^j)_{j=1}^m$ satisfies \cref{eq:what-wprime-constraint-relaxed}.
\end{itemize}
Since $\eprelax^2 / (5m\Cnrm \sqrt{d}) \geq \poly(m^{-1}, \Cnrm^{-1}, \epapx, \epneg, \epcvx)$, it follows that the ellipsoid algorithm finds a solution to \cref{eq:program-relaxed} in time $\poly(d,m,n, \log(dm\Cnrm / (\epapx \epneg \epcvx)))$ \cite[Theorem 2.4]{bubeck2014convex}. 

\paragraph{Correctness guarantees.} It remains to argue that a solution to this relaxed program still suffices to prove \cref{thm:muhat-coreset}, the main guarantee of \ESC, up to constant factors. Indeed, under the conditions of \cref{thm:muhat-coreset}, let $(\hat\mu_{h+1}^j)_{j=1}^m$ be a solution to \cref{eq:program-relaxed}. For every $j \in [m]$ define \[\tilde\mu_{h+1}^j = \begin{cases} 0 & \text{ if } \norm{\hat\mu_{h+1}^j}_1 < \epapx/m \\ \frac{\Cnrm}{\Cnrm + \eprelax} \cdot \hat\mu_{h+1}^j & \text{ otherwise.}\end{cases}\] 
We claim that $(\tilde \mu_{h+1}^j)_{j=1}^m$ is an $(3\epapx, 3\epneg, \Cnrm)$-\coreset (\cref{def:not-core-set}):
\begin{itemize}

\item \cref{it:emulator-C-bound} is immediate from \cref{eq:hatmu-cnrm-constraint-relaxed}. 

\item Following the original proof of \cref{thm:muhat-coreset}, it's easy to check that the approximate nonnegativity condition (\cref{it:emulator-epneg-bound}) holds for each $j \in [m]$ with error at most $\epneg\norm{\tilde\mu^j_{h+1}}_1 + \eprelax$ (in particular, this follows from \cref{eq:apx-nneg-final}). If $\tilde\mu_{h+1}^j \neq 0$, then $\norm{\tilde\mu_{h+1}^j}_1 \geq  \frac 12 \norm{\hat\mu_{h+1}^j}_1 \geq \epapx/(2m)$, so this error is at most $3\epneg\norm{\tilde\mu^j_{h+1}}_1$ by choice of $\eprelax \leq \epneg\epapx/m$. Of course, if $\tilde\mu_{h+1}^j=0$, then the approximate nonnegativity condition for $j$ is satisfied with error $0$.

\item Finally, since \cref{lem:muhat-approx} still holds with $\epcvx^2$ replaced by $\epcvx^2 + \eprelax^2 \leq 2\epcvx^2$, we get that the vectors $(\hat\mu_{h+1}^j)_{j=1}^m$ satisfy \cref{it:emulator-epapx-bound} with bound $\epapx\sqrt{2}$. Since $\norm{\tilde\mu_{h+1}^j - \hat\mu_{h+1}^j}_1 \leq \max\{\epapx/m, \eprelax \cdot \| \hat \mu_{h+1}^j \|_1 \} \leq \epapx/m$ for all $j \in [m]$ (where the last inequality uses that $\eprelax\leq \epapx/(2m\Cnrm)$ and $\norm{\hat\mu_{h+1}^j}_1 \leq \Cnrm+\eprelax\leq2\Cnrm$), it follows that the vectors $(\tilde\mu_{h+1}^j)_{j=1}^m$ satisfy \cref{it:emulator-epapx-bound} with bound $\epapx(1+\sqrt{2}) \leq 3\epapx$.
\end{itemize}
The claim follows.

Finally, we note that by decreasing the parameters $\epapx, \epneg$ that are passed to \cref{thm:muhat-coreset} by a factor of 3, we can ensure that the solution $(\hat \mu_{h+1}^j)_{j=1}^m$ to \cref{eq:program-relaxed} is in fact a $(\epapx, \epneg, \Cnrm)$-\coreset (see \cref{rmk:convex-program-efficient}).

\section{Reducing sparse linear regression to policy learning}\label{app:slr-to-rl}

In this appendix, we show that the boundedness assumption on the reward vector in our definition of a sparse linear MDP (\cref{def:sparse-lmdp}) is necessary for statistically and computationally efficient learning, barring breakthroughs in sparse linear regression. In particular, we show that even for the special case $H=1$ (i.e. sparse contextual bandits), polynomial dependence on the bound is necessary. The reduction is likely not novel, but we include it for completeness.

Suppose that there is an algorithm $\MH$ that, for any $\epsilon>0$, for any $d$-dimensional $k$-sparse linear MDP (\cref{defn:l1lmdp}) with $A$ actions and horizon $H=1$, where the (sparse) reward vector satisfies $\norm{\theta_1}_1 \leq B$, learns a policy $\pi$ with suboptimality at most $\epsilon$ using $N(d,k,A,B,\epsilon)$ samples and $T(d,k,A,B,\epsilon)$ runtime, with high probability (omitting the dependence on failure probability for simplicity). We do not make any assumption about the representation of the policy, other than that it be efficiently queryable: for any given state $x$, one can compute $\pi(x)$ (or more generally, for a stochastic policy $\pi$, draw a sample from $\pi(x)$) in time $T(d,k,A,B,\epsilon)$.

We consider an instance of noiseless sparse linear regression of the following form. Let $\MD \in \Delta([-1,1]^d)$ be a known covariate distribution (i.e., so that we can draw an arbitrary number of unlabelled samples from it). Fix some unknown $k$-sparse vector $w^\st \in \RR^d$ with $\norm{w^\st}_1 \leq B$ and $\langle x, w^\st\rangle \in [0,1]$ for all $x \in \supp(\MD)$. We are given $m$ independent labelled samples $(x^i,y^i)_{i=1}^m$ where $x^i \sim \MD$ and $y^i = \langle w^\st, x^i\rangle \in [0,1]$. Our goal is to find $\hat w \in \RR^d$ approximately minimizing the out-of-sample prediction error $\E_{x \sim \MD} \langle w^\st - \hat w, x\rangle^2$.

\paragraph{Preliminary notation.} Fix $\epdisc>0$. Define a horizon-$1$ linear MDP $M^{w^\st} = (1, \MX,\MA,\BP_1,\phi,\theta_1^{w^\st})$ as follows. Let the set of states be $\MX := \supp(\MD) \subseteq [-1,1]^d$, and let the set of actions be $\MA := \{0, \epdisc, 2\epdisc,\dots,1\}$. Define the feature mapping $\phi: \MX\times\MA \to \RR^{\binom{d}{2} + d + 1}$ as follows:
\[(x,a) \mapsto \phi(x,a) := (1-a^2, ax_1, \dots, ax_d, x_1^2, x_1x_2, \dots, x_d^2).\]
Let the initial distribution be $\BP_1 := \MD$. Finally, let the reward vector be %
\[\theta_1^{w^\st} := (1, 2w^\st_1,\dots,2w^\st_d,-(w^\st_1)^2, -2w^\st_1 w^\st_2, \dots, -(w^\st_d)^2)\]
so that $\langle \phi(x,a), \theta_1^{w^\st}\rangle = 1 - (\langle w^\st, x\rangle - a)^2$ for any $(x,a) \in \MX\times\MA$. Note that $M$ is a $\binom{d}{2}+d+1$-dimensional $\binom{k}{2}+k+1$-sparse linear MDP with $1/\epdisc$ actions and horizon $1$, and we have $\norm{\theta_1^{w^\st}}_1 \leq (B+1)^2$ by definition of $\theta_1^{w^\st}$ and the fact that $\norm{w^\st}_1 \leq B$. Also, there is a policy with expected reward at least $1 - \epdisc^2$: namely, the map $x \mapsto \epdisc \lfloor \langle w^\st,x\rangle / \epdisc\rfloor$.

\paragraph{Reduction.} We can now define our sparse linear regression algorithm. We invoke the algorithm $\MH$ with error parameter $\epdisc^2$, and simulate interaction with $M^{w^\st}$ using the samples $(x^i,y^i)_{i=1}^m$. In particular, at the beginning of an episode we give $\MH$ a new covariate $x^i \sim \MD$. Then $\MH$ chooses an action $a$, and we return $1 - (y^i - a)^2$. Note that this reward is equal to $\langle \phi(x^i, a), \theta_1^{w^\st}\rangle$, as desired. Also, it's clear that given $x$ and $a$ we can compute $\phi(x,a)$. Thus, so long as $N(\binom{d}{2}+d+1,\binom{k}{2}+k+1,\epdisc^{-1},(B+1)^2,\epdisc^2) \leq m$, we have that $\MH$ produces a policy $\hat\pi$ with suboptimality at most $\epdisc^2$, in time at most $T(\binom{d}{2}+d+1,\binom{k}{2}+k+1,\epdisc^{-1},(B+1)^2,\epdisc^2)$. As a result, $\hat\pi$ satisfies the bound
\[\E_{x\sim\MD} \E_{a \sim \hat\pi(x)} (\langle w^\st, x\rangle - a)^2 \leq 2\epdisc^2.\]

We now draw $n := \poly(d,\epdisc^{-1},B)$ unlabeled samples $(\tilde x^j)_{j=1}^n$ independently from $\MD$, and for each $j \in [n]$ draw $a^j \sim \hat\pi(\tilde x^j)$. (Note that we use here that $\MD$ is known and efficiently sampleable, which allows us to draw many samples from it without incurring a larger sample complexity cost.) Finally, we compute 
\[ \hat{w} := \argmin_{w \in \RR^d,\ \| w \|_2 \leq B} \frac{1}{n}\sum_{j=1}^n (\langle w, x^j\rangle - a^j)^2.\] 
Standard concentration arguments give that $\E_{x\sim \MD} \langle w^\st - \hat w, x\rangle^2 \leq 3\epdisc^2$ with high probability, as desired. Setting $\epsilon := 3\epdisc^2$, we get that this algorithm computes a regressor $\hat w$ with prediction error at most $\epsilon$, with high probability. Moreover, it uses $N(O(d^2),O(k^2),O(\epsilon^{-1/2}), O(B^2+1), O(\epsilon))$ samples and $\poly(d,\epsilon^{-1/2},B) \cdot T(O(d^2),O(k^2),\epsilon^{-1/2},B^2+1,\epsilon^{-1})$ time. 

\paragraph{Consequences.} Our algorithm for learning sparse linear MDPs with the above parameter bounds has sample complexity $\poly(k,A,H,B,\epsilon^{-1},\log d)$ and time complexity $\poly(d,A,H,B,\epsilon^{-1})$. Ignoring computational efficiency, it's possible to achieve sample complexity $\poly(k,A,H,\epsilon^{-1},\log dB)$, i.e. with a much weaker dependence on the norm bound $B$ \cite{jiang2017contextual}. However, in the above sparse linear regression setting, there is no known computationally efficient algorithm that achieves better than polynomial dependence on $B$ in the sample complexity (even when the covariate distribution $\mathcal{D}$ is known, and the sample complexity only measures labelled samples).\footnote{See e.g. the discussion after Theorem 7.20 in \cite{wainwright2019high} on rates achieved by Lasso under various conditions. The ``fast rate'' sample complexity has no dependence on $B$ (and better dependence on the desired prediction error) but requires the covariate distribution $\MD$ to satisfy a restricted eigenvalue condition. The ``slow rate'' sample complexity holds for any sub-Gaussian covariate distribution, but has polynomial dependence on $B$. No computationally efficient algorithm is known to achieve the fast rate without a restricted eigenvalue condition.

To be clear, the majority of the literature on sparse linear regression studies the more general setting where $\MD$ is unknown. However, even when $\MD$ is partially or completely known, significant algorithmic challenges remain \cite{kelner2022power, kelner2023feature}, and removing the restricted eigenvalue condition seems out-of-reach of current techniques. On a similar note, our assumptions above that $\MD$ has bounded support, and that $\langle x,w^\st\rangle \in [0,1]$ for all $x \in \supp(\MD)$, are somewhat non-standard. However, since they do not preclude the types of constructions that seem hard for current algorithms (e.g. where the covariates contain a sparse approximate linear dependence, and $w^\st$ is a large vector in the direction of this dependence), it's unclear whether these assumptions make the problem any easier. In fact, we would not be surprised if the general setting where $\MD$ is sub-Gaussian (but still known) can be reduced to our setting. 

See also \cite{zhang2014lower, gupte2020fine}, which prove lower bounds for the more challenging problem of outputting a {sparse} estimator, for related discussion.}
The above reduction implies that this same dependence may be necessary in the reinforcement learning setting as well. 

\begin{remark}
The above reduction only pertains to the norm of the reward vector, not the norm of the ``cumulative transitions'' $\sum_{x\in\MX} \mu_h(x)$. Of course, when $H=1$ there are no transitions. But by considering $H=2$, one can show that the norm bound on the cumulative transitions is necessary as well. In particular, rather than simulating interaction with the horizon-$1$ linear MDP $M^{w^\st}$, we simulate interaction with a horizon-$2$ linear MDP where at step $2$ there are two states $r_1$ and $r_0$, which have deterministic rewards $1$ and $0$ respectively. Essentially, $\theta_1^{w^\st}$ now defines $\mu_2(r_1)$, and we define $\mu_2(r_0)$ so that all probabilities sum to $1$. This can be done by augmenting $\phi(x,a)$ with an additional feature $a^2$. We omit the details.
\end{remark}

\section{Details for decision tree block MDPs}\label{sec:dec-tree-linear}

In this appendix, we show that decision tree block MDPs, as introduced in \cref{sec:bmdp}, are sparse linear MDPs with known feature mapping but unknown sparsity pattern. \cref{cor:dec-tree-mdp-informal} (stated formally below as \cref{cor:dec-tree-mdp}) then follows immediately from \cref{thm:main}.

\subsection{Preliminaries on block MDPs.}
We first formally define block MDPs \cite{du2019provably}.%
\begin{definition}[Block MDP (BMDP)]\label{def:bmdp}
  Let $M$ be an MDP with state space $\MX$, action space $\MA$, horizon $H$, initial distribution $\BP_1 \in \Delta(\MX)$, transition distributions $\BP_h : \MX \times \MA \ra \Delta(\MA)$, and rewards $\bfr_h : \MX \times \MA \ra [0,1]$. For a set $\MS$ and a mapping $\rho^\st : \MX \ra \MS$, we say that that $M$ is a \emph{block MDP (BMDP)} with \emph{latent state space} $\MS$ and \emph{decoding function} $\rho^\st$ if the following holds: for each $h \in [H-1]$, there are \emph{latent transition distributions} $\tilde \BP_h : \MS \times \MA \ra \Delta(\MS)$ and \emph{emission distributions} $\tilde \BO_{h+1} : \MS \ra \Delta(\MX)$, and for each $h \in [H]$ there is a \emph{latent reward function} $\tilde \bfr_h : \MS \times \MA \ra [0,1]$, so that for all $x,x' \in \MX$ and $a \in \MA$,
  \begin{itemize}
    \item $\BP_h(x'|x,a) = \tilde{\BP}_h(\rho^\st(x')|\rho^\st(x),a)\tilde{\BO}_{h+1}(x'|\rho^\st(x'))$, and
    \item $\bfr_h(x,a) = \tilde{\bfr}_h(\rho^\st(x),a)$.
    \end{itemize}
    Moreover, it is required that for all distinct $s,s' \in \MS$, the supports of $\tilde \BO_{h+1}(\cdot | s)$ and $\tilde \BO_{h+1}(\cdot | s')$ are disjoint. 
\end{definition}
The decoding function $\rho^\st$, latent transitions $\tilde \BP_h$, emissions $\tilde \BO_{h+1}$, and latent reward functions $\tilde \bfr_h$ are \emph{unknown} to the learning algorithm. However, the following \emph{realizability} assumption is made:
\begin{assumption}[Realizability for BMDPs]
  \label{asm:bmdp-rlz}
  For some \emph{known} class $\Phi \subset (\MX \ra \MS)$, we have that $\rho^\st \in \Phi$.
\end{assumption}
All prior work in the literature assumes some type of oracle to access the class $\Phi$: for instance, \cite{modi2021model, zhang2022efficient} assume access to an oracle which can solve max-min problems where the maximization is over $\phi \in \Phi$ as well as a discriminator class defined in terms of $\Phi$, and the minimization is over linear weight vectors. %
The oracle in \cite{mhammedi2023representation} only requires minimizing over $\phi \in \Phi$, but (as discussed further in \cref{sec:related}) is required to be proper, i.e., to output an element of $\Phi$, and is not known to be efficiently implementable for essentially any interesting concrete classes. 

\subsection{Learning decision tree BMDPs}

\paragraph{Decision tree BMDPs.}
We next introduce a concrete set of decoding functions $\Phi$ for which we can establish end-to-end computationally efficient learning algorithms for the corresponding family of BMDPs. 
Fix $n,s \in \NN$, and let $\Phi_{n,s}$ be the class of functions $\phi: \{0,1\}^n \to [s]$ that can be expressed as depth-$\log(s)$ decision trees, where, for example, depth-$1$ decision trees are functions of the form 
\[ \rho(x) = \begin{cases} s_0 & \text{ if } x_j = 0 \\ s_1 & \text{ if } x_j = 1 \end{cases}\] 
for $s_0,s_1 \in [s]$ and $j \in [n]$, and more generally, a depth-$k$ decision tree is a function of the form
\[ \rho(x) = \begin{cases} \psi^0(x) & \text{ if } x_j = 0 \\ \psi^1(x) & \text{ if } x_j = 1 \end{cases} \] 
for depth-$(k-1)$ decision trees $\psi^0,\psi^1$ and $j \in [n]$.

\begin{definition}[Decision tree BMDP]\label{defn:dec-tree-mdp}
  For $n,s \in \NN$, let $M$ be an MDP with state space $\MX = \{0,1\}^n$, action space $\MA$, and horizon $H \in \NN$. We say that $M$ is an \emph{$(n,s)$-decision tree BMDP} if it is a BMDP (\cref{def:bmdp}) with latent state space $\MS = [s]$ and some decoding function $\rho^\st \in \Phi_{n,s}$. 
\end{definition}

It is straightforward to see that $\log |\Phi_{n,s}| \leq O(s \log n)$, and therefore numerous existing results in the literature \cite{jiang2017contextual,jin2021bellman, zhang2022horizon,mhammedi2023efficient} imply that there is an algorithm that requires $\poly(A, H, s, \log n, 1/\ep)$ samples and which outputs an $\ep$-optimal policy. However, the best known computational cost of any such algorithm was at least $n^{O(s)}$ (see e.g. the discussion in \cite[Section~7]{modi2021model}, which likely also applies to \cite{zhang2022efficient}). In constrast, as a corollary of our results, we get an algorithm whose computational cost scales only as $n^{O(\log s)}$ (without sacrificing sample-efficiency). 

\paragraph{Clause feature mapping.} To apply \cref{thm:main}, we need to show that decision tree BMDPs are a special case of sparse linear MDPs. To do so, let $\mathfrak{C}_{n,s}$ be the set of Boolean conjunctions of length $\log(s)$ in the variables $x_1,\dots,x_n$ (e.g. $x_1 \land \lnot x_3 \land x_4$ is a conjunction of length three). Note that $|\mathfrak{C}_{n,s}| \leq \binom{2n}{\log s}$. We define a feature mapping $\phidt{n,s}: \{0,1\}^n \times \MA \to \RR^{|\mathfrak{C}_{n,s}| \cdot |\MA|}$. Each feature vector is indexed by clauses $\MC \in \mf{C}_{n,a}$  and actions $a \in \MA$. %
For a clause $\MC \in \mathfrak{C}_{n,s}$ and action $a \in \MA$, we define the component of $\phidt{n,s}$ corresponding to $(\MC, a)$, namely $\phidt{n,s}_{(\MC,a)}: \{0,1\}^n \times \MA \to \RR$, by
\begin{align}\label{eq:phi-dt} \phidt{n,s}_{(\MC,a)}(x',a') := \mathbbm{1}[(\MC(x') = \texttt{True}) \land (a = a')].\end{align}

Note that for any given $x'\in\MX,a'\in\MA$, one can evaluate $\phidt{n,s}(x',a')$ in time $n^{O(\log s)} \cdot |\MA|$. We are now ready to formally state our main result regarding efficient learnability of decision tree BMDPs:
\begin{corollary}[Formal version of \cref{cor:dec-tree-mdp-informal}]
  \label{cor:dec-tree-mdp}
Let $n,s,A,H \in \NN$ and $\epsilon,\delta>0$. Let $M$ be a block MDP on $\{0,1\}^n$ with (unknown) decoding function $\rho^\st \in \Phi_{n,s}$. Then with probability at least $1-\delta$, $\OPT(\epsilon,\delta)$ with feature mapping $\phi^{\mathsf{DT}}$ outputs a policy with suboptimality at most $\epsilon$. Moreover, the sample complexity of the algorithm is $\poly(s, A, H, \epsilon^{-1}, \log(n/\delta))$, and the time complexity is $\poly(n^{\log s}, A, H, \epsilon^{-1},\log(1/\delta))$.
\end{corollary}

\begin{proposition}
  \label{prop:phidt-sparse}
Let $n,s\in\NN$, and let $M$ be an $(n,s)$-decision tree BMDP. Then $M$ is a $sA$-sparse linear MDP with bound $s$ (\cref{def:sparse-lmdp}), where the feature mapping is $\phidt{n,s}$. 
\end{proposition}

\begin{proof}
Let $\rho^\st \in \Phi_{n,s}$ be the decoding function for $M$, and let $\tilde{\BP}_h$ and $\tilde{r}_h$ denote the latent transitions and latent rewards respectively, as per \cref{defn:dec-tree-mdp,def:bmdp}. Since $\rho^\st$ can be expressed as a depth-$\log(s)$ decision tree, there are conjunctions $\MC_1,\dots,\MC_s \in \mathfrak{C}$ and latent states $g_1,\dots,g_s \in [s]$ such that for every $x \in \{0,1\}^n$, there is exactly one $i \in [s]$ with $\MC_i(x) = \texttt{True}$, and moreover $\rho^\st(x) = g_i$. 

Now for any $h \in [H-1]$ and $x' \in \MX$, we define the vector $\mu_{h+1}(x') \in \RR^{|\mathfrak{C}|\times|\MA|}$ by
\[\mu_{h+1}(x')_{(\MC,a)} := \BO_{h+1}(x'|\rho^\st(x')) \cdot \sum_{i=1}^s \mathbbm{1}[\MC=\MC_i] \tilde{\BP}_h(\rho^\st(x')|g_i, a).\]
Then for any $h \in [H-1]$, $x,x'\in\MX$, and $a \in \MA$, we can observe that
\begin{align*}\langle \phidt{n,s}(x,a), \mu_{h+1}(x')\rangle 
&= \tilde{\BO}_{h+1}(x'|\rho^\st(x')) \cdot \sum_{i=1}^s \sum_{a'\in\MA} \phidt{n,s}_{(\MC_i,a')}(x,a) \tilde{\BP}_h(\rho^\st(x')|g_i,a') \\ 
&= \tilde{\BO}_{h+1}(x'|\rho^\st(x')) \cdot \sum_{i=1}^s \sum_{a'\in\MA} \One{\MC_i(x) = \texttt{True}} \mathbbm{1}[a'=a] \tilde{\BP}_h(\rho^\st(x')|g_i,a') \\ 
&= \tilde{\BO}_{h+1}(x'|\rho^\st(x')) \cdot \sum_{i=1}^s \One{\MC_i(x)=\texttt{True}} \tilde{\BP}_h(\rho^\st(x')|g_i,a) \\
&= \tilde{\BO}_{h+1}(x'|\rho^\st(x')) \cdot \tilde{\BP}_h(\rho^\st(x')|\rho^\st(x),a) = \BP_h(x'|x,a)
\end{align*}
where the first two equalities are by the definitions of $\mu_{h+1}(x')$ and $\phidt{n,s}(x,a)$ respectively, the fourth equality is by the representation of $\rho^\st$ described above, and the final equality is by \cref{defn:dec-tree-mdp}.

Next, for any $h \in [H]$, we define the vector $\theta_h \in \RR^{|\mathfrak{C}|\times|\MA|}$ by 
\[(\theta_h)_{(\MC,a)} := \sum_{i=1}^s \mathbbm{1}[\MC=\MC_i] \tilde{\bfr}_h(g_i, a).\] 
Similar to above, we can compute that for any $h \in [H]$, $x \in \MX$, and $a \in \MA$,
\[\langle \phidt{n,s}(x,a), \theta_h\rangle = \sum_{i=1}^s\One{\MC_i(x)=\texttt{True}} \tilde{\bfr}_h(g_i,a) = \tilde{\bfr}_h(\rho^\st(x),a) = \bfr_h(x,a)\] 
as desired. Thus, the transitions and rewards are linear in $\phidt{n,s}$. It's clear by definition that the vectors $(\mu_{h+1}(x'))_{h \in [H-1], x' \in \MX}$ and $(\theta_h)_{h\in[H]}$ have common support of size at most $sA$. %
It only remains to bound the norms of the vectors. Indeed, for any $h \in [H-1]$ and $x' \in \MX$, it's clear that $\norm{\mu_{h+1}(x')}_\infty \leq \BO_{h+1}(x'|\rho^\st(x'))$. Thus,
\[\sum_{x'\in\MX} \norm{\mu_{h+1}(x')}_\infty \leq \sum_{x'\in\MX} \BO_{h+1}(x'|\rho^\st(x')) = \sum_{i=1}^s \sum_{x'\in\MX: \rho^\st(x')=i} \BO_{h+1}(x'|i) = s.\]
Similarly, it's clear that $\norm{\theta_h}_\infty \leq 1$ for all $h \in [H]$. Finally, the fact that $\norm{\phidt{n,s}(x,a)}_\infty \leq 1$ for all $x\in\MX$, $a \in \MA$ is immediate from the definition.
\end{proof}

Now \cref{cor:dec-tree-mdp} is immediate from \cref{thm:main} with ambient dimension $d := n^{O(\log s)}$ and sparsity $k := sA$.

\section{Computational lower bounds for block MDPs}\label{sec:block-lb}

In this section, we state a (previously unwritten) result due to Sitan Chen, Fred Koehler, Morris Yau, and the second author. The result is that there is a family of block MDPs (BMDPs), defined by an explicit decoding function class $\Phi$, such that the time complexity of any learning algorithm must scale nearly polynomially in $|\Phi|$, assuming the hardness of learning noisy parity \cite{blum2003noise} (a standard conjecture in learning theory). In contrast, the optimal sample complexity for learning this family of BMDPs only scales with $\log |\Phi|$. 

We first describe the hard family of BMDPs. 
Fix $n \in \NN$. For any vector $w^\st \in \{0,1\}^n$, we define an MDP $M(w^\st)$ with horizon $H = 2$, action set $\MA = \{0,1\}$, and state space $\MX = \{0,1\}^n \cup \{ \mf s_0, \mf s_1 \}$. The initial state distribution is uniform over $\{0,1\}^n \subset \MX$. %
All rewards at step $h=1$ are 0, and the rewards at step $h=2$ are given by $\bfr_2(x,a) = \One{x = \mf s_1}$. Finally, the transitions are parametrized by $w^\st$, as follows: for $x \in \MX, a \in \MA$, we define
\begin{align}
\BP_1(\mf s_1 | x, a) = \frac 13 + \frac 13 \One{\lng x, w^\st \rng \equiv a \pmod{2}}, \qquad \BP_1(\mf s_0 | x,a) = 1- \BP_1(\mf s_1 | x,a)\nonumber.
\end{align}

Each MDP $M(w^\st)$ is a block MDP with latent state space $\MS = \{0,1\}$, and with decoding function given by $\rho^\st(x) = \rho_{w^\st}(x)$, where we define, for $w \in \{0,1\}^n$, 
\begin{align}
  \rho_w(x) := \begin{cases}
    \lng x, w \rng \pmod{2} &: x \in \{0,1\}^n \\
    0 &: x = \mf s_0 \\
    1 &: x = \mf s_1.
  \end{cases}\nonumber
\end{align}
Note that the latent state transitions and emission distributions of this block MDP do not depend on $w^\st$, and are given as follows (using the notation of \cref{def:bmdp}):
\begin{align}
  \tilde \BP_1(s' | s, a) = \frac 13 + \frac 13 \cdot \One{s=s'} \qquad & \forall a \in \MA, s \in \MS\nonumber\\
  \tilde \BO_2(\mf s_s | s) = 1 \qquad & \forall s \in \MS \nonumber.
\end{align}
The true decoding function $\rho^\st$ is contained in the function class $\Phi := \{\rho_w(x) \ : \ w \in \{0,1\}^n\}$, which has size $2^n$. Thus, a near-optimal policy can be found with sample complexity $\poly(n)$ (e.g., \cite{jiang2017contextual,jin2021bellman,du2021bilinear}). However, there is likely no computationally efficient learning algorithm:

\begin{proposition}%
  \label{prop:parity-hardness}
Suppose that improperly PAC learning a parity function over the uniform distribution on $\{0,1\}^n$ with noise level $1/3$ and constant advantage $\delta \in (0,1/2)$ requires time $\exp(n^{\Omega(1)})$. Then for any $\epsilon \in (0, 1/6-\delta/3]$, any algorithm that learns an $\epsilon$-near optimal policy in the family of block MDPs $\MM = \{M(w): w \in \{0,1\}^n\}$ also requires time $\exp(n^{\Omega(1)})$.
\end{proposition}

\begin{proof}
The problem of  PAC learning noisy parity functions (with noise level $1/3$) can be reduced to that of finding a near-optimal policy in the class $\MM$ of BMDPs. This follows because an interaction with $M(w^\st)$ can be simulated using a noisy parity sample $(x,y)$, i.e. where $x \sim \Unif(\{0,1\}^n)$ and $y \in \{0,1\}$ satisfies $\Pr[y \equiv \langle x, w^\st\rangle \pmod{2}] = 2/3$. Moreover, any policy $\pi: \MX \to \Delta(\{0,1\})$ has expected reward $\frac{1}{3}+\frac{1}{3}\Pr_{x \sim \{0,1\}^n}[\pi(x) \equiv \langle x, w^\st\rangle \pmod{2}]$. Thus, finding a policy with constant suboptimality $\epsilon \in (0, 1/6)$ gives, in the same running time, a (possibly improper) predictor with constant advantage at predicting $\langle x,w^\st\rangle \bmod{2}$. 
\end{proof}
We remark that the best-known algorithm for learning parity functions from noisy samples has time complexity $\exp(\Omega(n/\log n))$ \cite{blum2003noise}.

\begin{remark}
\cref{prop:parity-hardness} relies on hardness of improper learning, i.e. the problem of finding \emph{any} predictor $\pi$ such that $\Pr_{x\sim\{0,1\}^n}[\pi(x) \equiv \langle x,w^\st\rangle \pmod{2}] \geq 1/2+\delta$ for some constant $\delta>0$. However, by self-reducibility of parity functions, this is equivalent to hardness of proper learning (up to polynomial factors), i.e. the problem of finding $w \in \{0,1\}^n$ such that $\Pr_{x\sim\{0,1\}^n}[\langle x,w\rangle \equiv \langle x,w^\st\rangle \pmod{2}] \geq 1/2+\delta'$, for some (possibly different) constant $\delta'>0$. 
\end{remark}

\begin{remark}
  \label{rmk:supervised-rl}
The above reduction also straightforwardly extends to \emph{any} function class $\Phi$ of maps $\phi: \{0,1\}^n \to \{0,1\}$, providing a direct reduction from supervising learning to reinforcement learning in block MDPs with two latent states and decoding function class $\Phi$. If the improper learning problem is computationally hard for $\Phi$ (as is believed to be the case for many concrete classes even when $\log |\Phi| = \poly(n)$), then there is no hope for solving the corresponding reinforcement learning problem in a computationally efficient manner.
\end{remark}

\begin{remark}[Computational-statistical gaps in RL]
Since, as we have remarked above, the class $\MM$ can be learned \emph{statistically efficiently} (i.e., with $\poly(n)$ samples), \cref{prop:parity-hardness} yields a computational-statistical gap in reinforcement learning. Such gaps have recently been shown \cite{kane2022computational, liu2023exponential}, using a much more involved technique, for the problem of learning a near-optimal policy in the class of MDPs for which $Q^\st(x,a)$ and $V^\st(x)$ are assumed to be linear functions of known features. 
\end{remark}

\end{document}